\documentclass[twoside]{article}
\usepackage[accepted]{aistats2025}
\usepackage[american]{babel}
\usepackage{natbib}
\usepackage{mathtools}
\usepackage{booktabs}
\usepackage{tikz}

\usepackage[T1]{fontenc}    
\usepackage{hyperref}       
\usepackage{url}            
\usepackage{amsfonts}       
\usepackage{nicefrac}       
\usepackage{microtype}      
\usepackage{svg}
\usepackage{amsmath}
\usepackage{multirow}
\usepackage{centernot}
\usepackage{amsthm}
\usepackage{thmtools}
\usepackage{thm-restate}
\usepackage{cleveref}

\newtheorem{lemma}{Lemma}[section]

\newtheorem{definition}{Definition}[section]
\newtheorem{example}{Example}[section]

\usepackage{algorithmicx}
\usepackage[noend]{algpseudocode}
\usepackage{algorithm}

\usepackage{comment}

\usepackage{cases}

\usepackage{xcolor}
\usepackage{hyperref}
\usepackage{subcaption}
\usepackage{paralist}
\usepackage{comment}

\newcommand{\indep}{\perp\kern-6pt\perp}
\newcommand{\dep}{\centernot{\perp\kern-6pt\perp}}
\newcommand{\starleft}{*\kern-5pt}
\newcommand{\starright}{\kern-5pt*}

\usepackage{acronym}
\acrodef{AID}{adjustment identification distance}
\acrodef{ATE}{average treatment effect}
\acrodef{CI}{conditional independence}
\acrodef{CPDAG}{complete partially directed acyclic graph}
\acrodef{DAG}{directed acyclic graph}
\acrodef{KCI}{Kernel-based conditional independence}
\acrodef{LCD}{local causal discovery}
\acrodef{MB}{Markov blanket}
\acrodef{MEC}{Markov equivalence class}
\acrodef{SHD}{structural Hamming distance}
\acrodef{SNAP}{Sequential Non-Ancestor Pruning}

\begin{document}

\twocolumn[

\aistatstitle{SNAP: Sequential Non-Ancestor Pruning for Targeted Causal Effect Estimation With an Unknown Graph}

\aistatsauthor{ Mátyás Schubert \And Tom Claassen \And  Sara Magliacane }

\aistatsaddress{
University of Amsterdam\\
\And
Radboud University Nijmegen\\
\And
University of Amsterdam\\
}
]

\begin{abstract}

Causal discovery can be computationally demanding for large numbers of variables. 
If we only wish to estimate the causal effects on a small subset of target  variables, we might not need to learn the causal graph for all variables, but only a small subgraph that includes the targets and their adjustment sets. 
In this paper, we focus on identifying causal effects between target variables in a computationally and statistically efficient way.
This task combines causal discovery and effect estimation, aligning the discovery objective with the effects to be estimated.
We show that definite non-ancestors of the targets are unnecessary to learn causal relations between the targets and to identify efficient adjustments sets.
We sequentially identify and prune these definite non-ancestors with our Sequential Non-Ancestor Pruning (SNAP) framework, which can be used either as a preprocessing step to standard causal discovery methods, or as a standalone sound and complete causal discovery algorithm.
Our results on synthetic and real data show that both approaches substantially reduce the number of  independence tests and the computation time without compromising the quality of causal effect estimations.

\end{abstract}
\section{INTRODUCTION}
\label{sec:introduction}

Causal inference \citep{pearl2009causality} is fundamental to our scientific understanding and practical decision-making.
In many settings, we do not know the causal relations between the variables, which we can learn with \emph{causal discovery} methods \citep{glymour2019review}.
These methods can be computationally demanding for large numbers of variables.
In many cases, we are only interested in estimating the causal effects between a small subset of variables, which does not require recovering the causal graph over all variables.
We formalize this setting as \emph{targeted causal effect estimation with an unknown graph}, a task that focuses on identifying the causal effects $P(T_i|do(T_j))$  between pairs of target variables $T_i,T_j \in \mathbf{T}$, where $\mathbf{T}$ is a  subset of all variables $\mathbf{V}$ in a \emph{computationally and statistically efficient way}.

We assume that we are in a causally sufficient setting, i.e., there are no unobserved confounders or selection bias.
Under these assumptions, we can use constraint-based causal discovery algorithms \citep{spirtes2000causation} to identify the \ac{MEC} of the causal graph \citep{verma90equivalence}, represented by a mixed graph, called the \ac{CPDAG}.
The \ac{CPDAG} can then be used to identify valid adjustment sets for causal effect estimation \citep{perkovic2015complete}.
However, discovering the \acs{CPDAG} over all variables can scale poorly in terms of \ac{CI} tests for large numbers of nodes \citep{mokhtarian2021recursive}.

\Acl{LCD} methods \citep{wang2014discovering, gupta2023local} address this issue for a pair of targets, a treatment and an outcome, by identifying the parent adjustment set of the treatment, but they cannot learn other types of adjustments that are statistically more efficient.
\citet{maasch2024local} learn to group nodes according to their ancestral relationship to a treatment-outcome pair, which can be used to identify various adjustment sets, but not necessarily optimal ones.
These algorithms focus only on two target variables and assume that we know the causal relations between them.
\citet{watson2022causal} propose an algorithm to discover the causal relations between multiple targets, which they call \emph{foreground variables}, but assume 
that the other variables, the \emph{background variables}, are all non-descendants of the target variables.

In this paper we propose Sequential Non-Ancestor Pruning (SNAP), an approach that efficiently discovers the ancestral relationships between multiple target variables, as well as their adjustments sets, from observational data.
We show that only possible ancestors of the target variables are required to both identify these relationships and provide efficient adjustment sets.
Thus, during the process of causal discovery we identify definite non-ancestors of the targets, remove them from consideration, and continue causal discovery only on the remaining variables.
SNAP is also straightforward to combine with readily available causal discovery algorithms, significantly decreasing their complexity in terms of execution time.
Our contributions are:
\begin{compactitem}
   \item We introduce the \emph{targeted causal effect estimation with an unknown graph} task in which we focus on estimating causal effects for a set of targets in a computationally and statistically efficient way.
    \item To solve this task, we propose the \ac{SNAP} framework.
    SNAP can be used as a sound and complete standalone method,
    or as a preprocessing  step that can be combined with standard causal discovery methods.
    \item We evaluate  \ac{SNAP} on simulated and real-world data, showing their benefits in reducing conditional independence tests and computation time.
\end{compactitem}

\section{PRELIMINARIES}
\label{sec:preliminaries}

We focus on graphs $G=(\mathbf{V}, \mathbf{E})$ with nodes $\mathbf{V}$ and edges $\mathbf{E}$.
The edges $\mathbf{E}$ can be undirected edges $X-Y$, directed edges $X \to Y$ or $X \gets Y$ and bidirected edges $X \leftrightarrow Y$ that are directed both towards $X$ and $Y$.
If two nodes are connected by an edge in a graph $G$, we say that they are \emph{adjacent}, and denote the set of nodes adjacent to $X$ as $Adj_G(X)$.
An undirected graph is a graph with only undirected edges, and a directed graph is a graph with only directed edges.
A mixed graph can contain undirected, directed and bidirected edges.

A path between two nodes $X$ and $Y$ is a sequence of distinct adjacent nodes that starts at $X$ and ends at $Y$.
A directed path from $X$ to $Y$ is a path where each edge on the path is directed towards $Y$.
A directed cycle is a directed path from a node to itself.
A \ac{DAG} is a graph with only directed edges and without directed cycles.
If there is a directed path from a node $X$ to another node $Y$ in a graph $G$, then we say that $X$ is an ancestor of $Y$ and $Y$ is a descendant of $X$.
We denote the set of ancestors of $Y$ as $An_G(Y)$ and the set of descendants of $X$ as $De_G(X)$.
We say that a set of nodes $\mathbf{V}' \subseteq \mathbf{V}$ is an \emph{ancestral set} in a DAG $D$ if for all $V \in \mathbf{V}'$ it holds that $An_D(V) \subseteq \mathbf{V}'$.

A causal graph $D$ is a \ac{DAG} that describes the data generating process of a joint observational distribution $p$ over variables $\mathbf{V}$ \citep{pearl2009causality}.
We assume \emph{causal sufficiency}, i.e., the distribution $p$ is over all variables $\mathbf{V}$ and 
there are no latent confounders or selection bias.
Furthermore, we assume that the distribution $p$ is \emph{Markov and faithful} to $D$, i.e., that for all $X,Y \in \mathbf{V}$ and $\mathbf{S} \subseteq \mathbf{V}\setminus \{X,Y\}$ it holds that the \acl{CI} $X \indep Y | \mathbf{S}$ in the distribution $p$ is equivalent to d-separation $X \perp_d Y | \mathbf{S}$ in the true causal graph $D$.

Under these assumptions, constraint-based causal discovery algorithms use conditional independence (CI) tests to identify causal relations between the variables.
However, without further assumptions, \ac{CI} tests alone generally cannot identify all causal relationships, because multiple causal graphs with conflicting causal relations can imply the same set of conditional independence.
This set of causal graphs is called the \emph{Markov Equivalence Class} (\ac{MEC}) of $D$.
\citet{verma90equivalence} show that all graphs in a \ac{MEC} have the same adjacencies and v-structures, i.e., patterns of $X \to Z \gets Y$, such that $X$ and $Y$ are not adjacent. \citet{meek1995causal} shows how to identify some further edge orientations.

We  graphically represent a \ac{MEC} by a mixed graph, called the Completed Partially oriented DAG (\ac{CPDAG}).
In a \ac{CPDAG} $G$, directed edges are oriented the same way in all graphs in the \ac{MEC} of $D$, while undirected edges indicate conflicting orientations between graphs in the \ac{MEC}.
Due to conflicting orientations, some ancestral relationships might hold only in some graphs in the \ac{MEC}, but not all \citep{zhang2006causal, roumpelaki2016marginal}.
If $X$ is an ancestor of $Y$ in at least one graph in the \ac{MEC} represented by the CPDAG $G$, then $X$ is a \emph{possible ancestor} of $Y$.

\begin{definition}
\label{def:possan}
Given a \ac{CPDAG} $G$, the node $X$ is a \emph{possible ancestor} of the node $Y$, denoted as $X\in PossAn_G(Y)$, iff $X \in An_D(Y)$ in at least one graph $D$ in the MEC of $G$, or equivalently, iff there exists a \emph{possibly directed path} from $X$ to $Y$ in $G$, i.e., a path from $X$ to $Y$ that only contains undirected edges or directed edges pointing towards $Y$.
Otherwise, $X$ is a \emph{definite non-ancestor} of $Y$.
The possible ancestors for sets of nodes $\mathbf{V}'$ are the union of the possible ancestors of all nodes
$
    PossAn_G(\mathbf{V}') = \bigcup_{V \in \mathbf{V}'} PossAn(V).
$
\end{definition}

Possible ancestry is transitive in a CPDAG: if $X$ is a possible ancestor of $Y$, and $Y$ is a possible ancestor of $Z$, then $X$ is a possible ancestor of $Z$. 
Similarly, $Y$ is a possible descendant of $X$, i.e., $Y \in PossDe_G(X)$, if there is a possibly directed path from $X$ to $Y$ in $G$.

The induced subgraph of a graph $G = (\mathbf{V}, \mathbf{E})$ over a subset of variables $\mathbf{V}^* \subseteq \mathbf{V}$ is a graph, denoted as $G|_{\mathbf{V}^*}$, with edges $\mathbf{E}' \subseteq \mathbf{E}$ that are between pairs of variables in $\mathbf{V}^*$.
Let $G(\mathbf{V})$ be the \ac{CPDAG} according to the observational distribution over $\mathbf{V}$ and $G(\mathbf{V}^*)$ be the CPDAG according to the marginal observational distribution over $\mathbf{V}^*$.
\citet{guo2023variable} (Lemma D.1) show  that if $\mathbf{V}^*$ is an ancestral set in the underlying causal graph, then $G(\mathbf{V})|_{\mathbf{V}^*} = G(\mathbf{V}^*)$.

Given a \ac{CPDAG} $G$, we can estimate causal effects between its variables \citep{maathuis2009estimating}, either single values, if they are identifiable, or as a set of values.
If a causal effect is identifiable, there are various adjustment sets available to estimate it.
Parents \citep{maathuis2015generalized} and definite ancestors \citep{henckel2024adjustment} of the cause $X$ are valid adjustment sets for any outcome $Y$ that is not a parent of $X$.
If the causal effect is identifiable, then we can always estimate the effect of a cause $X$ on an outcome $Y$ in a CPDAG $G$ with the canonical adjustment \citep{perkovic2018complete}, defined as 
$
PossAn_G(\{X,Y\}) \setminus (Forb_G(X,Y) \cup \{X, Y \}),
$
where $Forb_G(X,Y)$ is the \emph{forbidden set} defined as $PossDe_G(Cn_G(X,Y))$ such that $Cn_G(X,Y)$ is the set of all nodes on  possibly directed paths from $X$ to $Y$ in $G$. While this adjustment always exists if the causal effect is identifiable, it is not necessarily the most statistically efficient adjustment.
\citet{henckel2022graphical} define the asymptotically optimal adjustment set to estimate the causal effect of $X$ on $Y$ as
$$
Pa_G(Cn_G(X,Y)) \setminus (Forb_G(X,Y) \cup \{X\}),
$$
where $Pa_G$ are the definite parents in $G$.
All the listed adjustment sets are subsets of the possible ancestors of the outcome $Y$.
\citet{guo2023variable} show that non-ancestors of the outcome are \emph{uninformative}, in the sense that they have no impact on the statistical efficiency of estimating the causal effect.
Furthermore, \citet{henckel2022graphical} show that the parents of the cause $X$ are usually strongly correlated with $X$, making the asymptotic variance of the parent adjustment set suboptimal.
\section{SEQUENTIAL NON-ANCESTOR PRUNING}
\label{sec:method}

Causal effect estimation requires knowledge about the causal graph.
If causal relations are not known a priori, we can use causal discovery methods to learn them from data.
In many settings, we might have access to a large set of variables, but we are only interested in the causal effects between a small set of \emph{target variables}. Standard \emph{global} causal discovery on all variables can estimate the \emph{identifiable} causal effects between targets, as well as valid adjustment sets, but it might be computationally inefficient, since it might also learn parts of the graph that are uninformative for the causal effects of interest. 
Alternatively, focusing on the target variables and only learning the causal relations between them might lead to confounded causal effects and less identifiable causal relations than we would learn in global causal discovery.

Our goal is to estimate the identifiable causal effects between the targets as statistically efficiently as global causal discovery methods, but in a more computationally efficient way than estimating the complete graph.
We formalize this goal as a specific task as follows:

\begin{definition}
\label{def:task}
Given a joint distribution $p$ over variables $\mathbf{V}$ and targets $\mathbf{T} \subseteq \mathbf{V}$,  we define \emph{targeted causal effect estimation with an unknown graph} as the task of estimating in a computationally and statistically efficient way the interventional distributions $P(T_i|do(T_j))$, for all possible pairs $T_i,T_j \in \mathbf{T}$.
\end{definition}

Local causal discovery methods, such as \citep{wang2014discovering, gupta2023local,maasch2024local}, address the computational efficiency by only learning a local structure in the neighborhood of a pair of target variables for which we know which is the cause and which is the effect a priori.
On the other hand, these methods are  not as statistically efficient as global causal discovery, since they cannot identify optimal adjustment sets. In particular, they either learn local adjustment sets relying on the parents of the treatment, which are suboptimal in terms of variance, or other valid adjustment sets based on groups of variables, which might not include the optimal adjustment sets. 

In this work, we address this task by progressively pruning definite non-ancestors of the target variables. 
We first show that definite non-ancestors of the targets are not needed to identify valid and statistically efficient adjustment sets. While \citet{guo2023variable} show that non-ancestors of the outcome are uninformative, our results show that they are also unnecessary to identify the causal structure between the targets and their (optimal) adjustment sets.
Then, we present \acf{SNAP}, a framework to progressively identify definite non-ancestors, inspired by low-order constraint-based causal discovery.
We provide all proofs for the following results in App.~\ref{app:proof}.

\subsection{Possible Ancestors Are All You Need}
\label{sec:pruning}
In this section, we show that definite non-ancestors of the target variables are not needed to identify statistically efficient adjustment sets for the causal effects between the targets. \cite{guo2023variable} show that non-ancestors of the target variables are not needed as part of the statistically efficient adjustment sets for the targets. Here, we do not know the non-ancestors of the target variables, since we learn the causal structure as a CPDAG $G$ and we can at best identify the set of possible ancestors of the targets $PossAn_G(\mathbf{T})$. Additionally, we show that definite non-ancestors of the targets will not have any effect on the orientations of the causal structure for the targets or their possible ancestors, and hence they can be safely ignored. 

In general, we might not be able to estimate the possible ancestors of the targets without reconstructing the whole CPDAG.
On the other hand, we can overestimate this set and consider a set $\mathbf{V}^*$ that contains  $PossAn(\mathbf{T})$. If $\mathbf{V}^*$ also contains all its possible ancestors, we then show that learning a CPDAG on it will provide the same results as learning a complete CPDAG and restricting it to $\mathbf{V}^*$. 
We formalize this results as follows:

\begin{restatable}[]{lemma}{lemmaancestors}
\label{thm:ancestors}
Let $G$ be a full \ac{CPDAG} over variables $\mathbf{V}$. 
Let $\mathbf{V}^* \subseteq \mathbf{V}$ be a possibly ancestral set of nodes, i.e. $PossAn_G(\mathbf{V}^*) \subseteq \mathbf{V}^*$.
Let $G(\mathbf{V})|_{\mathbf{V}^*}$ be the induced subgraph of $G$ over $\mathbf{V}^*$ and let $G(\mathbf{V}^*)$ be the CPDAG over variables $\mathbf{V}^*$. Then we have $G(\mathbf{V})|_{\mathbf{V}^*} = G(\mathbf{V}^*)$.
\end{restatable}

Lemma~\ref{thm:ancestors} implies that running a global causal discovery algorithm on a possibly ancestral set  that includes the possible ancestors of the targets, solves the task, allowing for the estimation of the causal effects between the targets in an as  statistically efficient way as learning the full CPDAG.
In particular, the restricted \ac{CPDAG} $G(\mathbf{V}^*)$ has the same canonical, parent, ancestor and optimal adjustment sets for pairs of targets $\mathbf{T}$ as the full \ac{CPDAG} $G(\mathbf{V})$. 
As discussed by \citet{guo2023variable} it also allows for more general optimal sets beyond valid adjustments.
Intuitively, the challenge is then how to identify as many definite non-ancestors as possible in a computationally efficient way.
In the next section, we propose two methods to address this challenge.

\subsection{The SNAP Framework}
\label{sec:snap}

If the possibly ancestral set $\mathbf{V}^*$, which contains the possible ancestors of the targets, is much smaller than the total number of nodes, then discovering only $G(\mathbf{V}^*)$, instead of the full CPDAG $G(\mathbf{V})$, can save considerable computational resources.
Furthermore, $G(\mathbf{V}^*)$ allows us to use statistically efficient adjustment sets.
Without background knowledge on a suitable $\mathbf{V}^*$, we approach the discovery of $G(\mathbf{V}^*)$ using an iterative framework, which we call \acl{SNAP}. We propose two implementations of this framework: SNAP$(k)$ (Alg.~\ref{alg:snap(k)}) and SNAP$(\infty)$ (Alg.~\ref{alg:snap(inf)}).

SNAP$(k)$ is a causal discovery algorithm that uses information about ancestry to progressively eliminate definite non-ancestors of targets, removing them from $\mathbf{V}^*$.
In particular, SNAP$(k)$ adapts the PC-style skeleton search, iteratively increasing conditioning set sizes of \ac{CI} tests.
At every iteration $i$, SNAP$(k)$ computes a partially oriented graph $\hat{G}^i$ by orienting v-structures according to the intermediate skeleton $\hat{U}^i$ and separating sets discovered so far.
Then, $\hat{G}^i$ is used to identify and eliminate a subset of the definite non-ancestors of the targets, and SNAP$(k)$ continues to the next iteration only on the remaining variables $\hat{V}^{i+1}$.
Thus, SNAP$(k)$ considers fewer and fewer variables as the size of the conditioning set increases.
In practice, we see that many non-ancestors are identified already by marginal tests, leading to significantly fewer higher order tests.

\begin{algorithm*}
\caption{Sequential Non-Ancestor Pruning - SNAP$(k)$}
\label{alg:snap(k)}
\footnotesize{
\begin{algorithmic}[1]
    \Require Vertices $\mathbf{V}$, targets $\mathbf{T} \subseteq \mathbf{V}$, Maximum order $k$
    \State $\hat{\mathbf{V}}^{0} \gets \mathbf{V}$, $\hat{U}^{-1} \gets$ fully connected undirected graph over $\mathbf{V}$
    \For{$i \in 0..k$}
        \State $\hat{U}^i \gets$ induced subgraph of $\hat{U}^{i-1}$ over $\hat{\mathbf{V}}^{i}$
        \For{$X \in \hat{\mathbf{V}}^{i}$, $Y \in Adj_{\hat{U}^i}(X)$} \Comment{Learn skeleton step at order $i$}
                \For{$\mathbf{S} \subseteq Adj_{\hat{U}^i}(X) \setminus \{Y\}$ s.t. $|\mathbf{S}| = i$}
                        \If{$X \perp\kern-6pt\perp Y | \mathbf{S}$}
                            \State Delete the edge $X - Y$ from ${\hat{U}^i}$
                            \State $sepset(X,Y) \gets sepset(Y,X) \gets \mathbf{S}$
                            \State \textbf{break}
                        \EndIf
                \EndFor
        \EndFor
        \If{$i < 2$} \Comment{Orient v-structures}
            \State$\hat{G}^i \gets$ OrientVstructPC($\hat{U}^i, sepset$) (Alg.~\ref{alg:pc-v})
        \Else
            \State$\hat{G}^i, sepset \gets$ OrientVstructRFCI($\hat{U}^i, sepset$) (Alg.~\ref{alg:rfci-v})
            \State $\hat{U}^i \gets$ skeleton of $\hat{G}^i$
        \EndIf
        \State $\hat{\mathbf{V}}^{i+1} \gets$ all $V \in \hat{\mathbf{V}}^{i}$ with a possibly directed path to any $T \in \mathbf{T}$ in $\hat{G}^i$ \Comment{Prune non-ancestors}
    \EndFor
\State $\hat{G}^k \gets$ induced subgraph of $\hat{G}^k$ over $\hat{\mathbf{V}}^{k+1}$    \State\Return{$\hat{\mathbf{V}}^{k+1}, \hat{G}^k$}
\end{algorithmic}
}
\end{algorithm*}

\begin{algorithm}
\caption{SNAP$(\infty)$ }
\footnotesize{
\label{alg:snap(inf)}
\begin{algorithmic}[1]
    \Require Vertices $\mathbf{V}$, targets $\mathbf{T} \subseteq \mathbf{V}$
    \State $\hat{\mathbf{V}}, \hat{G} \gets \text{SNAP}(\mathbf{V}, \mathbf{T}, |\mathbf{V}|-2)$
    \Repeat \Comment{Apply Meek rules}
        \If{$X \to Z - Y$ in $\hat{G}$ and $X \notin Adj_{\hat{G}}(Y)$}
            \State Orient $Z - Y$ as $Z \to Y$ in $\hat{G}$
        \EndIf
        \If{$X \to Z \to Y$, and $X - Y$ in $\hat{G}$}
            \State Orient $X - Y$ as $X \to Y$ in $\hat{G}$
        \EndIf
        \If{$X \to Z \gets Y, X - V - Y, V - Z$ in $\hat{G}$ and $X \notin Adj_{\hat{G}}(Y)$}
            \State Orient $V - Z$ as $V \to Z$ in $\hat{G}$
        \EndIf
    \Until{no edges can be oriented}
    \State $\hat{\mathbf{V}} \gets$ all $V \in \hat{\mathbf{V}}$ with a possibly directed path to any $T \in \mathbf{T}$ in $\hat{G}$
    \State $\hat{G} \gets$ induced subgraph of $\hat{G}$ over $\hat{\mathbf{V}}$
    \State\Return{$\hat{G}$}
\end{algorithmic}
}
\end{algorithm}

In \ac{SNAP} we orient v-structures in $\hat{G}^i$ using only conditional independence of a maximum order $i$ to identify definite non-ancestors.
This requires particular care, since in general \ac{CI} test results restricted to a maximum order can lead to conflicting v-structure orientations, which are oriented as bidirected edges, as shown by Fig.~\ref{fig:bidirected_example_main} for order $0$.
\citet{wienobst2020recovering} show that if \emph{all} \ac{CI} test results up to order $k$ are available, then these bidirected edges indicate that neither variable is the possible ancestor of the other.
Surprisingly, if we instead only perform a subset of \ac{CI} tests based on the skeleton search up to order $k$, as in anytime FCI \citep{pmlr-vR3-spirtes01a}, and then orient v-structures as in PC (Alg.~\ref{alg:pc-v}), this can lead to incorrect information about possible ancestry, as shown in this example:

\begin{figure}[t]
    \centering
    \begin{subfigure}{.25\linewidth}
        \centering
        \includegraphics[width=\linewidth]{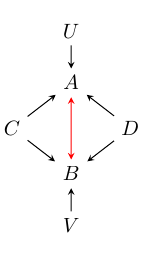}
        \caption{}
        \label{fig:bidirected_example_main}
    \end{subfigure}
    \hspace{.1\linewidth}
    \begin{subfigure}{.45\linewidth}
        \centering
        \includegraphics[width=\linewidth]{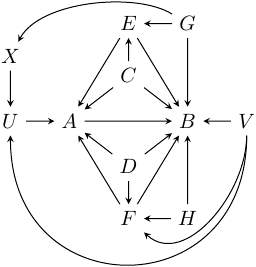}
        \caption{}
        \label{fig:rfci_example}
    \end{subfigure}
    \caption{Given the DAG in Fig.~\ref{fig:bidirected_example_main} denoted by black edges, $A$ and $B$ are marginally dependent with conflicting v-structures $U \to A \gets B$ and $A \to B \gets V$ if we orient v-structures at order 0, as shown by the red bidirected edge.
    Given the DAG in Fig.~\ref{fig:rfci_example}, performing skeleton search up to order 3 and then orienting v-structures as in PC (Alg.~\ref{alg:pc-v}) leads to an incorrect bidirected edge $A \leftrightarrow B $.
    Details in App.~\ref{app:vstruct}.}
\end{figure}

\begin{example}
\label{example:rfci}
Consider the DAG in Fig~\ref{fig:rfci_example}.
At order 3, skeleton search may find $A \perp\kern-6pt\perp G | \{E,C,X\}$ and remove the edge $A-G$, before testing $A \perp\kern-6pt\perp X | \{G,U,V\}$.
Then, finding $X \perp\kern-6pt\perp B | \{G,U,V\}$ removes the edge $X - B$.
Orienting v-structures as in PC (Alg.~\ref{alg:pc-v}) at order 3 creates the conflicting v-structures $X \to A \gets B$ and $A \to B \gets V$, resulting in the bidirected edge $A \leftrightarrow B$, even though $A \to B$ is in the true DAG.
\end{example}

To overcome this, we show in App~\ref{app:proof} that orienting v-structures as in RFCI \citep{colombo2012learning} (Alg.~\ref{alg:rfci-v}) leads to correct information about possible ancestry at any order of the PC-style skeleton search, in the causally sufficient case.
Since the RFCI orientation involves performing additional CI tests, we show that it is only required above order 1, which greatly reduces its overhead.
Then, as Theorem~\ref{thm:snap_is_sound} states, SNAP$(k)$ only removes definite non-ancestors of targets, and the remaining nodes form a possibly ancestral set.

\begin{restatable}[]{theorem}{snapsound}
\label{thm:snap_is_sound}
Given oracle conditional independence tests, at each step $i=0, \dots, k$ of \ac{SNAP}$(k)$ (Alg.~\ref{alg:snap(k)}) $\hat{\mathbf{V}}^{i+1}$ contains $PossAn(\mathbf{T})$. Additionally, $\hat{\mathbf{V}}^{i+1}$ is a possibly ancestral set, i.e. $PossAn_G(\hat{\mathbf{V}}^i) \subseteq \hat{\mathbf{V}}^{i+1}$.
\end{restatable}

As shown by Corollary~\ref{cor:snap-prefiltering}, we can stop SNAP$(k)$ early at any maximum iteration $k$ to obtain the remaining variables $\hat{\mathbf{V}}^{k+1}$, and subsequently run a sound and complete global causal discovery algorithm of our choice only on $\hat{\mathbf{V}}^{k+1}$ to obtain a CPDAG that we can use for targeted causal effect estimation.
We refer to this approach as prefiltering with SNAP$(k)$.

\begin{restatable}[]{corollary}{snapkprefiltering}
\label{cor:snap-prefiltering}
Given oracle conditional independence tests, a sound and complete causal discovery algorithm on $\hat{\mathbf{V}}^{k+1}$, the output of \ac{SNAP}(k), will return a CPDAG on $\hat{\mathbf{V}}^{k+1}$, which is the same as the induced subgraph of the full CPDAG $G(\mathbf{V})$ restricted to $\hat{\mathbf{V}}^{k+1}$. Additionally, this CPDAG will contain all informative adjustment sets for estimating causal effects between the targets $\mathbf{T}$.
\end{restatable}

Similarly, local algorithms by \citep{wang2014discovering} and \citep{gupta2023local} output the same parents and undirected neighbors when ran only on $\hat{\mathbf{V}}^{k+1}$ instead of all nodes.
Thus, standard discovery methods on $\hat{\mathbf{V}}^{k+1}$ provide us with adjustment sets to solve the targeted causal effect estimation task with an unknown graph.

Furthermore, \ac{SNAP}$(k)$ can be extended to obtain a stand-alone causal discovery algorithm.
We call this algorithm \ac{SNAP}$(\infty)$ and show it in Alg.~\ref{alg:snap(inf)}.
\ac{SNAP}$(\infty)$ runs \ac{SNAP}$(k)$ \emph{until completion} with $k=|\mathbf{V}|-2$.
Then, it applies Meek's rules on the resulting partially oriented graph $\hat{G}^k$, identifies definite non-ancestors one more time, which in this case will provide exactly the set of possible ancestors of the targets, and returns the induced subgraph of $\hat{G}^k$ over the remaining variables.
Theorem~\ref{thm:snap_is_complete} states that \ac{SNAP}$(\infty)$ is sound and complete over the possible ancestors of the targets.

\begin{restatable}[]{theorem}{snapcomplete}
\label{thm:snap_is_complete}
Given oracle conditional independence tests, let $\hat{G}$ be the output of graph of \ac{SNAP}$(\infty)$ for targets $\mathbf{T}$. Then, \ac{SNAP}$(\infty)$ returns $\mathbf{\hat{V}} = PossAn(\mathbf{T})$. Additionally, \ac{SNAP}$(\infty)$ is sound and complete over the possible ancestors $\mathbf{T}$, i.e.
$
\hat{G} = G|_{PossAn(\mathbf{T})}.
$
\end{restatable}

\subsection{Complexity Analysis}

In this section, we show that the worst-case computational complexity of \ac{SNAP}$(\infty)$ in terms of \ac{CI} tests matches the complexity of PC
for graphs with maximum degree $d_{max} \geq 2$.
The worst-case complexity for PC is determined by its skeleton search, which is $\mathcal{O}(|\mathbf{V}|^{d_{\max}+2})$ \citep{spirtes2000causation}, where $|\mathbf{V}|$ is the number of nodes and $d_{\max}$ is the maximum degree.
The worst-case complexity of \ac{SNAP}$(\infty)$ is given by the complexity of the skeleton search and the RFCI orientation rules (Algorithm~\ref{alg:rfci-v}).
We present a lemma that states that the worst-case complexity of the RFCI orientation rules is polynomial in the number of nodes.

\begin{restatable}[]{lemma}{rfcicomplexity}
\label{lem:rfci_complexity}
The worst-case complexity of the RFCI orientation rules (Algorithm~\ref{alg:rfci-v}) in terms of \ac{CI} tests performed is at most $\mathcal{O}(|\mathbf{V}|^4)$ \ac{CI} tests, where $|\mathbf{V}|$ is the number of nodes.
\end{restatable}

We provide a formal proof in App.~\ref{app:complexity}.
Intuitively, Algorithm~\ref{alg:rfci-v} performs two \ac{CI} tests for each unshielded triple, leading to $\mathcal{O}(|\mathbf{V}|^3)$ \ac{CI} tests.
Then, for each $\mathcal{O}(|\mathbf{V}|^2)$ independences found, it finds a minimal separating set in $\mathcal{O}(|\mathbf{V}|^2)$ \ac{CI} tests.
This results in $\mathcal{O}(|\mathbf{V}|^3 + |\mathbf{V}|^2 \cdot |\mathbf{V}|^2) = \mathcal{O}(|\mathbf{V}|^4)$ \ac{CI} tests.
By combining the worst-case complexity of the PC-style skeleton search and the RFCI orientation rules, can then show that the worst-case computational complexity of SNAP($\infty$) is at most
$\mathcal{O}(|\mathbf{V}|^{d_{\max}+2} + |\mathbf{V}|^4)$. 

We now show that for graphs that allow a maximum degree $d_{max} \geq 2$, which include most realistic graphs, the complexity of the skeleton search dominates the complexity of the RFCI orientation rules, which means that on these graphs SNAP($\infty$) matches PC in terms of worst-case complexity.

\begin{restatable}[]{corollary}{snapcomplexity}
\label{col:snap_complexity}
For graphs with maximum degree $d_{max} \geq 2$, the worst-case computational complexity of SNAP($\infty$) in terms of \ac{CI} tests is $\mathcal{O}(|\mathbf{V}|^{d_{\max}+2})$, which matches the complexity of PC.
\end{restatable}

Intuitively this makes sense, since the RFCI orientation rules are dominated by the PC style skeleton search for dense enough graphs and in the worst case for SNAP, the ancestors of the targets include the whole graph. On the other hand, the average case complexity is intuitively substantially lower, as also shown empirically. Since this analysis can be quite complex, we only show a rough approximation of the expected number of possible ancestors in App.~\ref{app:expectedanc} as an indication for SNAP's average case complexity.

\section{RELATED WORK}
\label{sec:related_work}

Causal discovery is the task of identifying causal relations between variables from data \citep{glymour2019review}.
We limit our scope to observational data and assume causal sufficiency, meaning there are no unobserved confounders or selection bias.
Under these assumptions, there are multiple algorithms in the literature that can be used to learn an equivalence class of causal graphs. In this paper, we mostly focus on constraint-based algorithms, which use \ac{CI} tests to constrain the possible causal relations. PC \citep{spirtes2000causation} is one of the most famous constraint-based methods, but it also suffers from scalability issues and reliability. 

Several methods have been proposed to reduce the number of CI tests needed, e.g., MARVEL \citep{mokhtarian2021recursive}, or to only consider \ac{CI} tests of low order, which are assumed to be more reliable. For example, \citet{textor2015learning} extract all possible \acp{CPDAG} given marginal independence tests.
\citet{wienobst2020recovering} and \citet{kocaoglu2024characterization} expand on this by showing how to extract information from \ac{CI} tests up to order $k$, i.e., tests with maximum conditioning set size $k$.

These methods need to perform all possible \ac{CI} tests up to order $k$. We take inspiration from these approaches, as well as anytime FCI \citep{pmlr-vR3-spirtes01a}, an extension of the more general causal discovery method FCI \citep{spirtes2000causation} that uses CI tests up to order $k$, to develop our incremental pruning of non-ancestors with increasing order.
As opposed to most of these works, we perform a small number of CI tests for each order.

Due to the complexity of discovering the structure over all variables, a large body of literature is concerned with collecting the set of variables that belong in some neighborhood of a single target variable, and estimating their causal effects on the target and vice versa. This task is often called  \emph{local causal discovery}.

Local causal discovery methods sequentially find parent-children-descendant sets \citep{yin2008partial} or \acp{MB}, i.e., the set of parents, children, and the parents of the children of the variables \citep{wang2014discovering,gao2015local,ling2020using},   
until all edges around the target are oriented.
\citet{zhou2010discover} extend PCD-by-PCD \citep{yin2008partial}  to orient edges not only in the immediate neighborhood of the target, but within the \textit{depth $k$} neighborhood.
\citet{dai2024local} perform \acl{LCD} on possibly cyclic models, assuming that they are linear and non-Gaussian.
\citet{statnikov15ultra} and \citet{shiragur2023meek} use both observational and experimental data to discovery the direct causes and direct effects of the target variable.
\citet{choo2023subset} use experimental data to identify a subset of  target edges.

One of the most related local causal discovery methods is Local Discovery using Eager Collider Checks (LDECC) \citep{gupta2023local}, which identifies the parents of the target by finding the MB of the target and then performing \ac{CI} tests similarly to PC.
Since these algorithms recover only the local structure around a target, causal effect estimation is limited to using the parent adjustment set.

Local Discovery by Partitioning (LDP) \citep{maasch2024local} overcomes this limitation by learning groups of nodes according to their ancestral relations to a target pair.
These groups can then be used as valid adjustment sets.
However, the groups can be identified only if the causal relationship between the targets is known a priori.
Furthermore, LDP might not recover the optimal adjustment set.

All previous methods consider only a single target pair, where one variable is known to be the treatment and the other the outcome.
Similarly to us, Confounder Blanket Learner (CBL) \citep{watson2022causal} recovers the causal order among multiple targets, but it assumes all other variables are non-descendants of the targets. Our approach does not require this assumption, and its output is complementary to this work.

\section{EXPERIMENTS}
\label{sec:experiments}

\paragraph{Baselines.} We evaluate \ac{SNAP}$(\infty)$, along with global algorithms PC \citep{spirtes2000causation}, MARVEL \citep{mokhtarian2021recursive} and FGES \citep{ramsey2017million}, the modified versions of local algorithms MB-by-MB \citep{wang2014discovering} and LDECC \citep{gupta2023local},  
and their combination with \ac{SNAP}$(0)$.
We choose \ac{SNAP}$(0)$ as it is the simplest prefiltering method, and it is already very effective in practice, as shown in App.~\ref{app:snapk}.
We do not compare to CBL \citep{watson2022causal}, since it requires that all non-target nodes are non-descendants of the targets.
Similarly, we do not compare to LDP \citep{maasch2024local}, since it requires known ancestral relationships between targets to estimate the groups of variables that can be used for adjustment.
We publish our code at \url{https://github.com/matyasch/snap}.

The local structure discovered by LDECC and MB-by-MB is sound regardless of the ancestral relationships between targets.
However, knowledge about the ancestral relationships is still needed for causal effect estimation.
Thus, we apply these methods on all targets separately and estimate if a target $X$ is an ancestor of another target $Y$ by testing their independence conditioned on the identified definite parents of $X$.

We also modify the publicly available implementation of local algorithms to speed up their execution time, as explained in App~\ref{sec:local_speed}.
We denote the modified methods as MB-by-MB* and LDECC*.
Following the authors, we use total conditioning \citep{pellet2008using} for Markov blanket discovery in MARVEL and  IAMB \citep{tsamardinos2003algorithms} in MB-by-MB* and LDECC*.
We follow Sec.~5.7.1 by \citet{mokhtarian2024recursive} to orient edges in MARVEL.

\begin{figure*}[t]
    \centering
    \includegraphics[width=.7\linewidth]{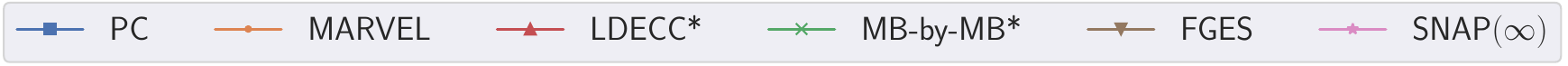}
    \begin{subfigure}[b]{\linewidth}
        \begin{subfigure}[b]{0.23\linewidth}
            \centering
            \caption*{ $\quad$ d-separation tests}
        \end{subfigure}
        \begin{subfigure}[b]{0.24\linewidth}
            \centering
            \caption*{$\quad$ Fisher-Z tests}
        \end{subfigure}
        \begin{subfigure}[b]{0.24\linewidth}
            \centering
            \caption*{KCI tests}
        \end{subfigure}
        \begin{subfigure}[b]{0.24\linewidth}
            \centering
            \caption*{$\chi^2$ tests}
        \end{subfigure}
        \includegraphics[width=\linewidth]{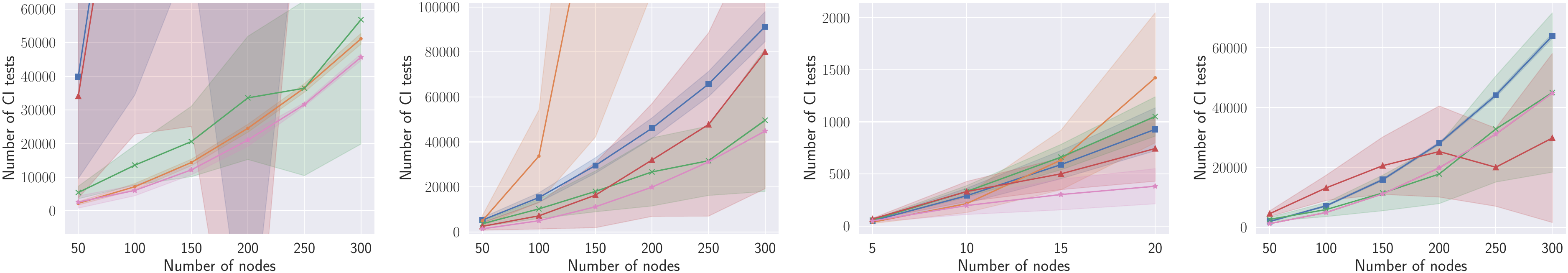}
    \end{subfigure}
    \begin{subfigure}[b]{\linewidth}
        \includegraphics[width=\linewidth]{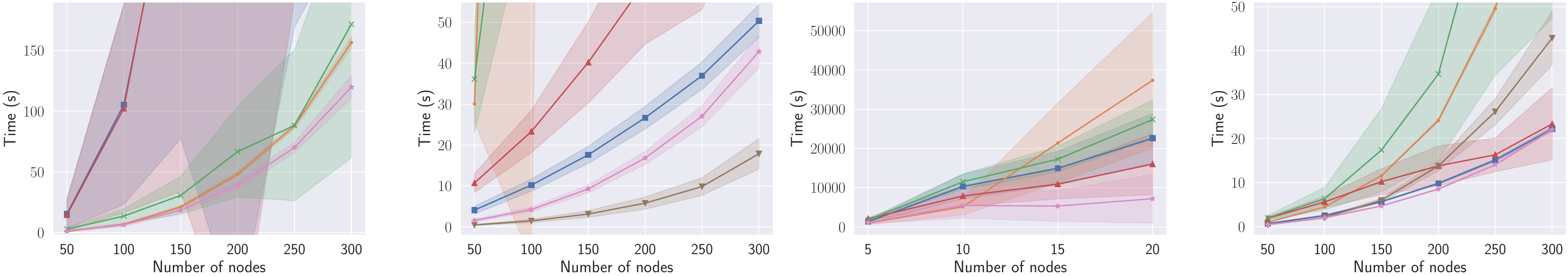}
    \end{subfigure}
    \caption{Number of \ac{CI} tests (top row) and computation time (bottom row) over number of nodes with $n_{\mathbf{T}}=4, \overline{d} = 3, d_{\max}=10$ and $n_{\mathbf{D}} = 1000$ data-points for different simulation settings in each column. The shadow area denotes the range of the standard deviation.
    SNAP $(\infty)$ is consistently one of the best methods.}
    \label{fig:test_and_time_per_node_unrest_std}
\end{figure*}

\paragraph{Metrics.} We compare the constraint-based algorithms in terms of the \ac{CI} tests they perform, and compare all algorithms, including FGES, in terms of computation time. To estimate the quality of the joint causal discovery and causal effect estimation, we report the intervention distance, i.e., the distance between the ground truth causal effect from the target $T$ to the target $T'$, and the predicted causal effect $\hat{\Theta}_{T, T'}$ for the same pair given an adjustment, which we define in App.~\ref{app:metrics}.
If the causal effect is identifiable from the CPDAG returned by a global algorithm or SNAP, then we estimate it using the optimal adjustment set \citep{henckel2022graphical} and $\hat{\Theta}_{T, T'}$ has a single value.

If the causal effect of $T$ on $T'$ is not identifiable, or the output is a local structure around $T$ by MB-by-MB* or LDECC*, then we estimate a set of possible causal effects $\hat{\Theta}_{T, T'}$ using the local structure around $T$ \citep{maathuis2009estimating}.
Following \citet{gradu2022valid}, we estimate causal effects using a separate set of samples after discovery.
Additionally, we also report the \Ac{SHD} between the ground truth and the estimated CPDAG, restricted to the possible ancestors of the target variables $\mathbf{T}$. For the subset of \emph{identifiable} causal effects, we  report the Adjustment identification distance \citep{henckel2024adjustment} for the parental, ancestor and optimal adjustment sets.

\paragraph{Synthetic data generation.}
We generate 100 random causal graphs with number of vertices $n_\mathbf{V} = 50, \dots, 300$, expected degree $\overline{d} = 2, \dots 4$ and maximum degree of $d_{\max} = 10$. For each graph, we generate $n_{\mathbf{D}} = 500, \dots, 10000$ number of data-points. For the linear Gaussian case, we sample edge weights from $[-3, -0.5] \cup [0.5, 3]$ with standard Gaussian noises.

For the binary data, we generate a conditional probability table according to the causal graph, where the  probabilities are sampled uniformly in  $[0,1]$.
For each graph, we sample randomly $n_\mathbf{T} = 2, 4, 6, 8$ targets.
To better show the general trends in the plots, we remove the top 5 and bottom 5 results for each method.

\subsection{Experimental Results}
We evaluate four settings: oracle \ac{CI} d-separation tests, Fisher-Z tests for the linear setting, KCI tests \citep{zhang2011kernel} for the linear setting, and $\chi^2$ test on binary data. For FGES, we use BIC \citep{schwarz1978estimating} for the linear setting and BDeu \citep{heckerman1995learning} for the binary setting. For KCI tests, we consider smaller graphs with $n_\mathbf{V} =5, \dots, 20$ nodes, because of computational constraints.
Fig.~\ref{fig:test_and_time_per_node_unrest_std} shows our results for $n_{\mathbf{T}}=4, \overline{d} = 3, d_{\max}=10$ and $n_{\mathbf{D}} = 1000$ 
in terms of the number of \ac{CI} tests performed and the execution time, over graphs with various number of nodes.

Under both metrics, SNAP$(\infty)$ performs consistently as one of the best methods across all domains, while the performance of other methods varies based on the setting. In particular, PC and LDECC* perform substantially more tests and are substantially slower for the oracle setting, while the gap is smaller for partial correlation and KCI tests. For the binary setting, LDECC* performs fewer tests, but is computationally comparable to SNAP$(\infty)$. On the contrary, MARVEL performs well with oracle tests, but requires more CI tests and is hence slower with linear Gaussian data. For KCI and $\chi^2$ tests, MARVEL performs comparably in the number of tests, but is substantially slower in  computation time. 
MB-by-MB* performs similarly to SNAP$(\infty)$ in terms of tests on most settings, except for KCI tests, but it is substantially slower for Fisher-Z and $\chi^2$ tests. FGES runs fastest when using the BIC score on linear data, but trails behind PC, LDECC* and SNAP$(\infty)$ on binary data using the BDeu score.

Fig.~\ref{fig:delta_time_per_node_unrest} shows the difference in computation time for the baselines methods and their versions combined with SNAP$(0)$, i.e., SNAP(k) with $k=0$, as a prefiltering method.
Prefiltering with SNAP$(0)$ can effectively improve the computation time in most settings.
Adding SNAP$(0)$ always improves the computation time of PC, MARVEL, and MB-by-MB*, and improves LDECC* on all domains except on binary data, while improving FGES on binary data. 
The corresponding figure for the number of CI tests and the computation times for each method with and without SNAP($0$) prefiltering (Fig.~\ref{fig:tests_and_time_per_node_unrest} in App.~\ref{app:additional_experiments}) shows the same trends.

In terms of causal effect estimation, our results in Fig.~\ref{fig:quality_per_node_unrest_std} and \ref{fig:quality_per_node_unrest} in App.~\ref{app:additional_experiments} show that all methods achieve comparable intervention distance, even though SNAP variants perform slightly worse in terms of \ac{SHD}. We show in
App~\ref{app:error_rates} that the difference is mostly due from pruning incorrectly some of the possible ancestors, and in App.~\ref{app:shd_opt} we demonstrate that this performance gap decreases substantially when we consider the \ac{SHD} of the induced CPDAG only over the targets and their oracle optimal adjustment sets.

\begin{figure}[t]
    \centering
    \includegraphics[width=\linewidth]{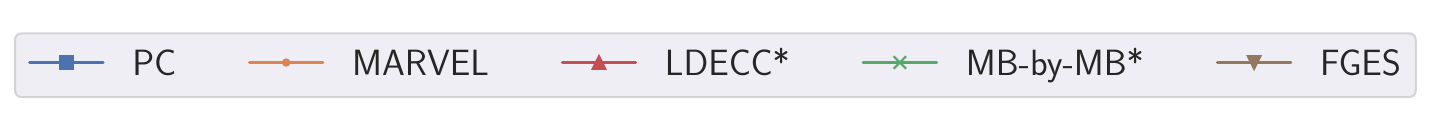}
    \includegraphics[width=.9\linewidth]{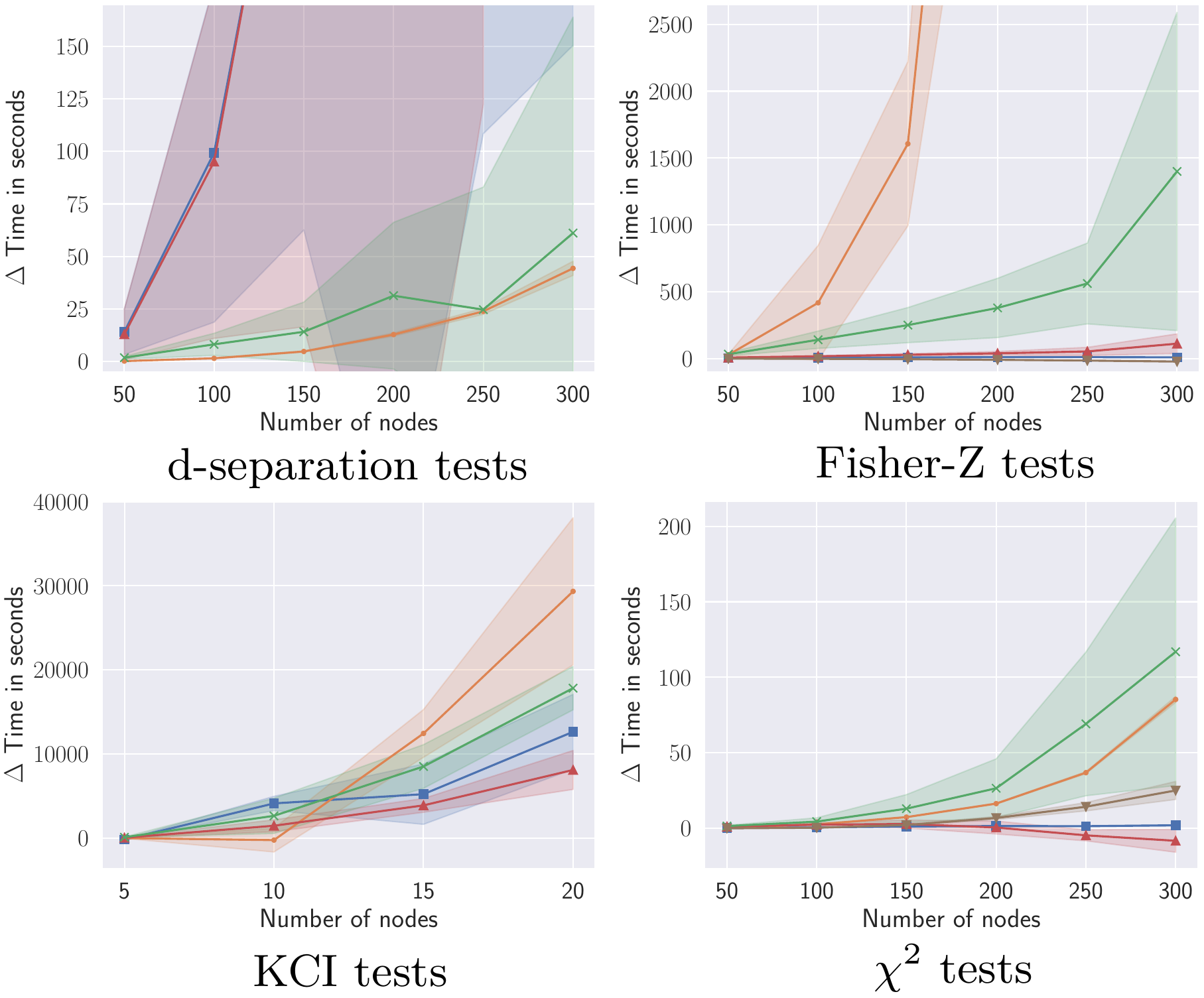}
    \caption{Difference in computation time for each baseline  with and without SNAP$(0)$ for $n_{\mathbf{T}}=4, \overline{d} = 3, d_{\max}=10$ and $n_{\mathbf{D}} = 1000$ data-points.
    Prefiltering with SNAP($0$) improves the baselines in most settings.
    \label{fig:delta_time_per_node_unrest}
    }    
\end{figure}

\begin{figure}
    \centering
    \includegraphics[width=.7\linewidth]{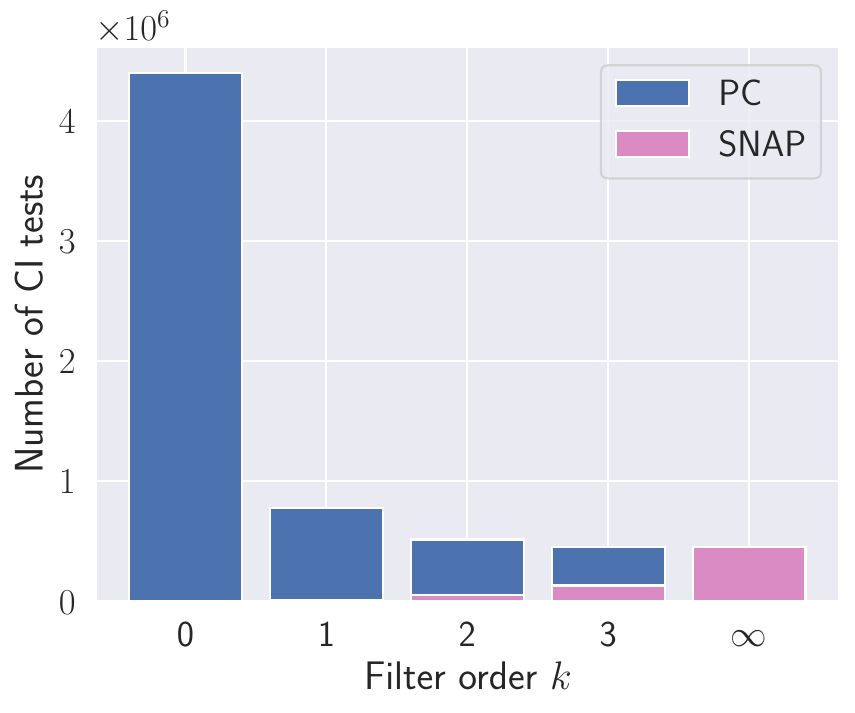}
    \caption{Number of d-separation tests performed by SNAP($k$) during prefiltering for $k=0,1,2,3$ and SNAP($\infty$) (pink), to which we add the number of tests PC performs after prefiltering (blue) on graphs with $|\mathbf{V}| = 50, n_{\mathbf{T}}=4, \overline{d} = 5$ and $d_{\max}=10$.
    }
    \label{fig:test_per_filter_order}
\end{figure}

In App.~\ref{app:identifiable}, we provide results on a subset of graphs where all the targets have \emph{identifiable} causal effects, and each target has to be the ancestor or the descendant of at least on other target.
Fig.~\ref{fig:appendix_per_node_ident_std} and \ref{fig:appendix_per_node_ident} show that forcing targets to be identifiable is the most beneficial for LDECC* and MB-by-MB*, due to the identifiable targets being more likely to have identifiable parents.
However, SNAP variants remain competitive in terms of computational performance even in this setting.
Fig.~\ref{fig:aid_per_node_ident_std} and \ref{fig:aid_per_node_ident} show that SNAP achieves mostly comparable, but in some cases slightly higher Adjustment Identification Distances (AID) than other methods. AID counts how many times the learned adjustment sets for various criteria are invalid given the ground truth causal graph. An invalid adjustment set can be still useful for finite sample prediction of causal effects, which seems to be the case in our experiments, given the comparable intervention distance across methods.

In App.~\ref{app:over_targets}, we vary the number of targets and fix all other parameters.
Our results show that the performance of SNAP($\infty$) is constant in the number of targets, unlike local methods, for which the performances worsen substantially with more targets.
This highlights the advantage of considering all targets jointly rather than independently one-by-one.
In App~\ref{app:over_degrees}, we show that SNAP variants remain top performers even on denser graphs, while the performance of the baselines depends highly on the setting.
App.~\ref{app:over_samples} shows that more data points may improve the \ac{SHD} of global methods, but have little impact on intervention distance.

App.~\ref{app:snapk} shows that SNAP$(0)$ is already almost as effective as SNAP$(1)$ and SNAP$(2)$ on the graphs considered in Fig.~\ref{fig:test_and_time_per_node_unrest_std} and \ref{fig:delta_time_per_node_unrest}.
However, prefiltering with $k > 0$ can be beneficial on denser graphs, especially in combination with PC.
Fig.~\ref{fig:test_per_filter_order} shows the number of d-separation tests performed by SNAP($k$) during prefiltering and PC after prefiltering on graphs with 50 nodes and an expected degree of 5.
These results show that increasing $k$ from 0 to 1 reduces the total number of \ac{CI} tests substantially, and increasing it to 2 and 3 further reduces it noticeably in this case.

App.~\ref{app:order} highlights how SNAP variants can reduce the number of higher order CI tests. This is particularly advantageous for the $\chi^2$ test on discrete binary data, where larger conditioning sets require a greater fraction of the samples to be excluded, diminishing the test’s power. Similarly, for KCI tests, larger conditioning set sizes are associated with a higher likelihood of type II errors \citep{zhang2011kernel}. By minimizing both the number and complexity of the tests, SNAP leverages these statistical advantages effectively

\paragraph{Real world data.}
We evaluate on the MAGIC-NIAB and Andes networks from bnlearn \citep{scutari2010bnlearn}, described in App.~\ref{app:real_networks}, as representative Gaussian and discrete models.
For both networks, we sample 100 times $n_{\mathbf{T}}=4$ identifiable targets and $n_{\mathbf{D}} = 1000$ number of data-points.
We report our results for MAGIC-NIAB in Tab.~\ref{tab:magic-niab}, which shows that SNAP reduces the number of CI tests and execution time compared to baselines. We show that the \ac{SHD} and intervention distance are comparable with most baselines in Tab.~\ref{tab:magic-niab_app} in App.~\ref{app:real-network-results}.  Tab.~\ref{tab:andes} shows that SNAP variants are also faster than baselines on the Andes data, while maintaining a comparable intervention distance.

\begin{table}
\centering
\footnotesize{
\begin{tabular}{lrrrr}
\hline
\multicolumn{1}{c}{} & \multicolumn{1}{c}{CI tests} & \multicolumn{1}{c}{Time} \\ \hline
PC                   & $12807 (\pm 2086)$           & $79.3 (\pm 24.7)$             \\
PC-SNAP($0$)           & $\mathbf{955 (\pm 10)}$       & $0.5 (\pm 0.1)$              \\ \hline
MARVEL               &{$8873 (\pm 3056)$}   & $27.3 (\pm 9.4)$                       \\
MARVEL-SNAP($0$)       & $960 (\pm 5)$                & $0.6 (\pm 0.1)$                       \\ \hline
LDECC*                & $18142 (\pm 2608)$           & $19.2 (\pm 4.0)$           \\
LDECC*-SNAP($0$)        & $981 (\pm 23)$               & $0.8 (\pm0.1)$        \\ \hline
MB-by-MB*             & $11464 (\pm 1995)$           & $25.7 (\pm 4.7)$        \\
MB-by-MB*-SNAP($0$)     & $972 (\pm 17)$               & $0.7 (\pm 0.2)$           \\ \hline
FGES                 & -                            & $0.7 (\pm 0.1)$             \\
FGES-SNAP($0$)         & -                            & $\mathbf{0.4 (\pm 0.1)}$     \\ \hline
SNAP($\infty$)       & $\mathbf{955 (\pm 10)}$       & $0.6 (\pm 0.2)$              \\ \hline
\end{tabular}
\caption{Results for the MAGIC-NIAB network. The best method is indicated in \textbf{bold}. Prefiltering with SNAP($0$) consistently improves most baselines.}
\label{tab:magic-niab}}
\end{table}

\section{CONCLUSIONS}
\label{sec:conclusion}

We propose SNAP, an efficient method for discovering the relevant portion of the CPDAG for causal effect estimation between target variables.
SNAP does not require prior knowledge of ancestral relationships and identifies statistically efficient adjustment sets.
We introduce two variants: SNAP($k$), which can be used as a preprocessing step for other discovery algorithms, and SNAP($\infty$), a stand-alone sound and complete discovery algorithm.
Our experiments show that both variants significantly reduce the number of CI tests and computation time while maintaining a comparable intervention distance. 
Future work will explore extending SNAP to causally insufficient settings.
\section*{Acknowledgements}

We thank SURF (\url{www.surf.nl}) for the support in using the National Supercomputer Snellius.

\bibliography{references}
\clearpage
\clearpage
\section*{Checklist}

 \begin{enumerate}

 \item For all models and algorithms presented, check if you include:
 \begin{enumerate}
   \item A clear description of the mathematical setting, assumptions, algorithm, and/or model. [Yes]
   \item An analysis of the properties and complexity (time, space, sample size) of any algorithm. [No]
   \item (Optional) Anonymized source code, with specification of all dependencies, including external libraries. [Yes]
 \end{enumerate}

 \item For any theoretical claim, check if you include:
 \begin{enumerate}
   \item Statements of the full set of assumptions of all theoretical results. [Yes]
   \item Complete proofs of all theoretical results. [Yes]
   \item Clear explanations of any assumptions. [Yes]     
 \end{enumerate}

 \item For all figures and tables that present empirical results, check if you include:
 \begin{enumerate}
   \item The code, data, and instructions needed to reproduce the main experimental results (either in the supplemental material or as a URL). [Yes]
   \item All the training details (e.g., data splits, hyperparameters, how they were chosen). [Yes]
   \item A clear definition of the specific measure or statistics and error bars (e.g., with respect to the random seed after running experiments multiple times). [Yes]
   \item A description of the computing infrastructure used. (e.g., type of GPUs, internal cluster, or cloud provider). [Yes]
 \end{enumerate}

 \item If you are using existing assets (e.g., code, data, models) or curating/releasing new assets, check if you include:
 \begin{enumerate}
   \item Citations of the creator If your work uses existing assets. [Yes]
   \item The license information of the assets, if applicable. [Yes]
   \item New assets either in the supplemental material or as a URL, if applicable. [Not Applicable]
   \item Information about consent from data providers/curators. [Not Applicable]
   \item Discussion of sensible content if applicable, e.g., personally identifiable information or offensive content. [Not Applicable]
 \end{enumerate}

 \item If you used crowdsourcing or conducted research with human subjects, check if you include:
 \begin{enumerate}
   \item The full text of instructions given to participants and screenshots. [Not Applicable]
   \item Descriptions of potential participant risks, with links to Institutional Review Board (IRB) approvals if applicable. [Not Applicable]
   \item The estimated hourly wage paid to participants and the total amount spent on participant compensation. [Not Applicable]
 \end{enumerate}

 \end{enumerate}

\newpage
\appendix
\clearpage

\onecolumn
\aistatstitle{SNAP: Sequential Non-Ancestor Pruning for Targeted Causal Effect Estimation With an Unknown Graph \\
Supplementary Materials}

\section{ORIENTING V-STRUCTURES}
\label{app:vstruct}

Algorithms~\ref{alg:pc-v} and \ref{alg:rfci-v} correspond to orienting v-structures as done in the PC algorithm \citep{spirtes2000causation} and the RFCI algorithm \citep{colombo2012learning} algorithms. Both algorithms use the meta-symbol $*$ as a wildcard to denote both edge tails and arrowheads. This allows edges to be oriented as bi-directed. Bi-directed edges indicate \emph{conflicting v-structures} that arise when performing \ac{CI} tests only up to some order $k$ \citep{wienobst2020recovering}. Fig.~\ref{fig:bidirected_cpdag} shows an example for how a bidirected edge can arise when performing \ac{CI} tests up to order 0.

Lemma~\ref{lem:vstructure} describes the conditions for correctly orienting v-structures.
In particular, the dependency conditions  $X \dep Z | sepset(X,Y)$ and $Z \dep Y | sepset(X,Y)$ are the same as in Lemma~3.1 by \citet{colombo2012learning}.
Thus, we adopt Algorithm~4.4 in \citep{colombo2012learning} to orient v-structures, as shown by Algorithm~\ref{alg:rfci-v}.
Fig.~\ref{fig:rfci_example} shows an example where PC-style v-structure orientations, as given by Algorithm~\ref{alg:pc-v}, would orient an edge bidirected, even though the edge exists in the true CPDAG.

\begin{algorithm}
\centering
\caption{OrientVstructPC: Orienting v-structures in the PC algorithm \citep{spirtes2000causation}}
\label{alg:pc-v}
\begin{algorithmic}[1]
    \Require Skeleton $\hat{U}$, separating sets $sepset$
    \State$\hat{G} \gets \hat{U}$
    \For{$X - Z - Y$ in $\hat{\mathbf{E}}$ such that $X \notin Adj_{\hat{U}}(Y)$}
        \If{$Z \notin sepset(X,Y)$}
            \State Orient $X \starleft-\starright Z \starleft-\starright Y$ as $X \starleft\to Z \gets\starright Y$ in $\hat{G}$
        \EndIf
    \EndFor
\State\Return{$\hat{G}$}
\end{algorithmic}
\end{algorithm}

\begin{algorithm}
\centering
\caption{OrientVstructRFCI: Orienting v-structures in the RFCI algorithm \citep{colombo2012learning}}
\label{alg:rfci-v}
\begin{algorithmic}[1]
    \Require Skeleton $\hat{U}$, separating sets $sepset$
    \State $M \gets \{(X,Z,Y) \text{ such that } X - Z - Y$ in $\hat{\mathbf{E}}$ \text{ and } $X \notin Adj_{\hat{U}}(Y)\}$
    \State $L \gets \{\}$
    \Repeat
        \State $X,Z,Y \gets$ choose an unshielded triple from $M$
        \If{$Z \notin sepset(X,Y)$}
            \If{$X \dep Z | sepset(X,Y)$ and $Z \dep Y | sepset(X,Y)$}
                \State $L \gets L \cup \{(X,Y,Z)\}$ \Comment{Add to legitimate v-structures}
            \Else
                \For{$V \in \{X,Y\}$}
                    \If{$V \indep Z | sepset(X,Y)$}
                        \State $\mathbf{S} \gets sepset(X,Y)$
                        \Comment{Find minimal separating set}
                        \State done $\gets$ False
                        \While{not done}
                            \State done $\gets$ True
                            \For{$S \in \mathbf{S}$}
                                \If{$V \indep Z | \mathbf{S} \setminus \{S\}$}
                                    \State $\mathbf{S} \gets \mathbf{S} \setminus \{S\}$
                                    \State done $\gets$ False
                                    \State \textbf{break}
                                \EndIf
                            \EndFor
                        \EndWhile
                        \State $sepset(V,Z) \gets sepset(Z,V) \gets \mathbf{S}$
                        \State $M \gets M \cup$ all triangles $(\min(V,Z),\cdot,\max(V,Z))$ in $\hat{U}$ \Comment{Update new unshielded triples.}
                        \State $M \gets M \setminus$ all triples in $M$ of the form $(V,Z,\cdot), (Z,V,\cdot), (\cdot,V,Z)$ and $(\cdot,Z,V)$
                        \State $L \gets L \setminus$ all triples in $L$ of the form $(V,Z,\cdot), (Z,V,\cdot), (\cdot,V,Z)$ and $(\cdot,Z,V)$
                        \State Delete the edge $V - Z$ from $\hat{U}$
                    \EndIf
                \EndFor
            \EndIf
        \EndIf
    \Until{$M$ is empty}
    \State$\hat{G} \gets \hat{U}$
    \For{$(X, Z, Y) \in L$}
        \State Orient $X \starleft-\starright Z \starleft-\starright Y$ as $X \starleft\to Z \gets\starright Y$ in $\hat{G}$
    \EndFor
\State\Return{$\hat{G}, sepset$}
\end{algorithmic}
\end{algorithm}

\begin{figure}[ht!]
    \centering
    \begin{subfigure}[b]{0.49\linewidth}
    \centering
    \begin{tikzpicture}[xscale = 1,yscale = 1.1]
        \node (a) at (0,0) {$A$};
        \node (b) at (1,0) {$B$};
        \node (c) at (0.5, 0.7) {$C$}
          edge [-stealth] (a)
          edge [-stealth] (b);
        \node (d) at (0.5, -0.7) {$D$}
          edge [-stealth] (a)
          edge [-stealth] (b);
        \node (u) at (-1,0) {$U$}
          edge [-stealth] (a);
        \node (v) at (2,0) {$V$}
          edge [-stealth] (b);
    \end{tikzpicture}
    \caption{Underlying true DAG}
    \label{fig:bidirected_dag}
    \end{subfigure}
    \begin{subfigure}[b]{0.49\linewidth}
    \centering
    \begin{tikzpicture}[xscale = 1,yscale = 1.1]
        \node (a) at (0,0) {$A$};
        \node (b) at (1,0) {$B$}
          edge [stealth-stealth, red] (a);
        \node (c) at (0.5, 0.7) {$C$}
          edge [-stealth] (a)
          edge [-stealth] (b);
        \node (d) at (0.5, -0.7) {$D$}
          edge [-stealth] (a)
          edge [-stealth] (b);
        \node (u) at (-1,0) {$U$}
          edge [-stealth] (a);
        \node (v) at (2,0) {$V$}
          edge [-stealth] (b);
    \end{tikzpicture}
    \caption{Partially directed graph $G^0$ found by SNAP}
    \label{fig:bidirected_cpdag}
    \end{subfigure}
  \caption{Example of a bidirected edge, based on Figure 2 in \citep{wienobst2020recovering}. Fig.~\ref{fig:bidirected_dag} shows the underlying true DAG. Fig.~\ref{fig:bidirected_cpdag} shows the partially directed graph $G^0$ found by SNAP (with any targets) at iteration $i=0$ after performing only marginal \ac{CI} tests. Nodes $A$ and $B$ are not separated by the empty set, thus there is still an edge between them. This results in two conflicting v-structures, $U \to A \gets B$ and $A \to B \gets V$, indicated by the bidirected edge $A \leftrightarrow B$. As shown by \citet{wienobst2020recovering}, given low-order \ac{CI} tests, nodes with such a bidirected edge between them are not adjacent in the true CPDAG.}
  \label{fig:bidirected_example}
\end{figure}
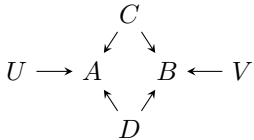
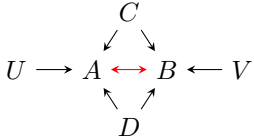

\section{PROOFS}
\label{app:proof}

\subsection{Proof for Lemma~\ref{thm:ancestors}}\label{app:ancestors}

\lemmaancestors*

\begin{proof}
    \citet{lauritzen1996graphical} (Proposition 3.22) and \citet{guo2023variable} (Lemma D.1) show that this holds for true causal graphs $D$ and ancestral sets.
    If two causal graphs are equal, then their \acp{CPDAG} are also equal.
    Since we assume that $\mathbf{V}^*$ is possibly ancestral (defined in Theorem.~\ref{thm:ancestors}), i.e., it contains all its own possible ancestors $PossAn_G(\mathbf{V^*}) \subseteq \mathbf{V^*}$, this implies that it also contains all its (actual) ancestors in the (unknown) ground truth DAG $D$. Hence, $\mathbf{V^*}$ is  ancestral in the unknown ground truth graph $D$, i.e. $An_D(\mathbf{V^*}) \subseteq \mathbf{V^*}$, in the MEC represented by the CPDAG $G$.
\end{proof}

\subsection{Proof for Theorem~\ref{thm:snap_is_sound}} \label{proof:snap_is_sound}

Theorem~\ref{thm:snap_is_sound} states that at each iteration $i$ of the SNAP(k) algorithm (Algorithm~\ref{alg:snap(k)}), the set of considered nodes $\hat{\mathbf{V}}^{i}$ contains all possible ancestors of the targets $\mathbf{T}$.
Additionally, $\hat{\mathbf{V}}^{i}$ is possibly ancestral, which means that it contains all of its own possible ancestors in the CPDAG $G$, i.e. $PossAn_G(\hat{\mathbf{V}}^{i}) \subseteq \hat{\mathbf{V}}^{i}$.  

To prove this result, we first show that no edge between the nodes $\hat{\mathbf{V}}^{i}$ are wrongly deleted during the execution of SNAP(k).
In other words, the intermediate undirected graph $\hat{U}^i$ over the subset $\hat{\mathbf{V}}^{i} \subseteq \mathbf{V}$ is a supergraph of $U|_{\hat{\mathbf{V}}^{i}}$, the induced subgraph of the true skeleton $U$ over $\hat{\mathbf{V}}^{i}$, at each iteration $i$. This means that  $\hat{U}^i$ contains all edges in $U|_{\hat{\mathbf{V}}^{i}}$, and possibly some additional edges.

\begin{lemma}
\label{lem:skeleton}
Given oracle conditional independence tests, at any iteration $i = 0,..,k$ of Algorithm~\ref{alg:snap(k)}, the undirected graph $\hat{U}^{i}$ is a supergraph of $U|_{\hat{\mathbf{V}}^{i}}$, the induced subgraph of the true skeleton $U$ over $\hat{\mathbf{V}}^{i}$.

\begin{proof}
In the Algorithm~\ref{alg:snap(k)} we only remove edges in $\hat{U}^i$ between two nodes $X, Y \in \hat{\mathbf{V}}^{i}$ for which we can find a separating set.
By faithfulness, we assume that these nodes are therefore also non-adjacent in the ground truth skeleton $U$.
Hence, the resulting skeleton $\hat{U}^i$ over the variables $\hat{\mathbf{V}}^{i}$ is a supergraph of the induced subgraph of the true skeleton, denoted by  $U|_{\hat{\mathbf{V}}^{i}}$.
\end{proof}

\end{lemma}

Next, we show that no possible ancestor of any target in $\mathbf{T}$ gets wrongly eliminated at any iteration $i$ of SNAP(k).
We first show that if we used the rules to orient v-structures from RFCI, described in Algorithm~\ref{alg:rfci-v}, for all iterations, this would hold.
We then prove that these rules are not necessary for $i=0,1$, but we can instead just use the standard rules for orienting v-structures in PC, described in Algorithm~\ref{alg:pc-v}.

\begin{lemma}
\label{lem:vstructure1}
Let $G$ be the CPDAG of the ground truth DAG $D$ and $\mathbf{T} \subseteq \mathbf{V}$ a set of targets.
Given oracle conditional independence tests, at any iteration $i = 0,..,k$ of SNAP (k) (Algorithm~\ref{alg:snap(k)}), let $\hat{U}^i$ be the undirected graph with nodes $\hat{\mathbf{V}}^{i}$ estimated at iteration $i$ and $\hat{G}^i$ be an initially undirected graph with the same skeleton. After orienting v-structures in $\hat{G}^i$ using the RFCI orientation rules, described in Algorithm~\ref{alg:rfci-v}, in SNAP(k) we collect in $\hat{\mathbf{V}}^{i+1}$ only the nodes $V \in \hat{\mathbf{V}}^{i}$ that have a possibly directed path to any $T \in \mathbf{T}$.
Then it holds that $PossAn_G(\mathbf{T}) \subseteq \hat{\mathbf{V}}^{i+1}$ and that $\hat{\mathbf{V}}^{i+1}$ is \emph{possibly ancestral}, i.e.,  
for all $V \in \hat{\mathbf{V}}^{i+1}$ it holds that $PossAn_G(V) \subseteq \hat{\mathbf{V}}^{i+1}$.

\end{lemma}
\begin{proof}
As shown in Lemma~\ref{lem:skeleton}, given oracle conditional independence tests at any iteration $i$ at Line 17 the skeleton $\hat{U}^i$ is a supergraph of the true skeleton $U$ restricted to the variables $\hat{\mathbf{V}}^{i}$.
We will now need to prove that orienting v-structures following the additional criteria in Algorithm~\ref{alg:rfci-v}, the orientations that we create in $\hat{G}^i$ are not in conflict with any possible ancestor of the variables $\hat{\mathbf{V}}^i$ in the ground truth CPDAG $G$. Our proof by contradiction follows the proof of Lemma~2 in \cite{wienobst2020recovering}. 

Let us assume that for an unshielded triple  $X \starleft-\starright Z \starleft-\starright Y$, the following three conditions hold: $Z \notin sepset(X,Y)$, $X \dep Z | sepset(X,Y)$ and $Z \dep Y | sepset(X,Y)$. Let us assume additionally that
$Z \in PossAn_G(X)$.
Then, there exists some DAG $D'$ in the \ac{MEC} represented by $G$ for which $Z \in An_{D'}(X)$. This means that there is a directed path from $Z \to \dots \to X$ in $D'$.
Since $Z \dep Y | sepset(X,Y)$, there exists at least one open path between $Y$ and $Z$ not blocked by $sepset(X,Y)$ in $D'$. 
However, since $X \indep Y | sepset(X,Y)$, there has to be a node $W \in sepset(X,Y)$ on the directed path from $Z \to \dots  W \dots \to X$ that blocks this directed path in $D'$.
Otherwise, the open path between $Y$ and $Z$ and the directed path from $Z$ to $X$ would not be blocked by $sepset(X,Y)$.
However, we also know that $X \dep Z | sepset(X,Y)$ implying that there is some other path between $X$ and $Z$ that is not blocked by $sepset(X,Y)$.
Thus, there is a path between $X$ and $Z$ and between $Z$ and $Y$ that are not blocked by $sepset(X,Y)$ in  $D'$.
In this case, $X \indep Y | sepset(X,Y)$ can only hold if $Z$ is a collider when connecting these two paths, because $Z \notin sepset(X,Y)$.
However, this path is then unblocked by $W$ since it is a descendant of $Z$ and it is in $sepset(X,Y)$.
It follows that $X \indep Y | sepset(X,Y)$ cannot hold in $D'$, where $Z \in An_{D'}(X)$ and hence we get a contradiction, proving that $Z \not \in PossAn_G(X)$.
Thus, if an edge is oriented as $X \starleft\to Z$ in $\hat G^i$ after satisfying the above three conditions, it holds that $Z \notin PossAn_{G}(X)$, which means that $Z$ is a definite non-ancestor of $X$ is $G$.

From this and Lemma~\ref{lem:skeleton} it follows that if there is a possibly directed path from $V' \in \hat{\mathbf{V}}^i$ to $V \in \hat{\mathbf{V}}^i$ defined by a sequence of nodes $\langle V', \dots V \rangle$ in the true CPDAG $G$, then a path with the same sequence of nodes also exists from $V'$ to $V$ in $\hat G^i$ and it is also possibly directed.
Hence, if $V' \in PossAn_G(V)$, then there is also a possibly directed path from $V'$ to $V$ in $\hat G^i$.

In Line 15 of Algorithm~\ref{alg:snap(k)} we collect in $\hat{\mathbf{V}}^{i+1}$ only the nodes $V \in \hat{\mathbf{V}}^{i}$ that have a possibly directed path to any $T \in \mathbf{T}$ in $G^i$.

Let us consider a node $V$ that has a possibly directed path to a target $T \in \mathbf{T}$.
If a node $V'$ has a possibly directed path to $V$, then it also has a possibly directed path to $T$ through $V$.
Hence, no $V'$ that has a possibly directed to $V$ is removed from the nodes under consideration $\hat{\mathbf{V}}^{i+1}$.
We then conclude that in the true CPDAG $G$, for all $V \in \mathbf{\hat{V}}^{i}$:
$$
PossAn_G(V) \subseteq  \mathbf{\hat{V}}^{i}.
$$

\end{proof}

The Lemma above states that if we use the orientations rules for v-structures in RFCI, described in Algorithm~\ref{alg:rfci-v}, our Algorithm~\ref{alg:snap(k)} will identify a superset of the ground truth possible ancestors for all the variables that we consider at iteration $i$. In the following, we show an optimization that allows us to avoid checking the two additional dependencies for iterations $i=0,1$, thus allowing us to reduce the computation and test time.
While the case of $i=0$ is simple, the case of $i=1$ requires particular care, so we first state the following helper lemma.

\begin{lemma}
\label{lem:separated}
Let $\{W,X,Y,Z\}$ be nodes in a DAG $D$, with $X$ and $Y$ d-connected given the empty set.
If $X$ is d-separated from $Y$ given $Z$, and $Z$ is d-separated from $Y$ given $W$, then $X$ is d-separated from $Y$ given $W$, i.e.,
$$X \not\perp_d Y \  \land  \ X \perp_d Y|Z \ \land\ Z\perp_d Y| W \implies X \perp_d Y|W.$$

\begin{proof}
As $X$ and $Y$ are d-connected given the empty set, they are connected by a trek in $D$, i.e., paths without colliders, thus either directed from $X$ to $Y$, from $Y$ to $X$, or of the form $\langle X \gets ... \gets V \to ... \to Y \rangle$.
If $Z$ d-separates $X$ and $Y$, then $Z$ blocks all treks between $X$ and $Y$.
Hence, all treks between $X$ and $Y$ are of the form $\langle X, ..., Z, ..., Y \rangle$ in $D$. Additionally,  these treks imply that $Z$ is d-connected to $Y$.
Similarly, if $Z$ and $Y$ are d-separated by $W$, then $W$ must block all treks between $Z$ and $Y$.
Since any subpath of a trek is also a trek, and conditioning on any node along the trek will block it, it follows that any trek between $X$ and $Y$ blocked by $Z$ is also blocked by $W$, and so $W$ also blocks all treks between $X$ and $Y$.
Therefore, all treks between $X$ and $Y$ are of the form $\langle X,\dots,Z,\dots,W,\dots,Y \rangle$.

Next, we show, by contradiction, that conditioning on $W$ also does not unblock any other paths between $X$ and $Y$.
Thus assume the opposite, i.e., that $W$ unblocks a path $\pi$ between $X$ and $Y$.
This path must contain one or more colliders $V_i$, as all treks (paths without colliders) are blocked by $W$.
Thus, $\pi = \langle X, ..., V_1, ..., V_2, ..., V_k, ..., Y \rangle$ in $D$, such that $\forall i: V_i \in An_D(W)$, and $\{V_1,..,V_k\}$ are all colliders along $\pi$, i.e.\ $\pi$ is of the form $X \ldots \to V_1 \gets \ldots \to V_2 \gets (\ldots) \to V_k \gets \ldots Y$ in $D$. Note that possibly for one collider along $\pi$ we have $V_j = W$ (as by definition $W \in An_D(W)$), but not for more than one, otherwise $\pi$ would not be a path (sequence of distinct vertices).
We show that the existence of such an unblocked path $\pi$ given $W$ would imply the existence of an unblocked path between $X$ and $Y$ given $Z$, and/or the existence of an unblocked path between $Z$ and $Y$ given $W$, both contrary to the given.

First, we know that $Z$ is not on $\pi$, otherwise $Z$ and $Y$ would be d-connected given $W$ via a subpath of $\pi$.
Furthermore, $W \notin An_D(Z)$, otherwise $\pi$ would also be unblocked given $Z$.
Therefore, all treks between $Z$ and $W$ must be \emph{into} $W$ (and so, by extension, all treks between $Z$ and $Y$ must be of the form $\langle Z\ \ldots \to W \to \ldots \to Y \rangle$, i.e., \emph{into} $Y$).
But then, if there are one or more colliders $V_i \in An_D(Z)$, then let $V_j \in An_D(Z)$ be the one closest to $Y$ along $\pi$.
Then the path $Z \gets \ldots \gets V_j \gets \ldots (\to V_k \gets) \ldots Y$, i.e.,\  the concatenation of the directed path $V_j \to \ldots \to Z$ and the subpath $\langle V_j,..,Y \rangle$ from $\pi$, would be unblocked given $W$, contrary to the given. Note that $V_j \neq W$, as we already concluded that $W \notin An_D(Z)$.

So in particular, $V_k \notin An_D(Z)$.
But that implies that there are unblocked paths $\langle Z\ \ldots \to W \rangle$ and $\langle W \gets \ldots \gets V_k \gets \ldots Y \rangle$ in $D$ (given the empty set), that are both \emph{into} $W$.
Then, conditioning on $W$ would open up an unblocked path between $Z$ and $Y$ given $W$, contrary to the given.

Therefore, the assumption that there is an unblocked path $\pi$ between $X$ and $Y$ given $W$ leads to a contradiction, and hence $X$ is d-separated from $Y$ given $W$.

\end{proof}
\end{lemma}

\begin{restatable}[]{lemma}{lemmavstruct}
\label{lem:vstructure}
Let $G$ be the ground truth CPDAG. At any iteration $i = 0,..,k$ of Alg.~\ref{alg:snap(k)}, let $\hat{G}^i$ be the estimated mixed graph at iteration $i$ after orienting the v-structures with Alg.~\ref{alg:pc-v} for $i=0,1$ and Alg.~\ref{alg:rfci-v} for $i\geq2$.
Then for all $V \in \hat{\mathbf{V}}^{i+1}$ it holds that
$
PossAn_G(V) \subseteq  \hat{\mathbf{V}}^{i+1}$.
\end{restatable}

\begin{proof}
Lemma~\ref{lem:vstructure1} shows that the result holds if we only orient unshielded triples $X \starleft-\starright Z \starleft-\starright Y$, if $Z \notin sepset(X,Y)$, and if additionally the following two dependencies hold: $X \dep Z | sepset(X,Y)$ and $Z \dep Y | sepset(X,Y)$, as tested in the RFCI orientation rules in Algorithm~\ref{alg:rfci-v}.

At iterations $i\geq2$, Algorithm~\ref{alg:rfci-v} explicitly tests these dependencies, hence we can apply the result directly.

We now show that for iterations $i=0,1$ we do not need to test them explicitly, since they will also hold even if we use the simpler Algorithm~\ref{alg:pc-v} to orient v-structures. In particular:
\paragraph{Case $i=0$:} At iteration $i=0$ every pair of variables is tested for independence with the empty conditioning set during the skeleton step.
Thus, if there is still an edge between $X$ and $Z$ and between $Z$ and $Y$, then they are dependent given the empty set, or in other words, $X \dep Z$ and $Y \dep Z$ always hold.

\paragraph{Case $i=1$:} At iteration $i=1$, we check that $Z \notin sepset(X,Y)$. 
Given that we are testing separating set sizes of maximum order $1$, then $sepset(X,Y)$ is either empty or 
a single variable.
If $sepset(X,Y)$ is empty, then we know that $X \dep Z$ and $Y \dep Z$ also holds, since every pair of variables is tested for marginal independence and these variables are still adjacent after the marginal tests. 

Otherwise, the separating set is a single variable, $sepset(X,Y) = \{W\}$.
We will prove how this implies $X \dep Z | sepset(X,Y)$. The same proving strategy can be then applied to proving $Y \dep Z | sepset(X,Y)$, by substituting $X$ with $Y$. 

Let us consider two cases, when $W$ is adjacent to $X$ or $Z$, and when it is not adjacent to either.
\begin{itemize}
\item If $W$ is still adjacent to $X$ or $Z$ in $G^1$ after the skeleton search phase, then the skeleton search phase tested $X \indep Z | \{W\}$.
Since $Z$ is still adjacent to $X$, this test must have returned dependence, i.e., $X \dep Z | sepset(X,Y)$.

\item If instead $W$ is not adjacent to neither $X$ nor $Z$ in $G^1$, then the skeleton search found a node that blocks all unblocked paths between $W$ and $X$, or $W$ and $Z$.
Without loss of generality, we focus on the first case for $W$ and $X$, since the second case for $W$ and $Z$ follows the same logic.
Thus, if $W$ is not adjacent to $X$, there exists a node $V$, such that $X \indep W | \{V\}$.
According to Lemma~\ref{lem:separated}, if
$X \dep Z$, $ X \indep W | \{V\}$ and $X \indep Z | \{W\}$, then $X \indep Z | \{V\}$.
We consider two cases, one in which $V$ is still adjacent to $X$ and the other in which it is not.
\begin{itemize}
    \item If $V$ is still adjacent to $X$ in $G^1$ after the skeleton search phase, then the skeleton search phase tested already $X \indep Z | \{V\}$.
    Since $Z$ is still adjacent to $X$, this test must have returned dependence, i.e., $X \dep Z | \{V\}$, a contradiction with what Lemma~\ref{lem:separated} implies.
    Since we know the first two conditions of the Lemma already hold, if $V$ is still adjacent to $X$, then $X \indep Z | \{W\}$ cannot hold, i.e., $X \dep Z | \{W\}$, which means $X \dep Z | sepset(X,Y)$.

    \item If instead $V$ is also \emph{not} adjacent to $X$ in $G^1$, then it means the search found a node $V'$ such that $V$ is d-separated from $X$ given $V'$, i.e., $X \indep V | \{V'\}$.
    Then, by two successive applications of Lemma~\ref{lem:separated}, if $X \indep V | \{V'\}$ and $X \indep W | \{V\}$, then it holds that $X \indep W | \{V'\}$, and furthermore, if $X \indep W | \{V'\}$ and $X \indep Z | \{W\}$, then it holds that $X \indep Z | \{V'\}$.
    However, by the same logic as before, if $V'$ is still adjacent to $X$ in $\hat{G}^1$, then the skeleton search tested $X \indep Z | \{V'\}$, which returned dependence, i.e., $X \dep Z | \{V'\}$, a contradiction.
    Thus, if $V'$ is still adjacent to $X$, then $X \indep Z | \{W\}$ cannot hold, i.e., $X \dep Z | \{W\}$, which means $X \dep Z | sepset(X,Y)$.
    \vspace{2mm}

    Similarly, if $V'$ is also \emph{not} adjacent to $X$ in $G^1$, due to some other node $V''$, then we can follow the same logic by successively applying Lemma~\ref{lem:separated} three times to arrive at the same conclusion, that either $X \dep Z | \{W\}$ or $V''$ is also \emph{not} adjacent to $X$ in $G^1$.
    Given that there are only a finite number of nodes in the graph, we can continue this argument until we find a node $V^*$ such that $V^*$ is still adjacent to $X$ in $G^1$, in which case $X \indep Z | \{W\}$ cannot hold, i.e., $X \dep Z | \{W\}$, which means $X \dep Z | sepset(X,Y)$.
    \end{itemize}
\end{itemize}

The same derivation applies also for $Y$, meaning that both $X \dep Z | sepset(X,Y)$ and $Y \dep Z | sepset(X,Y)$ has to hold.
We then showed that at iterations $i=0,1$, the two dependencies $X \dep Z$ and $Y \dep Z$ always hold even if not tested explicitly.
Hence, the result of Lemma~\ref{lem:vstructure1} applies for iterations $i=0,1$ even when using Algorithm~\ref{alg:pc-v} to orient v-structures.
\end{proof}

\snapsound*

\begin{proof}

Lemma~\ref{lem:vstructure1} and Lemma~\ref{lem:vstructure} show that if $\hat{\mathbf{V}}^{i+1}$ is the set of remaining nodes obtained at line 15 in iteration $i$ of Algorithm~\ref{alg:snap(k)}, then $PossAn_G(\mathbf{T}) \subseteq \hat{\mathbf{V}}^{i+1}$ and that $\hat{\mathbf{V}}^{i+1}$ is \emph{possibly ancestral}, i.e.,  
for all $V \in \hat{\mathbf{V}}^{i+1}$ it holds that
$
PossAn_G(V) \subseteq  \hat{\mathbf{V}}^{i+1}$.
Since $\hat{\mathbf{V}}^{i+1}$ is not modified later, this also holds for iteration $i+1$.

\end{proof}

\subsection{Proof for Corollary~\ref{cor:snap-prefiltering}}

\snapkprefiltering*

\begin{proof}
    The proof follows from the application of Theorem~\ref{thm:snap_is_sound} which returns a possibly ancestral set of nodes that contains $PossAn(\mathbf{T})$, and from the application of Theorem~\ref{thm:ancestors}, that shows that a possibly ancestral set will have a valid CPDAG that is the same as restricting the CPDAG to the set.
\end{proof}

\subsection{Proof for Theorem~\ref{thm:snap_is_complete}}
\label{proof:snap_is_complete}

To prove Theorem~\ref{thm:snap_is_complete}, we first show that the skeleton and the v-structures are complete in SNAP($k$) (Algorithm~\ref{alg:snap(k)}) with respect to the ground truth CPDAG, when we run SNAP(k) until completion, i.e., $k=|\mathbf{V}|-2$.

\begin{lemma}\label{snapk_completion}
Let $k=|\mathbf{V}|-2$ and $\hat{G}^k$ be the mixed graph estimated by SNAP($k$) (Algorithm~\ref{alg:snap(k)}) at iteration $k$. Then $\hat{G}^k$ has the same skeleton and v-structures $G|_{\hat{\mathbf{V}}^k}$, the induced subgraph of the true CPDAG $G$ over $\hat{\mathbf{V}}^k$.

\begin{proof}
Let $D$ be the ground truth DAG with true skeleton $U$ and true \ac{CPDAG} $G$.
Let $\hat{U}^\infty$, $\hat{G}^\infty$ and $\hat{\mathbf{V}}^{\infty}$ denote the final skeleton, mixed graph and remaining nodes estimated at iteration $k=|\mathbf{V}|-2$.

Lemma~\ref{lem:skeleton} shows that $\hat{U}^{\infty}$ is a supergraph of $U|_{\hat{\mathbf{V}}^\infty}$, the induced subgraph of the true skeleton $U$ over $\hat{\mathbf{V}}^{\infty}$.
Theorem~\ref{thm:snap_is_sound} shows that $\hat{\mathbf{V}}^{\infty}$ is a possibly ancestral set, i.e. $PossAn_G(\hat{\mathbf{V}}^\infty) \subseteq \hat{\mathbf{V}}^\infty$.
Since the parents of any node are a subset of its possible ancestors,  $\hat{\mathbf{V}}^{\infty}$ contains all parents for all variables in $\hat{\mathbf{V}}^{\infty}$.
Then it follows that every node in $V \in \hat{\mathbf{V}}^{\infty}$ 
is adjacent in $\hat{U}^\infty$ to its ground truth parents in $D$, i.e. $Pa_D(V) \subseteq Adj_{\hat{U}^\infty}(V)$.

Since $|\mathbf{V}|-2$ is the maximum size of any conditioning set for $\mathbf{V}$ vertices, this means that SNAP$(k=|\mathbf{V}|-2)$ is allowed to test independence with any size of conditioning sets.
Hence, each pair of nodes $X,Y \in \hat{\mathbf{V}}^{\infty}$ still adjacent in $\hat{U}^\infty$ has been tested for independence using sets that include both $Pa_D(X)$ and $Pa_D(Y)$.
Thus, if $X$ and $Y$ are adjacent in $\hat{U}^\infty$, then since we assume faithfulness, they cannot be d-separated in the ground truth graph $D$, and hence they are also adjacent in the true skeleton $U$.
Due to Lemma~\ref{lem:skeleton}, every non-adjacent pair in $\hat{U}^\infty$ is also non-adjacent in the induced subgraph of the true CPDAG $G$ over $\hat{\mathbf{V}}^{\infty}$.
We conclude, that $\hat{U}^\infty$ is identical to $U|_{\hat{\mathbf{V}}^{\infty}}$, the induced subgraph of the true skeleton over $\hat{\mathbf{V}}^{\infty}$.

We obtain $\hat{G}^\infty$ by orienting v-structures in $\hat{U}^\infty$ with Algorithm~\ref{alg:pc-v} for $k=0,1$ or Algorithm~\ref{alg:rfci-v} for $k \geq 2$.
If two nodes are adjacent in $\hat{U}^k$, there exists no set separating them.
Thus, for any unshielded triple $X \starleft-\starright Z \starleft-\starright Y$ with $Z \notin sepset(X,Y)$, it automatically holds that $X \dep Z | sepset(X,Y)$ and $Z \dep Y | sepset(X,Y)$.
From Lemma~3.1 in \citep{colombo2012learning} it follows that every v-structure is correctly oriented.
Since, unshielded triples in $U^\infty$ and $U|_{\hat{\mathbf{V}}^{\infty}}$ are identical, it follows that the v-structures in $\hat{G}^\infty$ are identical to the v-structures in $G|_{\hat{\mathbf{V}}^{\infty}}$, the induced subgraph of $G$ over $\hat{\mathbf{V}}^{\infty}$.
\end{proof}
\end{lemma}

\snapcomplete*
\begin{proof}
Given vertices $\mathbf{V}$, the first step of SNAP$(\infty)$ (Algorithm~\ref{alg:snap(inf)}) runs SNAP$(k)$ (Algorithm~\ref{alg:snap(k)}) with maximum iterations $k=|\mathbf{V}|-2$, and obtains the remaining nodes $\hat{\mathbf{V}}$ and the resulting mixed graph $\hat{G}$.
From Theorem~\ref{thm:snap_is_sound} it follows that $\hat{\mathbf{V}}$ contains $PossAn_G(\mathbf{T})$.
From Lemma~\ref{snapk_completion} it follows that $\hat{G}$ has the same skeleton and v-structure as $G|_{\hat{\mathbf{V}}}$, the induced subgraph of the true CPDAG $G$ over the remaining nodes $\hat{\mathbf{V}}$.
Thus, $\hat{G}$ and $G|_{\hat{\mathbf{V}}}$ belong to the same equivalence class.

The second step of SNAP$(\infty)$ applies Meek's rules on $\hat{G}$, which are sound and complete in terms of orientations and hence finally $\hat{G}$ is exactly  $G|_{\hat{\mathbf{V}}}$.
The final nodes $\hat{\mathbf{V}}$ are obtained at line 10 of Algorithm~\ref{alg:snap(inf)} by collecting every node with a possibly directed path to any target in $\hat{G}$.
Then, $\hat{\mathbf{V}}$ are exactly the possible ancestors of $\mathbf{T}$ according to the true CPDAG $G$, and the final graph returned by SNAP$(\infty)$ is $\hat{G} = G|_{PossAn(\mathbf{T})}$.
\end{proof}

\subsection{Computational complexity}
\label{app:complexity}
In this section, we discuss the computational complexity of SNAP($\infty$) in terms of \ac{CI} tests performed.
In the worst-case scenario, all nodes are possible ancestors of the targets and no nodes can be ever pruned at line 15 of Algorithm~\ref{alg:snap(k)}.
Thus, the worst-case complexity of unique \ac{CI} tests performed by SNAP($\infty$) equals to the complexity of the PC-style skeleton search plus the complexity of the RFCI orientation rules (Algorithm~\ref{alg:rfci-v}).
The computational complexity of PC is $\mathcal{O}(|\mathbf{V}|^{d_{\max}+2})$ \citep{spirtes2000causation}.

We did not find a formal complexity analysis of the RFCI orientation rules in the literature, so we show in Lemma~\ref{lem:rfci_complexity} that its worst-case complexity in terms of \ac{CI} tests is at most $\mathcal{O}(|\mathbf{V}|^4)$.

\rfcicomplexity*
\begin{proof}
    The worst-case computational complexity of the RFCI orientation rules depends on the maximum number of unshielded triples in a graph.
    In a graph with variables $\mathbf{V}$, there are at most $|\mathbf{V}|$ central nodes $Z$ in all unshielded triples $X - Z -Y$.
    If each central node $Z$ has at most $d_{\max}$ neighbors, then it has at most $d_{\max} \cdot (d_{\max} -1)$ pairs of neighbors $(X,Y)$.
    Hence, there are at most $\mathcal{O}(|\mathbf{V}| \cdot d_{\max} \cdot (d_{\max} -1))$ number of unshielded triples in the graph. 
    In the worst case, the central nodes $Z$ are never in the separating sets of any pairs of their neighbors, i.e. $ Z \notin sepset(X,Y)$.
    Then, the RFCI orientation rules perform two \ac{CI} tests for each unshielded triple, resulting in $\mathcal{O}(|\mathbf{V}| \cdot d_{\max} \cdot (d_{\max} -1))$ \ac{CI} tests.
    For simplicity, we can upper bound this to $\mathcal{O}(|\mathbf{V}|^3)$, by taking advantage of the fact that the maximum degree $d_{\max}$ can be at most $|\mathbf{V}| - 1$, e.g., when we consider a single node connected to all other nodes.

    Each time an independence is found by these $\mathcal{O}(|\mathbf{V}|^3)$ \ac{CI} tests, the RFCI orientation rules find a corresponding minimal separating set and remove the corresponding edge from the skeleton.
    When the edge is removed, it stops being part of any unshielded triples and the RFCI orientation rules do not test the independence of the corresponding pair again.
    Thus, the amount of independence found is upper bounded by the number of edges in the skeleton, which is at most $|\mathbf{V}|(|\mathbf{V}| -1)/2$ in a fully connected skeleton upper bounded by $|\mathbf{V}|^2$.
    A corresponding minimal separating set can be found in at most $\mathcal{O}(|\mathbf{V}|^2)$ number of \ac{CI} tests \citep{tian1998finding}.
    This results in $\mathcal{O}(|\mathbf{V}|^2 \cdot |\mathbf{V}|^2)= \mathcal{O}(|\mathbf{V}|^4)$ additional CI tests.

    In total, the number of \ac{CI} tests to check unshielded triples is upper bounded by $\mathcal{O}(|\mathbf{V}|^3)$, and the number of \ac{CI} tests to find minimal separating sets for each independence found is upper bounded $\mathcal{O}(|\mathbf{V}|^2 \cdot |\mathbf{V}|^2) = \mathcal{O}(|\mathbf{V}|^4)$.
    This means that the number of \ac{CI} tests performed by the RFCI orientation rules is upper bounded by $\mathcal{O}(|\mathbf{V}|^3 + |\mathbf{V}|^4) = \mathcal{O}(|\mathbf{V}|^4)$.
\end{proof}

From Lemma~\ref{lem:rfci_complexity} and the classical result on PC-style skeleton search by \citet{spirtes2000causation}, it follows that the worst-case computational complexity of SNAP($\infty$) is at most
$$
\mathcal{O}(|\mathbf{V}|^{d_{\max}+2} + |\mathbf{V}|^4),
$$
where the first part refers to the skeleton search and the second part refers to the RFCI orientation rules. Since a maximum degree $d_{\max}$ of 2 means that each node can have at most 2 neighbors, which implies very sparse graph, many real-world graphs will have a higher maximum degree $d_{\max}$. For these graphs, the worst-case complexity of SNAP will match the worst-case complexity of PC, since it will be dominated by the skeleton search.

\snapcomplexity*
\begin{proof}
We can easily show that the complexity of the skeleton search dominates the complexity of the RFCI orientation rules for ${d_{\max}} \geq 2$, since it holds that
$$
\mathcal{O}(|\mathbf{V}|^{d_{\max}+2} + |\mathbf{V}|^4) = \mathcal{O}(|\mathbf{V}|^{d_{\max}+2}).
$$
\end{proof}

According to our empirical results, SNAP($\infty$) performs substantially fewer \ac{CI} tests than PC in practice.
However, there exist counter examples where SNAP($\infty$) performs more \ac{CI} tests than PC due to the tests performed by the RFCI orientation rules.
We show such an example graph on Figure~\ref{fig:more_tests_example} with targets $\mathbf{T}=\{X_1, X_2\}$.
In this example, all nodes are possible ancestors of the targets, thus no nodes can be pruned.
Furthermore, SNAP($\infty$) performs one additional \ac{CI} test at order $k=2$ during the RFCI orientation rule phase, that PC does not perform.
Thus, in total, SNAP($\infty$) performs one more test than PC.

\begin{figure}[ht!]
    \centering
    \begin{tikzpicture}[xscale = 1,yscale = 1.1]
        \node (x1) at (1,0) {$X_1$};
        \node (x2) at (0,0) {$X_2$}
          edge [-stealth] (x1);
        \node (x3) at (2,1) {$X_3$}
          edge [-stealth] (x2);
        \node (x4) at (2,-1) {$X_4$}
          edge [-stealth] (x1)
          edge [-stealth] (x3);
        \node (x5) at (-1,1) {$X_5$}
          edge [-stealth] (x2)
          edge [-stealth] (x3);
        \node (x6) at (-1,-1) {$X_6$}
          edge [-stealth] (x2)
          edge [-stealth] (x4)
          edge [-stealth] (x5);
    \end{tikzpicture}
    \caption{An example graph with targets $\mathbf{T}=\{X_1, X_2\}$ on which SNAP($\infty$) performs more \ac{CI} tests than PC.}
    \label{fig:more_tests_example}
\end{figure}
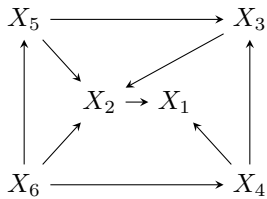

\subsection{Rough approximation of the expected number of possible ancestors}
\label{app:expectedanc}

In practice, the performance of SNAP  depends on the number of possible ancestors of targets in the graph.
We can get an idea of the expected number of \ac{CI} tests by estimating the expected number of possible ancestors.
We could not find an easy solution to this problem in the literature, so we provide a rough approximation for the expected size of the possible ancestors of $\mathbf{T}$ with empirical results that show that it is a substantial overestimation.

Assume $|\mathbf{T}|$ targets are chosen uniformly from all nodes $\mathbf{V}$ without replacement.
Given the topological order of the graph $1..|\mathbf{V}|$, if the highest target in the ordering is at rank $i$, then the rest of the $|\mathbf{T}-1|$ number of targets had to be chosen from nodes at orders $1..i-1$.
Choosing $|\mathbf{T}-1|$ targets from $i-1$ nodes can be done in $\binom{i-1}{|\mathbf{T}|-1}$ many ways.
Then, the probability of the highest target having rank $i$ is given by the number of ways that can be achieved, i.e., $\binom{i-1}{|\mathbf{T}|-1}$, normalised by all the ways one can choose $|\mathbf{T}|$ targets from all $|\mathbf{V}|$ amount of nodes, i.e. $\binom{|\mathbf{V}|}{|\mathbf{T}|}$, resulting in the probability $\binom{i-1}{|\mathbf{T}|-1} / \binom{|\mathbf{V}|}{|\mathbf{T}|}$.
We get the expected rank of the highest ranking target by taking the expectation over all possible highest ranks $i = |\mathbf{T}| .. |\mathbf{V}|$ as follows:
\begin{equation}\label{eq:exp_anc}
    \hat{M} = \frac{1}{ \binom{|\mathbf{V}|}{|\mathbf{T}|} } \sum_{i = |\mathbf{T}|}^{|\mathbf{V}|} i \binom{{i-1}}{|\mathbf{T}|-1}
\end{equation}

We can now overestimate the size of the possible ancestors of $\mathbf{T}$ by simply considering it to be $M$.
As $|\mathbf{T}|$ increases, this upper bound will become tighter, but it is still a loose overestimation.
As shown in Table~\ref{tab:exp_anc}, we empirically find that the number of possible ancestors is much lower than the above estimate.

\begin{table*}[ht]
\begin{tabular}{@{}lrrrrrr@{}}
\toprule
Nodes        & \multicolumn{1}{c}{50} & \multicolumn{1}{c}{100} & \multicolumn{1}{c}{150} & \multicolumn{1}{c}{200} & \multicolumn{1}{c}{250} & \multicolumn{1}{c}{300} \\ \midrule
$\bar{M}$   & $19.64 (\pm 7.16)$     & $23.95 (\pm 10.85)$     & $26.88 (\pm 13.60)$     & $29.49 (\pm 18.66)$     & $28.18 (\pm 15.50)$     & $33.78 (\pm 17.97)$     \\
$\hat{M}$ & $40.8$                 & $80.8$                  & $120.8$                 & $160.8$                 & $200.8$                 & $240.8$                 \\ \bottomrule
\end{tabular}
\caption{Estimates for the expected number of possible ancestors empirically ($\bar{M}$ ) over 100 seeds with various numbers of nodes, $n_{\mathbf{T}}=4, \overline{d} = 3$ and $d_{\max}=10$ in the first row, and by using the Equation~\ref{eq:exp_anc} in the second row ($\hat{M}$).}
\label{tab:exp_anc}
\end{table*}

It is important to note that for any number of possible ancestors of $\mathbf{T}$, SNAP($k$) + PC with $k=0,1$ performs at most as many CI tests as PC, since we do not apply the RFCI rules yet for these $k$. Hence, while such cases can limit the benefits of pruning, low order pruning ($k=0,1$) does not have any downside in terms of CI tests.

\section{EXPERIMENTAL DETAILS}
\label{app:experimental_details}

For our experiments, we used the following libraries: igraph \citep{csardi2006igraph} (GNU GPL version 2 or later), networkx \citep{hagberg2008exploring} (3-Clause BSD License), bnlearn \citep{scutari2010bnlearn} (MIT License), pcalg \citep{kalisch2012pcalg} (GNU GPL version 2 or later), dagitty \citep{textor2016dagitty} (GNU GPL), tetrad \citep{ramsey23pytetrad} (MIT License), causal-learn \citep{zheng2024causal} (MIT License) and RCD \citep{mokhtarian2024recursive} (BSD 2-Clause License).
In particular, we used the \ac{CI} test implementations and some orientation rules from the \texttt{causal-learn} package \citep{zheng2024causal}.
All experiments were performed on AMD Rome CPUs, using 48 CPU cores and 84 GiB of memory.
We let each experiment run for at most 24 hours.
If an experiment did not finish in the given time, we do not report any results for it.

\subsection{Speeding up local algorithms}
\label{sec:local_speed}
For LDECC* and MB-by-MB* we adapt the publicly available code\footnote{GitHub repository of \citet{gupta2023local}: \url{https://github.com/acmi-lab/local-causal-discovery}} of \citet{gupta2023local}.
In particular, we noticed that the running times of these algorithms were much slower than indicated by the number of tests they performed.
We found that this was caused by unnecessary \texttt{deepcopy} operations in their subroutines, thus we removed those.
This means that running times reported here are potentially faster than what one would achieve by using the original codebase.

\subsection{Real-world networks}
\label{app:real_networks}
The MAGIC-NIAB network has 66 nodes, an average degree 3, and is parameterized by linear Gaussian data.
We show the network structure of MAGIC-NIAB in Fig.~\ref{fig:magic-niab}.
The Andes network has 223 nodes, an average degree 3.03, and is parameterized by binary data.
We show the network structure of Andes in Fig.~\ref{fig:andes}.

\begin{figure}
    \centering
    \includegraphics[width=.8\linewidth]{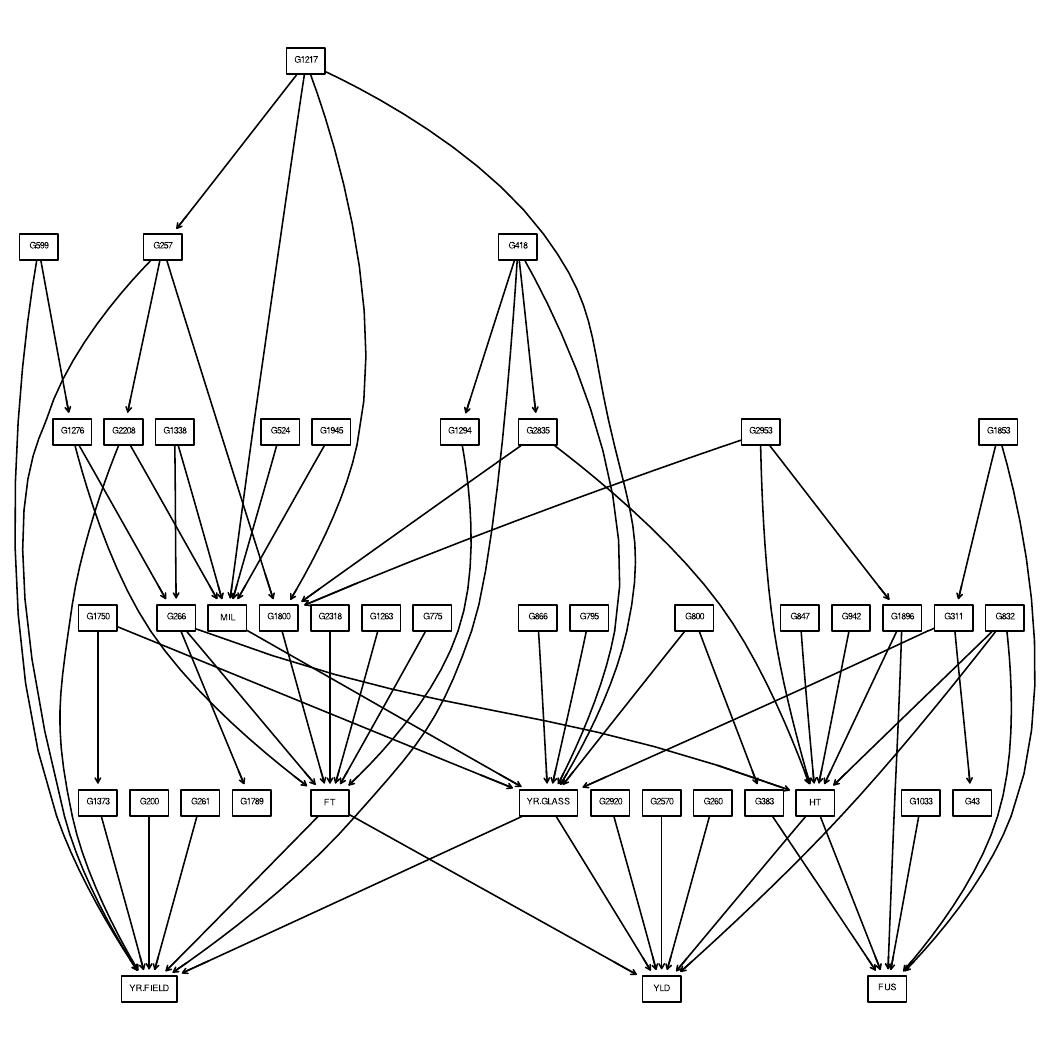}
    \caption{The MAGIC-NIAB network.}
    \label{fig:magic-niab}
\end{figure}

\begin{figure}
    \centering
    \includegraphics[width=\linewidth]{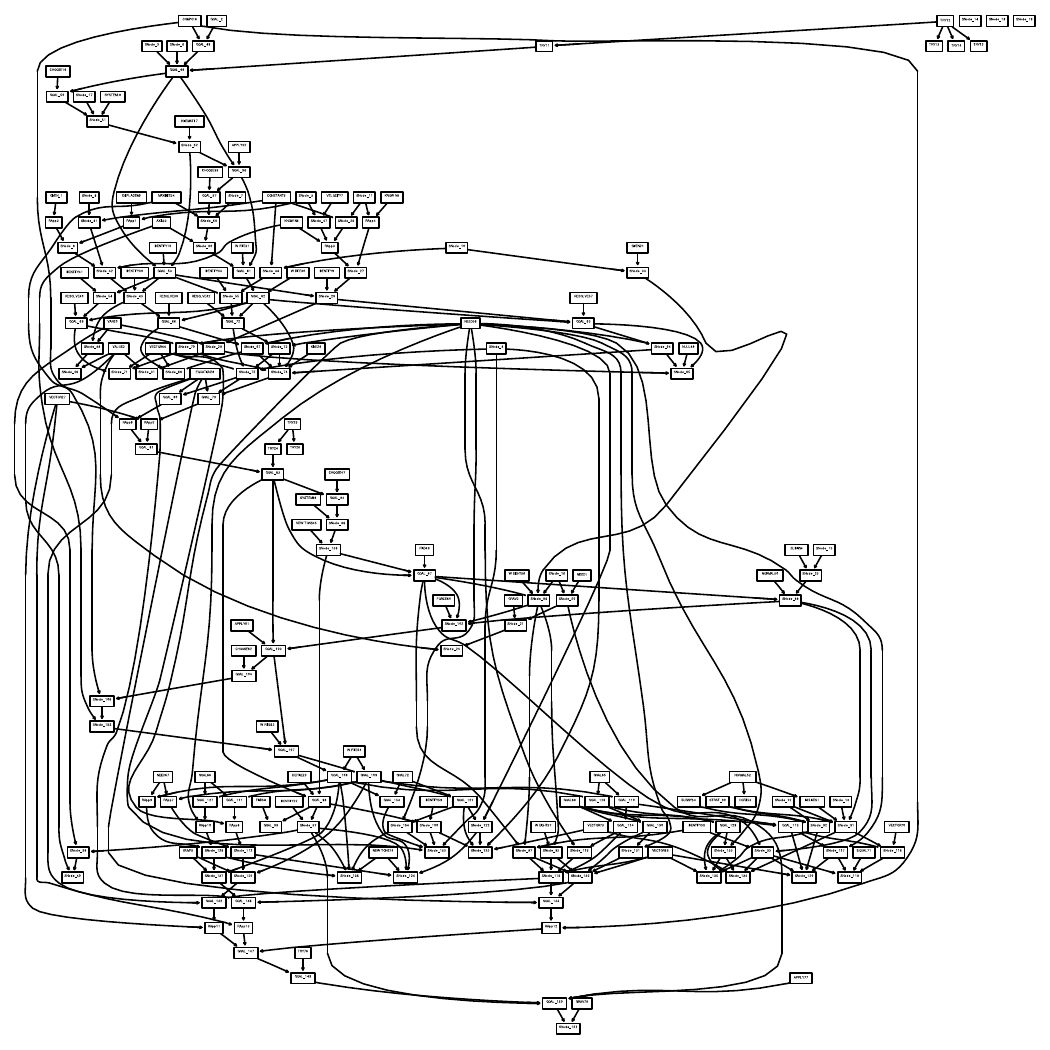}
    \caption{The Andes network.}
    \label{fig:andes}
\end{figure}

\subsection{Intervention distance} \label{app:metrics}
As one of the main metrics to evaluate the quality of the causal effect estimation, we considered the intervention distance between the estimated causal effect and the ground truth causal effect between pairs of target variables.

We developed this metric in order to be able to estimate the quality of the output even in cases in which the causal effects were not identifiable in the estimated causal graph. This allowed us to compare with all baselines on all the simulated graphs, instead of only focusing on a subset of graphs on which all baselines agreed that the causal effects were identifiable, which could potentially be a very biased set of graphs.

Formally, we then define the intervention distance as follows:
\begin{definition}
\label{def:intervention_distance}
Given targets $\mathbf{T} \subseteq \mathbf{V}$, for $T,T' \in \mathbf{T}$ let $\theta^*_{T, T'}$ be the true causal effect of $T$ on $T'$ and let $\hat{\Theta}_{T, T'}$ be the set of estimated causal effects of $T$ on $T'$ according to the output of an algorithm.
Then the intervention distance is defined as
$$
    \frac{1}{|\mathbf{T}|(|\mathbf{T}|-1) }\sum_{T \in \mathbf{T}} \sum_{T' \in \mathbf{T} \setminus \{T\}} \frac{1}{|\hat{\Theta}_{T, T'}|} \sum_{\hat{\theta} \in \hat{\Theta}_{T, T'}} |\theta^*_{T, T'}-\hat{\theta} |
$$
\end{definition}

\section{COMPLETE EXPERIMENTS}
\label{app:additional_experiments}
This section contains additional results to the ones shown in the main paper, as well as our results for all ablations.
App.~\ref{app:over_nodes} shows additional results on all metrics for graphs with various number of nodes.
App.~\ref{app:identifiable} shows our results for graphs with various number of nodes and with identifiable targets.
We present various ablations in App.~\ref{app:over_targets}, \ref{app:over_degrees} and \ref{app:over_samples} for varying number of targets, expected degree and number of data samples respectively.
App.~\ref{app:snapk} compares SNAP($k$) for $k=0,1,2$.
App.~\ref{app:order} provides a visual explanation on how SNAP variants avoid higher order \ac{CI} tests.
Finally, App.~\ref{app:real-network-results} shows all of our results for the MAGIC-NIAB and the Andes networks.

\subsection{Various number of nodes}
\label{app:over_nodes}
In this section, we consider the same setting presented in the main paper, in Figures~\ref{fig:test_and_time_per_node_unrest_std} and \ref{fig:delta_time_per_node_unrest}, i.e., graphs with various number of nodes, $n_{\mathbf{T}}=4, \overline{d} = 3, d_{\max}=10$ and $n_{\mathbf{D}} = 1000$ data-points.
Fig.~\ref{fig:quality_per_node_unrest_std} shows that intervention distance is not significantly different between most methods.
Fig~\ref{fig:tests_and_time_per_node_unrest} highlights the advantages of different SNAP variants over various settings.
In particular, SNAP($0$) combined with MARVEL or MB-by-MB seems to perform the fewest d-separation tests.
On the other hand, SNAP($\infty$) seems to do best with KCI tests.
On binary data, even though the computation time of FGES trails behind other methods, when combined with SNAP($0$) it becomes the fastest method on this setting.

Figures~\ref{fig:quality_per_node_unrest_std} and \ref{fig:quality_per_node_unrest} show the quality of the estimated structures in terms of \ac{SHD} and intervention distance.
Our results show that even though SNAP variants achieve higher \ac{SHD}, especially when using Fisher-Z tests, their intervention distance is competitive with baselines.
This indicates that structural metrics may not be suitable when the final objective is causal effect estimation.
Additionally, our results for intervention distance in the case of d-separation tests validate that parent adjustment sets are suboptimal for causal effect estimation.

\begin{figure}
    \centering
    \includegraphics[width=.6\linewidth]{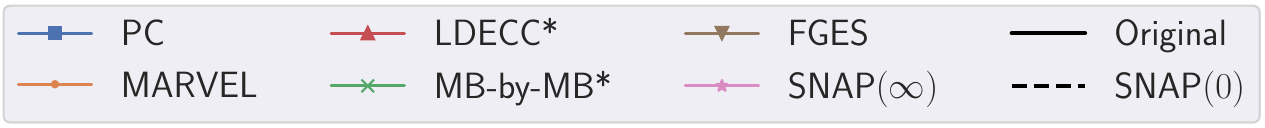}
    \begin{subfigure}[b]{\linewidth}
        \begin{subfigure}[b]{0.24\linewidth}
            \includegraphics[width=\linewidth]{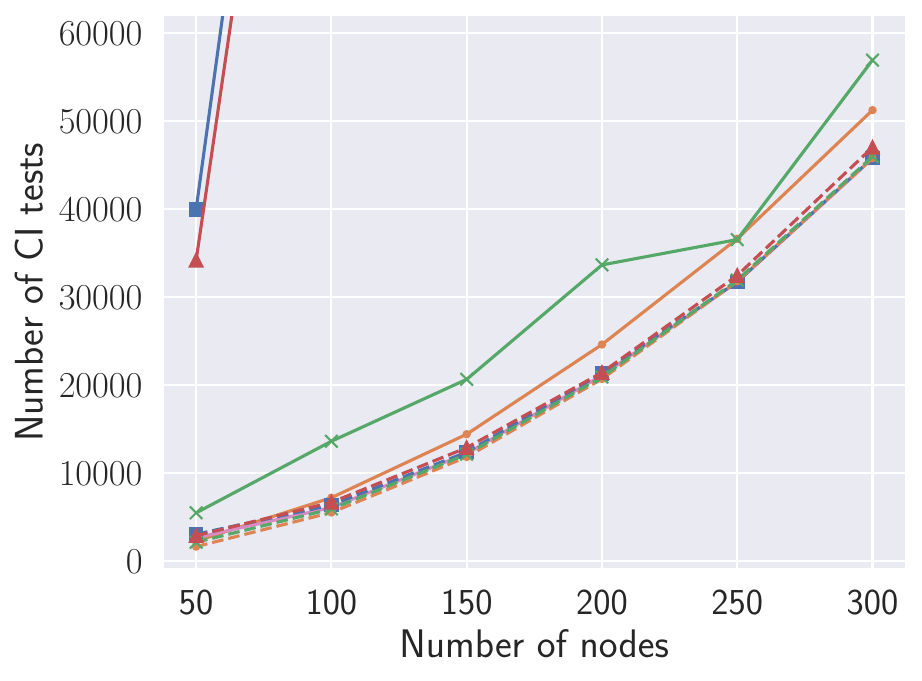}
            \caption*{d-separation tests}
            \label{fig:test_per_node_unrest_dsep}
        \end{subfigure}
        \begin{subfigure}[b]{0.24\linewidth}
            \includegraphics[width=\linewidth]{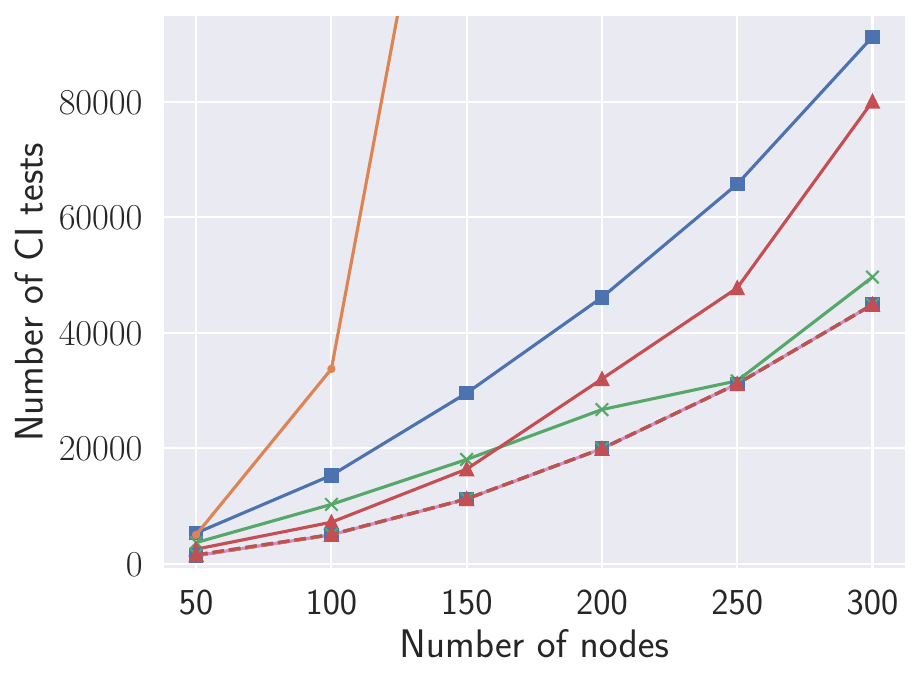}
            \caption*{Fisher-Z tests}
            \label{fig:test_per_node_unrest_fshz}
        \end{subfigure}
        \begin{subfigure}[b]{0.24\linewidth}
            \includegraphics[width=\linewidth]{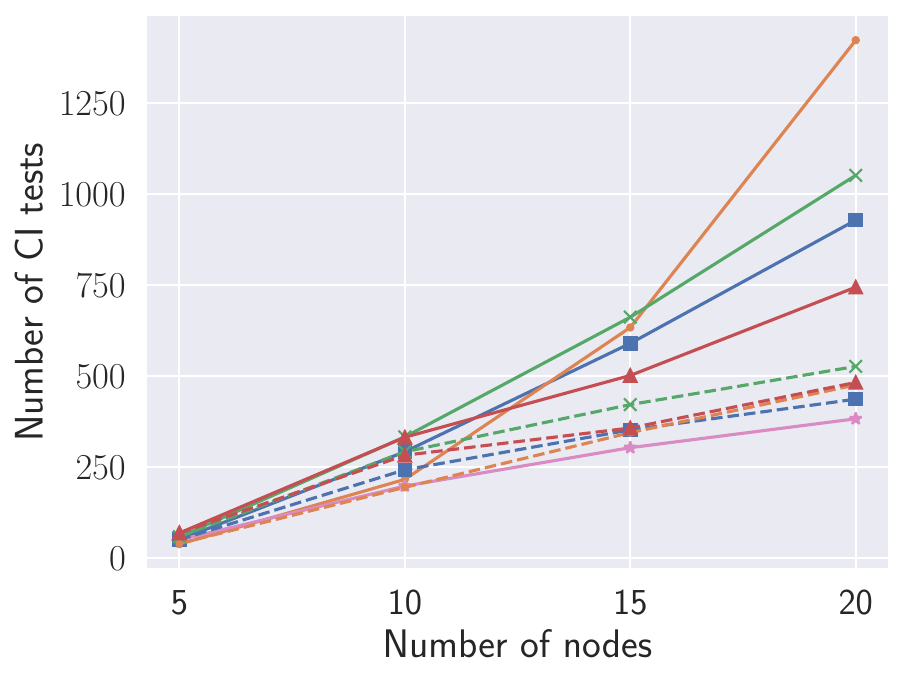}
            \caption*{KCI tests}
            \label{fig:test_per_node_unrest_kci}
        \end{subfigure}
        \begin{subfigure}[b]{0.24\linewidth}
            \includegraphics[width=\linewidth]{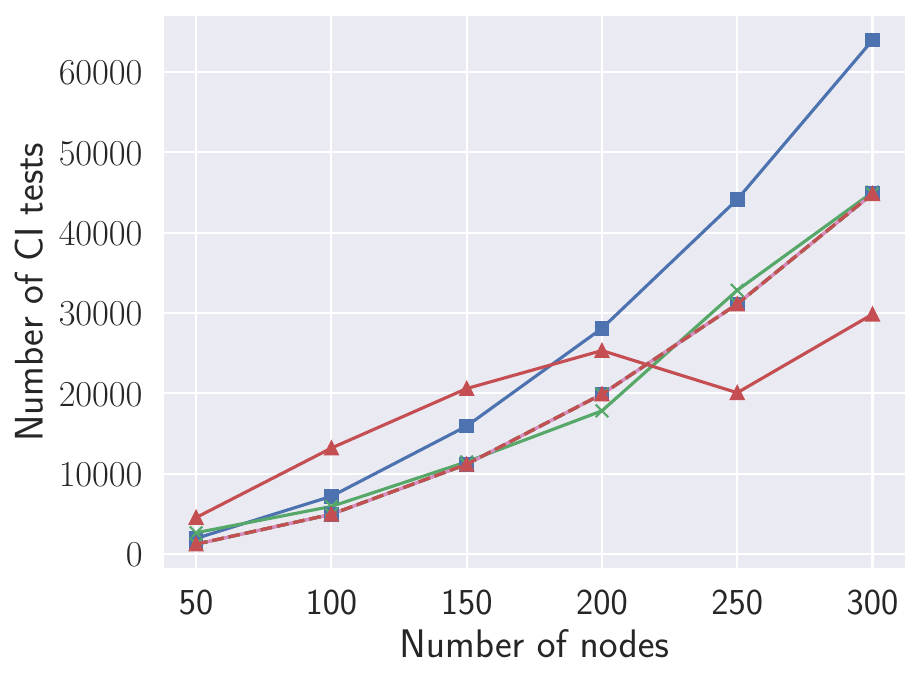}
            \caption*{$\chi^2$ tests}
            \label{fig:test_per_node_unrest_chsq}
        \end{subfigure}
        \caption{Number of \ac{CI} tests.}
        \label{test_per_node_unrest}
    \end{subfigure}
    \begin{subfigure}[b]{\linewidth}
        \begin{subfigure}[b]{0.25\linewidth}
            \includegraphics[width=\linewidth]{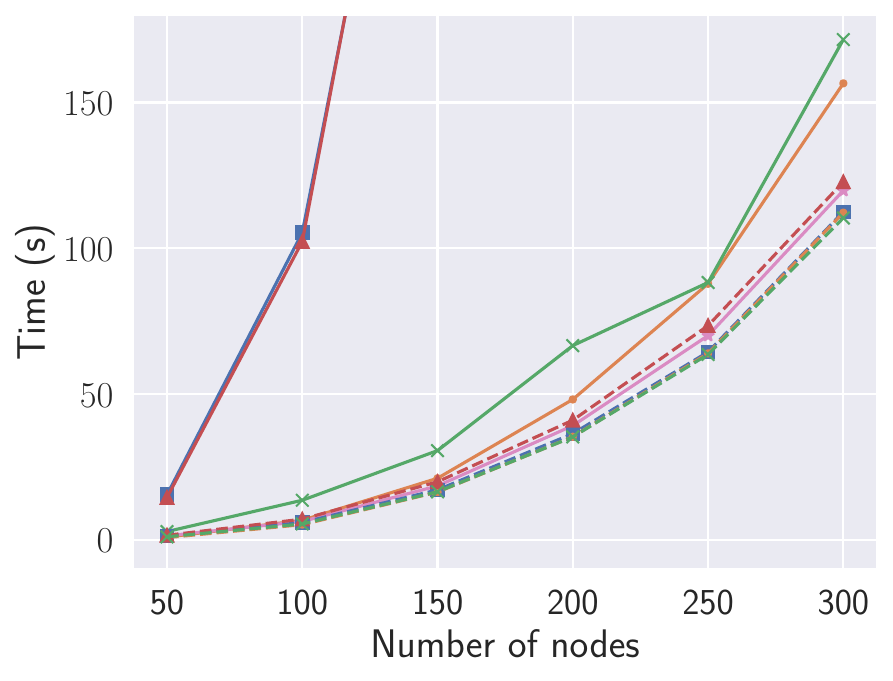}
            \caption*{d-separation tests}
            \label{fig:time_per_node_unrest_dsep}
        \end{subfigure}
        \begin{subfigure}[b]{0.24\linewidth}
            \includegraphics[width=\linewidth]{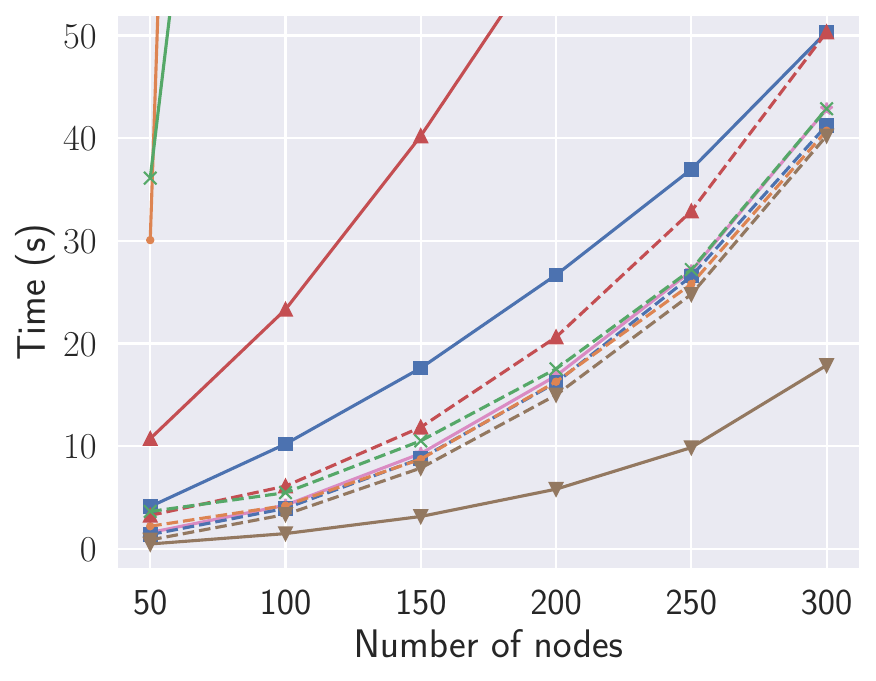}
            \caption*{Fisher-Z tests}
            \label{fig:time_per_node_unrest_fshz}
        \end{subfigure}
        \begin{subfigure}[b]{0.25\linewidth}
            \includegraphics[width=\linewidth]{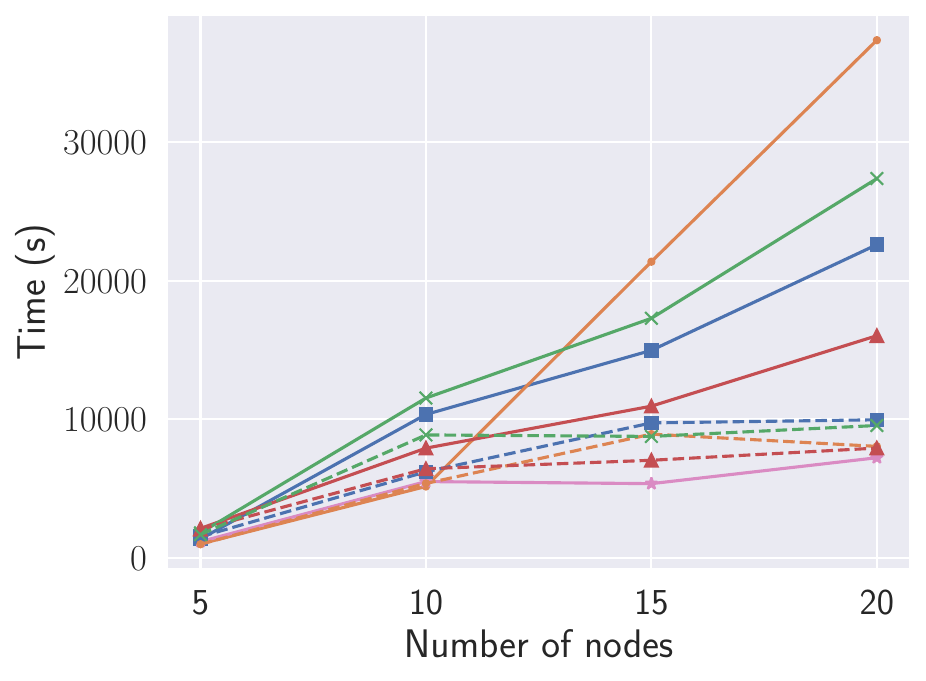}
            \caption*{KCI tests}
            \label{fig:time_per_node_unrest_kci}
        \end{subfigure}
        \begin{subfigure}[b]{0.24\linewidth}
            \includegraphics[width=\linewidth]{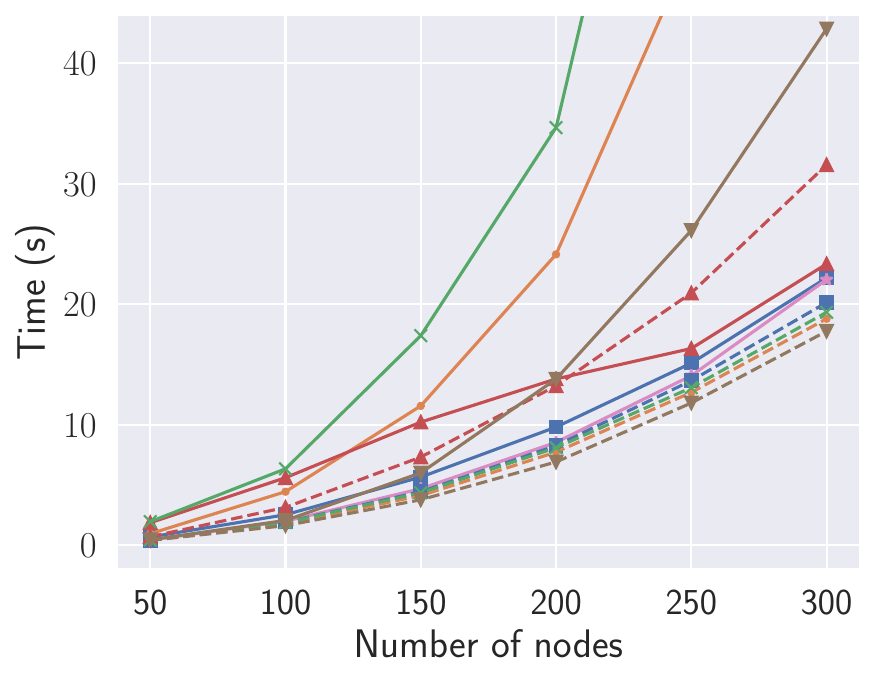}
            \caption*{$\chi^2$ tests}
            \label{fig:time_per_node_unrest_chsq}
        \end{subfigure}
        \caption{Computation time.}
        \label{time_per_node_unrest}
    \end{subfigure}
    \caption{Number of \ac{CI} tests and computation time for baseline methods combined with SNAP$(0)$ over number of nodes, with $n_{\mathbf{T}}=4, \overline{d} = 3, d_{\max}=10$ and $n_{\mathbf{D}} = 1000$ data-points.
    }
    \label{fig:tests_and_time_per_node_unrest}
\end{figure}

\begin{figure}
    \centering
    \includegraphics[width=.6\linewidth]{experiments/legend_small.pdf}
    \begin{subfigure}[b]{\linewidth}
        \begin{subfigure}[b]{0.24\linewidth}
            \caption*{d-separation tests}
        \end{subfigure}
        \begin{subfigure}[b]{0.24\linewidth}
            \caption*{Fisher-Z tests}
        \end{subfigure}
        \begin{subfigure}[b]{0.24\linewidth}
            \caption*{KCI tests}
        \end{subfigure}
        \begin{subfigure}[b]{0.24\linewidth}
            \caption*{$\chi^2$ tests}
        \end{subfigure}
    \end{subfigure}
    \begin{subfigure}[b]{\linewidth}
        \begin{subfigure}[b]{0.24\linewidth}
            \includegraphics[width=\linewidth]{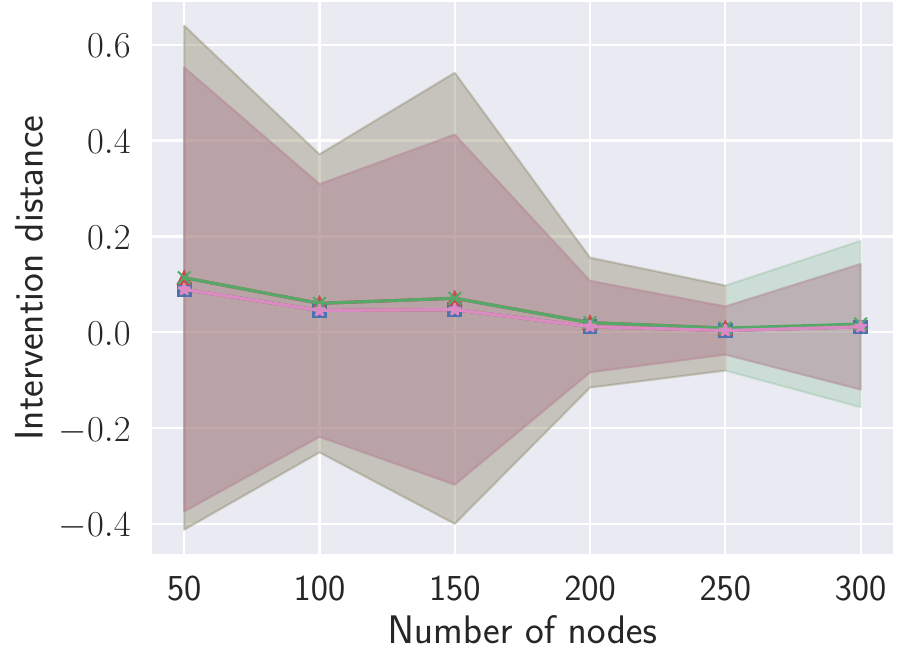}
        \end{subfigure}
        \begin{subfigure}[b]{0.24\linewidth}
            \includegraphics[width=\linewidth]{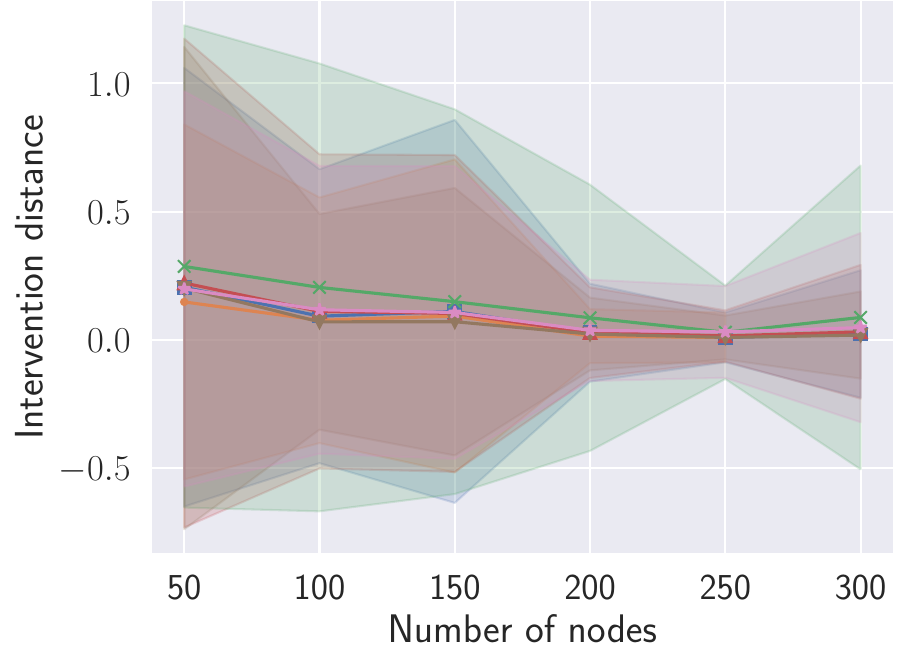}
        \end{subfigure}
        \begin{subfigure}[b]{0.24\linewidth}
            \includegraphics[width=\linewidth]{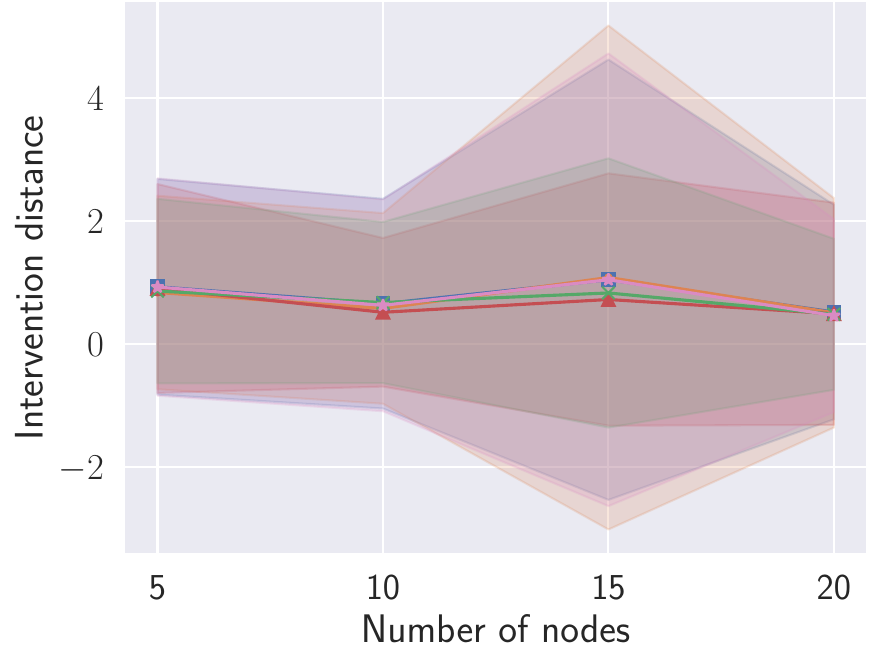}
        \end{subfigure}
        \begin{subfigure}[b]{0.24\linewidth}
            \includegraphics[width=\linewidth]{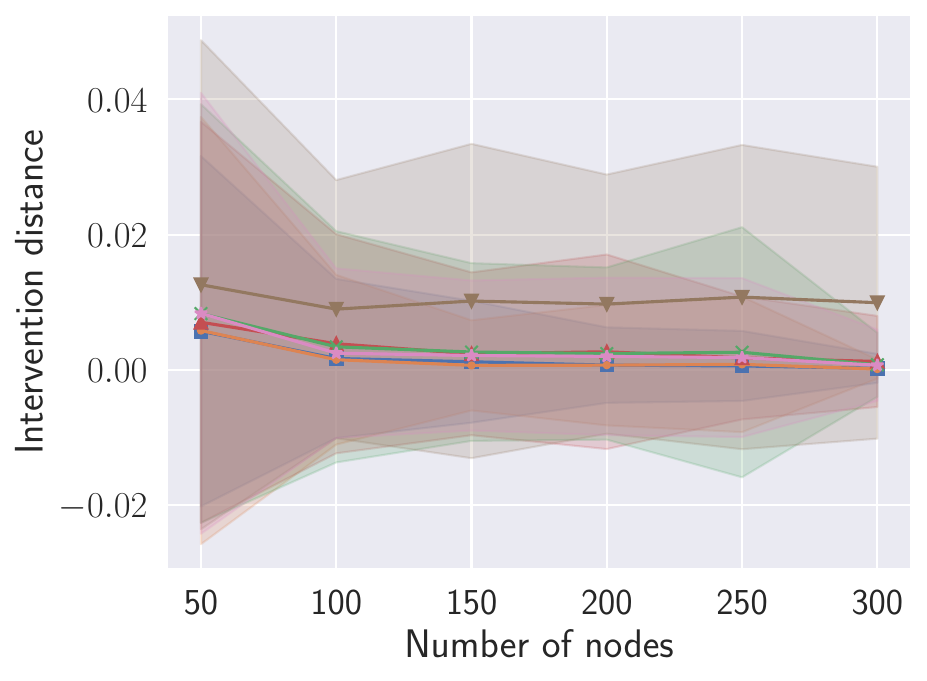}
        \end{subfigure}
        \caption{Intervention distance.}
        \label{fig:int_dist_per_node_unrest_std}
    \end{subfigure}
    \begin{subfigure}[b]{\linewidth}
        \begin{subfigure}[b]{0.24\linewidth}
            \includegraphics[width=\linewidth]{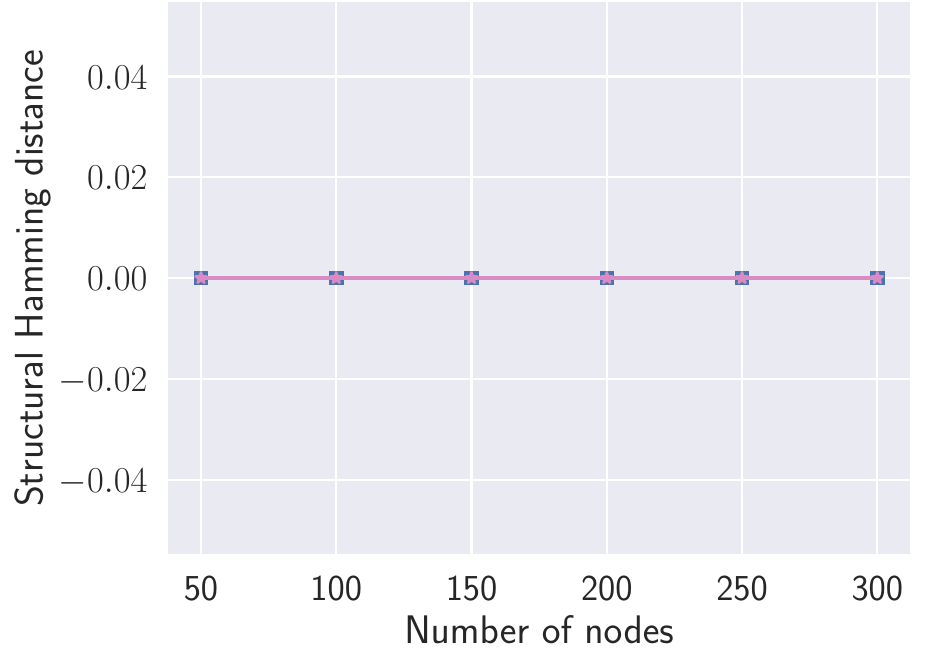}
        \end{subfigure}
        \begin{subfigure}[b]{0.24\linewidth}
            \includegraphics[width=\linewidth]{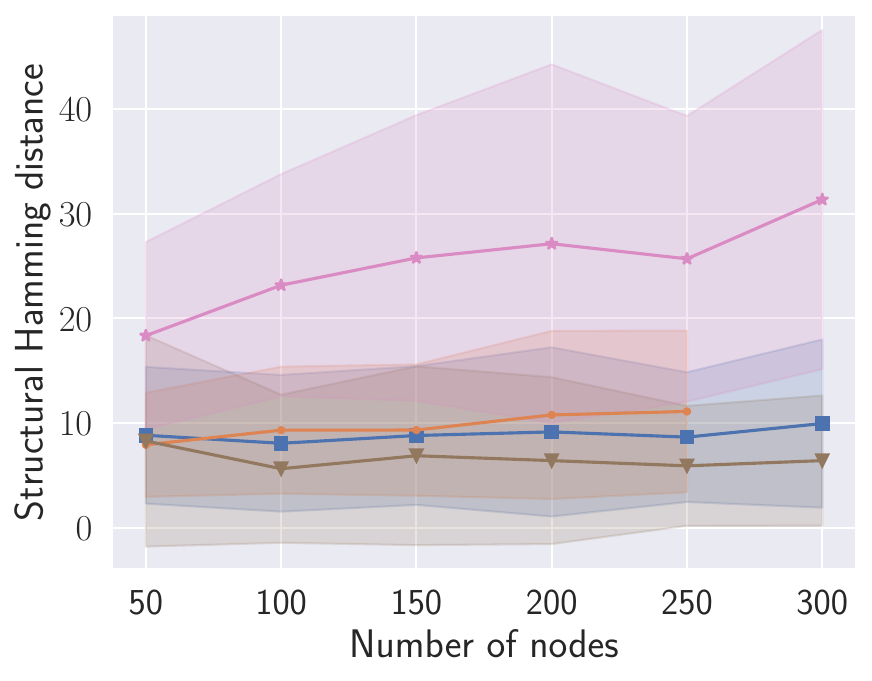}
        \end{subfigure}
        \begin{subfigure}[b]{0.24\linewidth}
            \includegraphics[width=\linewidth]{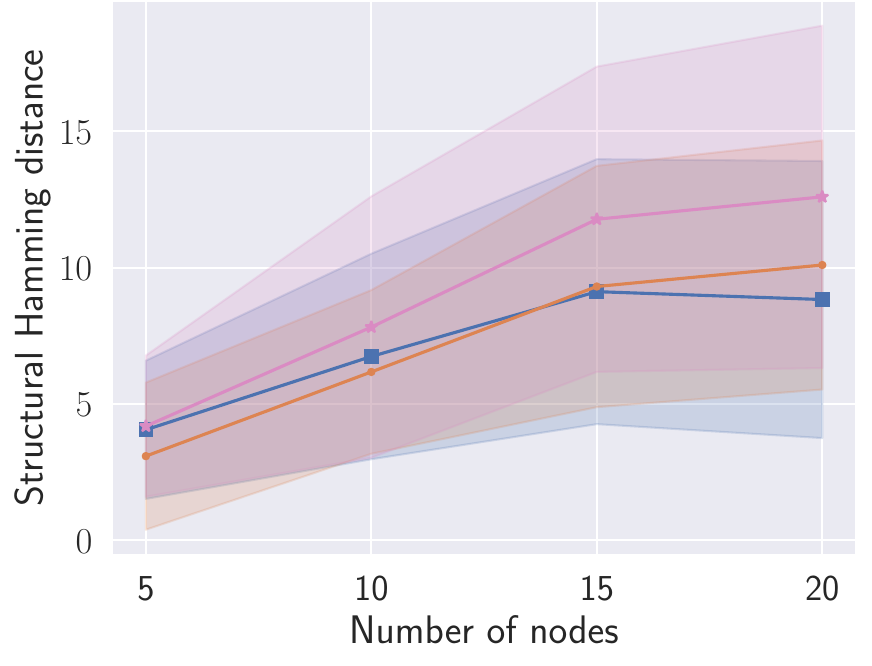}
        \end{subfigure}
        \begin{subfigure}[b]{0.24\linewidth}
            \includegraphics[width=\linewidth]{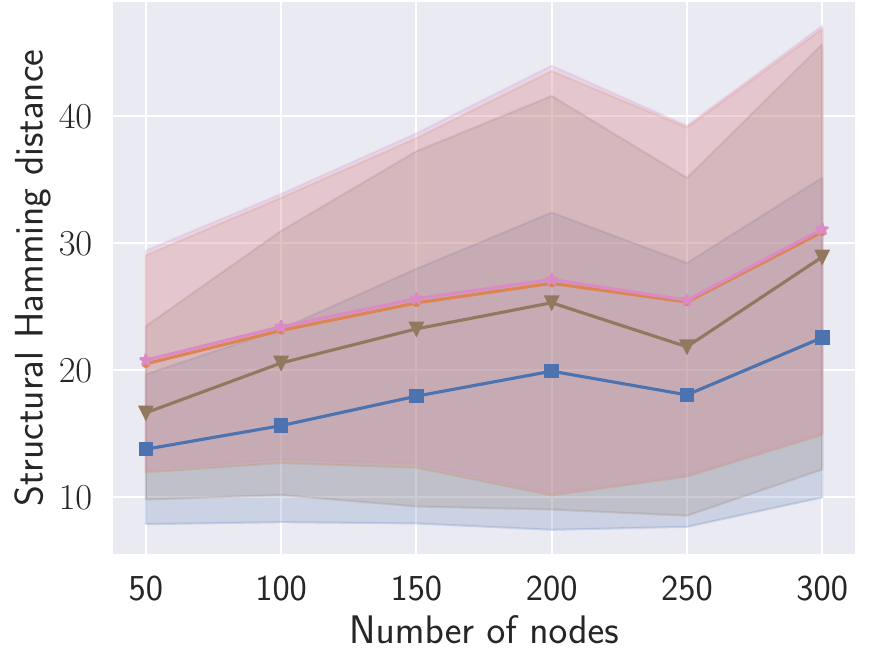}
        \end{subfigure}
        \caption{\Acf{SHD}.}
        \label{fig:shd_per_node_unrest_std}
    \end{subfigure}
    \caption{Quality of estimation over number of nodes, with $n_{\mathbf{T}}=4, \overline{d} = 3, d_{\max}=10$ and $n_{\mathbf{D}} = 1000$ data-points. We compute the intervention distance in the d-separation tests case using random linear Gaussian data according to the discovered structure.}
    \label{fig:quality_per_node_unrest_std}
\end{figure}

\begin{figure}
    \centering
    \includegraphics[width=.6\linewidth]{experiments/legend_big.pdf}
    \begin{subfigure}[b]{\linewidth}
        \begin{subfigure}[b]{0.24\linewidth}
            \includegraphics[width=\linewidth]{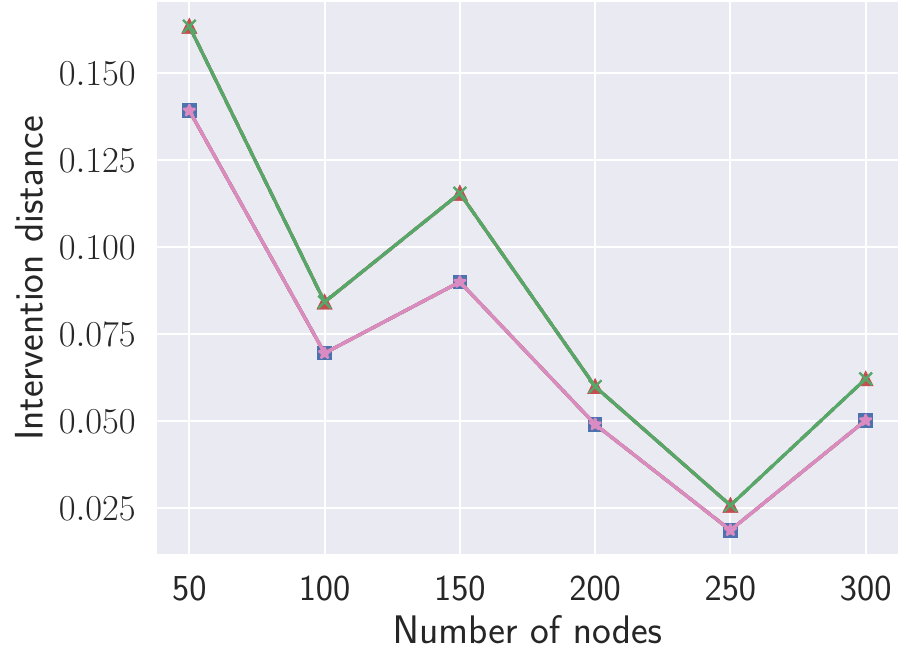}
            \caption*{d-separation tests}
            \label{fig:int_dist_per_node_unrest_dsep_abs}
        \end{subfigure}
        \begin{subfigure}[b]{0.24\linewidth}
            \includegraphics[width=\linewidth]{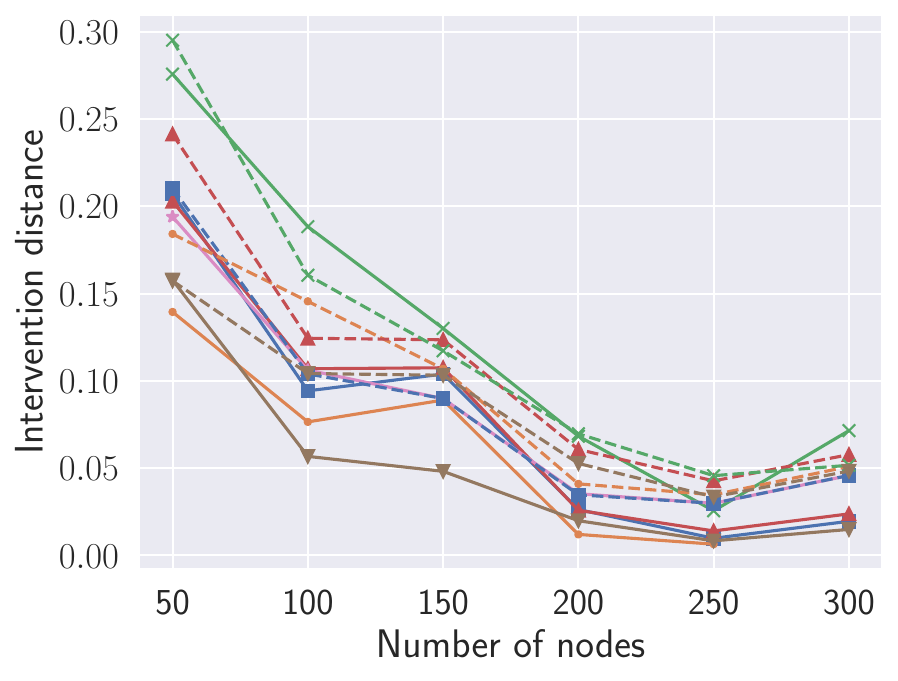}
            \caption*{Fisher-Z tests}
            \label{fig:int_dist_per_node_unrest_fshz_abs}
        \end{subfigure}
        \begin{subfigure}[b]{0.24\linewidth}
            \includegraphics[width=\linewidth]{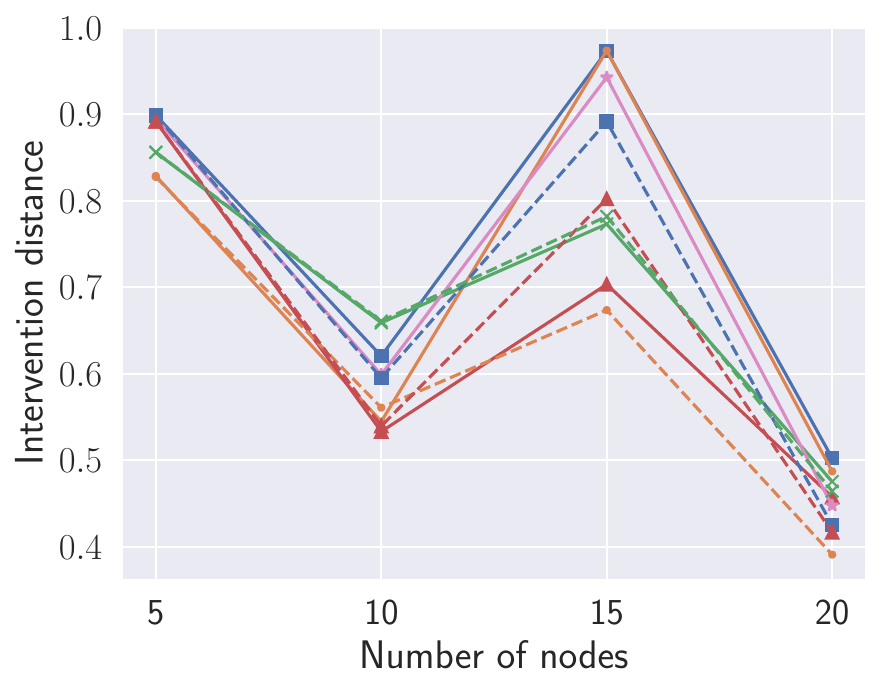}
            \caption*{KCI tests}
            \label{fig:int_dist_per_node_unrest_kci_abs}
        \end{subfigure}
        \begin{subfigure}[b]{0.24\linewidth}
            \includegraphics[width=\linewidth]{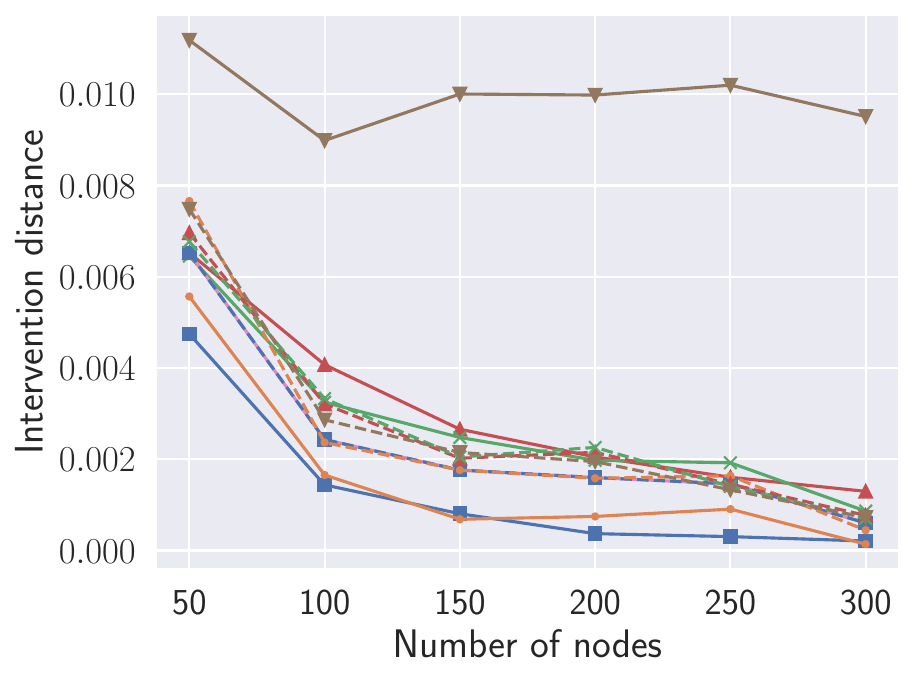}
            \caption*{$\chi^2$ tests}
            \label{fig:int_dist_per_node_unrest_chsq_abs}
        \end{subfigure}
        \caption{Intervention distance.}
        \label{fig:int_dist_per_node_unrest}
    \end{subfigure}
    \begin{subfigure}[b]{\linewidth}
        \begin{subfigure}[b]{0.24\linewidth}
            \includegraphics[width=\linewidth]{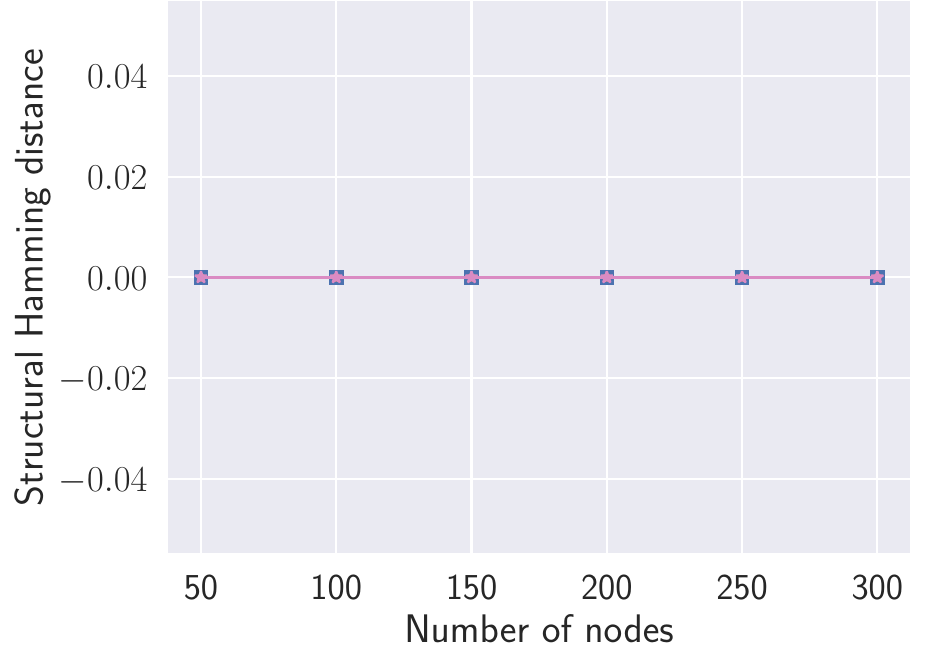}
            \caption*{d-separation tests}
            \label{fig:shd_per_node_unrest_dsep}
        \end{subfigure}
        \begin{subfigure}[b]{0.24\linewidth}
            \includegraphics[width=\linewidth]{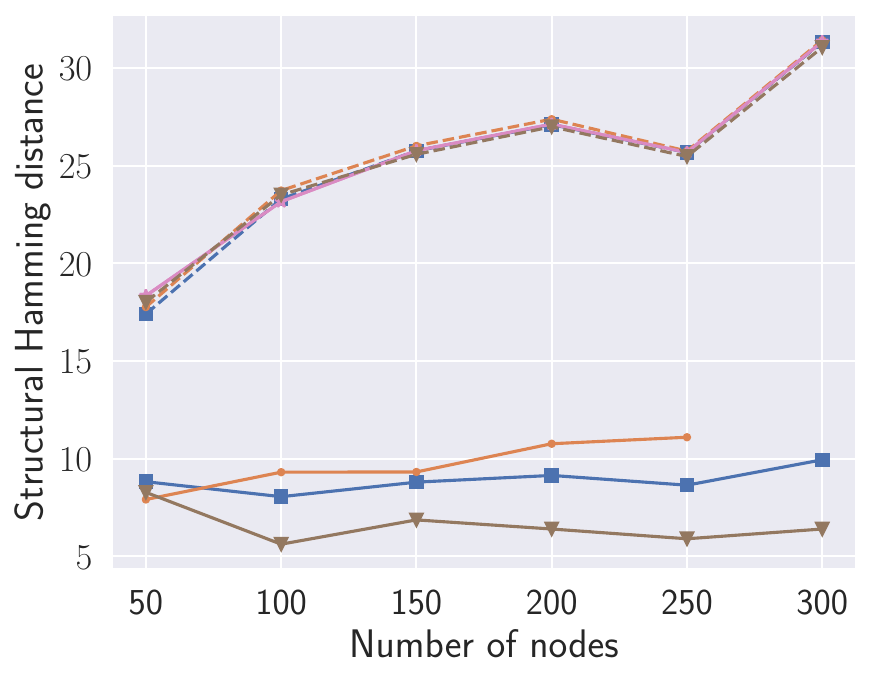}
            \caption*{Fisher-Z tests}
            \label{fig:shd_per_node_unrest_fshz}
        \end{subfigure}
        \begin{subfigure}[b]{0.24\linewidth}
            \includegraphics[width=\linewidth]{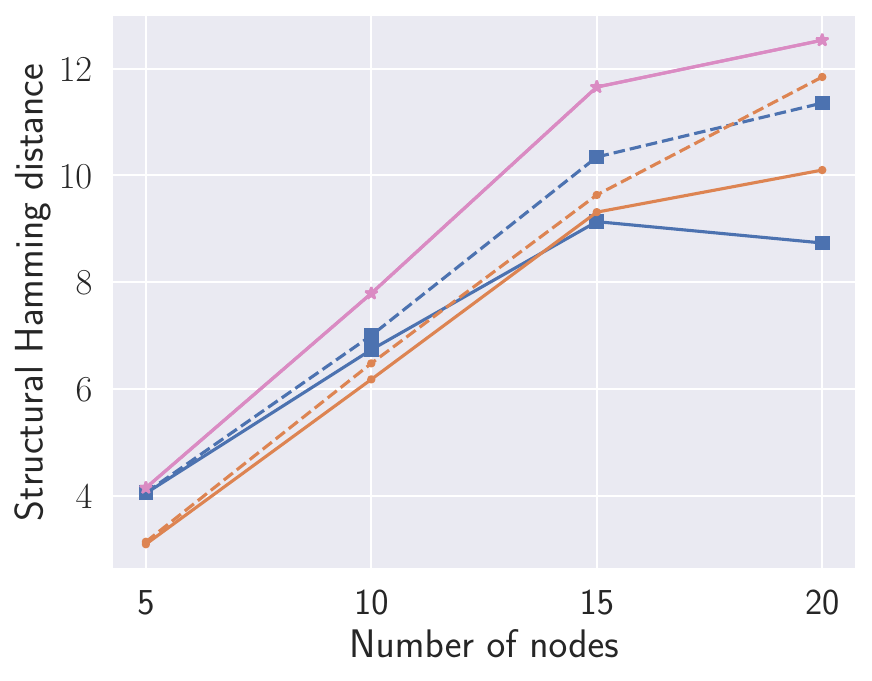}
            \caption*{KCI tests}
            \label{fig:shd_per_node_unrest_kci}
        \end{subfigure}
        \begin{subfigure}[b]{0.24\linewidth}
            \includegraphics[width=\linewidth]{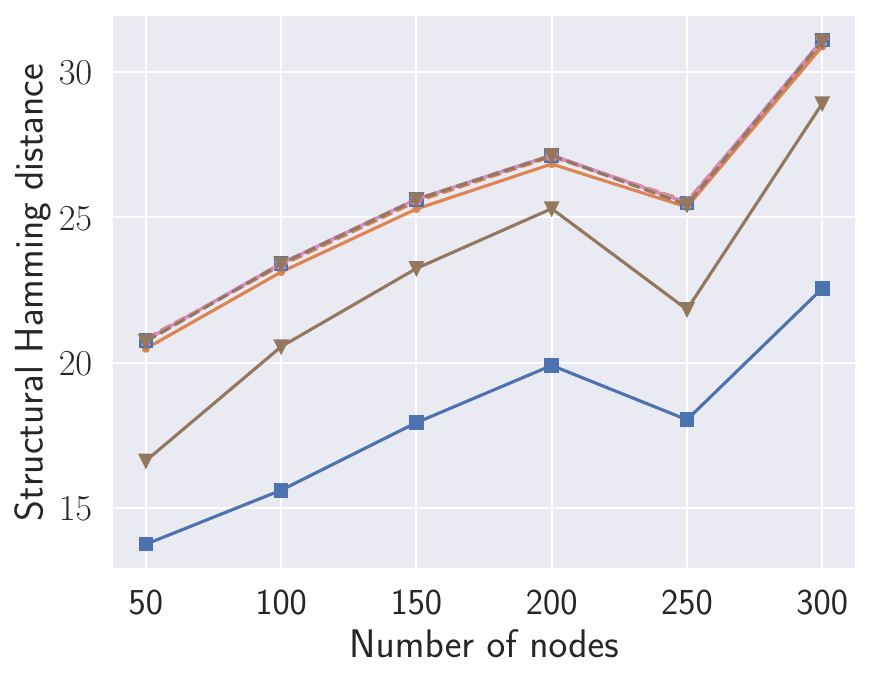}
            \caption*{$\chi^2$ tests}
            \label{fig:shd_per_node_unrest_chsq}
        \end{subfigure}
        \caption{\Acf{SHD}.}
        \label{fig:shd_per_node_unrest}
    \end{subfigure}
    \caption{Quality of estimation over number of nodes for baseline methods combined with SNAP$(0)$, with $n_{\mathbf{T}}=4, \overline{d} = 3, d_{\max}=10$ and $n_{\mathbf{D}} = 1000$ data-points. We compute the intervention distance in the d-separation tests case using random linear Gaussian data according to the discovered structure.}
    \label{fig:quality_per_node_unrest}
\end{figure}

\subsection{Identifiable targets}
\label{app:identifiable}
In this section, we sample target sets for experiment that are \emph{identifiable}.
We consider a set of targets identifiable if the causal effect is identifiable from the true CPDAG between each pair of targets.
To ensure that these identifiable causal effects are not mostly zero, when sampling identifiable targets, we also require that they are the ancestor or the descendant of at least one other target.
This makes intervention distance more meaningful, as the true CPDAG should achieve a near zero distance, while incorrect and overly sparse \acp{CPDAG} should achieve higher distances.
Furthermore, it allows us to measure the \acf{AID}.

Figures~\ref{fig:appendix_per_node_ident_std} and \ref{fig:appendix_per_node_ident} shows that enforcing targets to be identifiable does not have a significant impact on the results.
\Cref{fig:aid_per_node_ident_std,fig:aid_per_node_ident} show results for \acf{AID}, where SNAP variants are comparable to most methods.
In particular, SNAP$(0)$ improves MARVEL on linear Gaussian data with Fisher-Z tests.
SNAP variants seem to struggle the most on binary data in terms of adjustment identification distance.

\begin{figure}
    \centering
    \includegraphics[width=.6\linewidth]{experiments/legend_small.pdf}
    \begin{subfigure}[b]{\linewidth}
        \begin{subfigure}[b]{0.24\linewidth}
            \includegraphics[width=\linewidth]{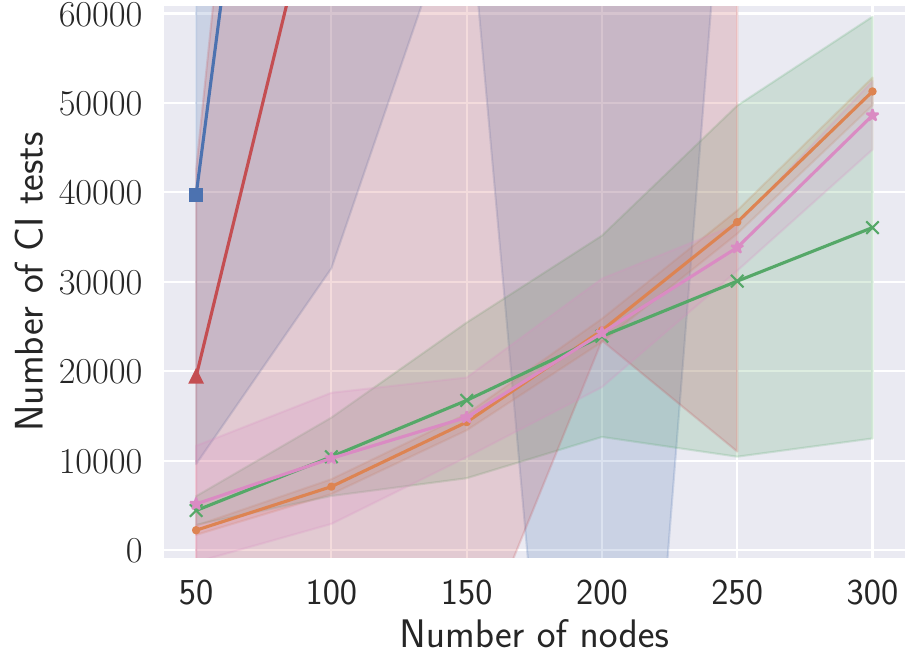}
            \caption*{d-separation tests}
            \label{fig:test_per_node_ident_dsep_std}
        \end{subfigure}
        \begin{subfigure}[b]{0.24\linewidth}
            \includegraphics[width=\linewidth]{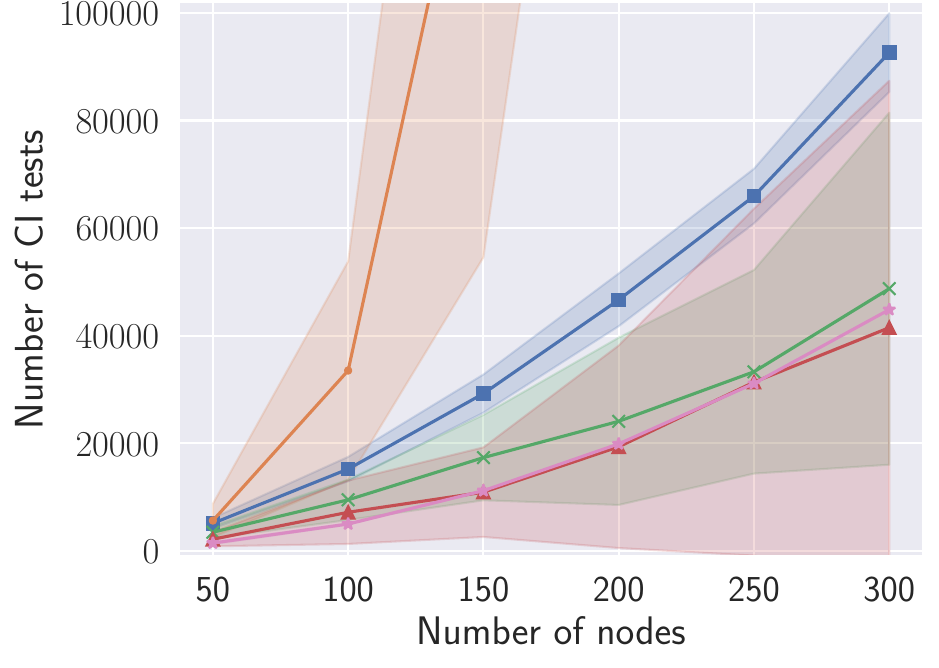}
            \caption*{Fisher-Z tests}
            \label{fig:test_per_node_ident_fshz_std}
        \end{subfigure}
        \begin{subfigure}[b]{0.24\linewidth}
            \includegraphics[width=\linewidth]{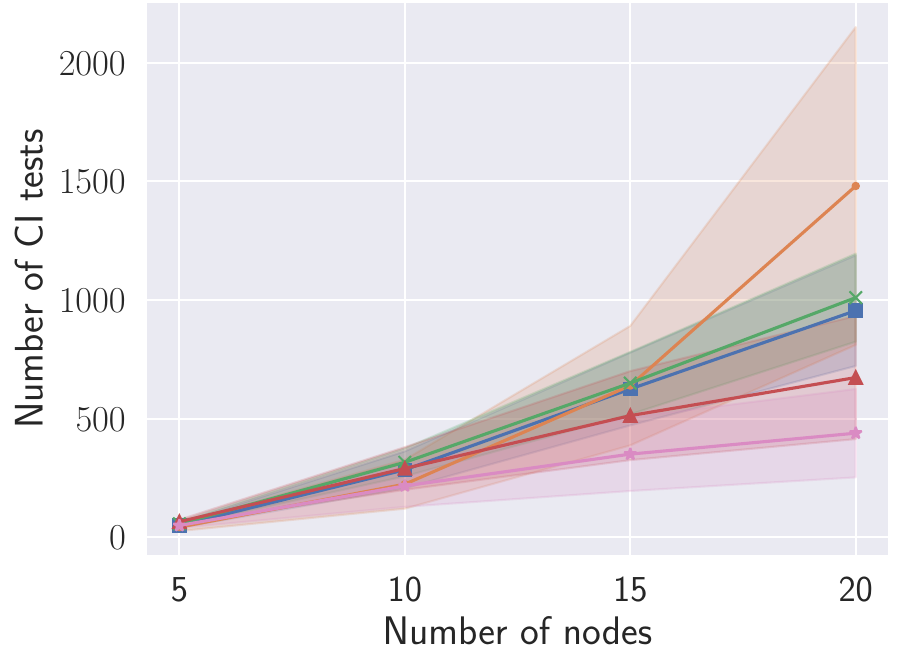}
            \caption*{KCI tests}
            \label{fig:test_per_node_ident_kci_std}
        \end{subfigure}
        \begin{subfigure}[b]{0.24\linewidth}
            \includegraphics[width=\linewidth]{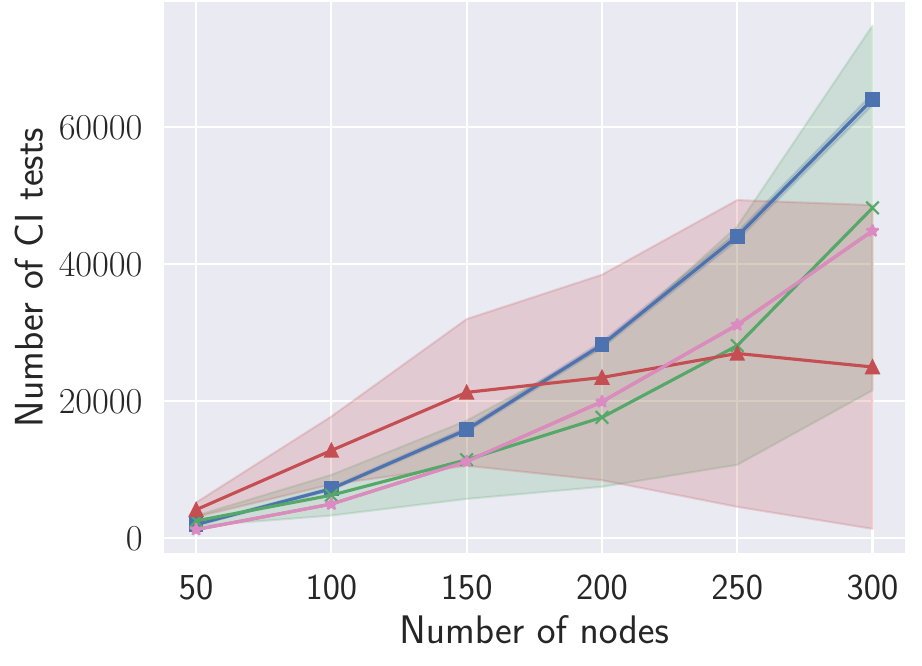}
            \caption*{$\chi^2$ tests}
            \label{fig:test_per_node_ident_chsq_std}
        \end{subfigure}
        \caption{Number of \ac{CI} tests.}
        \label{fig:test_per_node_ident_std}
    \end{subfigure}
    \begin{subfigure}[b]{\linewidth}
        \begin{subfigure}[b]{0.24\linewidth}
            \includegraphics[width=\linewidth]{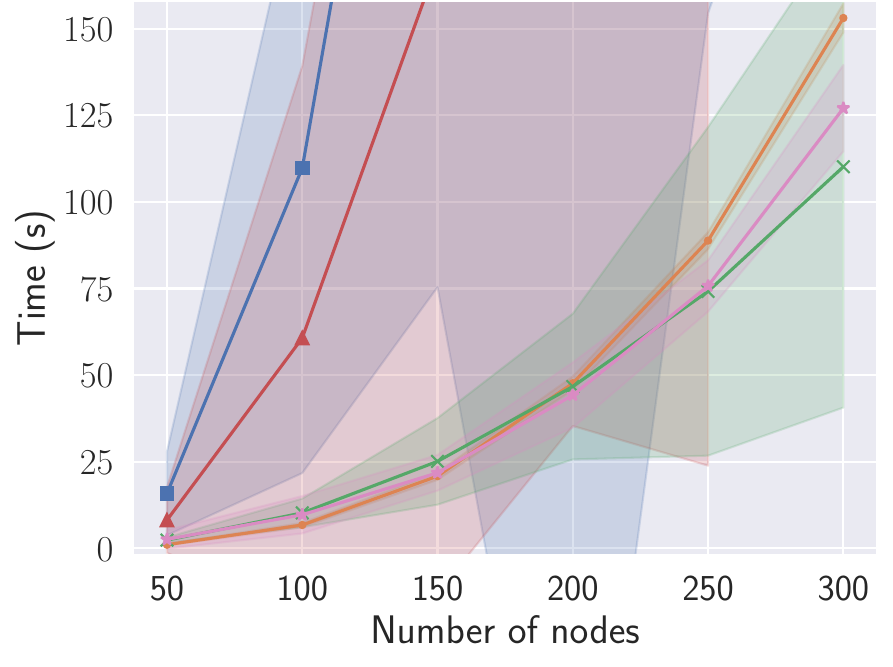}
            \caption*{d-separation tests}
            \label{fig:time_per_node_ident_dsep_std}
        \end{subfigure}
        \begin{subfigure}[b]{0.24\linewidth}
            \includegraphics[width=\linewidth]{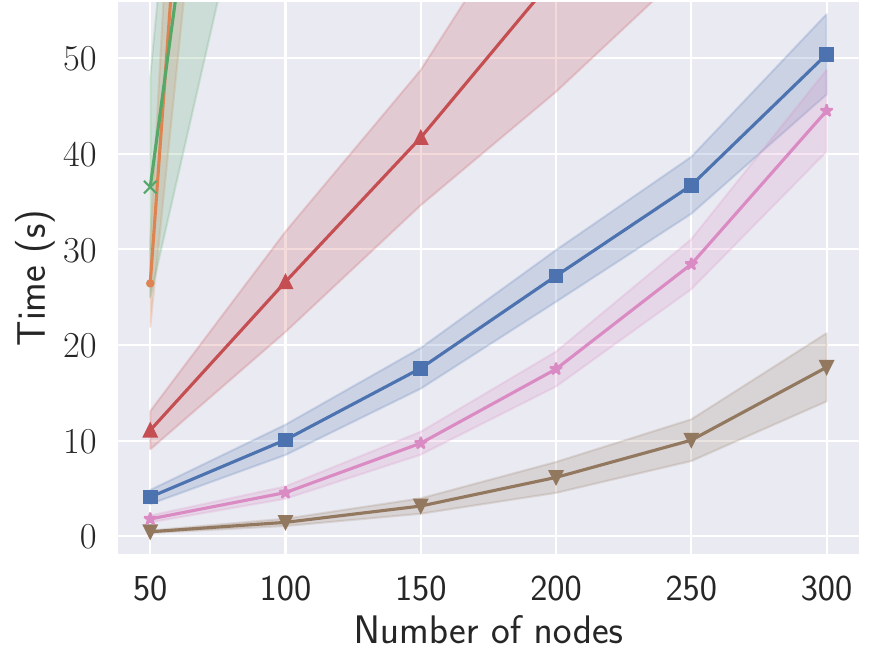}
            \caption*{Fisher-Z tests}
            \label{fig:time_per_node_ident_fshz_std}
        \end{subfigure}
        \begin{subfigure}[b]{0.24\linewidth}
            \includegraphics[width=\linewidth]{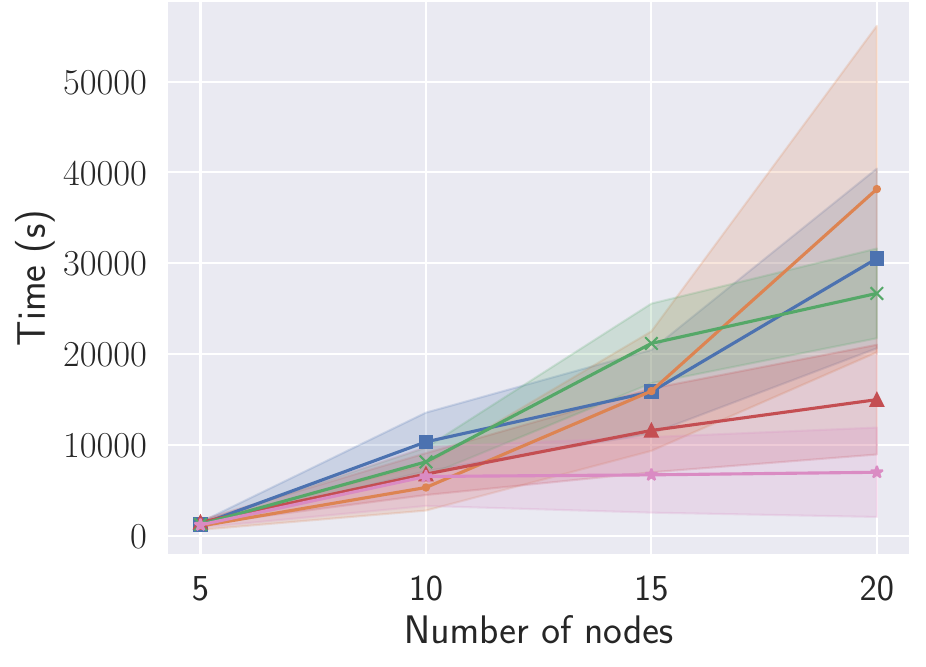}
            \caption*{KCI tests}
            \label{fig:time_per_node_ident_kci_std}
        \end{subfigure}
        \begin{subfigure}[b]{0.24\linewidth}
            \includegraphics[width=\linewidth]{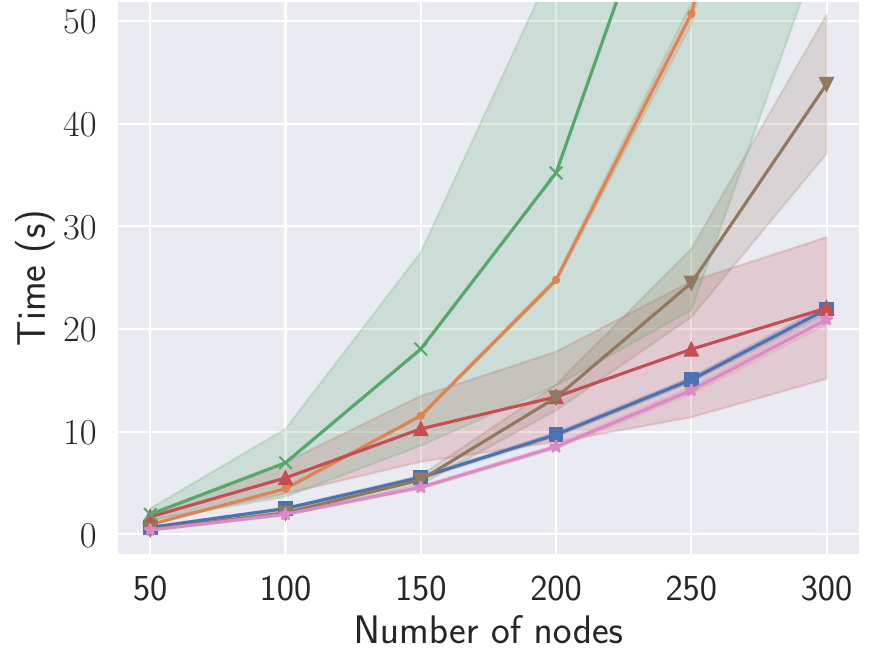}
            \caption*{$\chi^2$ tests}
            \label{fig:time_per_node_ident_chsq_std}
        \end{subfigure}
        \caption{Computation time.}
        \label{fig:time_per_node_ident_std}
    \end{subfigure}
    
    \begin{subfigure}[b]{\linewidth}
        \begin{subfigure}[b]{0.24\linewidth}
            \includegraphics[width=\linewidth]{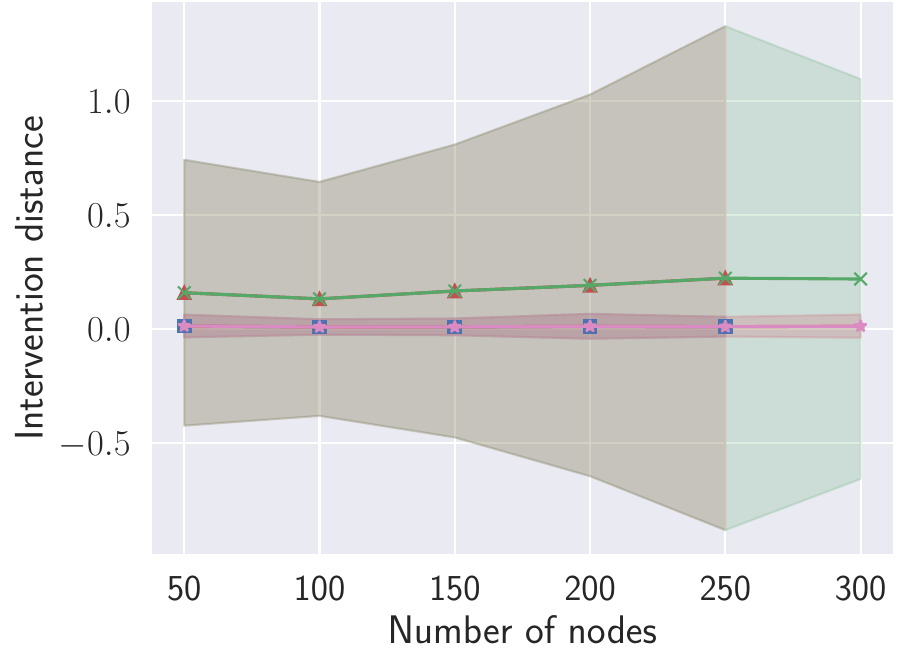}
            \caption*{d-separation tests}
            \label{fig:int_dist_per_node_ident_dsep_abs_std}
        \end{subfigure}
        \begin{subfigure}[b]{0.24\linewidth}
            \includegraphics[width=\linewidth]{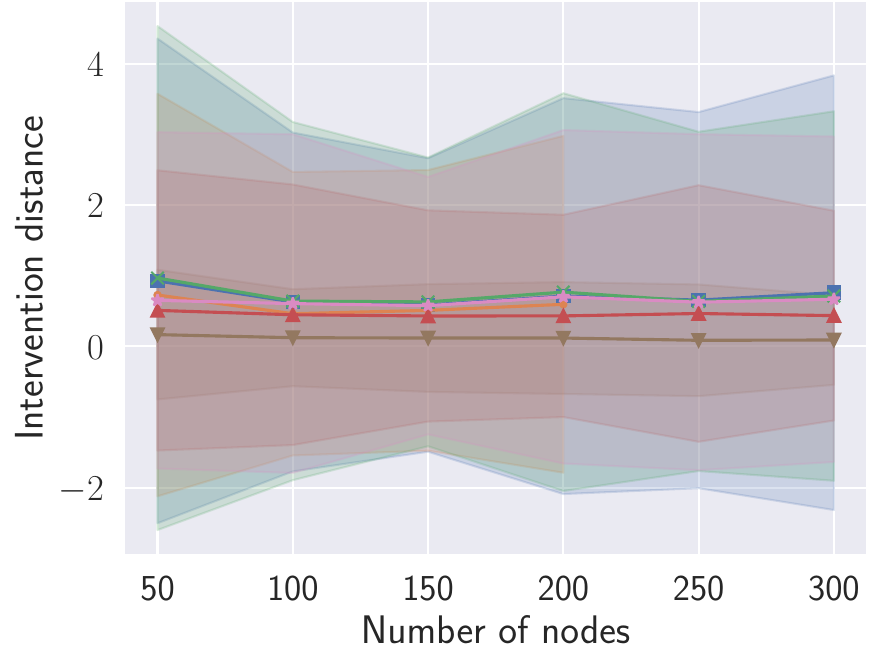}
            \caption*{Fisher-Z tests}
            \label{fig:int_dist_per_node_ident_fshz_abs_std}
        \end{subfigure}
        \begin{subfigure}[b]{0.24\linewidth}
            \includegraphics[width=\linewidth]{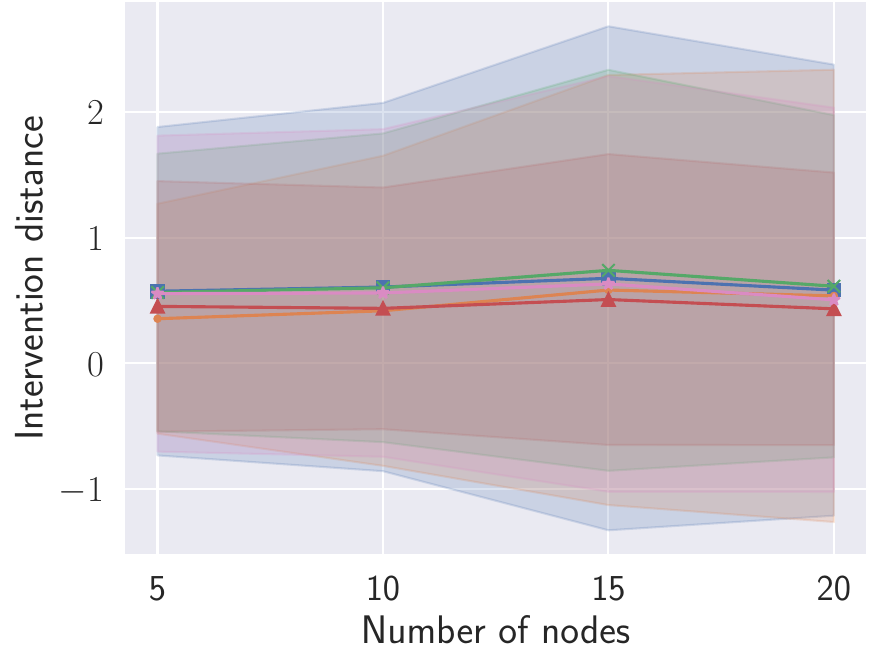}
            \caption*{KCI tests}
            \label{fig:int_dist_per_node_ident_kci_abs_std}
        \end{subfigure}
        \begin{subfigure}[b]{0.24\linewidth}
            \includegraphics[width=\linewidth]{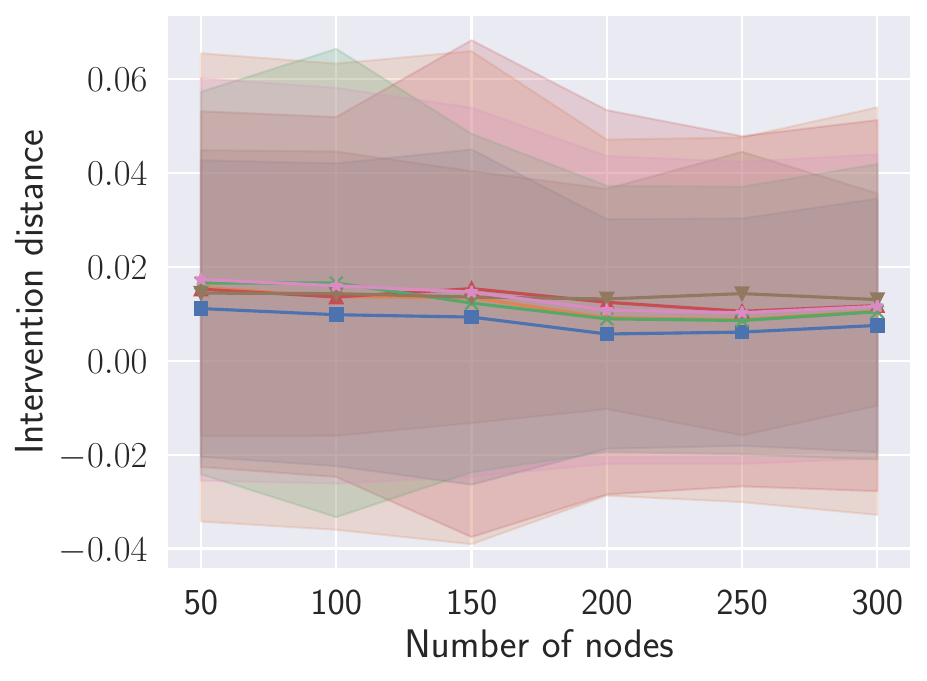}
            \caption*{$\chi^2$ tests}
            \label{fig:int_dist_per_node_ident_chsq_abs_std}
        \end{subfigure}
        \caption{Intervention distance.}
        \label{fig:int_dist_per_node_ident_std}
    \end{subfigure}
    \begin{subfigure}[b]{\linewidth}
        \begin{subfigure}[b]{0.24\linewidth}
            \includegraphics[width=\linewidth]{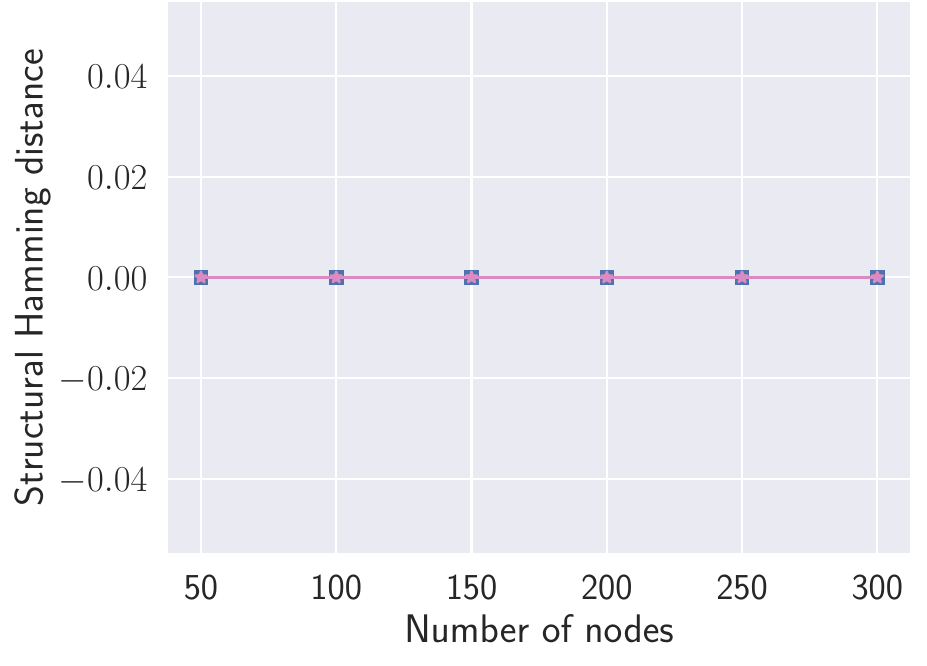}
            \caption*{d-separation tests}
            \label{fig:shd_per_node_ident_dsep_std}
        \end{subfigure}
        \begin{subfigure}[b]{0.24\linewidth}
            \includegraphics[width=\linewidth]{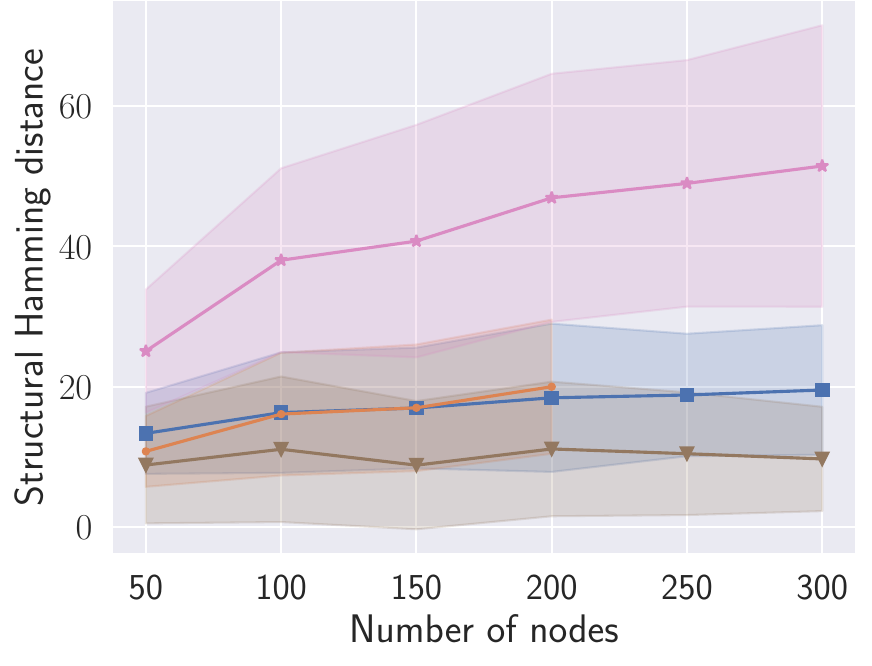}
            \caption*{Fisher-Z tests}
            \label{fig:shd_per_node_ident_fshz_std}
        \end{subfigure}
        \begin{subfigure}[b]{0.24\linewidth}
            \includegraphics[width=\linewidth]{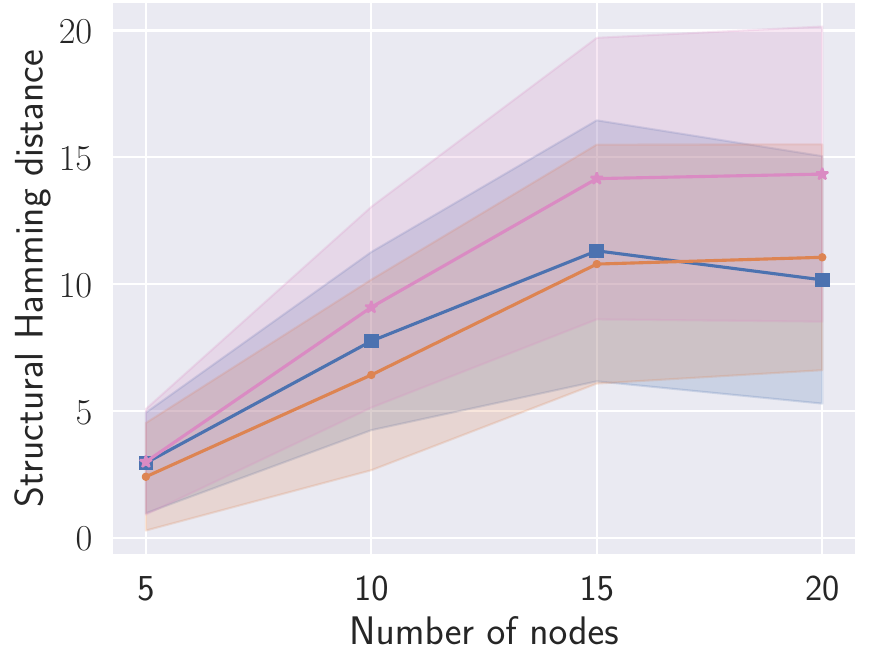}
            \caption*{KCI tests}
            \label{fig:shd_per_node_ident_kci_std}
        \end{subfigure}
        \begin{subfigure}[b]{0.24\linewidth}
            \includegraphics[width=\linewidth]{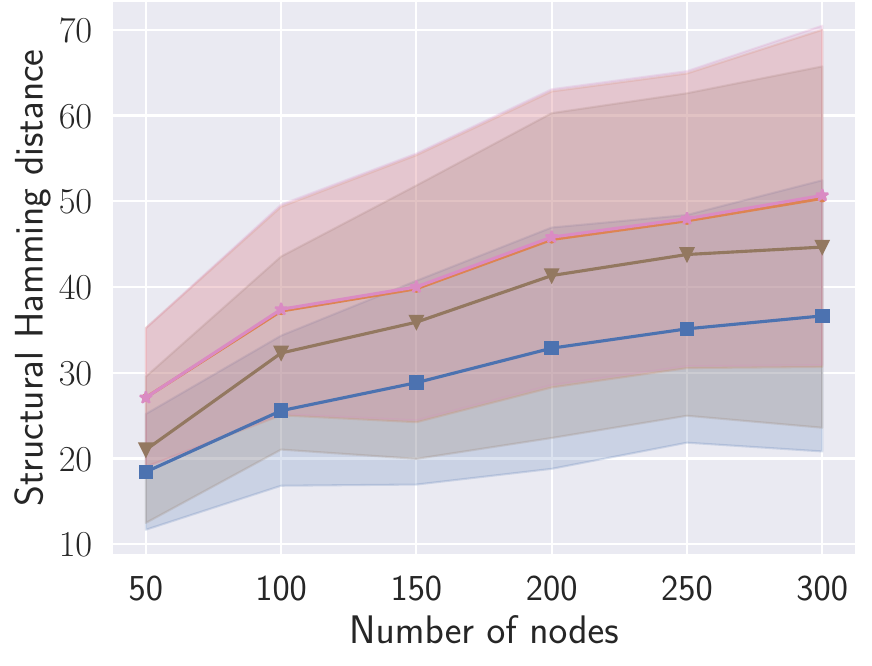}
            \caption*{$\chi^2$ tests}
            \label{fig:shd_per_node_ident_chsq_std}
        \end{subfigure}
        \caption{\Acf{SHD}.}
        \label{fig:shd_per_node_ident_std}
    \end{subfigure}
    \caption{Additional results over number of nodes for identifiable targets, with $n_{\mathbf{T}}=4, \overline{d} = 3, d_{\max}=10$ and $n_{\mathbf{D}} = 1000$ data-points. The shadow area denotes the range of the standard deviation. We compute the intervention distance in the d-separation tests case using random linear Gaussian data according to the discovered structure.}
    \label{fig:appendix_per_node_ident_std}
\end{figure}

\begin{figure}
    \centering
    \includegraphics[width=.6\linewidth]{experiments/legend_big.pdf}
    \begin{subfigure}[b]{\linewidth}
        \begin{subfigure}[b]{0.24\linewidth}
            \includegraphics[width=\linewidth]{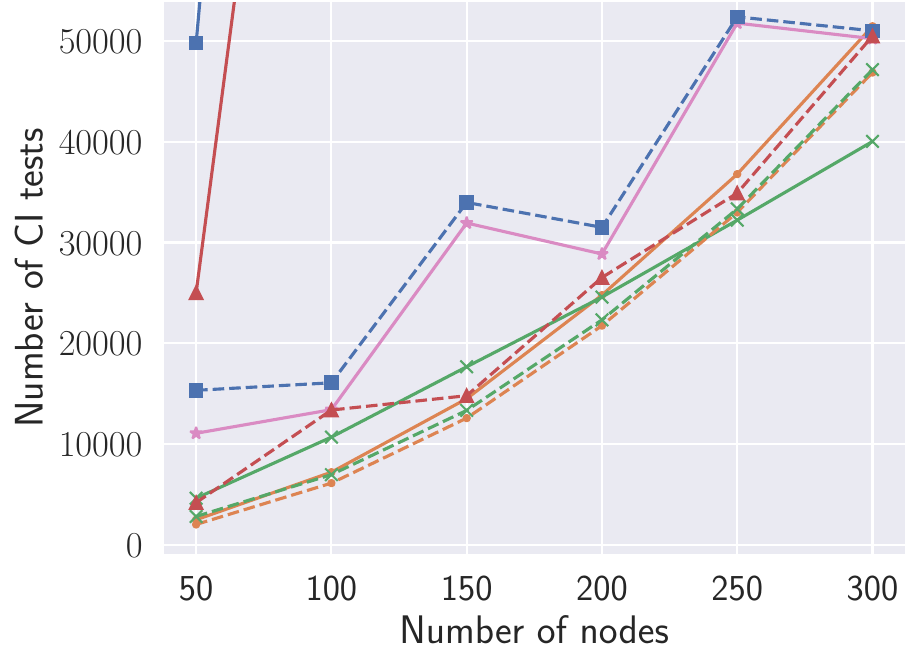}
            \caption*{d-separation tests}
            \label{fig:test_per_node_ident_dsep}
        \end{subfigure}
        \begin{subfigure}[b]{0.24\linewidth}
            \includegraphics[width=\linewidth]{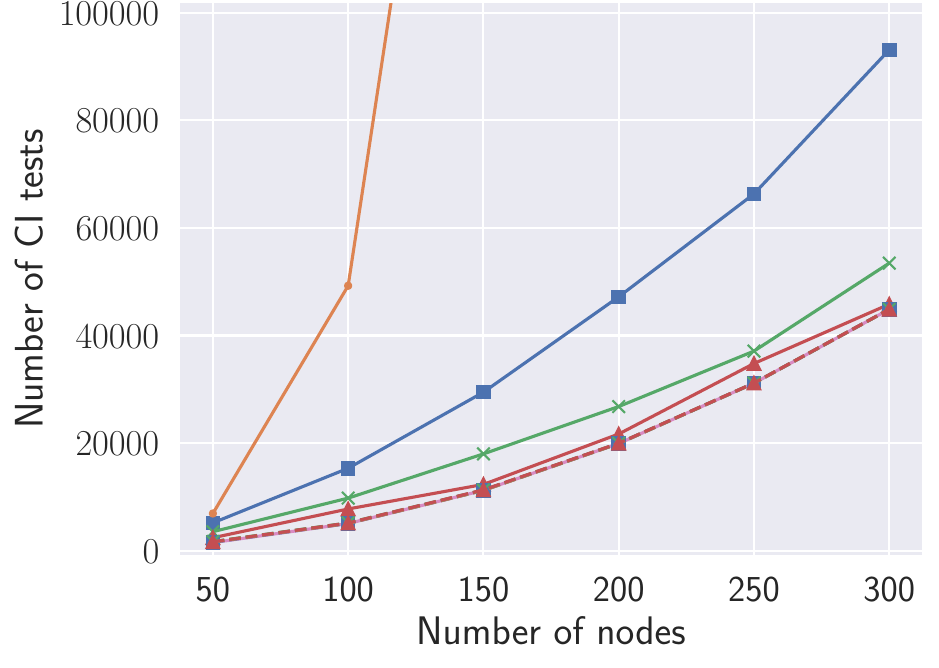}
            \caption*{Fisher-Z tests}
            \label{fig:test_per_node_ident_fshz}
        \end{subfigure}
        \begin{subfigure}[b]{0.24\linewidth}
            \includegraphics[width=\linewidth]{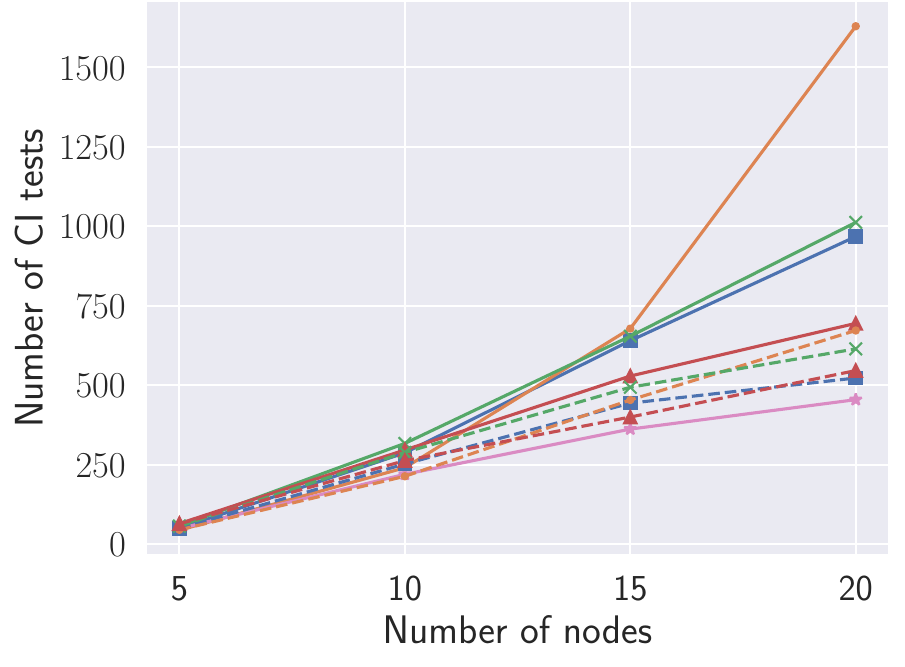}
            \caption*{KCI tests}
            \label{fig:test_per_node_ident_kci}
        \end{subfigure}
        \begin{subfigure}[b]{0.24\linewidth}
            \includegraphics[width=\linewidth]{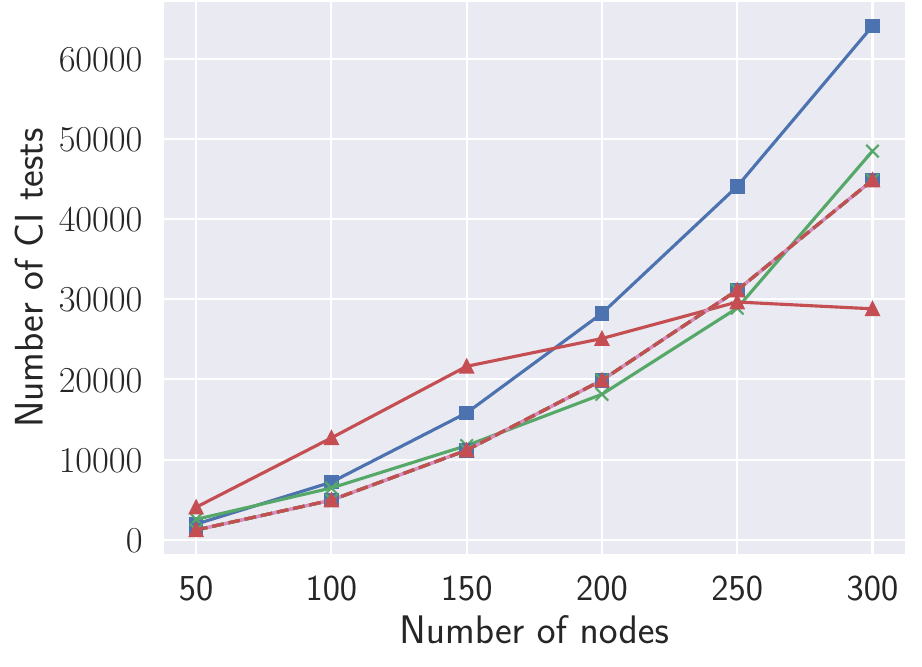}
            \caption*{$\chi^2$ tests}
            \label{fig:test_per_node_ident_chsq}
        \end{subfigure}
        \caption{Number of \ac{CI} tests.}
        \label{fig:test_per_node_ident}
    \end{subfigure}
    \begin{subfigure}[b]{\linewidth}
        \begin{subfigure}[b]{0.24\linewidth}
            \includegraphics[width=\linewidth]{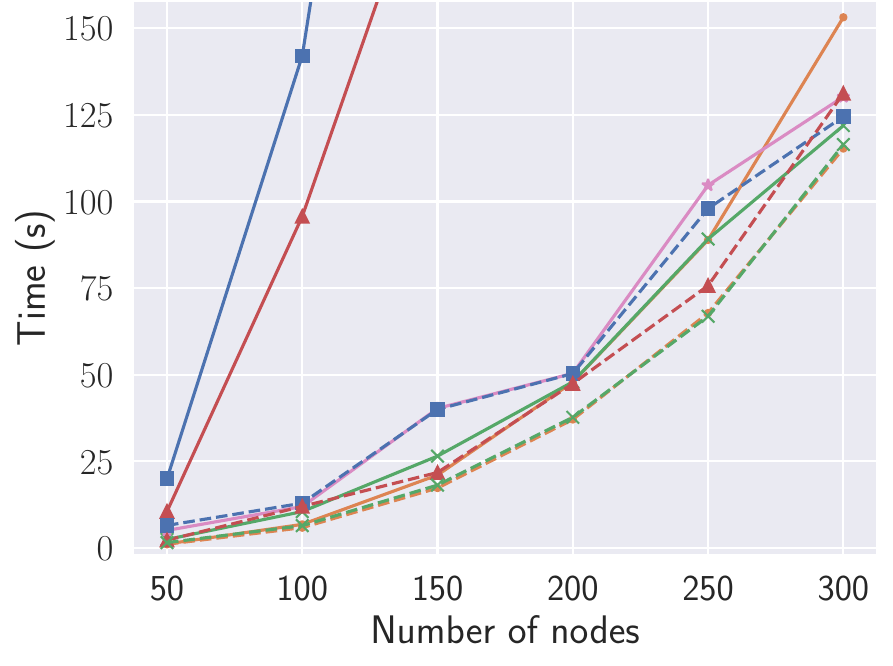}
            \caption*{d-separation tests}
            \label{fig:time_per_node_ident_dsep}
        \end{subfigure}
        \begin{subfigure}[b]{0.24\linewidth}
            \includegraphics[width=\linewidth]{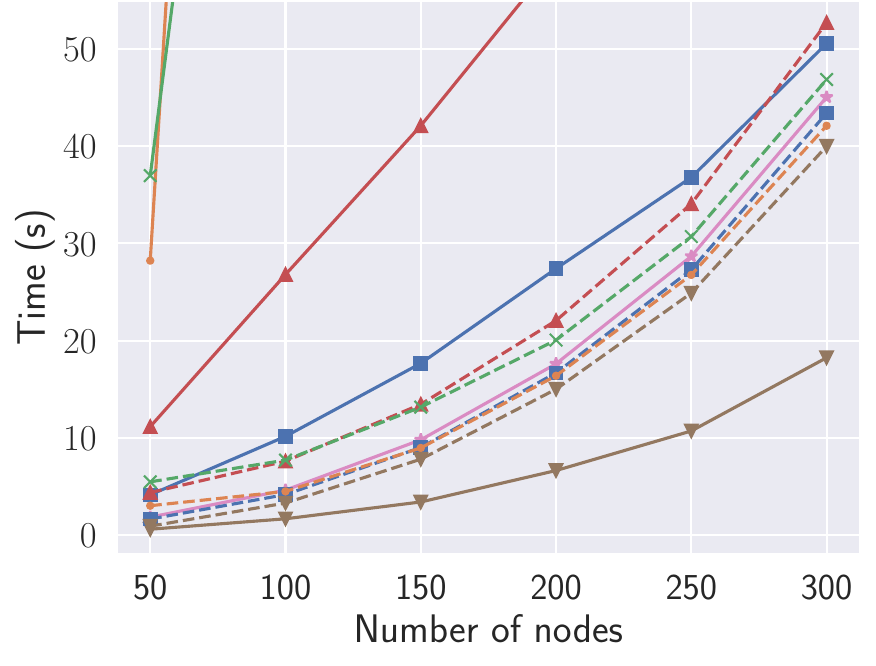}
            \caption*{Fisher-Z tests}
            \label{fig:time_per_node_ident_fshz}
        \end{subfigure}
        \begin{subfigure}[b]{0.24\linewidth}
            \includegraphics[width=\linewidth]{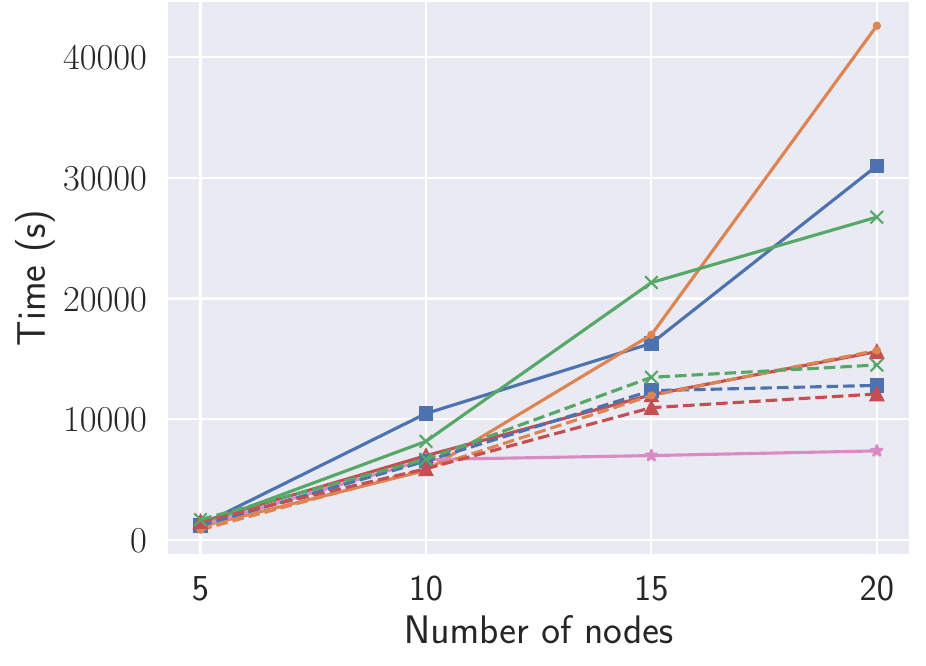}
            \caption*{KCI tests}
            \label{fig:time_per_node_ident_kci}
        \end{subfigure}
        \begin{subfigure}[b]{0.24\linewidth}
            \includegraphics[width=\linewidth]{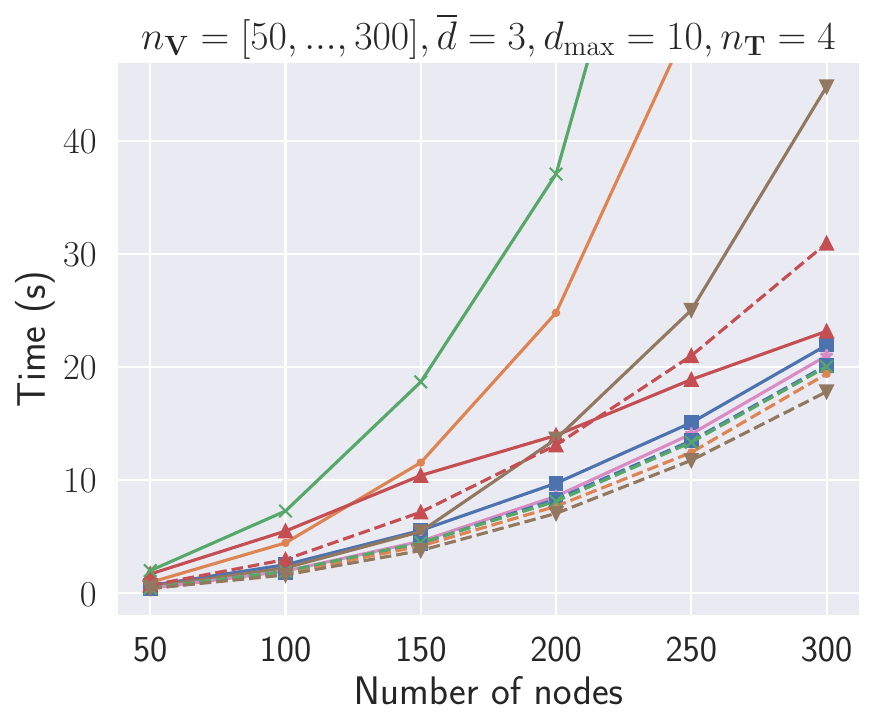}
            \caption*{$\chi^2$ tests}
            \label{fig:time_per_node_ident_chsq}
        \end{subfigure}
        \caption{Computation time.}
        \label{fig:time_per_node_ident}
    \end{subfigure}
    \begin{subfigure}[b]{\linewidth}
        \begin{subfigure}[b]{0.24\linewidth}
            \includegraphics[width=\linewidth]{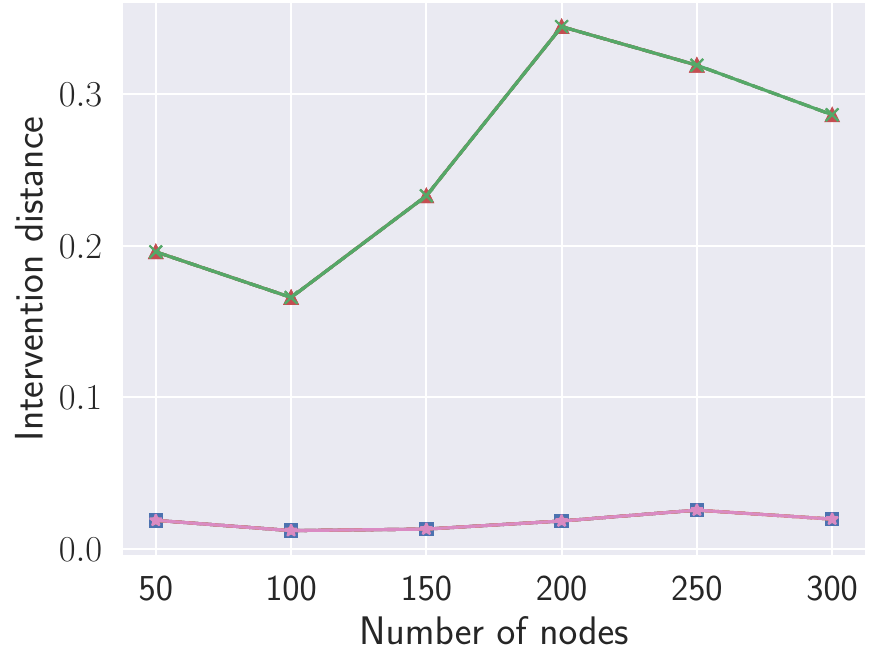}
            \caption*{d-separation tests}
            \label{fig:int_dist_per_node_ident_dsep_abs}
        \end{subfigure}
        \begin{subfigure}[b]{0.24\linewidth}
            \includegraphics[width=\linewidth]{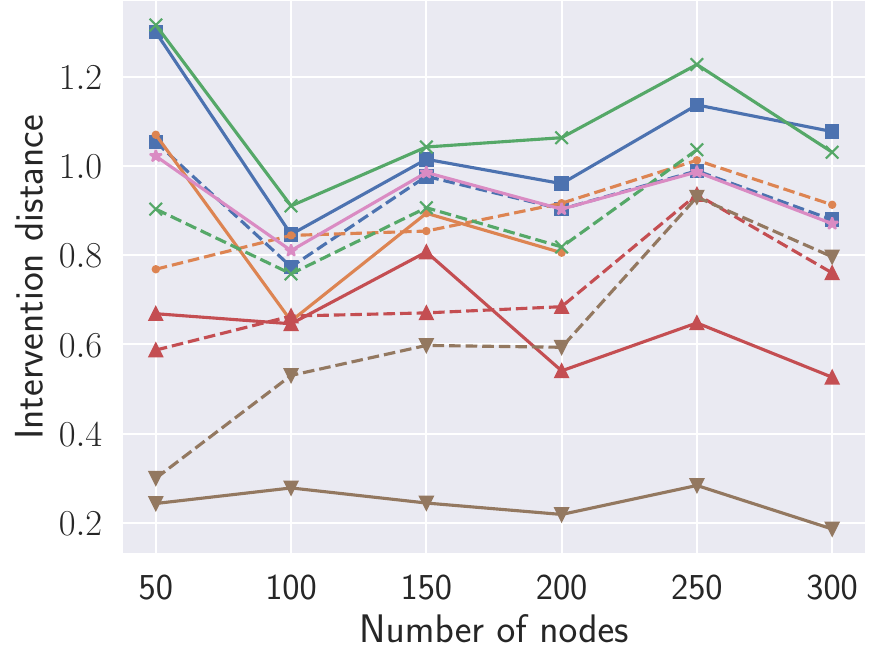}
            \caption*{Fisher-Z tests}
            \label{fig:int_dist_per_node_ident_fshz_abs}
        \end{subfigure}
        \begin{subfigure}[b]{0.24\linewidth}
            \includegraphics[width=\linewidth]{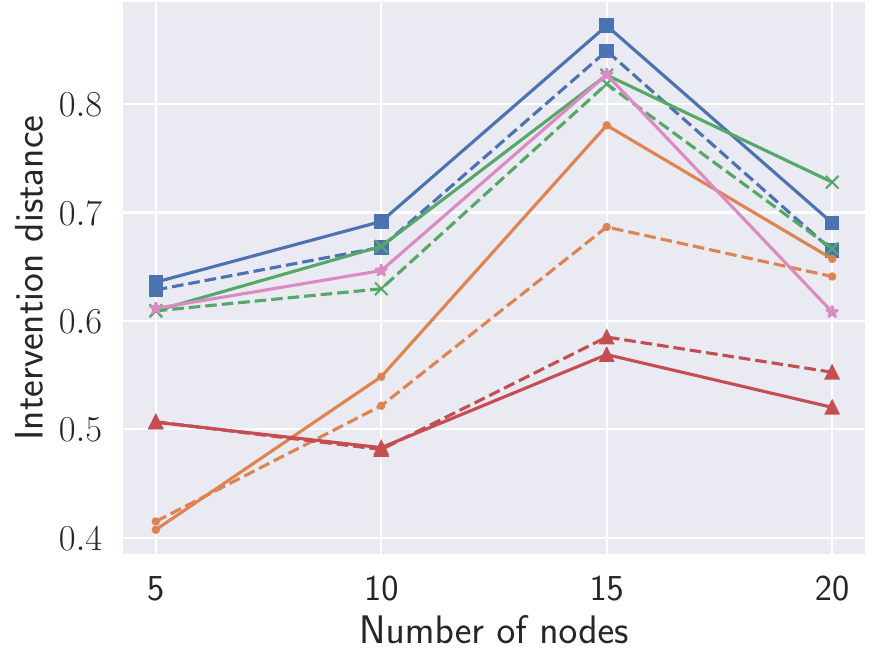}
            \caption*{KCI tests}
            \label{fig:int_dist_per_node_ident_kci_abs}
        \end{subfigure}
        \begin{subfigure}[b]{0.24\linewidth}
            \includegraphics[width=\linewidth]{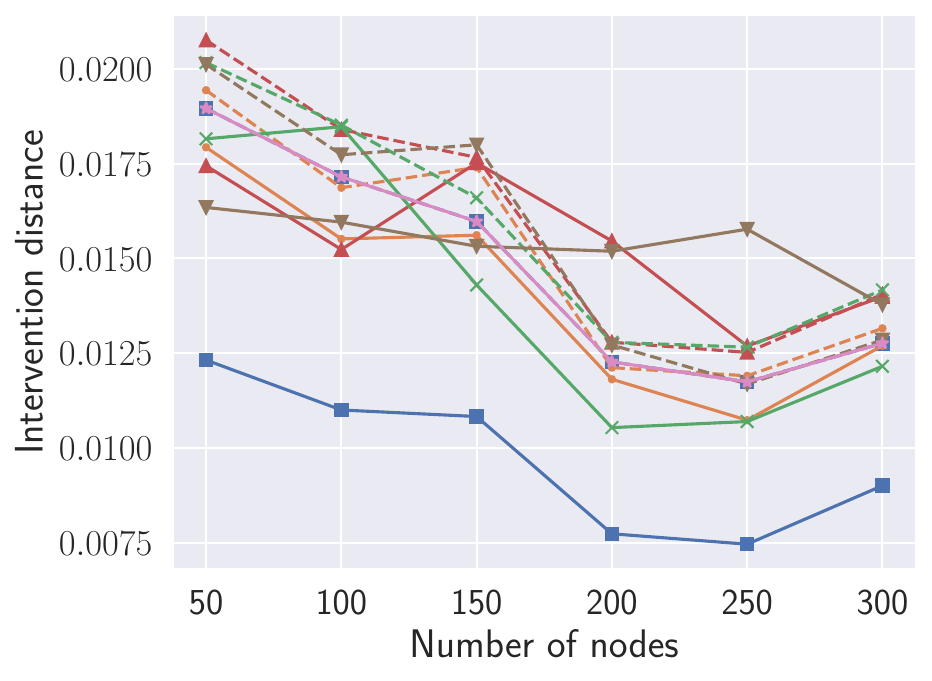}
            \caption*{$\chi^2$ tests}
            \label{fig:int_dist_per_node_ident_chsq_abs}
        \end{subfigure}
        \caption{Intervention distance.}
        \label{fig:int_dist_per_node_ident}
    \end{subfigure}
    \begin{subfigure}[b]{\linewidth}
        \begin{subfigure}[b]{0.24\linewidth}
            \includegraphics[width=\linewidth]{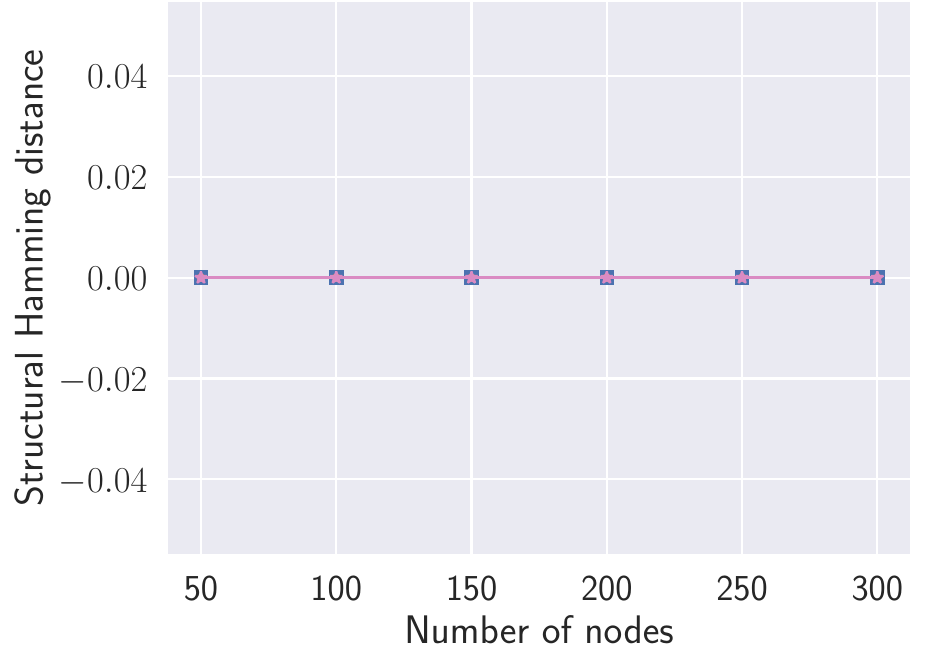}
            \caption*{d-separation tests}
            \label{fig:shd_per_node_ident_dsep}
        \end{subfigure}
        \begin{subfigure}[b]{0.24\linewidth}
            \includegraphics[width=\linewidth]{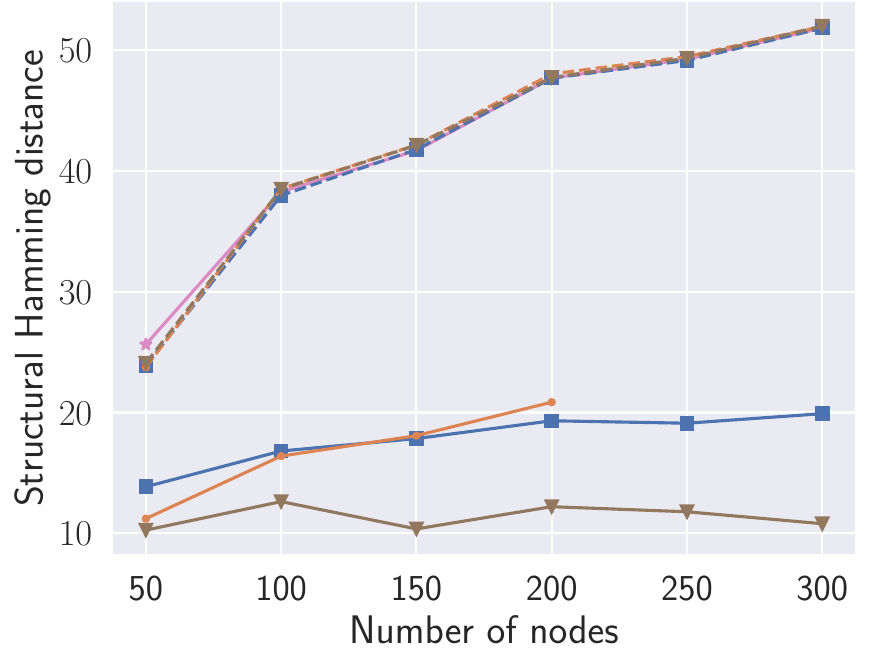}
            \caption*{Fisher-Z tests}
            \label{fig:shd_per_node_ident_fshz}
        \end{subfigure}
        \begin{subfigure}[b]{0.24\linewidth}
            \includegraphics[width=\linewidth]{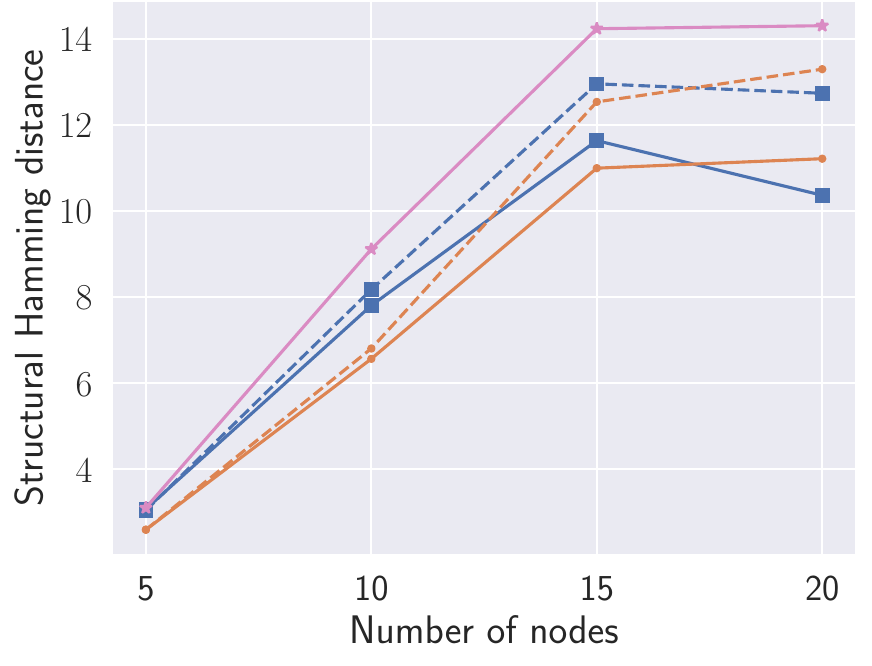}
            \caption*{KCI tests}
            \label{fig:shd_per_node_ident_kci}
        \end{subfigure}
        \begin{subfigure}[b]{0.24\linewidth}
            \includegraphics[width=\linewidth]{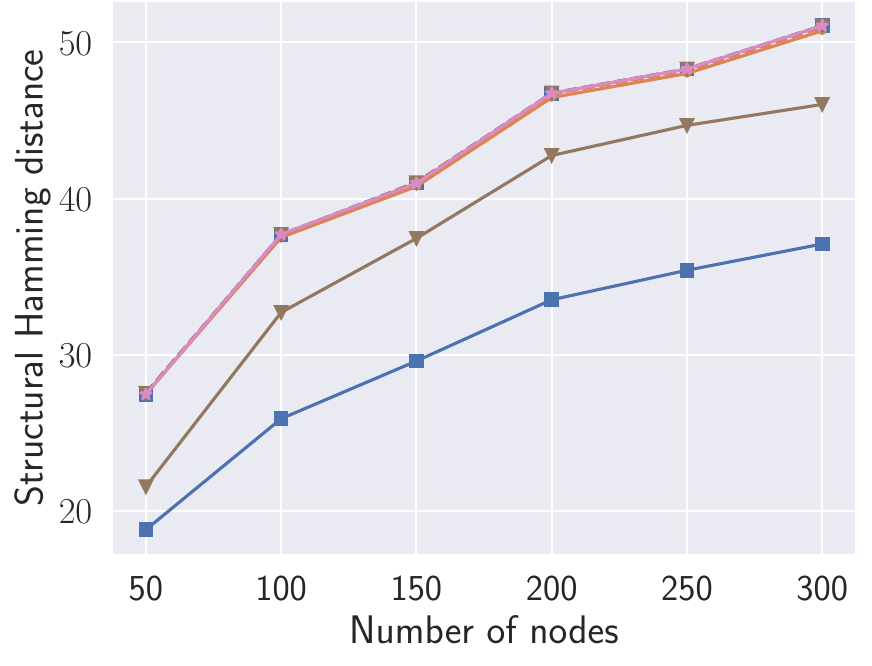}
            \caption*{$\chi^2$ tests}
            \label{fig:shd_per_node_ident_chsq}
        \end{subfigure}
        \caption{\Acf{SHD}.}
        \label{fig:shd_per_node_ident}
    \end{subfigure}
    \caption{Additional results over number of nodes for identifiable targets, with $n_{\mathbf{T}}=4, \overline{d} = 3, d_{\max}=10$ and $n_{\mathbf{D}} = 1000$ data-points. We also show baseline methods combined with SNAP$(0)$. We compute the intervention distance in the d-separation tests case using random linear Gaussian data according to the discovered structure.}
    \label{fig:appendix_per_node_ident}
\end{figure}

\begin{figure}
    \centering
    \includegraphics[width=.6\linewidth]{experiments/legend_small.pdf}
    \begin{subfigure}[b]{\linewidth}
        \begin{subfigure}[b]{0.24\linewidth}
            \includegraphics[width=\linewidth]{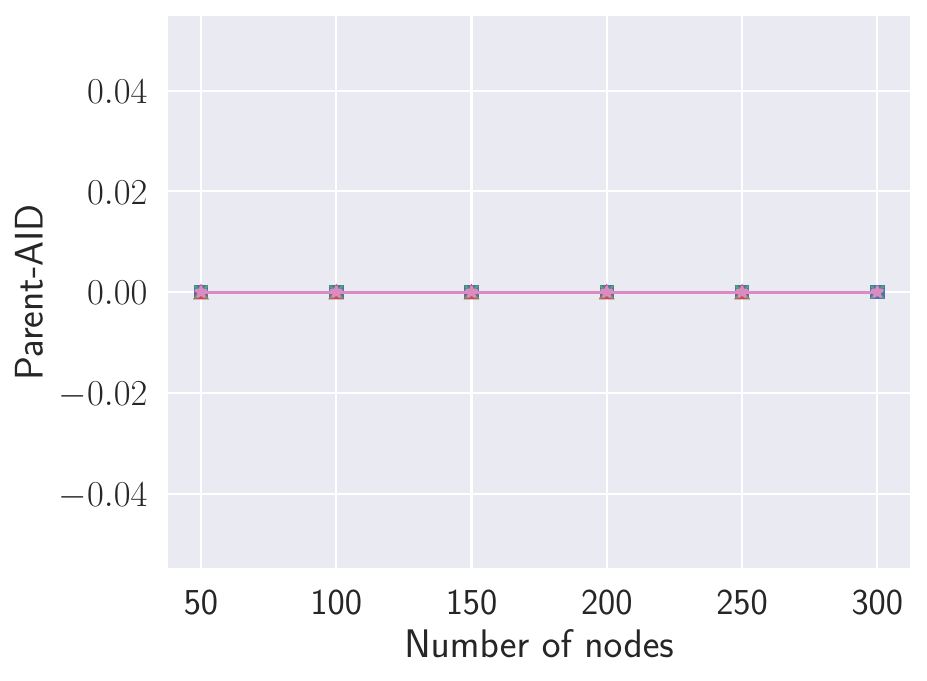}
            \caption*{d-separation tests}
            \label{fig:parent_aid_per_node_ident_dsep_std}
        \end{subfigure}
        \begin{subfigure}[b]{0.24\linewidth}
            \includegraphics[width=\linewidth]{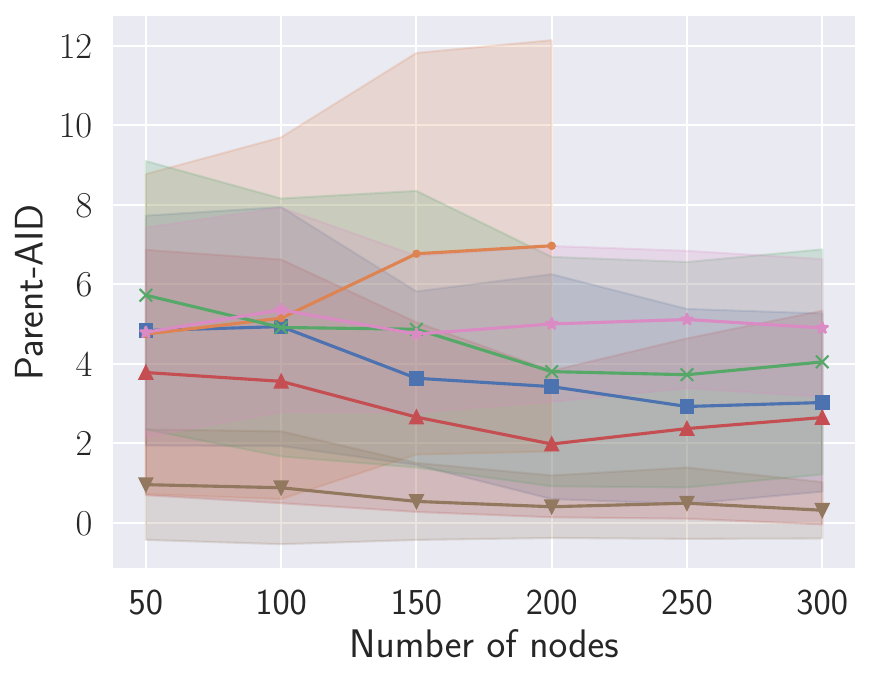}
            \caption*{Fisher-Z tests}
            \label{fig:parent_aid_per_node_ident_fshz_std}
        \end{subfigure}
        \begin{subfigure}[b]{0.24\linewidth}
            \includegraphics[width=\linewidth]{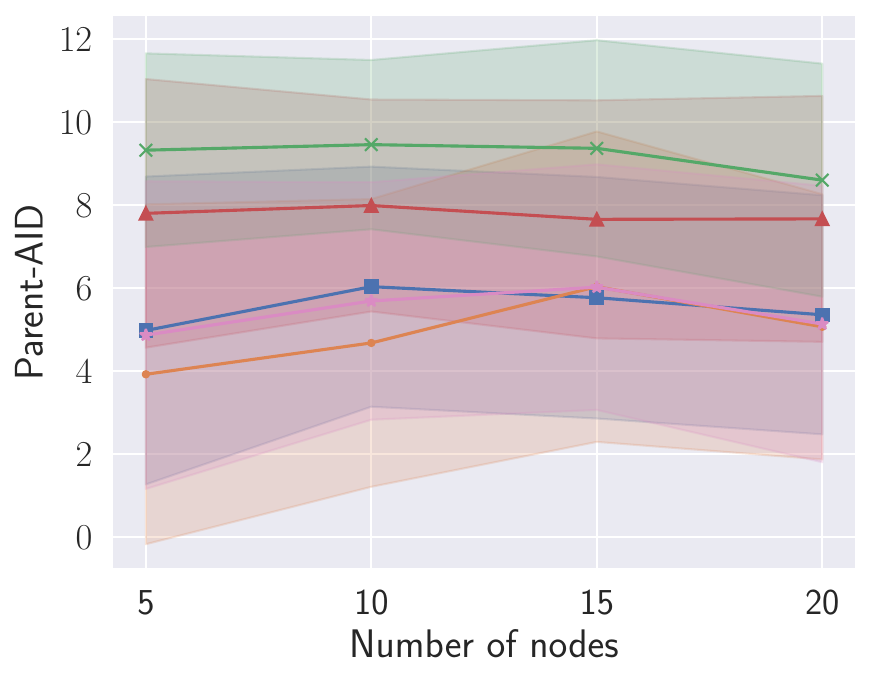}
            \caption*{KCI tests}
            \label{fig:parent_aid_per_node_ident_kci_std}
        \end{subfigure}
        \begin{subfigure}[b]{0.24\linewidth}
            \includegraphics[width=\linewidth]{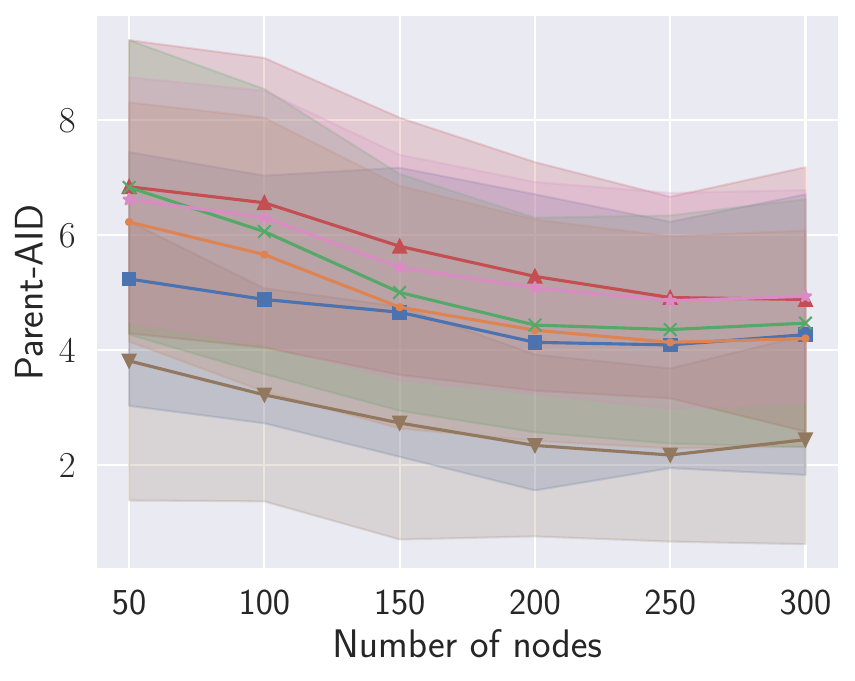}
            \caption*{$\chi^2$ tests}
            \label{fig:parent_aid_per_node_ident_chsq_std}
        \end{subfigure}
        \caption{Parent-AID.}
        \label{fig:parent_aid_per_node_ident_std}
    \end{subfigure}
    \begin{subfigure}[b]{\linewidth}
        \begin{subfigure}[b]{0.24\linewidth}
            \includegraphics[width=\linewidth]{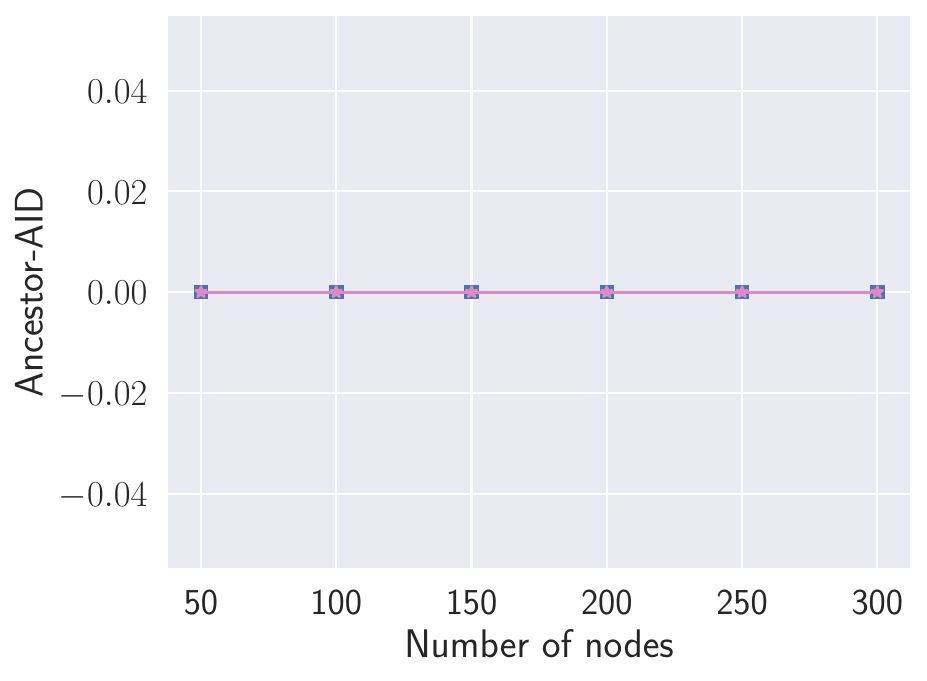}
            \caption*{d-separation tests}
            \label{fig:ancestor_aid_per_node_ident_dsep_std}
        \end{subfigure}
        \begin{subfigure}[b]{0.24\linewidth}
            \includegraphics[width=\linewidth]{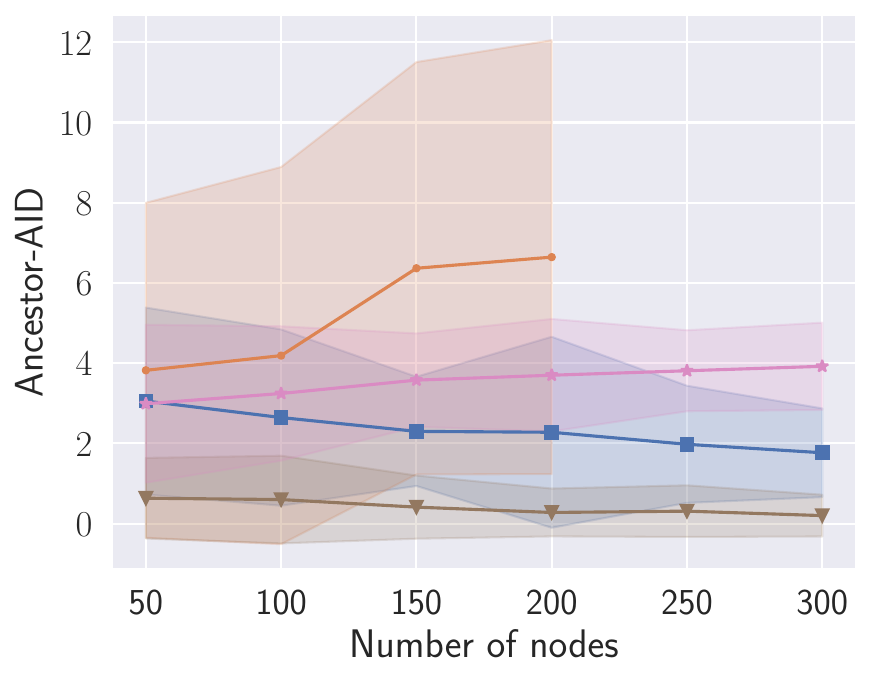}
            \caption*{Fisher-Z tests}
            \label{fig:ancestor_aid_per_node_ident_fshz_std}
        \end{subfigure}
        \begin{subfigure}[b]{0.24\linewidth}
            \includegraphics[width=\linewidth]{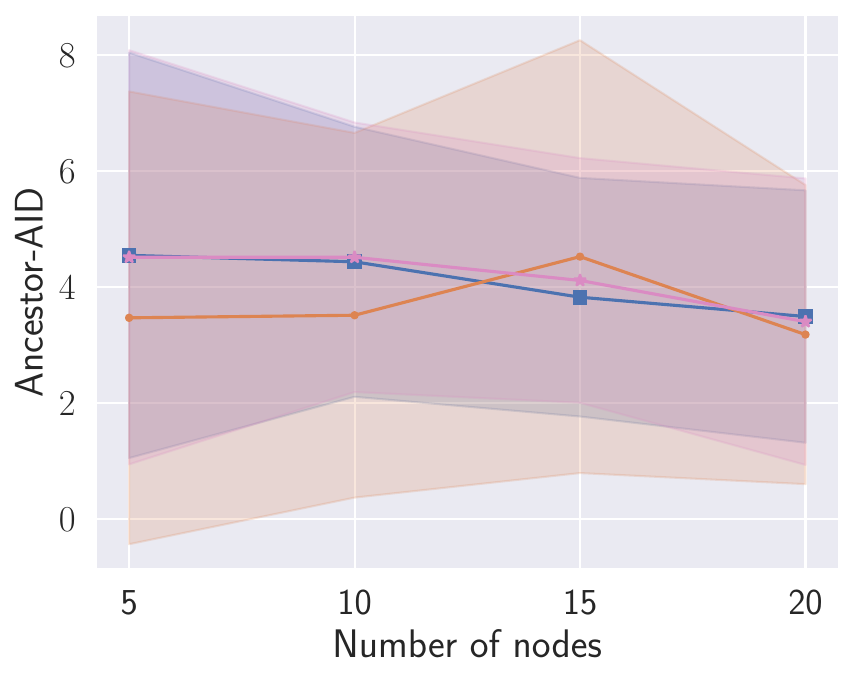}
            \caption*{KCI tests}
            \label{fig:ancestor_aid_per_node_ident_kci_std}
        \end{subfigure}
        \begin{subfigure}[b]{0.24\linewidth}
            \includegraphics[width=\linewidth]{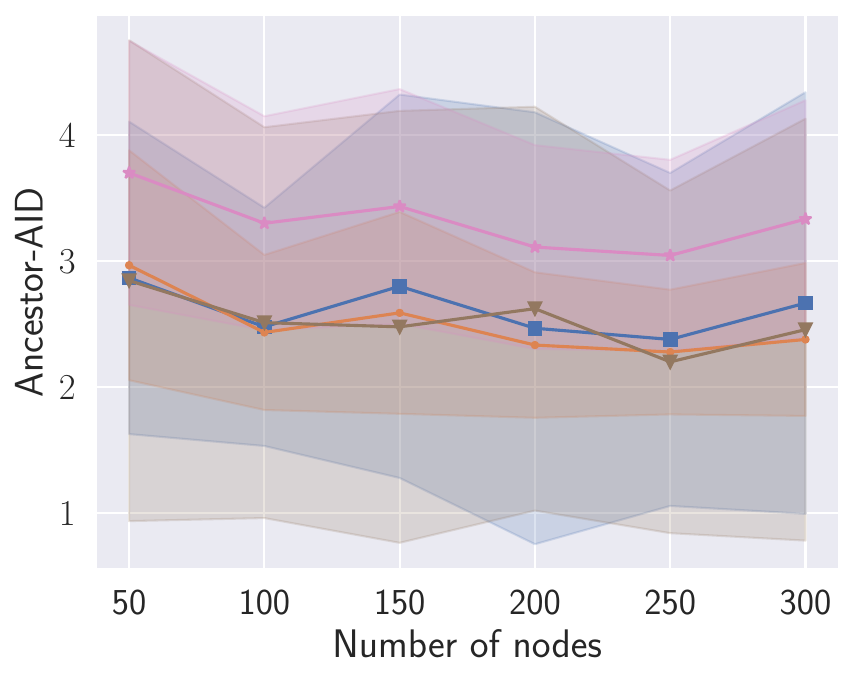}
            \caption*{$\chi^2$ tests}
            \label{fig:ancestor_aid_per_node_ident_chsq_std}
        \end{subfigure}
        \caption{Ancestor-AID.}
        \label{fig:ancestor_aid_per_node_ident_std}
    \end{subfigure}
    \begin{subfigure}[b]{\linewidth}
        \begin{subfigure}[b]{0.24\linewidth}
            \includegraphics[width=\linewidth]{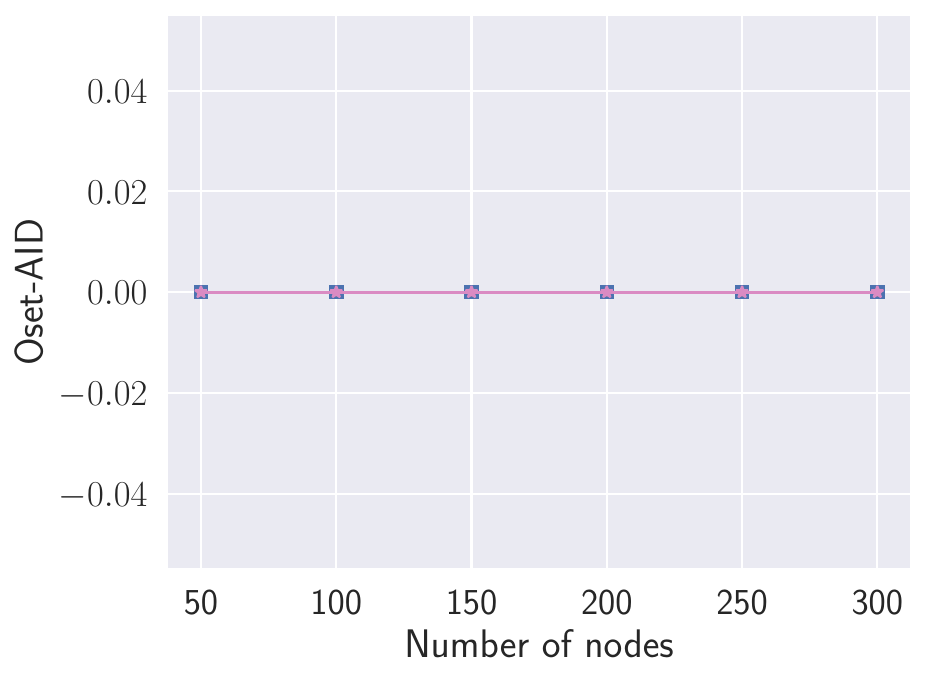}
            \caption*{d-separation tests}
            \label{fig:oset_aid_per_node_ident_dsep_std}
        \end{subfigure}
        \begin{subfigure}[b]{0.24\linewidth}
            \includegraphics[width=\linewidth]{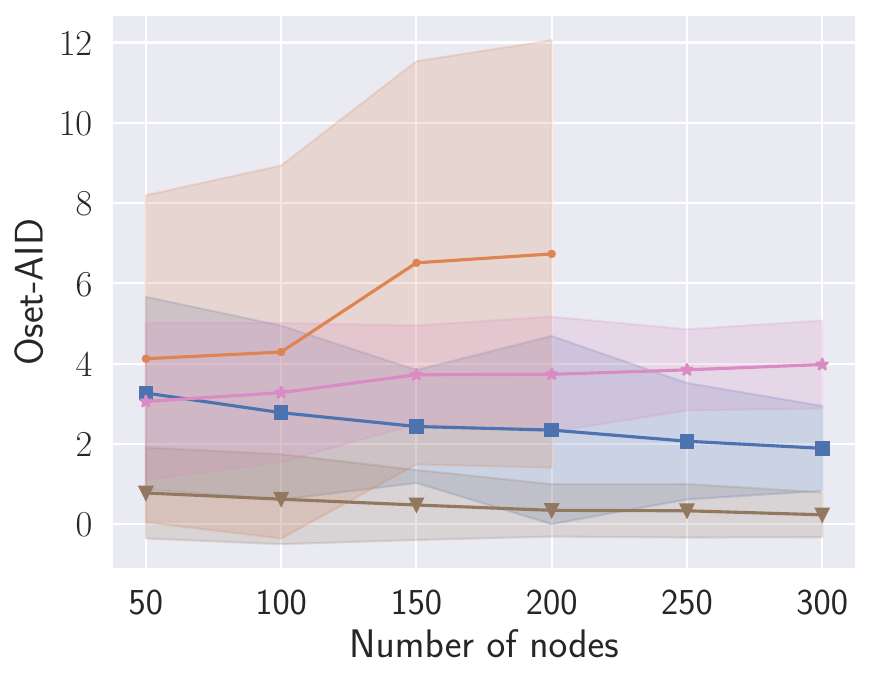}
            \caption*{Fisher-Z tests}
            \label{fig:oset_aid_per_node_ident_fshz_std}
        \end{subfigure}
        \begin{subfigure}[b]{0.24\linewidth}
            \includegraphics[width=\linewidth]{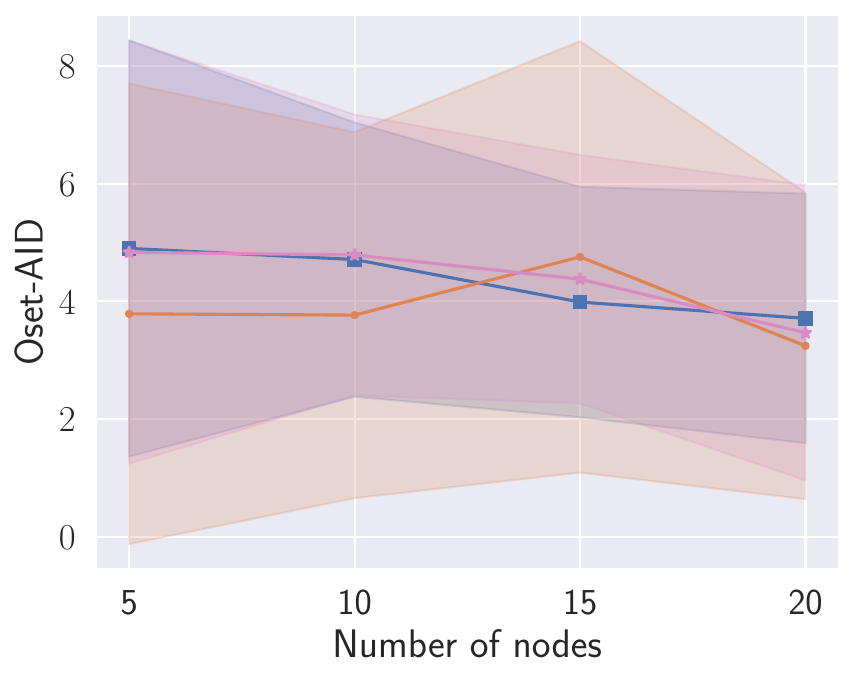}
            \caption*{KCI tests}
            \label{fig:oset_aid_per_node_ident_kci_std}
        \end{subfigure}
        \begin{subfigure}[b]{0.24\linewidth}
            \includegraphics[width=\linewidth]{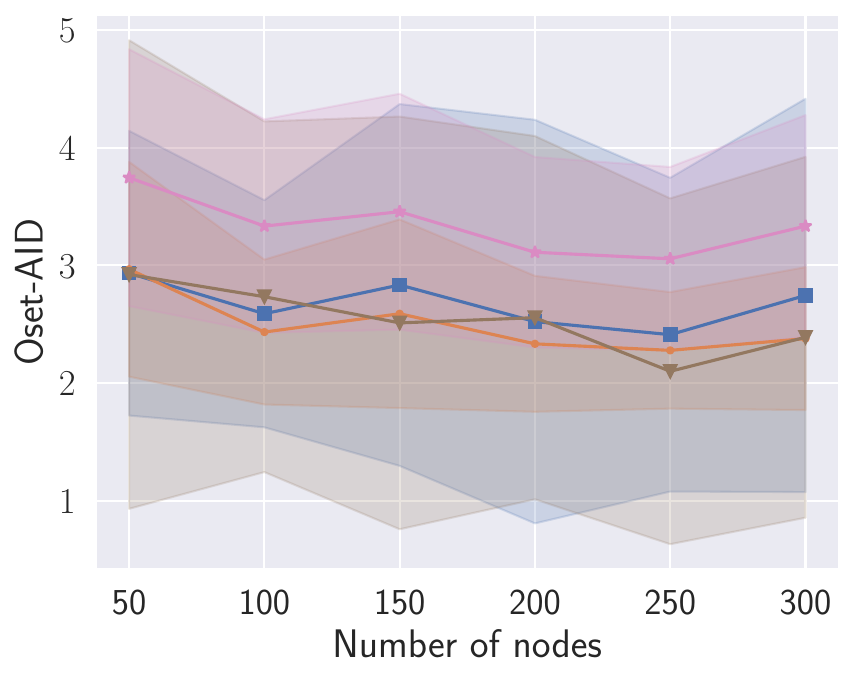}
            \caption*{$\chi^2$ tests}
            \label{fig:oset_aid_per_node_ident_chsq_std}
        \end{subfigure}
        \caption{Oset-AID.}
        \label{fig:oset_aid_per_node_ident_std}
    \end{subfigure}
    \caption{Adjustment identification distance over number of nodes for identifiable targets, with $n_{\mathbf{T}}=4, \overline{d} = 3, d_{\max}=10$ and $n_{\mathbf{D}} = 1000$ data-points. The shadow area denotes the range of the standard deviation.}
    \label{fig:aid_per_node_ident_std}
\end{figure}

\begin{figure}
    \centering
    \includegraphics[width=.6\linewidth]{experiments/legend_big.pdf}
    \begin{subfigure}[b]{\linewidth}
        \begin{subfigure}[b]{0.24\linewidth}
            \includegraphics[width=\linewidth]{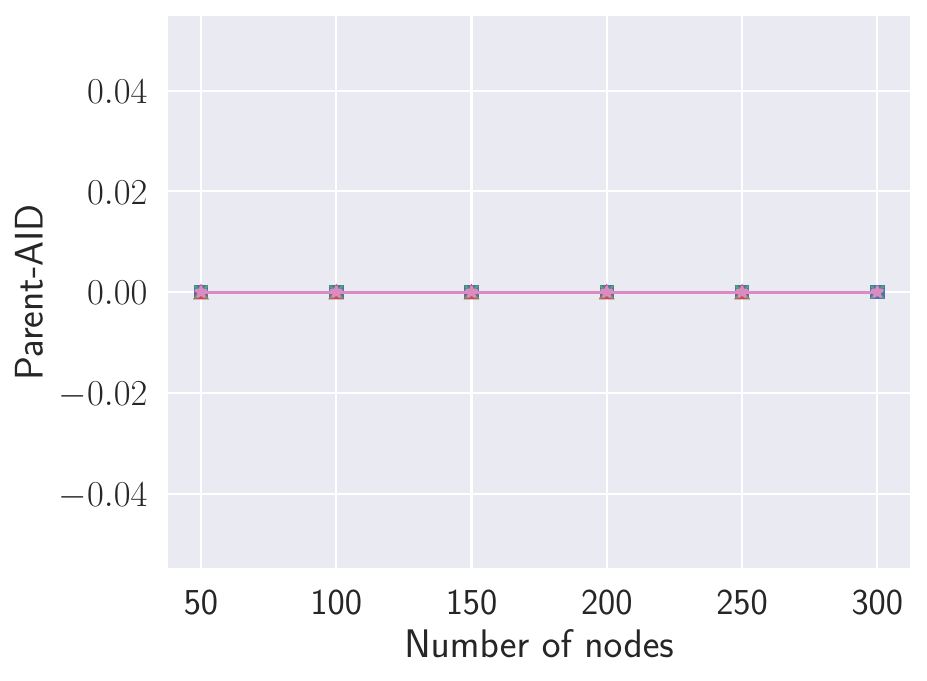}
            \caption*{d-separation tests}
            \label{fig:parent_aid_per_node_ident_dsep}
        \end{subfigure}
        \begin{subfigure}[b]{0.24\linewidth}
            \includegraphics[width=\linewidth]{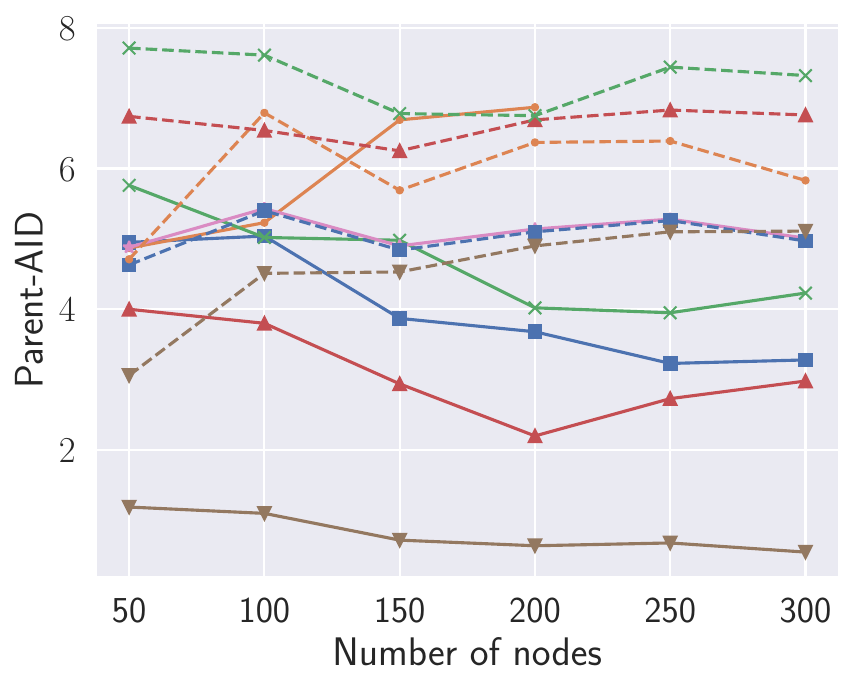}
            \caption*{Fisher-Z tests}
            \label{fig:parent_aid_per_node_ident_fshz}
        \end{subfigure}
        \begin{subfigure}[b]{0.24\linewidth}
            \includegraphics[width=\linewidth]{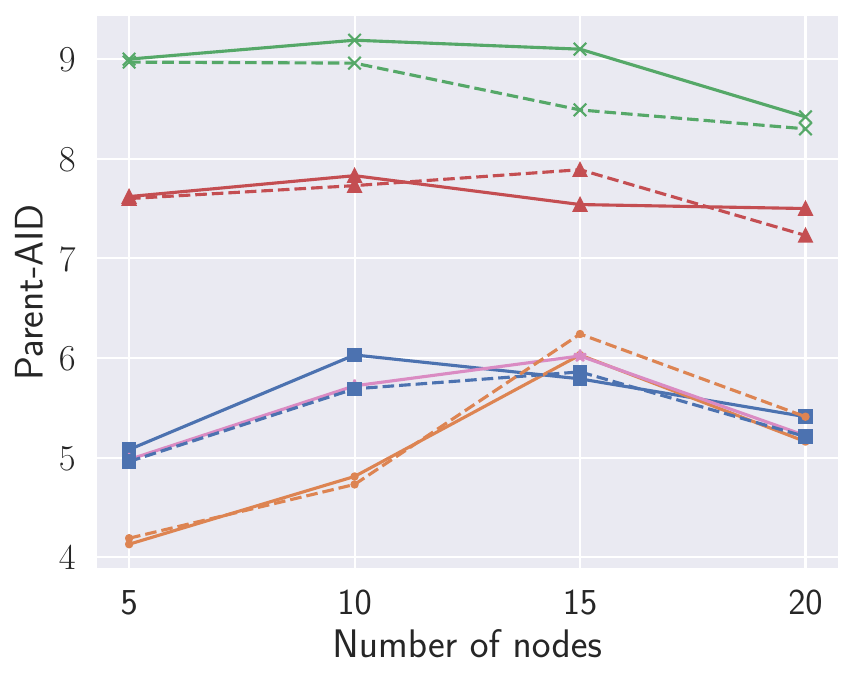}
            \caption*{KCI tests}
            \label{fig:parent_aid_per_node_ident_kci}
        \end{subfigure}
        \begin{subfigure}[b]{0.24\linewidth}
            \includegraphics[width=\linewidth]{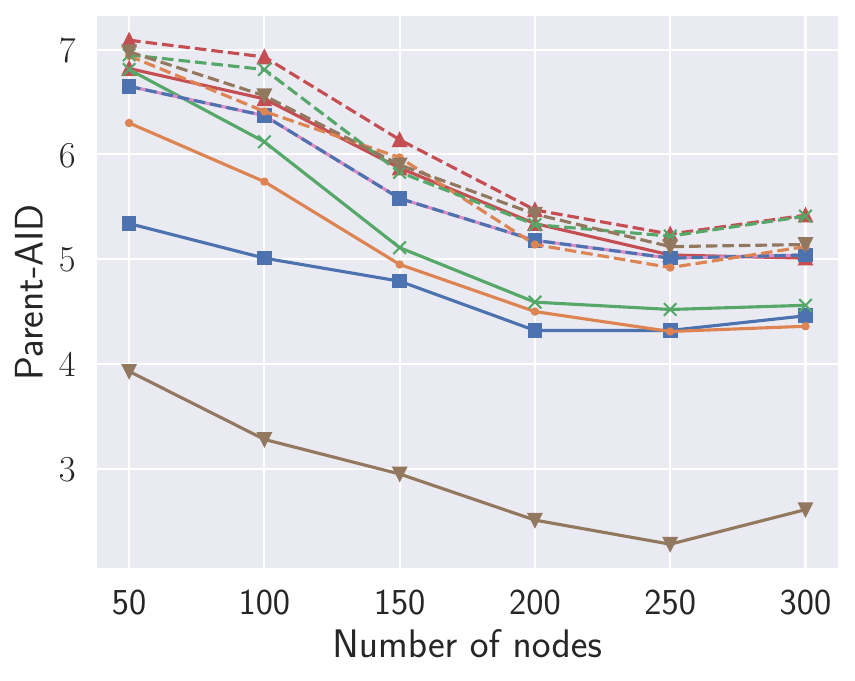}
            \caption*{$\chi^2$ tests}
            \label{fig:parent_aid_per_node_ident_chsq}
        \end{subfigure}
        \caption{Parent-AID.}
        \label{fig:parent_aid_per_node_ident}
    \end{subfigure}
    \begin{subfigure}[b]{\linewidth}
        \begin{subfigure}[b]{0.24\linewidth}
            \includegraphics[width=\linewidth]{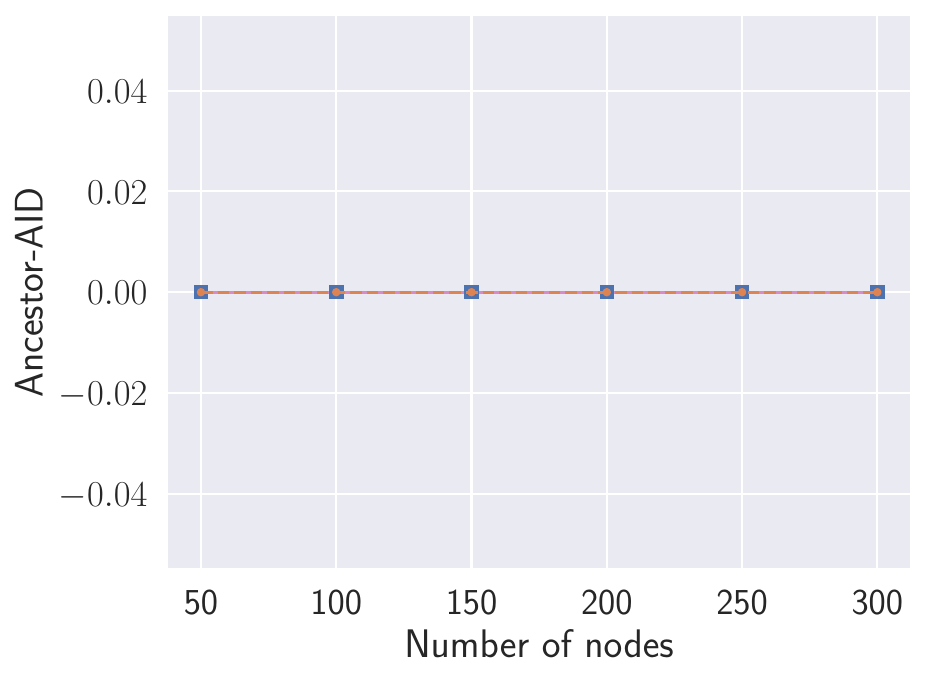}
            \caption*{d-separation tests}
            \label{fig:ancestor_aid_per_node_ident_dsep}
        \end{subfigure}
        \begin{subfigure}[b]{0.24\linewidth}
            \includegraphics[width=\linewidth]{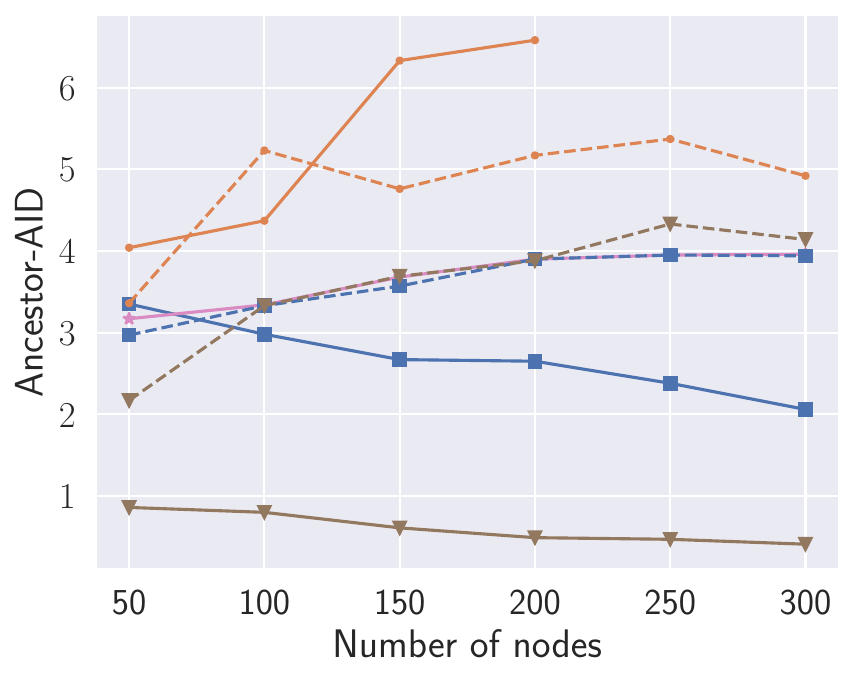}
            \caption*{Fisher-Z tests}
            \label{fig:ancestor_aid_per_node_ident_fshz}
        \end{subfigure}
        \begin{subfigure}[b]{0.24\linewidth}
            \includegraphics[width=\linewidth]{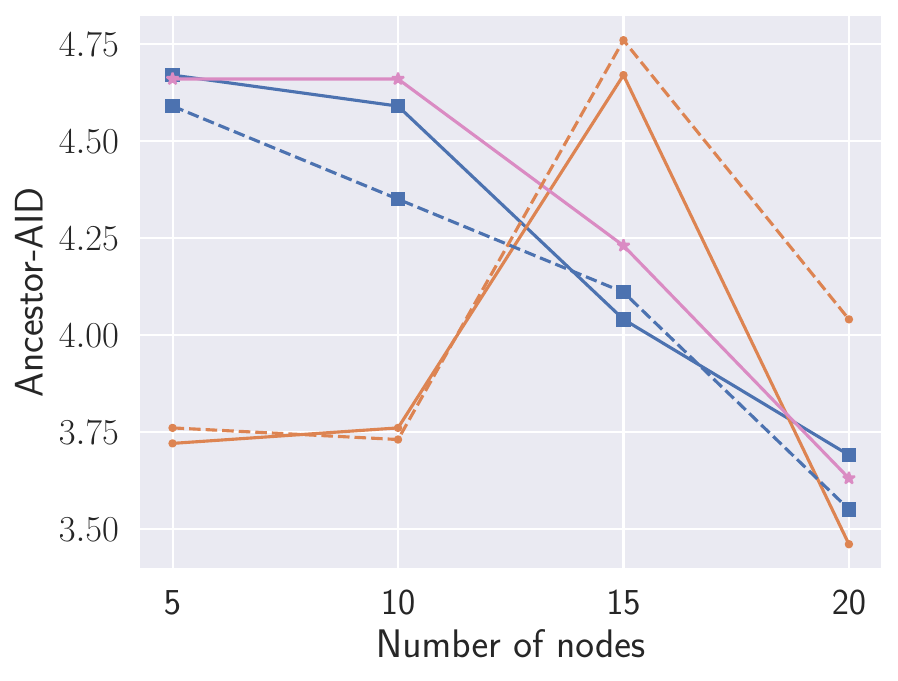}
            \caption*{KCI tests}
            \label{fig:ancestor_aid_per_node_ident_kci}
        \end{subfigure}
        \begin{subfigure}[b]{0.24\linewidth}
            \includegraphics[width=\linewidth]{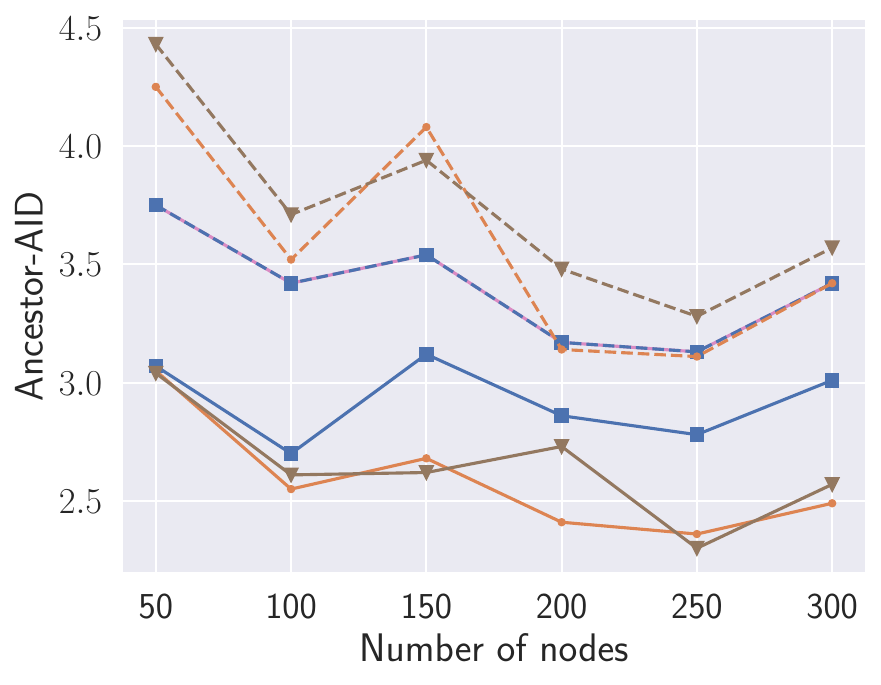}
            \caption*{$\chi^2$ tests}
            \label{fig:ancestor_aid_per_node_ident_chsq}
        \end{subfigure}
        \caption{Ancestor-AID.}
        \label{fig:ancestor_aid_per_node_ident}
    \end{subfigure}
    \begin{subfigure}[b]{\linewidth}
        \begin{subfigure}[b]{0.24\linewidth}
            \includegraphics[width=\linewidth]{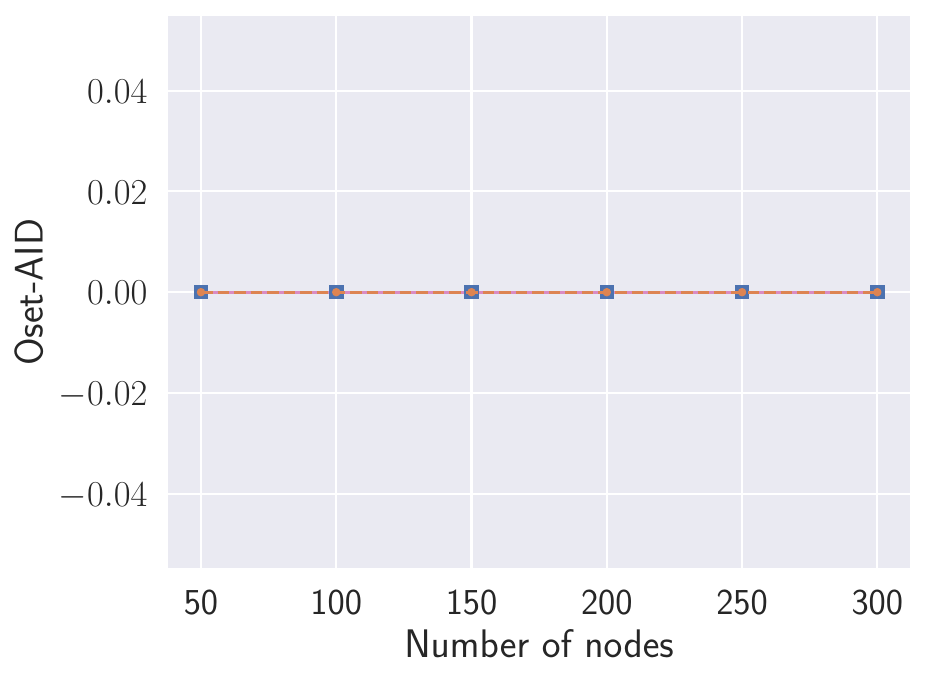}
            \caption*{d-separation tests}
            \label{fig:oset_aid_per_node_ident_dsep}
        \end{subfigure}
        \begin{subfigure}[b]{0.24\linewidth}
            \includegraphics[width=\linewidth]{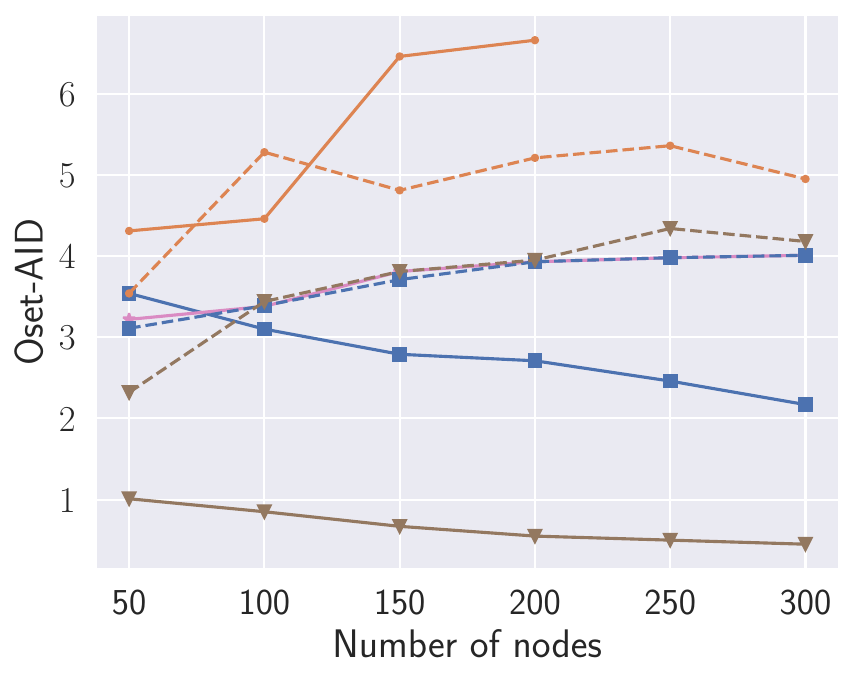}
            \caption*{Fisher-Z tests}
            \label{fig:oset_aid_per_node_ident_fshz}
        \end{subfigure}
        \begin{subfigure}[b]{0.24\linewidth}
            \includegraphics[width=\linewidth]{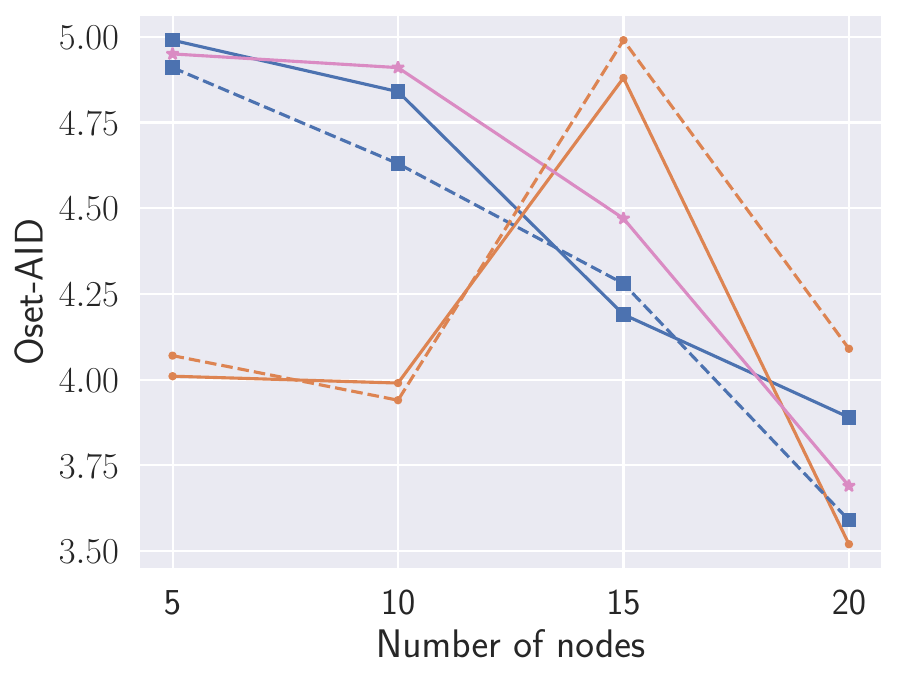}
            \caption*{KCI tests}
            \label{fig:oset_aid_per_node_ident_kci}
        \end{subfigure}
        \begin{subfigure}[b]{0.24\linewidth}
            \includegraphics[width=\linewidth]{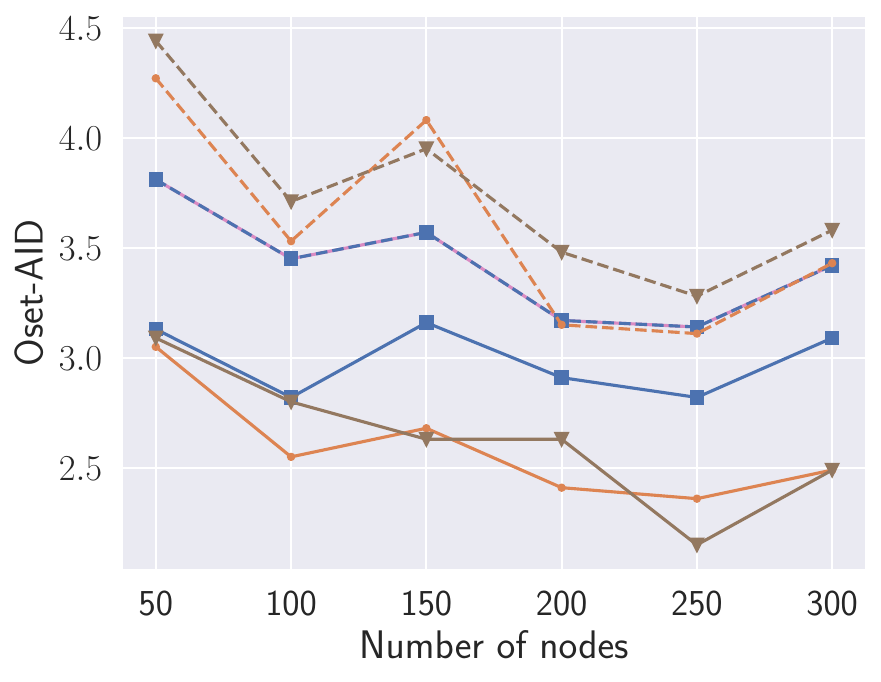}
            \caption*{$\chi^2$ tests}
            \label{fig:oset_aid_per_node_ident_chsq}
        \end{subfigure}
        \caption{Oset-AID.}
        \label{fig:oset_aid_per_node_ident}
    \end{subfigure}
    \caption{Adjustment identification distance over number of nodes for identifiable targets, with $n_{\mathbf{T}}=4, \overline{d} = 3, d_{\max}=10$ and $n_{\mathbf{D}} = 1000$ data-points. We also show baseline methods combined with SNAP$(0)$.}
    \label{fig:aid_per_node_ident}
\end{figure}

\subsection{Various numbers of targets}
\label{app:over_targets}

\Cref{fig:computation_per_target_unrest_std,fig:quality_per_target_unrest_std,fig:computation_per_target_unrest,fig:quality_per_target_unrest} and \Cref{fig:computation_per_target_ident_std,fig:quality_per_target_ident_std,fig:computation_per_target_ident,fig:quality_per_target_ident} show results for varying number of random and identifiable targets respectively, with all other parameters fixed.
In terms of \ac{CI} tests and computation time, SNAP is much more robust against different number of targets than local methods, while maintaining a comparable intervention distance.
Furthermore, SHD is only worse compared to other methods when using Fisher-Z tests on linear Gaussian data.

\begin{figure}
    \centering
    \includegraphics[width=.6\linewidth]{experiments/legend_big.pdf}
    \begin{subfigure}[b]{\linewidth}
        \centering
        \begin{subfigure}[b]{0.24\linewidth}
            \includegraphics[width=\linewidth]{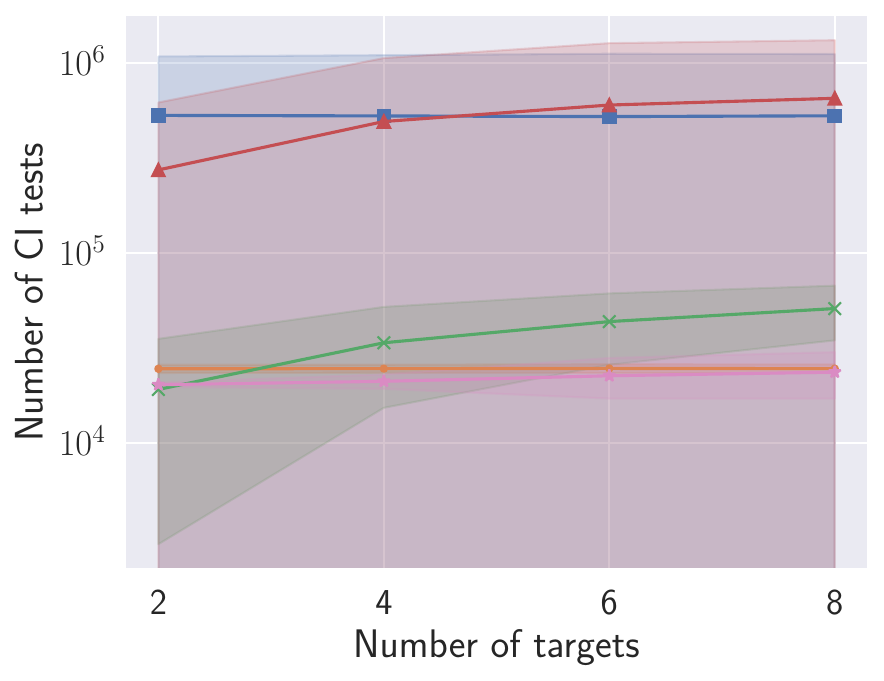}
            \caption*{d-separation tests}
            \label{fig:test_per_target_unrest_dsep_std}
        \end{subfigure}
        \begin{subfigure}[b]{0.24\linewidth}
            \includegraphics[width=\linewidth]{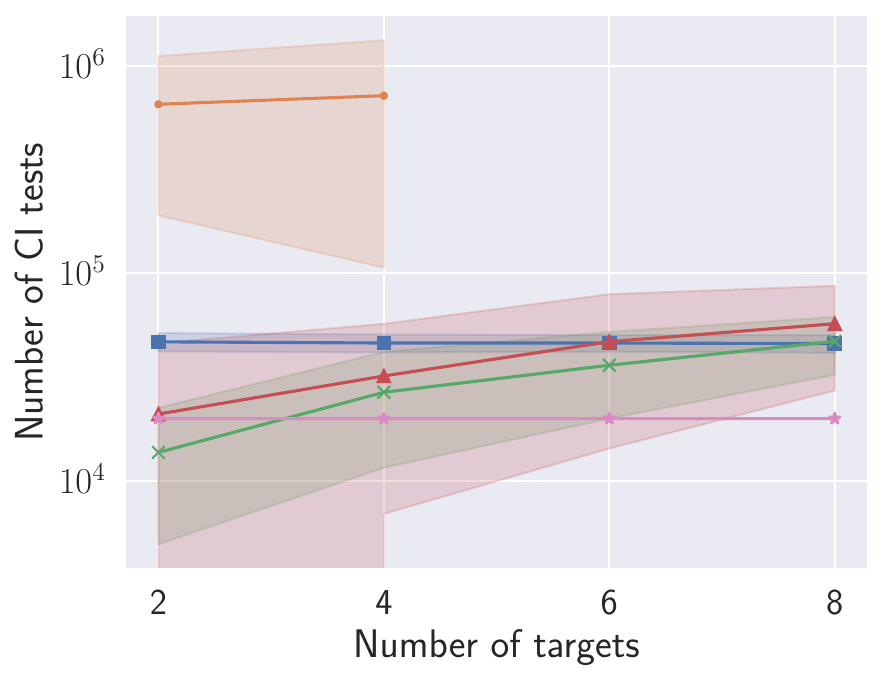}
            \caption*{Fisher-Z tests}
            \label{fig:test_per_target_unrest_fshz_std}
        \end{subfigure}
        \begin{subfigure}[b]{0.24\linewidth}
            \includegraphics[width=\linewidth]{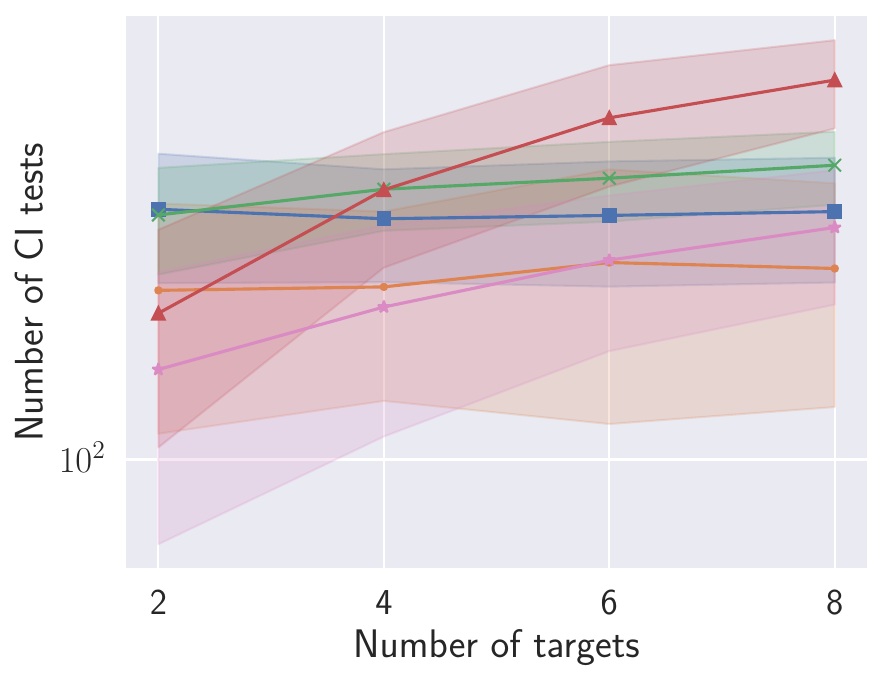}
            \caption*{KCI tests}
            \label{fig:test_per_target_unrest_kci_std}
        \end{subfigure}
        \begin{subfigure}[b]{0.24\linewidth}
            \includegraphics[width=\linewidth]{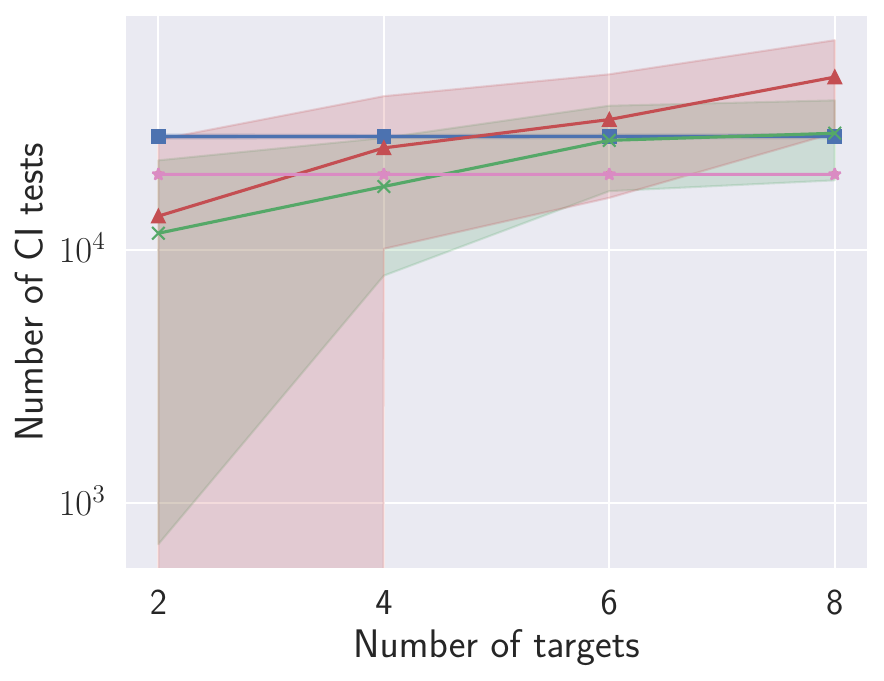}
            \caption*{$\chi^2$ tests}
            \label{fig:test_per_target_unrest_chsq_std}
        \end{subfigure}
        \caption{Number of \ac{CI} tests.}
        \label{fig:test_per_target_unrest_std}
    \end{subfigure}
    \begin{subfigure}[b]{\linewidth}
        \centering
        \begin{subfigure}[b]{0.24\linewidth}
            \includegraphics[width=\linewidth]{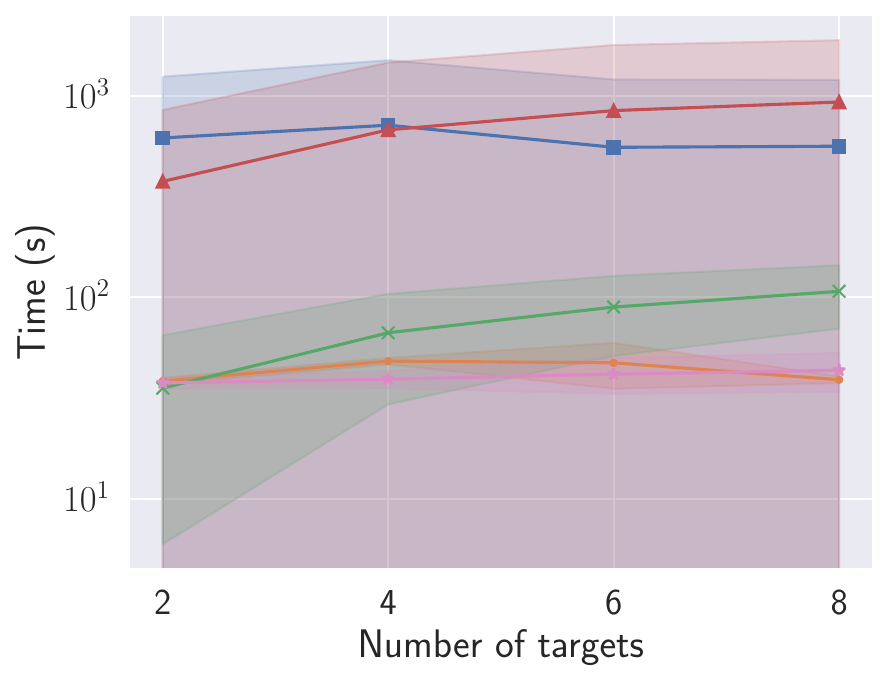}
            \caption*{d-separation tests}
            \label{fig:time_per_target_unrest_dsep_std}
        \end{subfigure}
        \begin{subfigure}[b]{0.24\linewidth}
            \includegraphics[width=\linewidth]{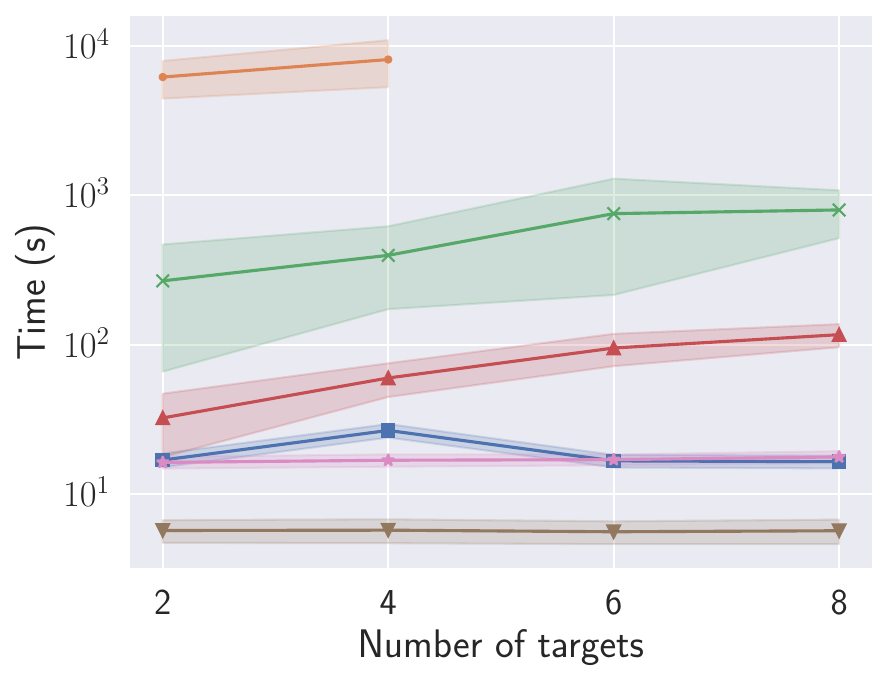}
            \caption*{Fisher-Z tests}
            \label{fig:time_per_target_unrest_fshz_std}
        \end{subfigure}
        \begin{subfigure}[b]{0.24\linewidth}
            \includegraphics[width=\linewidth]{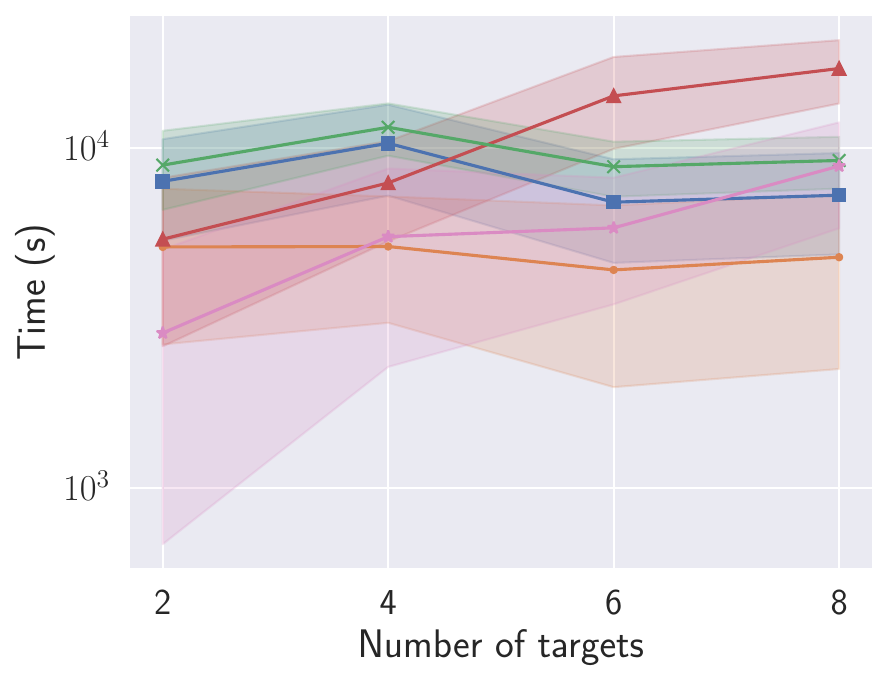}
            \caption*{KCI tests}
            \label{fig:time_per_target_unrest_kci_std}
        \end{subfigure}
        \begin{subfigure}[b]{0.24\linewidth}
            \includegraphics[width=\linewidth]{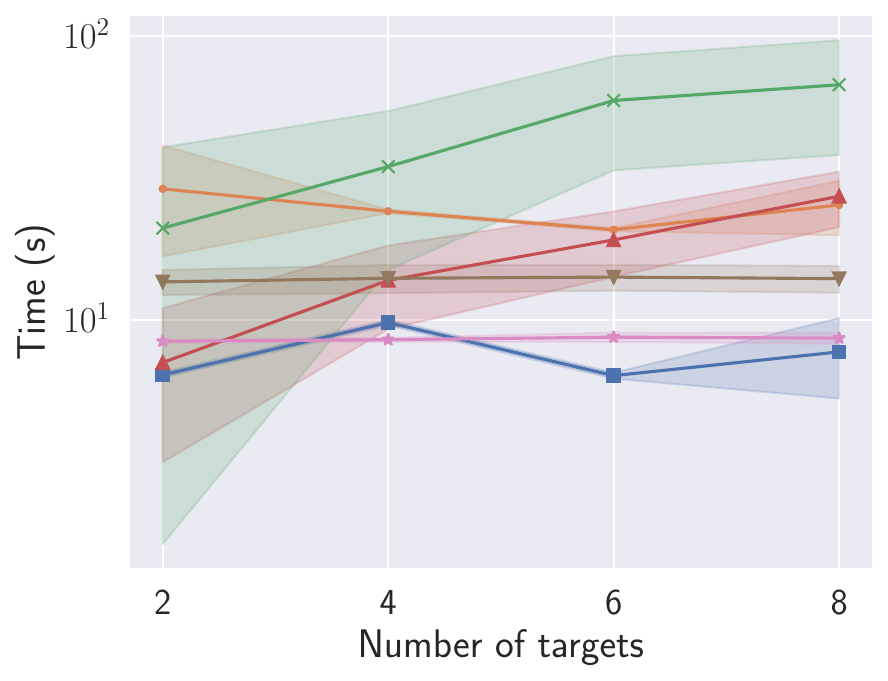}
            \caption*{$\chi^2$ tests}
            \label{fig:time_per_target_unrest_chsq_std}
        \end{subfigure}
        \caption{Computation time.}
        \label{fig:time_per_target_unrest_std}
    \end{subfigure}
    \caption{Number of \ac{CI} tests and computation time over number of targets, with $n_{\mathbf{V}}=10$ for KCI tests and $n_{\mathbf{V}}=200$ otherwise, $\overline{d} = 3, d_{\max}=10$ and $n_{\mathbf{D}} = 1000$ data-points. The shadow area denotes the range of the standard deviation. We plot values on a $\log$ scale.}
    \label{fig:computation_per_target_unrest_std}
\end{figure}

\begin{figure}
    \centering
    \includegraphics[width=.6\linewidth]{experiments/legend_big.pdf}
    \begin{subfigure}[b]{\linewidth}
        \centering
        \begin{subfigure}[b]{0.24\linewidth}
            \includegraphics[width=\linewidth]{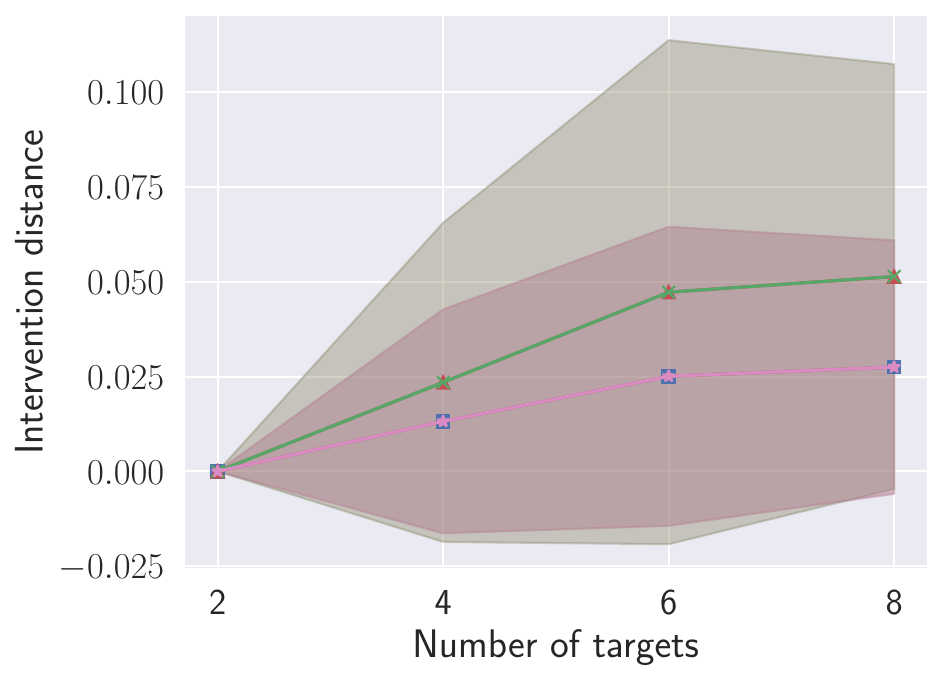}
            \caption*{d-separation tests}
            \label{fig:int_dist_per_target_unrest_dsep_abs_std}
        \end{subfigure}
        \begin{subfigure}[b]{0.24\linewidth}
            \includegraphics[width=\linewidth]{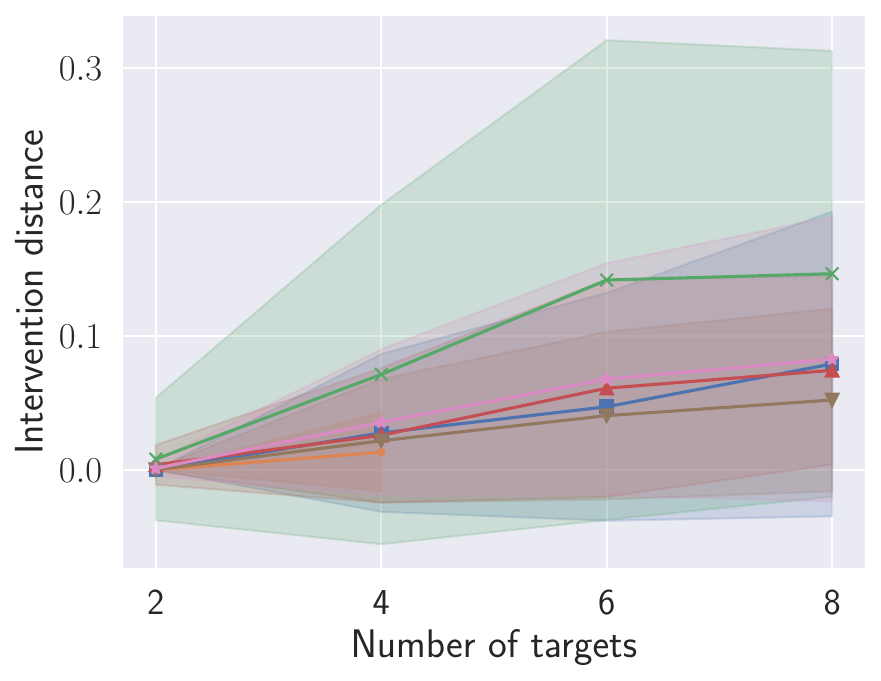}
            \caption*{Fisher-Z tests}
            \label{fig:int_dist_per_target_unrest_fshz_abs_std}
        \end{subfigure}
        \begin{subfigure}[b]{0.24\linewidth}
            \includegraphics[width=\linewidth]{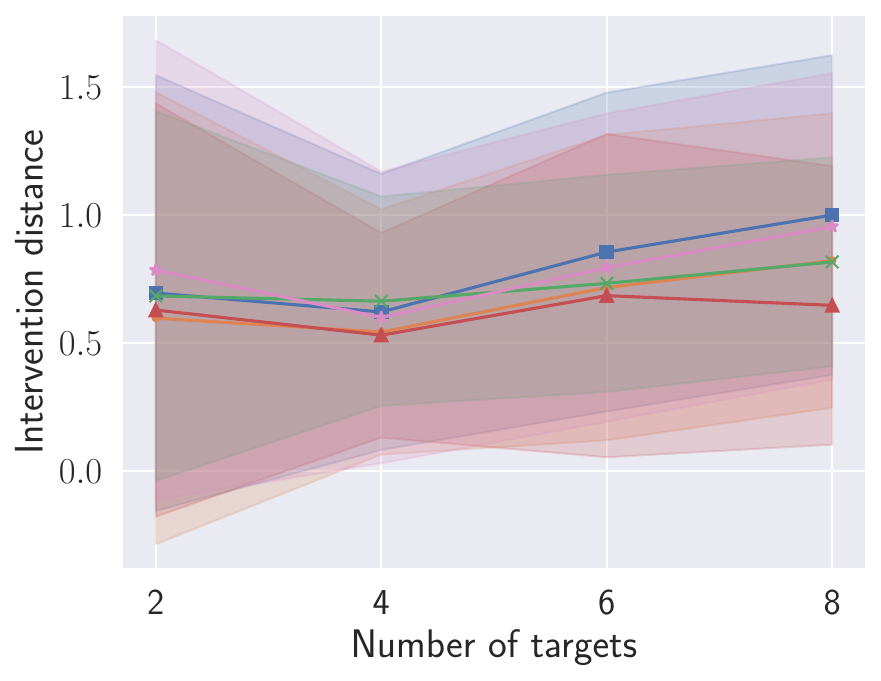}
            \caption*{KCI tests}
            \label{fig:int_dist_per_target_unrest_kci_abs_std}
        \end{subfigure}
        \begin{subfigure}[b]{0.24\linewidth}
            \includegraphics[width=\linewidth]{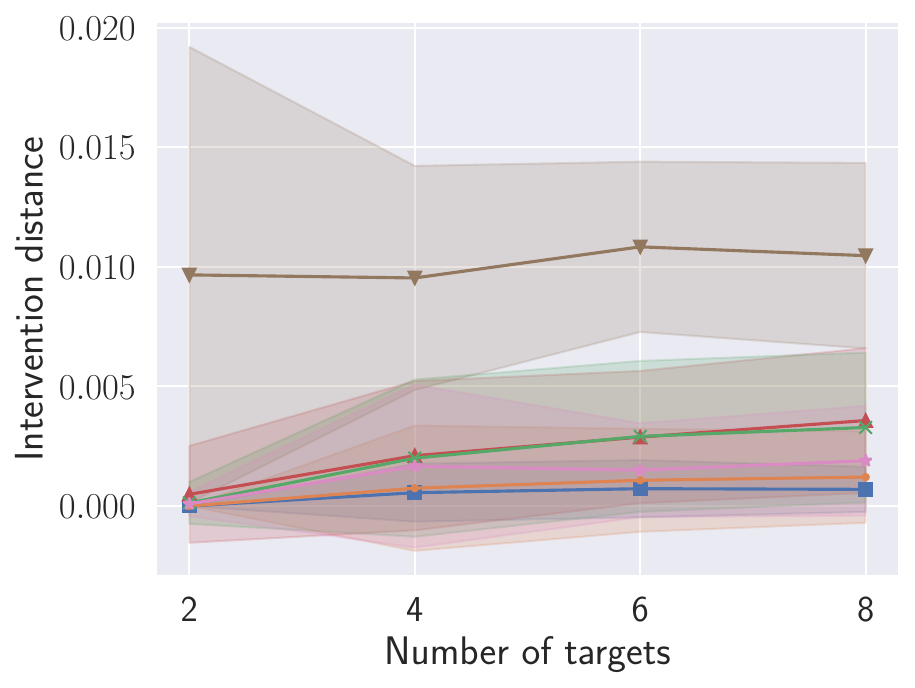}
            \caption*{$\chi^2$ tests}
            \label{fig:int_dist_per_target_unrest_chsq_abs_std}
        \end{subfigure}
        \caption{Intervention distance}
        \label{fig:int_dist_per_target_unrest_abs_std}
    \end{subfigure}
    \begin{subfigure}[b]{\linewidth}
        \centering
        \begin{subfigure}[b]{0.24\linewidth}
            \includegraphics[width=\linewidth]{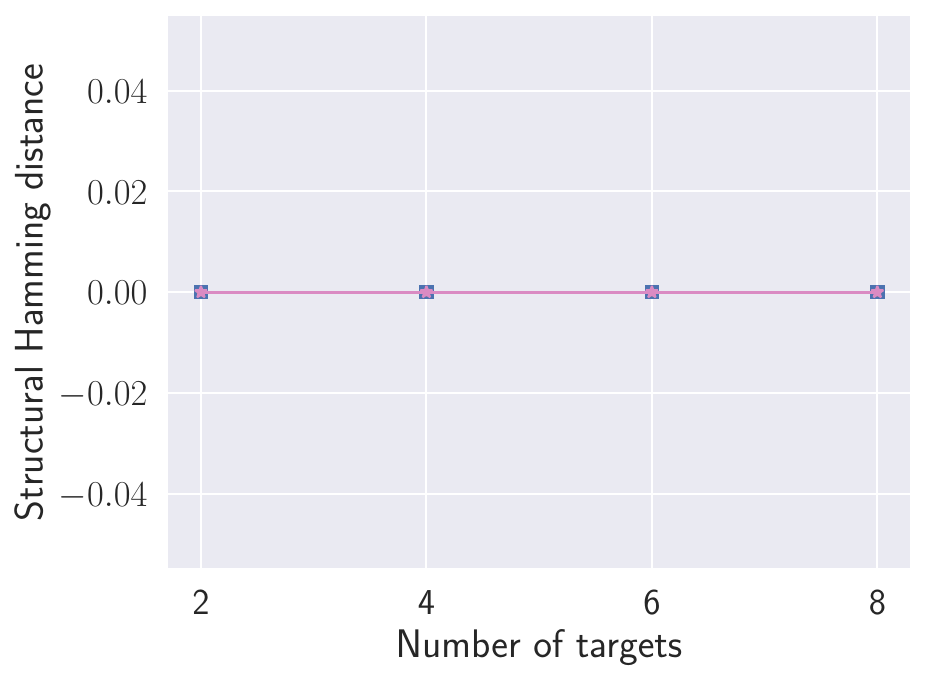}
            \caption*{d-separation tests}
            \label{fig:shd_per_target_unrest_dsep_std}
        \end{subfigure}
        \begin{subfigure}[b]{0.24\linewidth}
            \includegraphics[width=\linewidth]{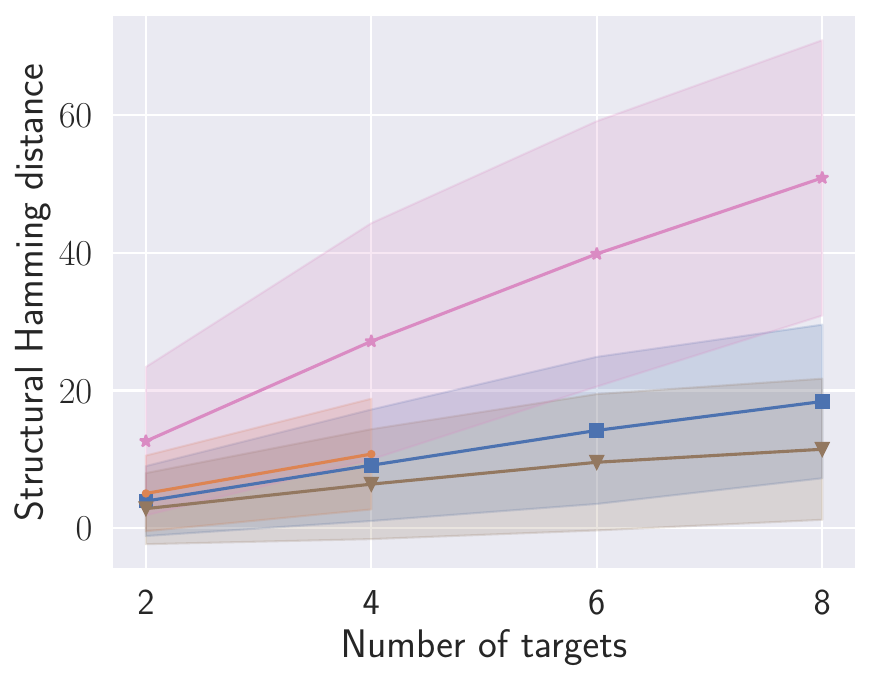}
            \caption*{Fisher-Z tests}
            \label{fig:shd_per_target_unrest_fshz_std}
        \end{subfigure}
        \begin{subfigure}[b]{0.24\linewidth}
            \includegraphics[width=\linewidth]{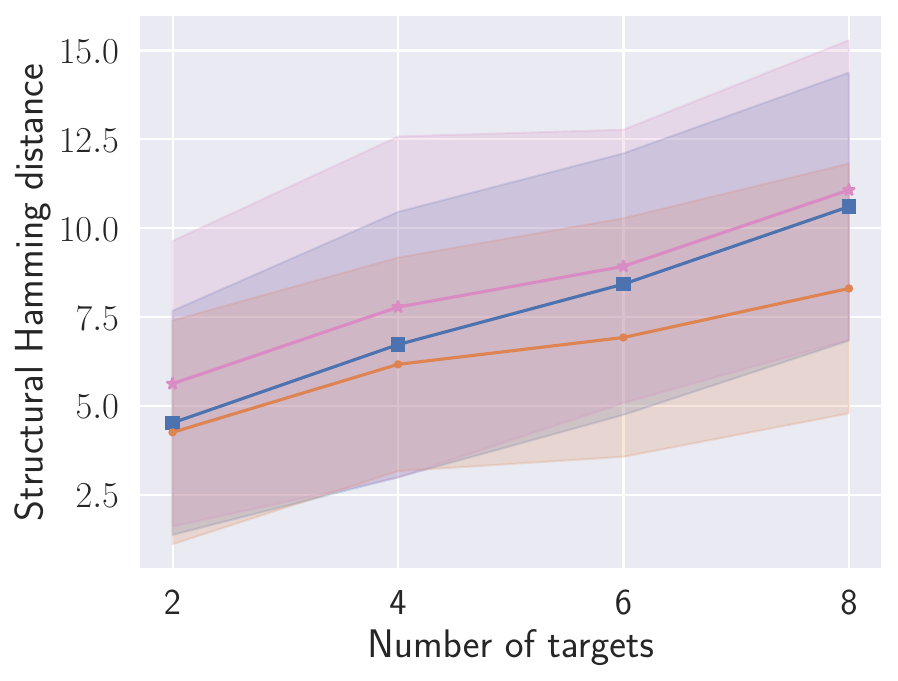}
            \caption*{KCI tests}
            \label{fig:shd_per_target_unrest_kci_std}
        \end{subfigure}
        \begin{subfigure}[b]{0.24\linewidth}
            \includegraphics[width=\linewidth]{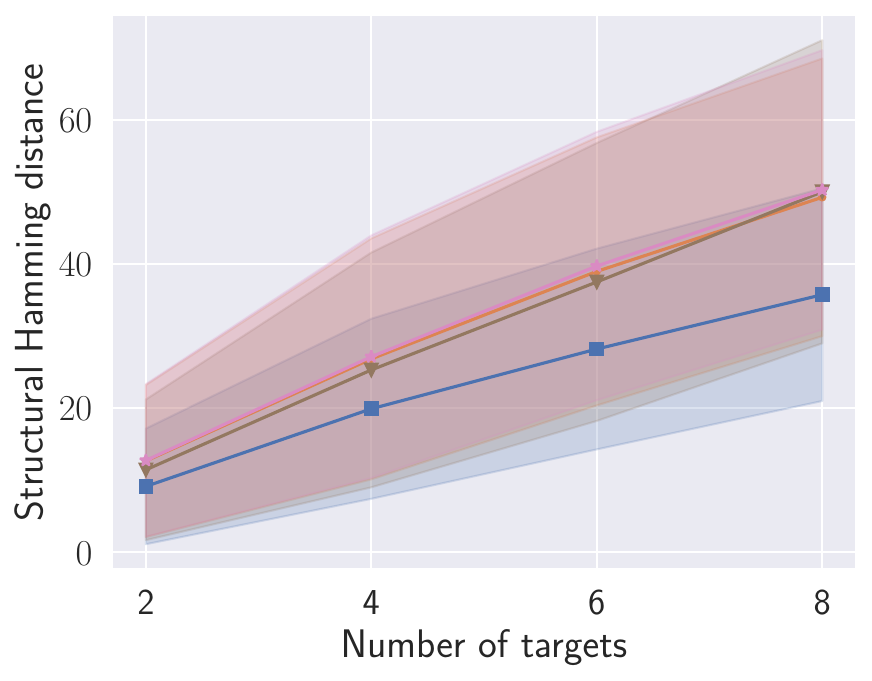}
            \caption*{$\chi^2$ tests}
            \label{fig:shd_per_target_unrest_chsq_std}
        \end{subfigure}
        \caption{\Acl{SHD}}
        \label{fig:shd_per_target_unrest_std}
    \end{subfigure}
    \caption{Estimation quality over number of targets, with $n_{\mathbf{V}}=10$ for KCI tests and $n_{\mathbf{V}}=200$ otherwise, $\overline{d} = 3, d_{\max}=10$ and $n_{\mathbf{D}} = 1000$ data-points. The shadow area denotes the range of the standard deviation.}
    \label{fig:quality_per_target_unrest_std}
\end{figure}

\begin{figure}
    \centering
    \includegraphics[width=.6\linewidth]{experiments/legend_big.pdf}
    \begin{subfigure}[b]{\linewidth}
        \centering
        \begin{subfigure}[b]{0.24\linewidth}
            \includegraphics[width=\linewidth]{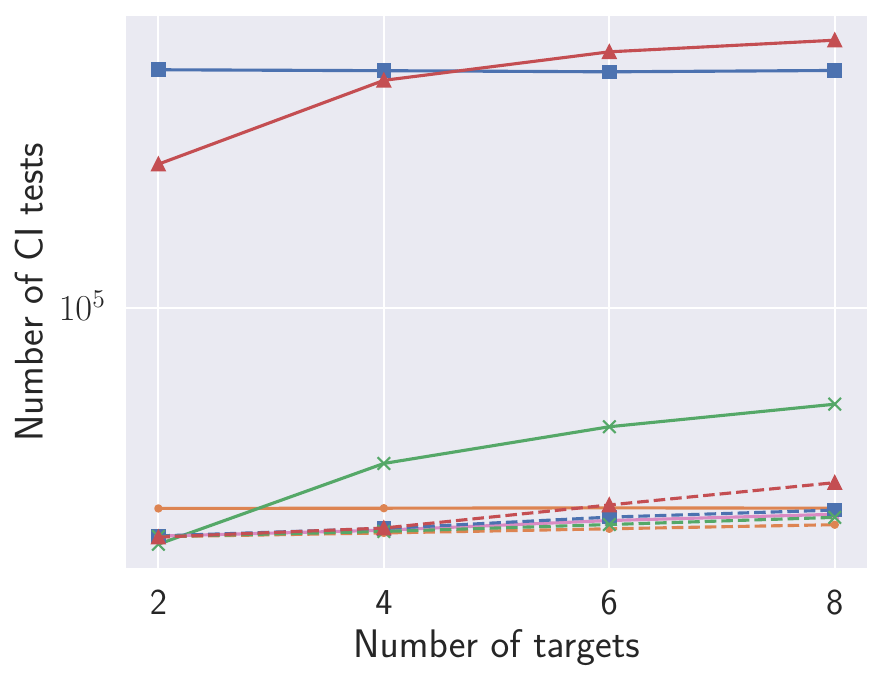}
            \caption*{d-separation tests}
            \label{fig:test_per_target_unrest_dsep}
        \end{subfigure}
        \begin{subfigure}[b]{0.24\linewidth}
            \includegraphics[width=\linewidth]{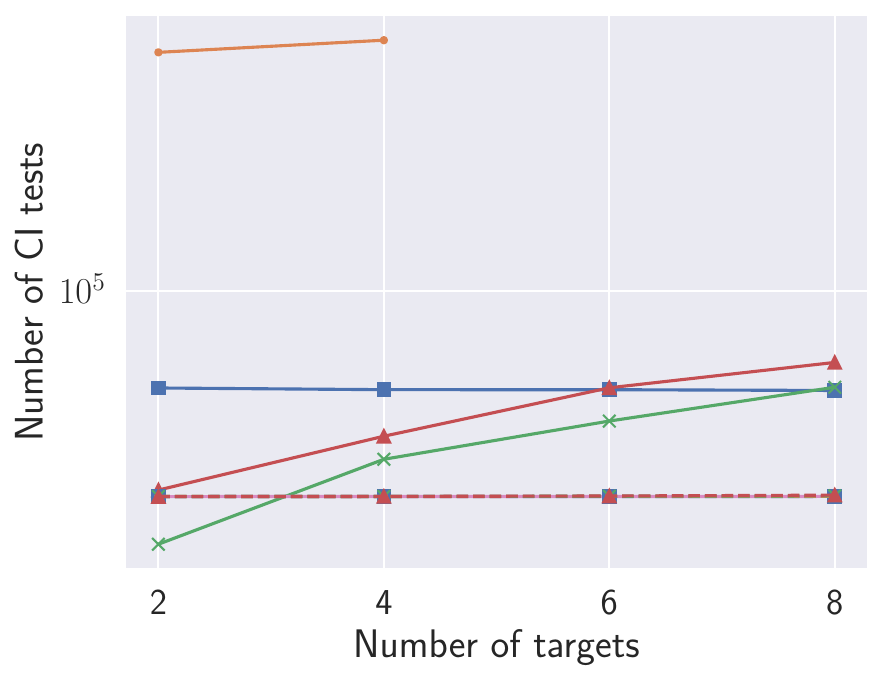}
            \caption*{Fisher-Z tests}
            \label{fig:test_per_target_unrest_fshz}
        \end{subfigure}
        \begin{subfigure}[b]{0.24\linewidth}
            \includegraphics[width=\linewidth]{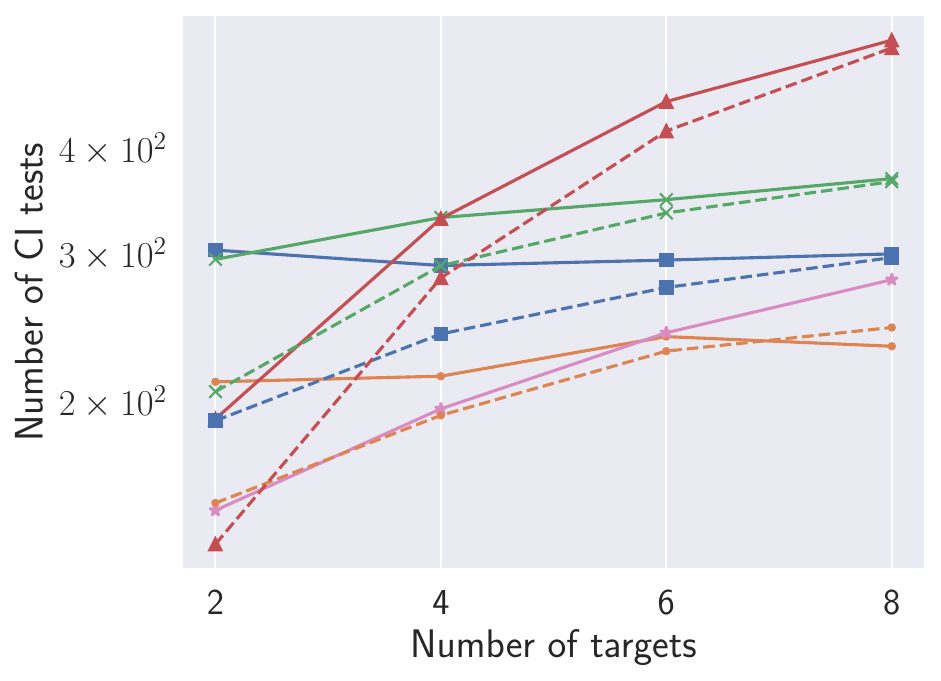}
            \caption*{KCI tests}
            \label{fig:test_per_target_unrest_kci}
        \end{subfigure}
        \begin{subfigure}[b]{0.24\linewidth}
            \includegraphics[width=\linewidth]{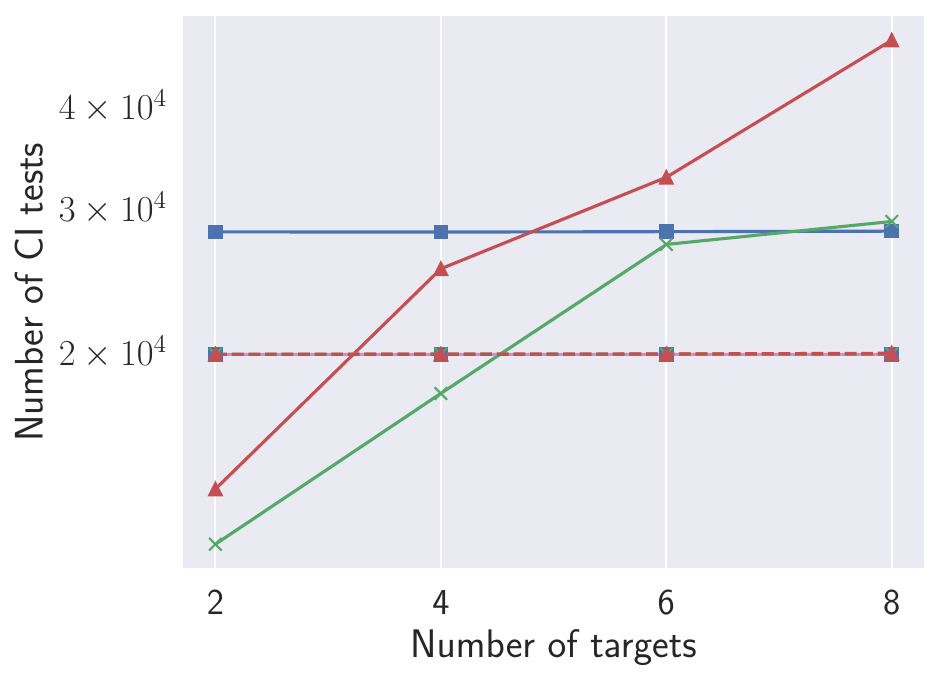}
            \caption*{$\chi^2$ tests}
            \label{fig:test_per_target_unrest_chsq}
        \end{subfigure}
        \caption{Number of \ac{CI} tests.}
        \label{fig:test_per_target_unrest_no}
    \end{subfigure}
    \begin{subfigure}[b]{\linewidth}
        \centering
        \begin{subfigure}[b]{0.24\linewidth}
            \includegraphics[width=\linewidth]{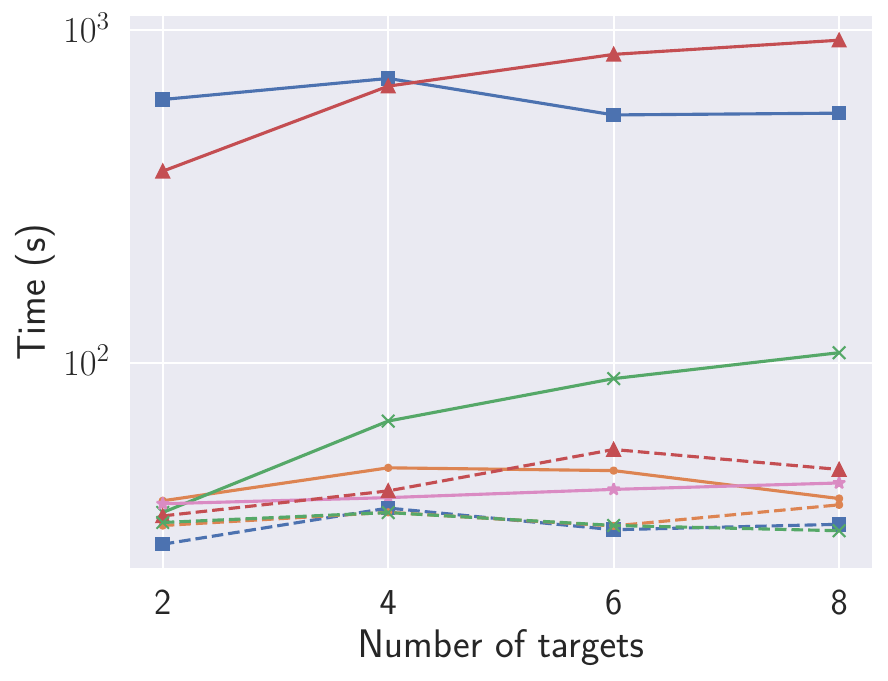}
            \caption*{d-separation tests}
            \label{fig:time_per_target_unrest_dsep}
        \end{subfigure}
        \begin{subfigure}[b]{0.24\linewidth}
            \includegraphics[width=\linewidth]{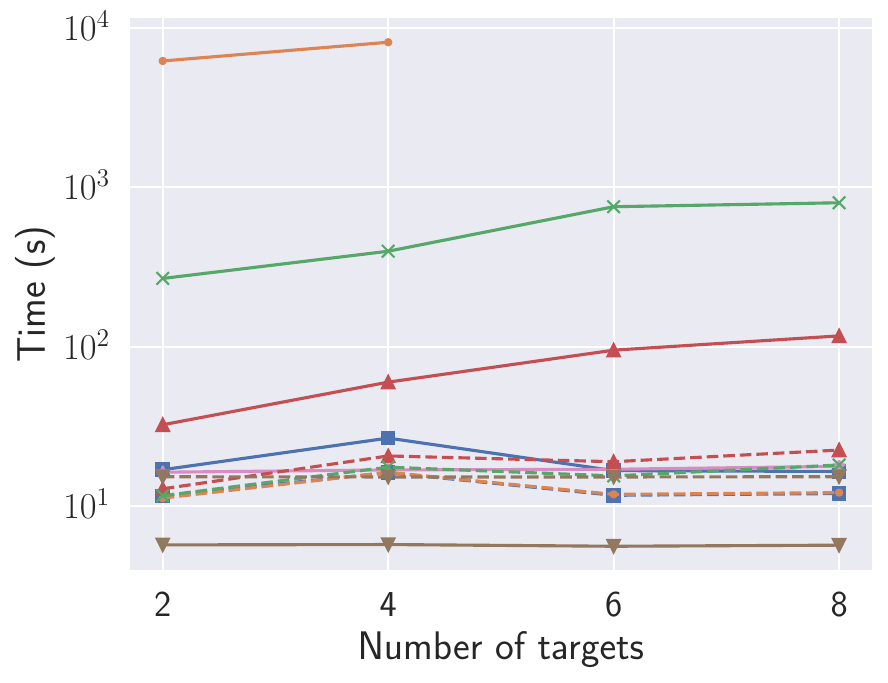}
            \caption*{Fisher-Z tests}
            \label{fig:time_per_target_unrest_fshz}
        \end{subfigure}
        \begin{subfigure}[b]{0.24\linewidth}
            \includegraphics[width=\linewidth]{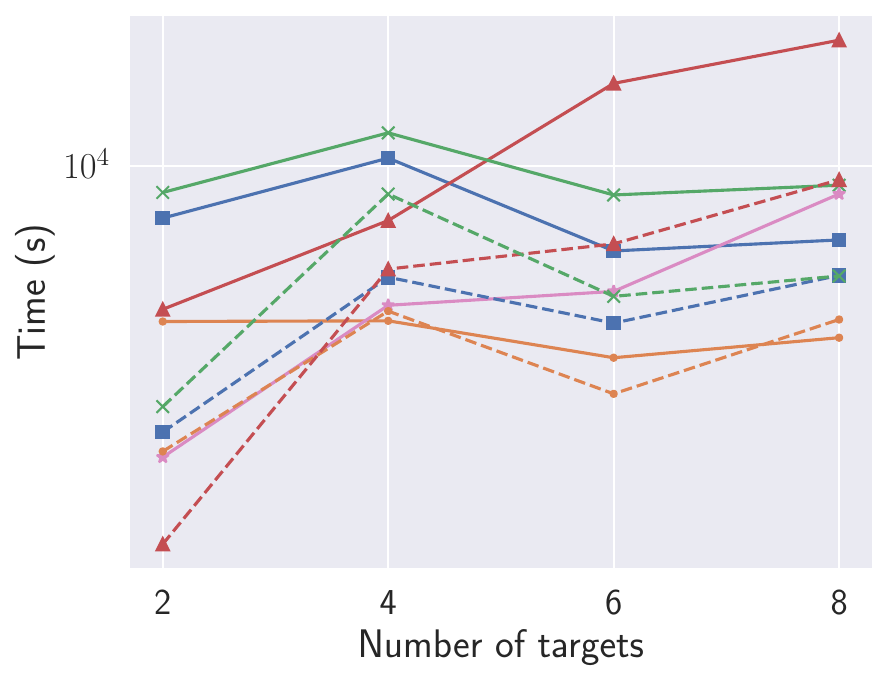}
            \caption*{KCI tests}
            \label{fig:time_per_target_unrest_kci}
        \end{subfigure}
        \begin{subfigure}[b]{0.24\linewidth}
            \includegraphics[width=\linewidth]{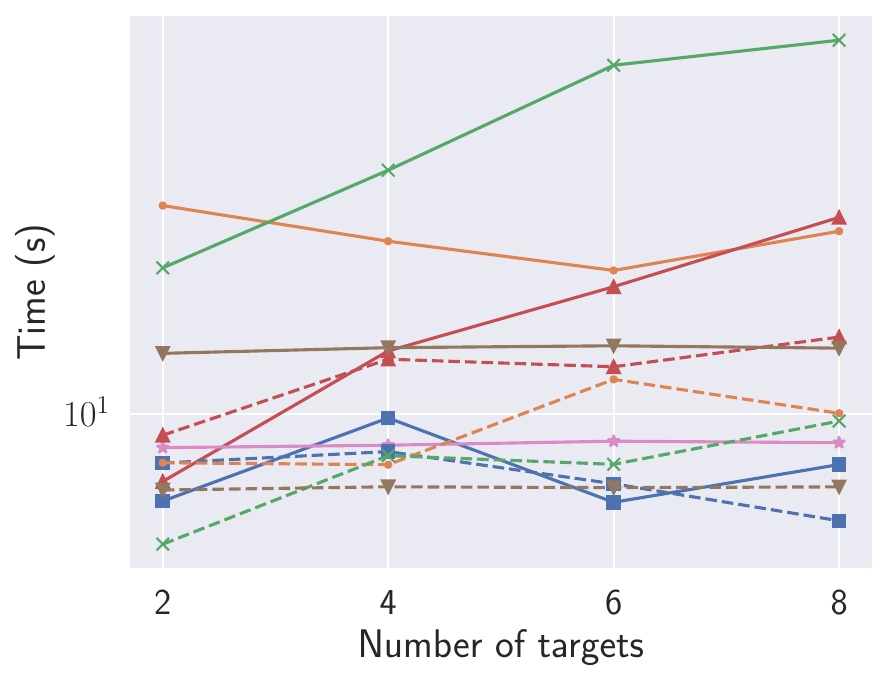}
            \caption*{$\chi^2$ tests}
            \label{fig:time_per_target_unrest_chsq}
        \end{subfigure}
        \caption{Computation time.}
        \label{fig:time_per_target_unrest}
    \end{subfigure}
    \caption{Number of \ac{CI} tests and computation time over number of targets, with $n_{\mathbf{V}}=10$ for KCI tests and $n_{\mathbf{V}}=200$ otherwise, $\overline{d} = 3, d_{\max}=10$ and $n_{\mathbf{D}} = 1000$ data-points. We also show baseline methods combined with SNAP$(0)$. We plot values on a $\log$ scale.}
    \label{fig:computation_per_target_unrest}
\end{figure}

\begin{figure}
    \centering
    \includegraphics[width=.6\linewidth]{experiments/legend_big.pdf}
    \begin{subfigure}[b]{\linewidth}
        \centering
        \begin{subfigure}[b]{0.24\linewidth}
            \includegraphics[width=\linewidth]{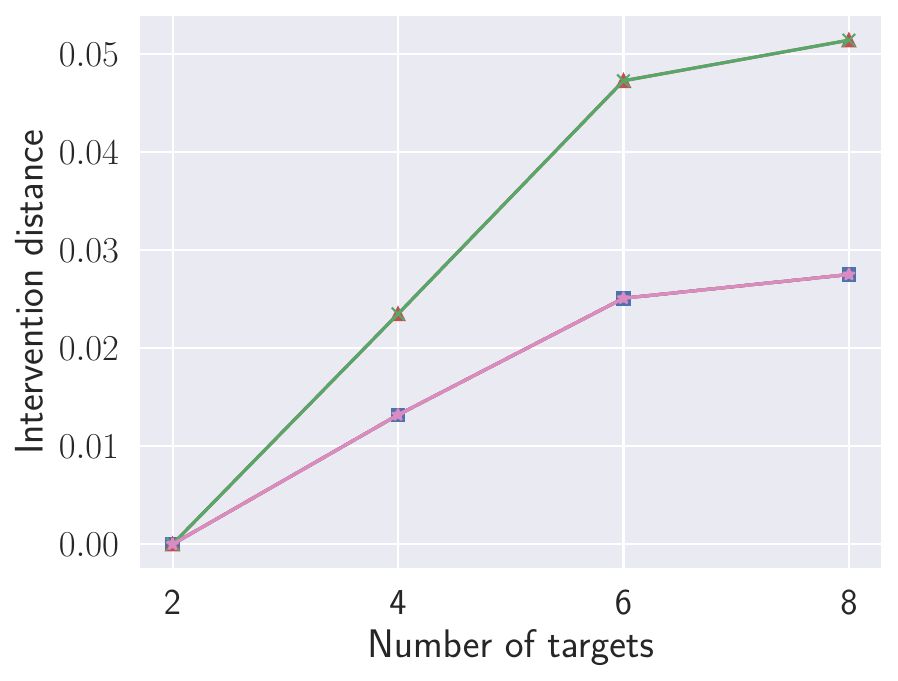}
            \caption*{d-separation tests}
            \label{fig:int_dist_per_target_unrest_dsep_abs}
        \end{subfigure}
        \begin{subfigure}[b]{0.24\linewidth}
            \includegraphics[width=\linewidth]{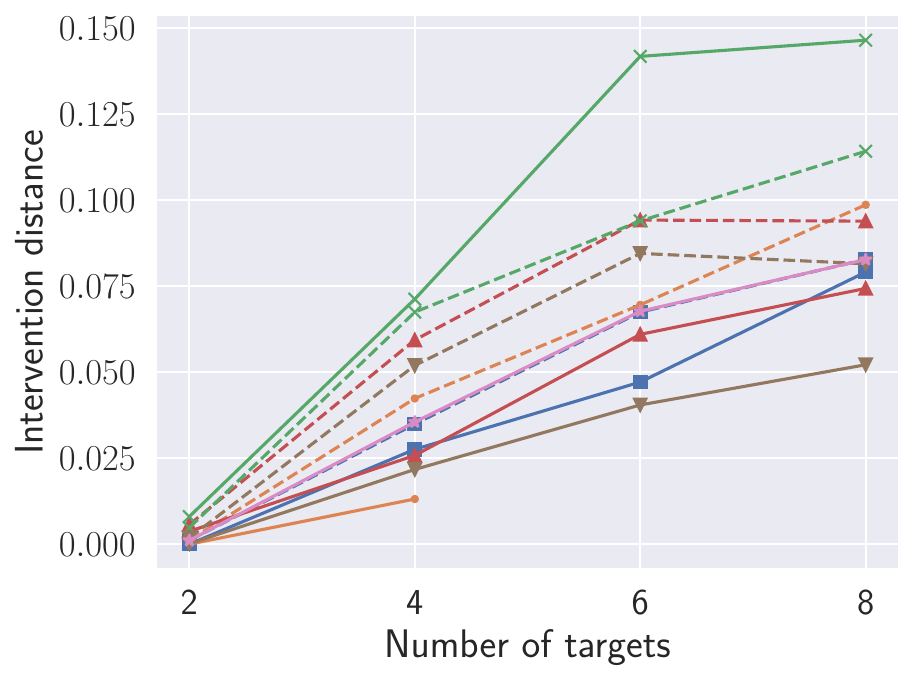}
            \caption*{Fisher-Z tests}
            \label{fig:int_dist_per_target_unrest_fshz_abs}
        \end{subfigure}
        \begin{subfigure}[b]{0.24\linewidth}
            \includegraphics[width=\linewidth]{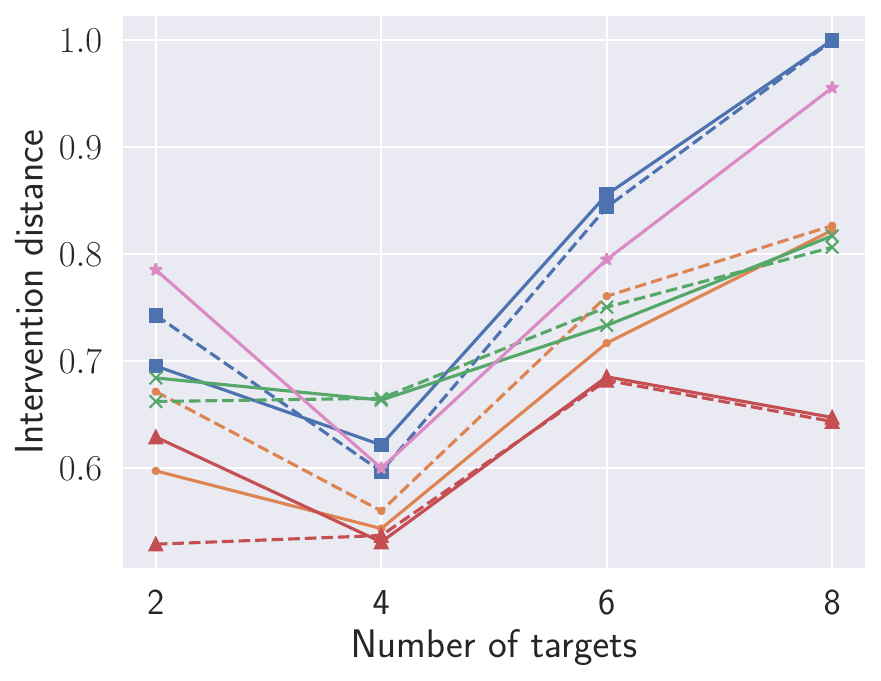}
            \caption*{KCI tests}
            \label{fig:int_dist_per_target_unrest_kci_abs}
        \end{subfigure}
        \begin{subfigure}[b]{0.24\linewidth}
            \includegraphics[width=\linewidth]{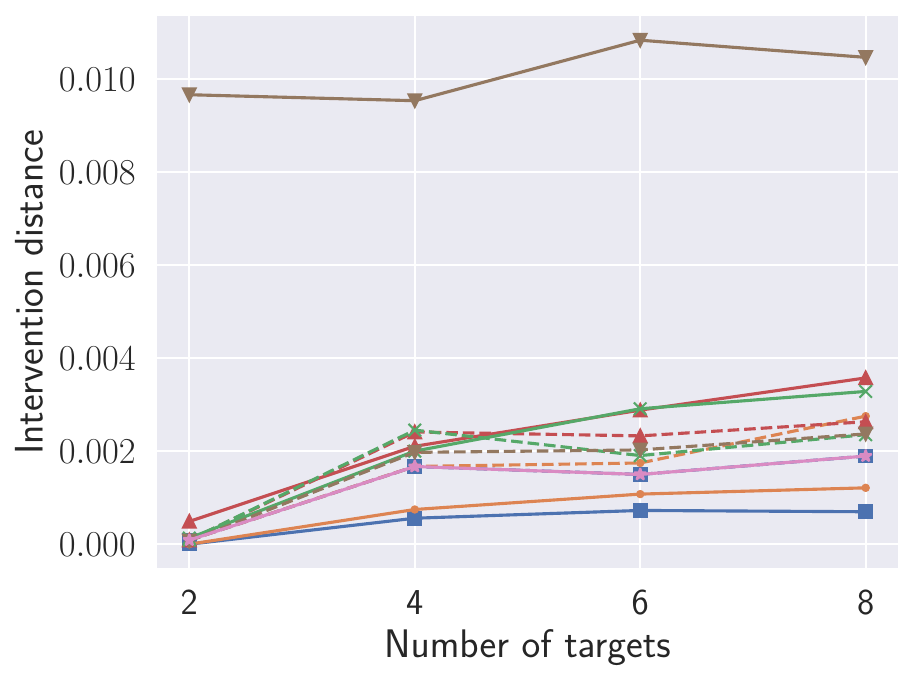}
            \caption*{$\chi^2$ tests}
            \label{fig:int_dist_per_target_unrest_chsq_abs}
        \end{subfigure}
        \caption{Intervention distance}
        \label{fig:int_dist_per_target_unrest_abs}
    \end{subfigure}
    \begin{subfigure}[b]{\linewidth}
        \centering
        \begin{subfigure}[b]{0.24\linewidth}
            \includegraphics[width=\linewidth]{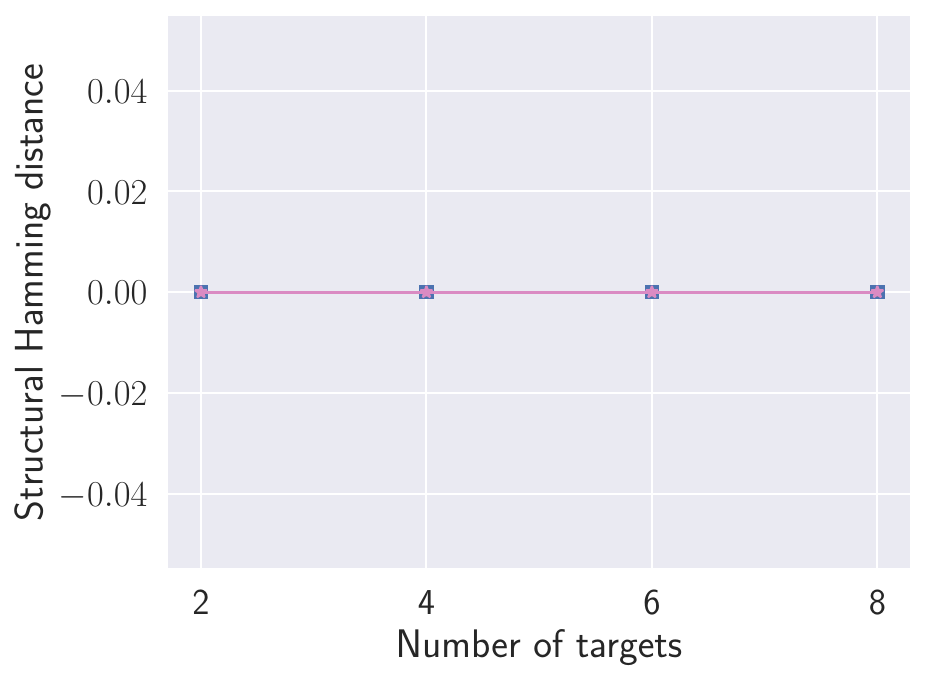}
            \caption*{d-separation tests}
            \label{fig:shd_per_target_unrest_dsep}
        \end{subfigure}
        \begin{subfigure}[b]{0.24\linewidth}
            \includegraphics[width=\linewidth]{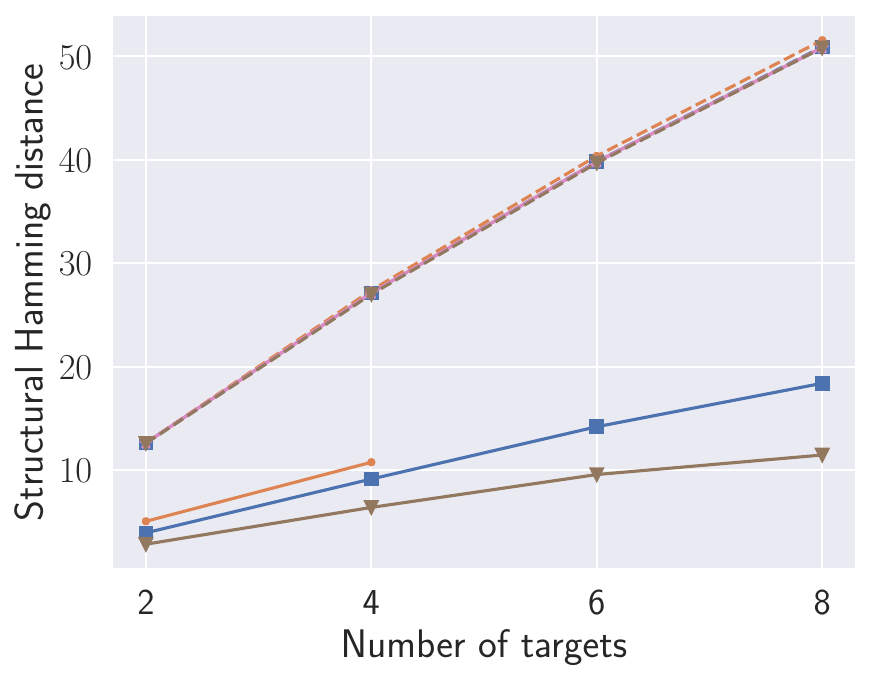}
            \caption*{Fisher-Z tests}
            \label{fig:shd_per_target_unrest_fshz}
        \end{subfigure}
        \begin{subfigure}[b]{0.24\linewidth}
            \includegraphics[width=\linewidth]{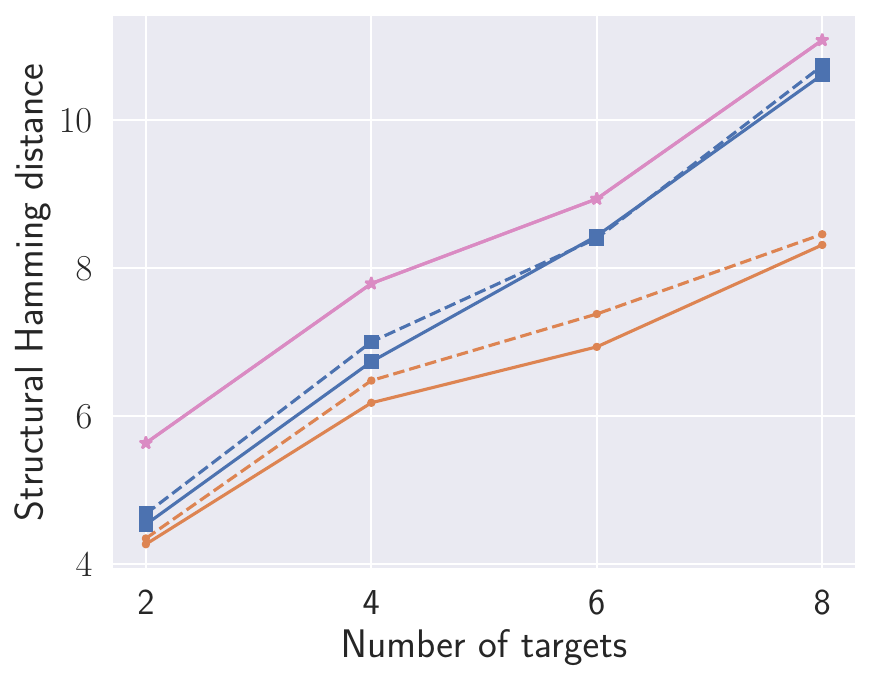}
            \caption*{KCI tests}
            \label{fig:shd_per_target_unrest_kci}
        \end{subfigure}
        \begin{subfigure}[b]{0.24\linewidth}
            \includegraphics[width=\linewidth]{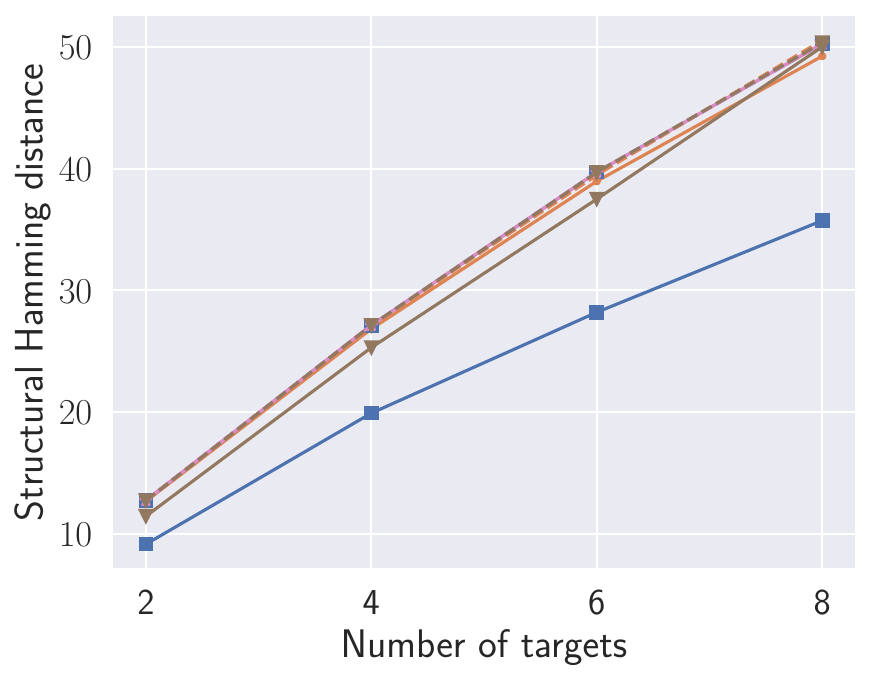}
            \caption*{$\chi^2$ tests}
            \label{fig:shd_per_target_unrest_chsq}
        \end{subfigure}
        \caption{\Acl{SHD}}
        \label{fig:shd_per_target_unrest}
    \end{subfigure}
    \caption{Estimation quality over number of targets, with $n_{\mathbf{V}}=10$ for KCI tests and $n_{\mathbf{V}}=200$ otherwise, $\overline{d} = 3, d_{\max}=10$ and $n_{\mathbf{D}} = 1000$ data-points. We also show baseline methods combined with SNAP$(0)$.}
    \label{fig:quality_per_target_unrest}
\end{figure}


\begin{figure}
    \centering
    \includegraphics[width=.6\linewidth]{experiments/legend_big.pdf}
    \begin{subfigure}[b]{\linewidth}
        \centering
        \begin{subfigure}[b]{0.24\linewidth}
            \includegraphics[width=\linewidth]{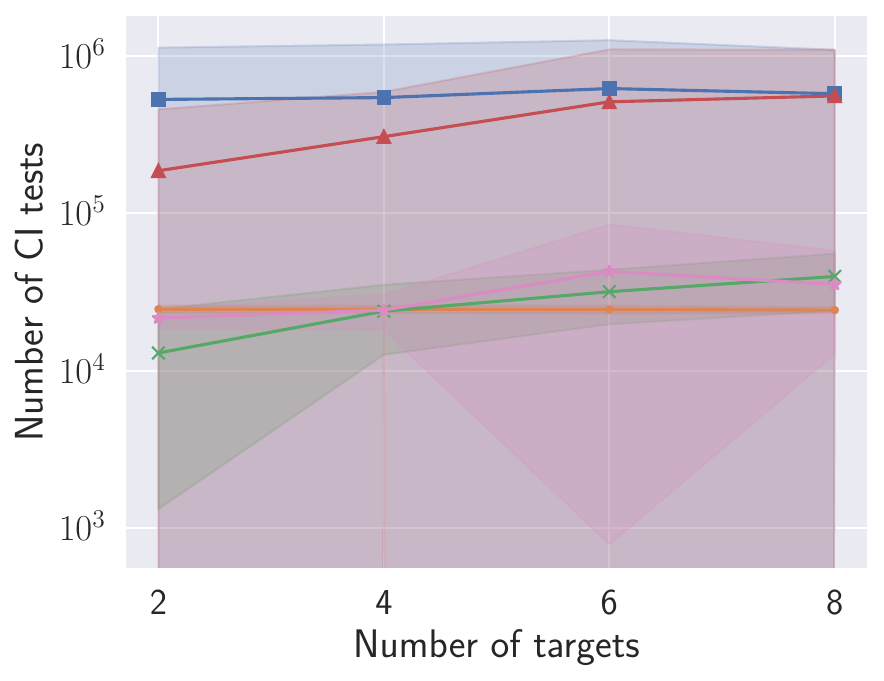}
            \caption*{d-separation tests}
            \label{fig:test_per_target_ident_dsep_std}
        \end{subfigure}
        \begin{subfigure}[b]{0.24\linewidth}
            \includegraphics[width=\linewidth]{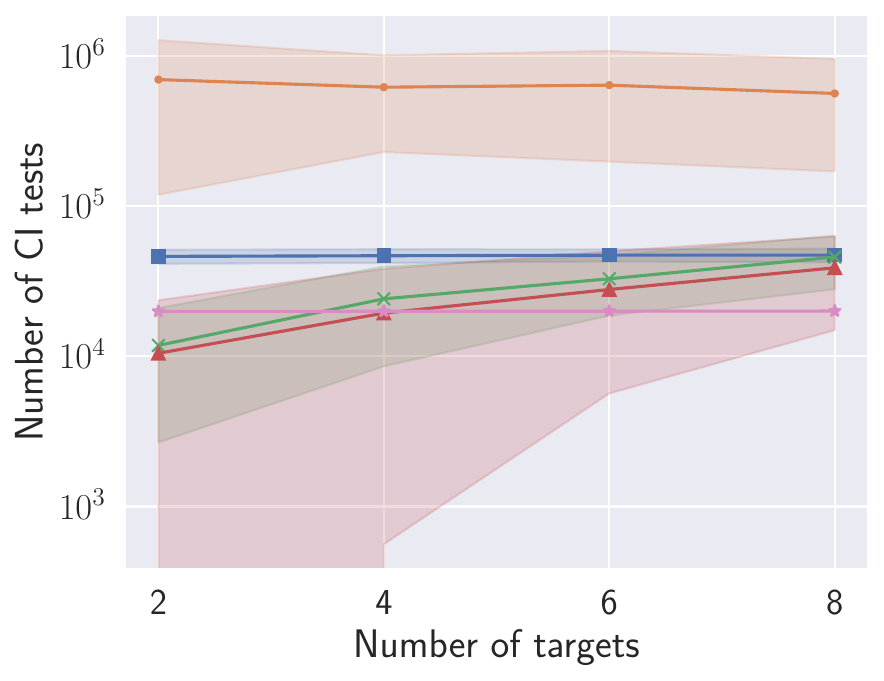}
            \caption*{Fisher-Z tests}
            \label{fig:test_per_target_ident_fshz_std}
        \end{subfigure}
        \begin{subfigure}[b]{0.24\linewidth}
            \includegraphics[width=\linewidth]{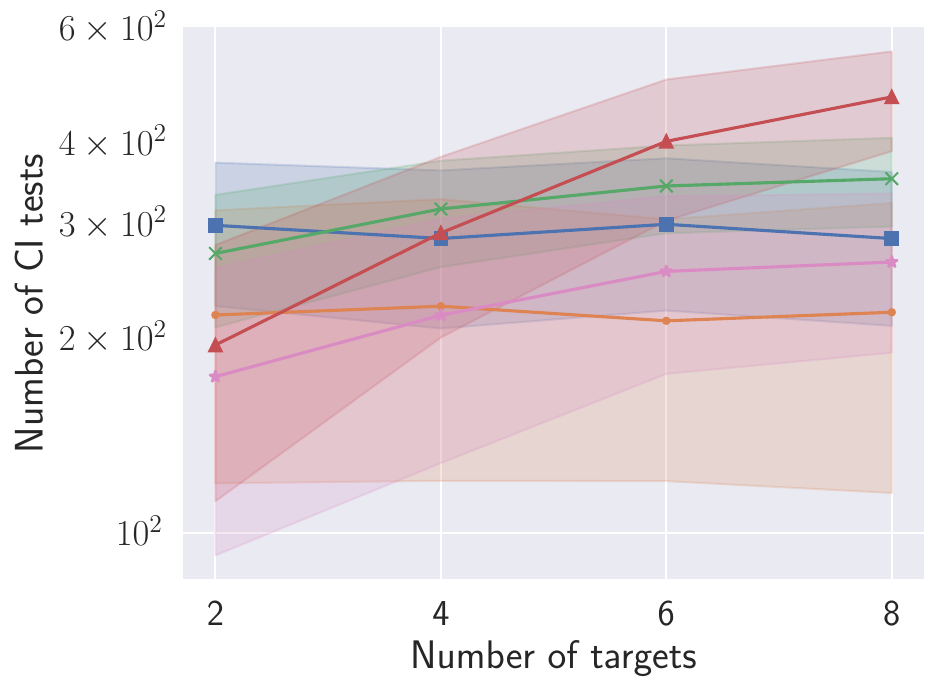}
            \caption*{KCI tests}
            \label{fig:test_per_target_ident_kci_std}
        \end{subfigure}
        \begin{subfigure}[b]{0.24\linewidth}
            \includegraphics[width=\linewidth]{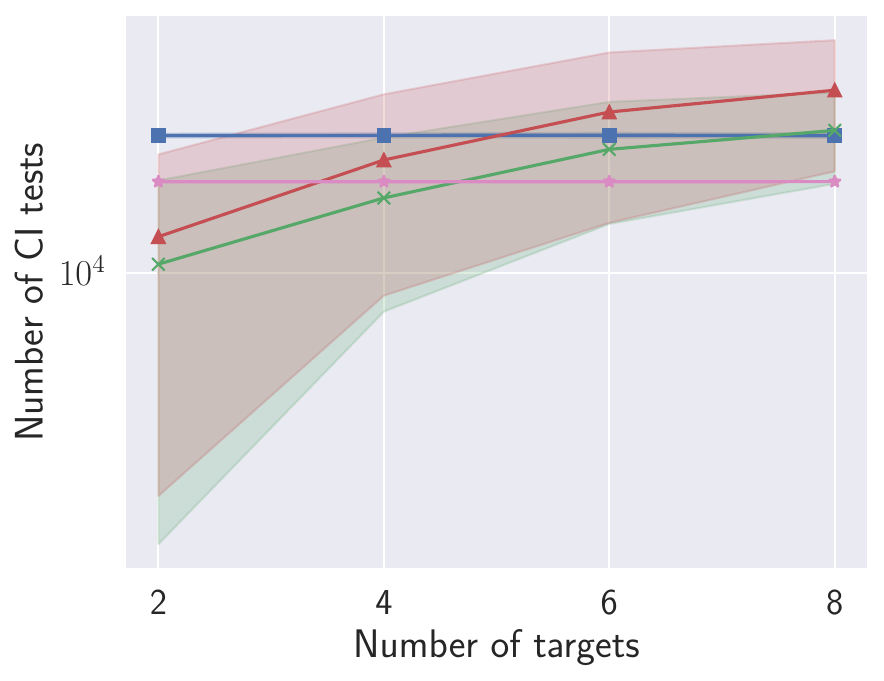}
            \caption*{$\chi^2$ tests}
            \label{fig:test_per_target_ident_chsq_std}
        \end{subfigure}
        \caption{Number of \ac{CI} tests.}
        \label{fig:test_per_target_ident_std}
    \end{subfigure}
    \begin{subfigure}[b]{\linewidth}
        \centering
        \begin{subfigure}[b]{0.24\linewidth}
            \includegraphics[width=\linewidth]{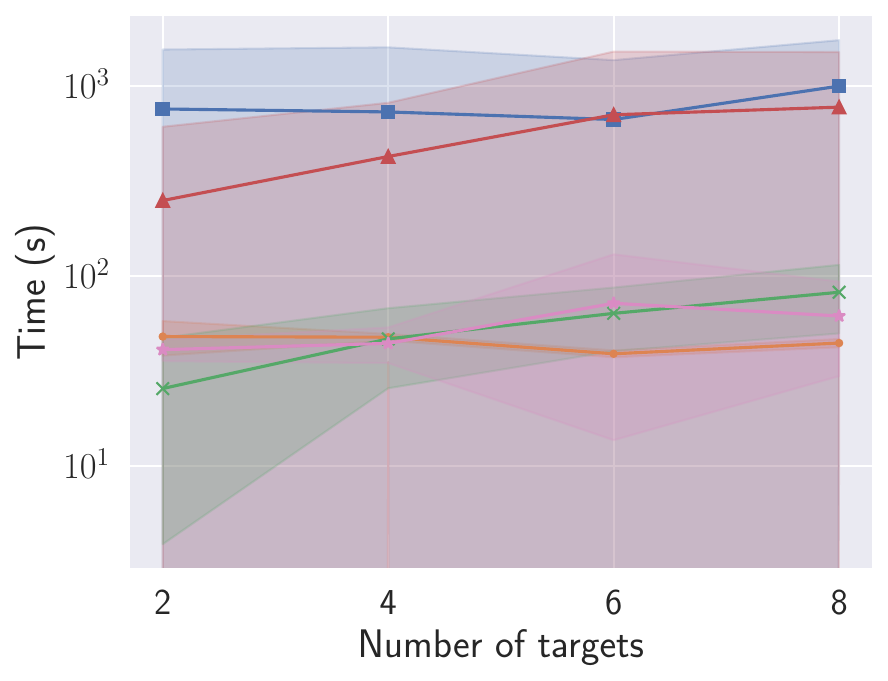}
            \caption*{d-separation tests}
            \label{fig:time_per_target_ident_dsep_std}
        \end{subfigure}
        \begin{subfigure}[b]{0.24\linewidth}
            \includegraphics[width=\linewidth]{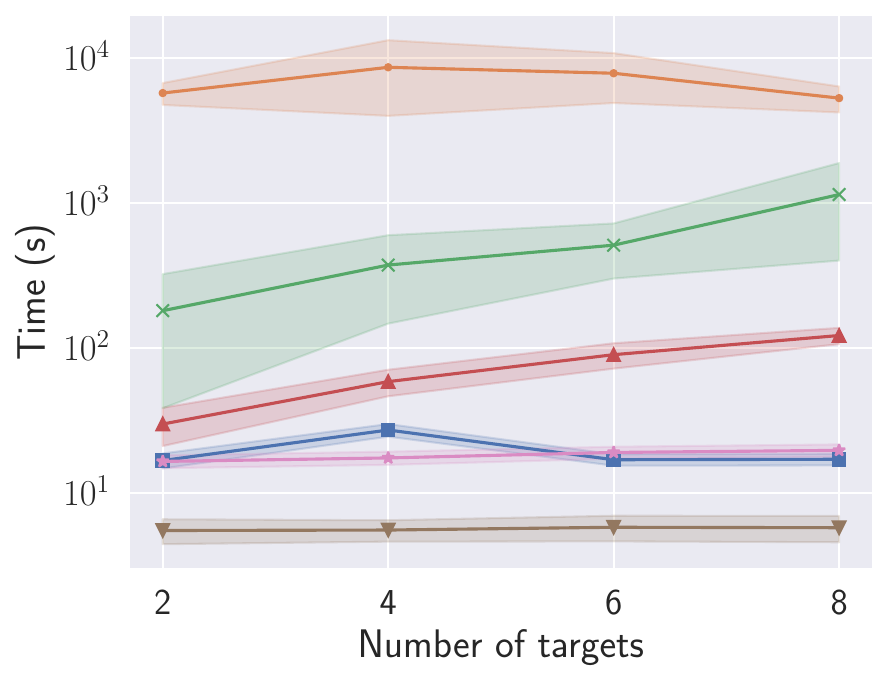}
            \caption*{Fisher-Z tests}
            \label{fig:time_per_target_ident_fshz_std}
        \end{subfigure}
        \begin{subfigure}[b]{0.24\linewidth}
            \includegraphics[width=\linewidth]{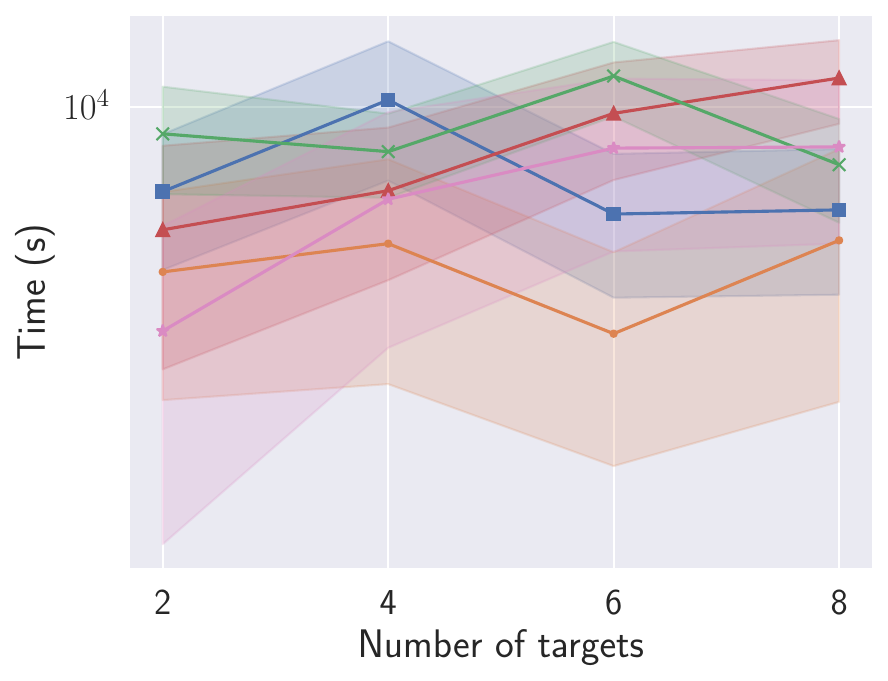}
            \caption*{KCI tests}
            \label{fig:time_per_target_ident_kci_std}
        \end{subfigure}
        \begin{subfigure}[b]{0.24\linewidth}
            \includegraphics[width=\linewidth]{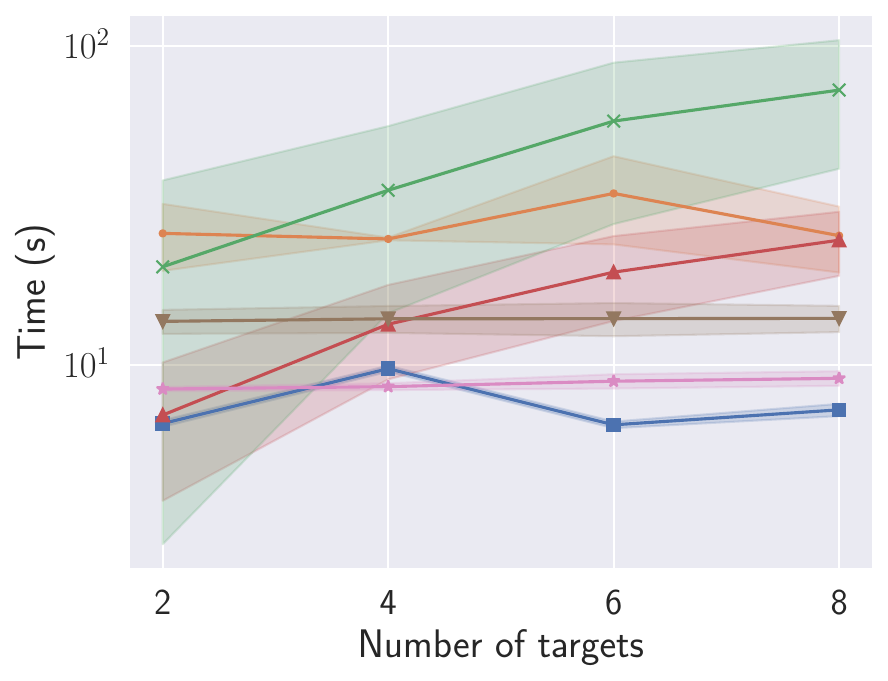}
            \caption*{$\chi^2$ tests}
            \label{fig:time_per_target_ident_chsq_std}
        \end{subfigure}
        \caption{Computation time.}
        \label{fig:time_per_target_ident_std}
    \end{subfigure}
    \caption{Number of \ac{CI} tests and computation time over number of identifiable targets, with $n_{\mathbf{V}}=10$ for KCI tests and $n_{\mathbf{V}}=200$ otherwise, $\overline{d} = 3, d_{\max}=10$ and $n_{\mathbf{D}} = 1000$ data-points. We plot values on a $\log$ scale. The shadow area denotes the range of the standard deviation.}
    \label{fig:computation_per_target_ident_std}
\end{figure}

\begin{figure}
    \centering
    \includegraphics[width=.6\linewidth]{experiments/legend_big.pdf}
    \begin{subfigure}[b]{\linewidth}
        \centering
        \begin{subfigure}[b]{0.24\linewidth}
            \includegraphics[width=\linewidth]{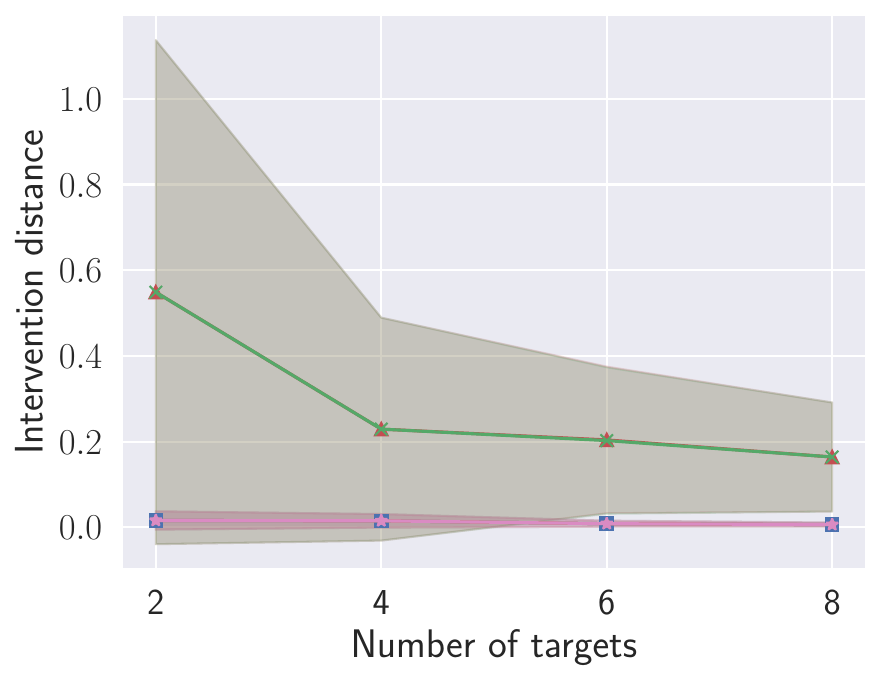}
            \caption*{d-separation tests}
            \label{fig:int_dist_per_target_ident_dsep_abs_std}
        \end{subfigure}
        \begin{subfigure}[b]{0.24\linewidth}
            \includegraphics[width=\linewidth]{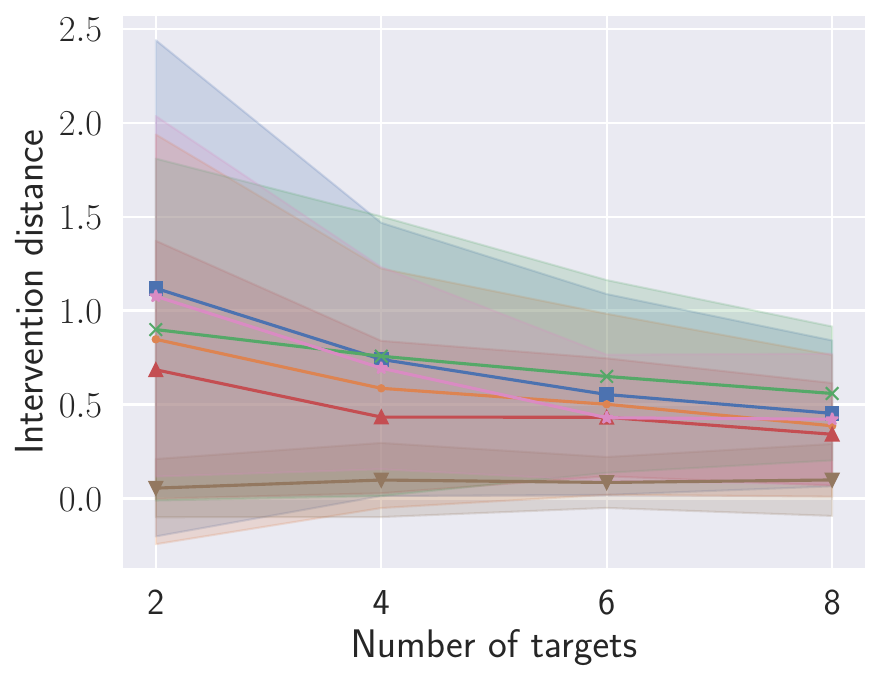}
            \caption*{Fisher-Z tests}
            \label{fig:int_dist_per_target_ident_fshz_abs_std}
        \end{subfigure}
        \begin{subfigure}[b]{0.24\linewidth}
            \includegraphics[width=\linewidth]{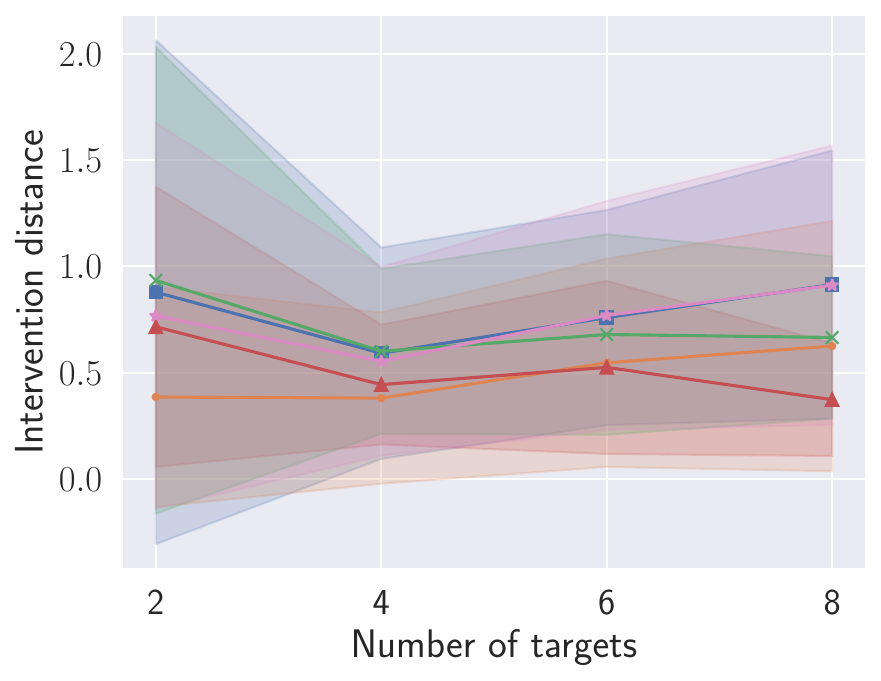}
            \caption*{KCI tests}
            \label{fig:int_dist_per_target_ident_kci_abs_std}
        \end{subfigure}
        \begin{subfigure}[b]{0.24\linewidth}
            \includegraphics[width=\linewidth]{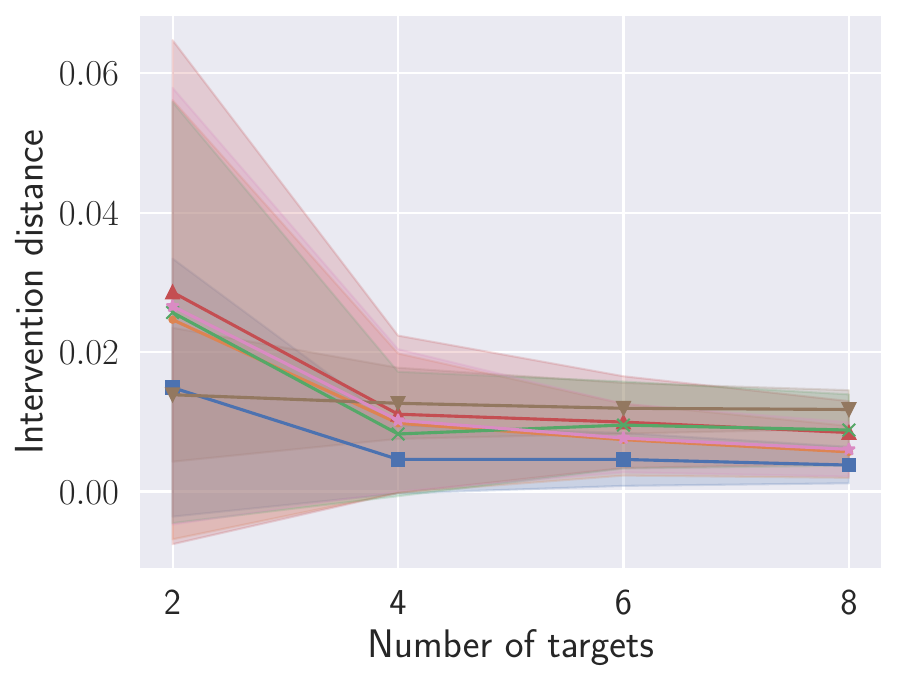}
            \caption*{$\chi^2$ tests}
            \label{fig:int_dist_per_target_ident_chsq_abs_std}
        \end{subfigure}
        \caption{Intervention distance}
        \label{fig:int_dist_per_target_ident_abs_std}
    \end{subfigure}
    \begin{subfigure}[b]{\linewidth}
        \centering
        \begin{subfigure}[b]{0.24\linewidth}
            \includegraphics[width=\linewidth]{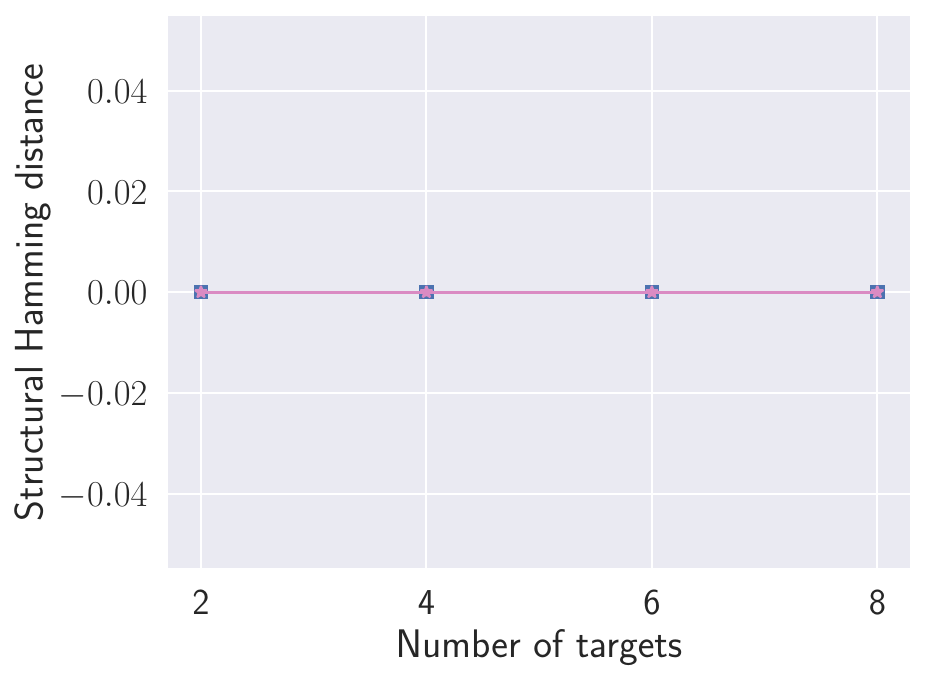}
            \caption*{d-separation tests}
            \label{fig:shd_per_target_ident_dsep_std}
        \end{subfigure}
        \begin{subfigure}[b]{0.24\linewidth}
            \includegraphics[width=\linewidth]{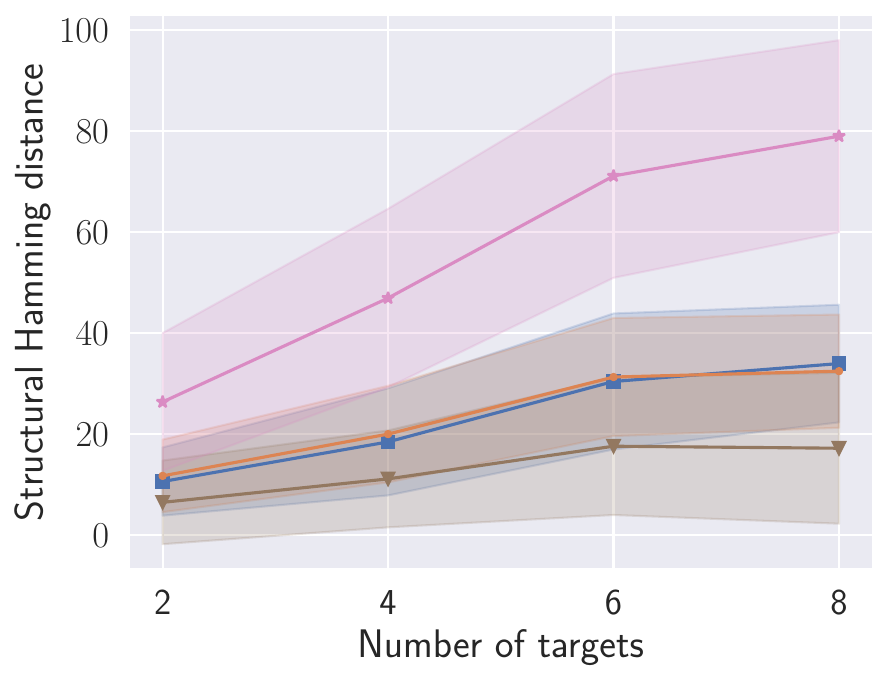}
            \caption*{Fisher-Z tests}
            \label{fig:shd_per_target_ident_fshz_std}
        \end{subfigure}
        \begin{subfigure}[b]{0.24\linewidth}
            \includegraphics[width=\linewidth]{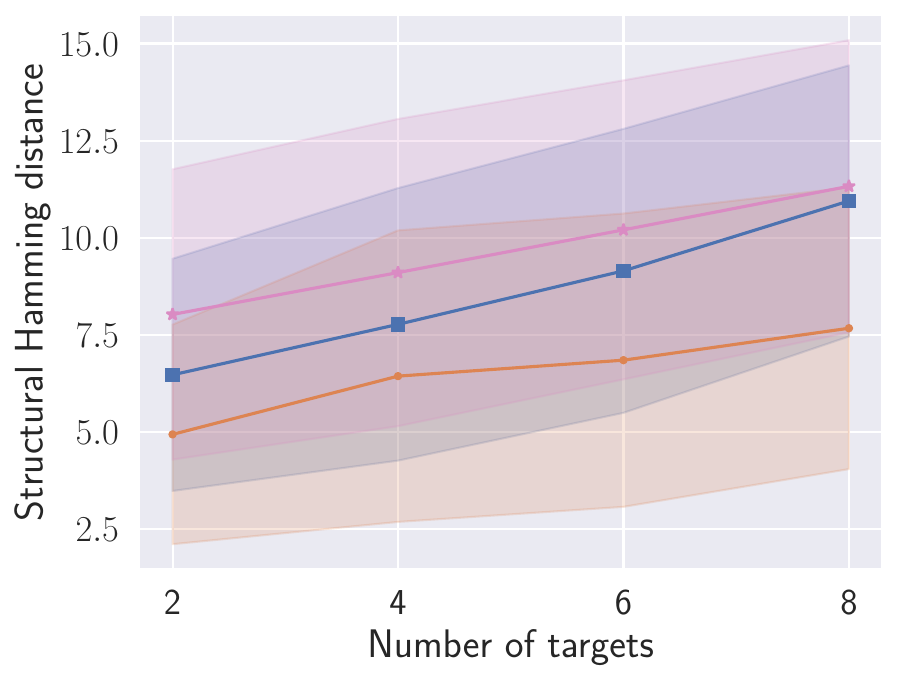}
            \caption*{KCI tests}
            \label{fig:shd_per_target_ident_kci_std}
        \end{subfigure}
        \begin{subfigure}[b]{0.24\linewidth}
            \includegraphics[width=\linewidth]{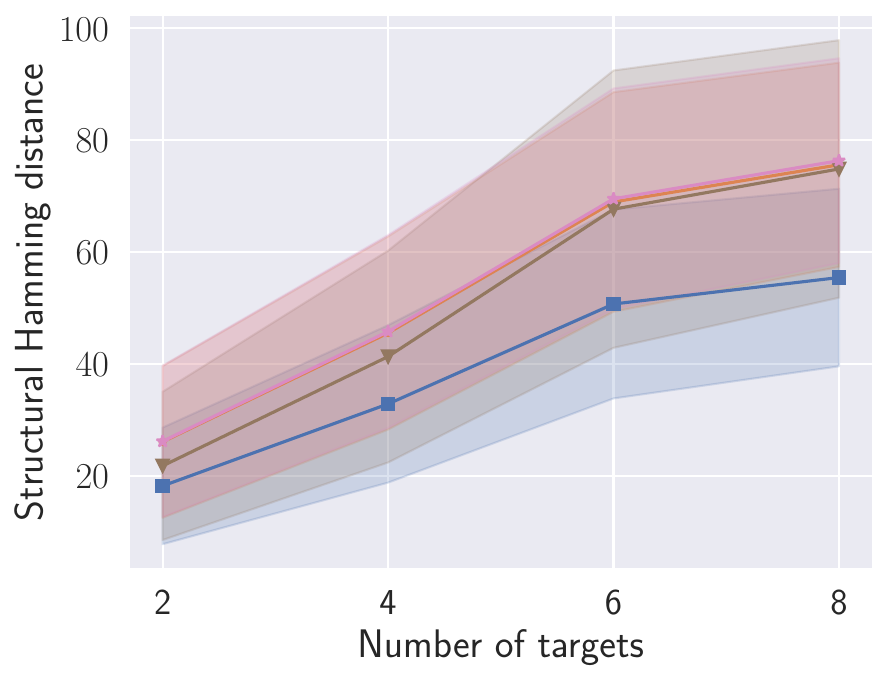}
            \caption*{$\chi^2$ tests}
            \label{fig:shd_per_target_ident_chsq_std}
        \end{subfigure}
        \caption{\Acl{SHD}}
        \label{fig:shd_per_target_ident_std}
    \end{subfigure}
    \caption{Estimation quality over number of identifiable targets, with $n_{\mathbf{V}}=10$ for KCI tests and $n_{\mathbf{V}}=200$ otherwise, $\overline{d} = 3, d_{\max}=10$ and $n_{\mathbf{D}} = 1000$ data-points. The shadow area denotes the range of the standard deviation.}
    \label{fig:quality_per_target_ident_std}
\end{figure}

\begin{figure}
    \centering
    \includegraphics[width=.6\linewidth]{experiments/legend_big.pdf}
    \begin{subfigure}[b]{\linewidth}
        \centering
        \begin{subfigure}[b]{0.24\linewidth}
            \includegraphics[width=\linewidth]{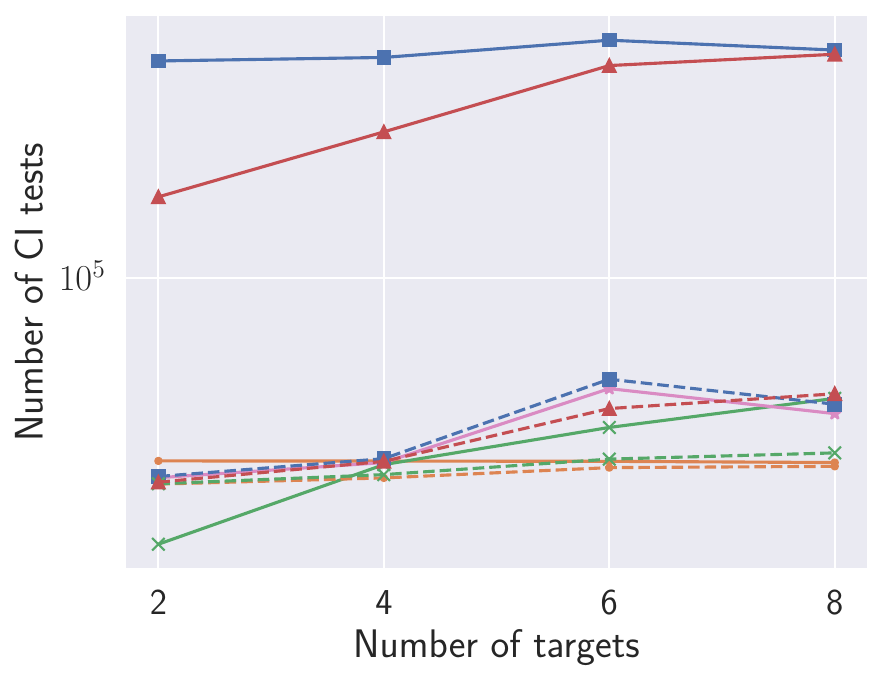}
            \caption*{d-separation tests}
            \label{fig:test_per_target_ident_dsep}
        \end{subfigure}
        \begin{subfigure}[b]{0.24\linewidth}
            \includegraphics[width=\linewidth]{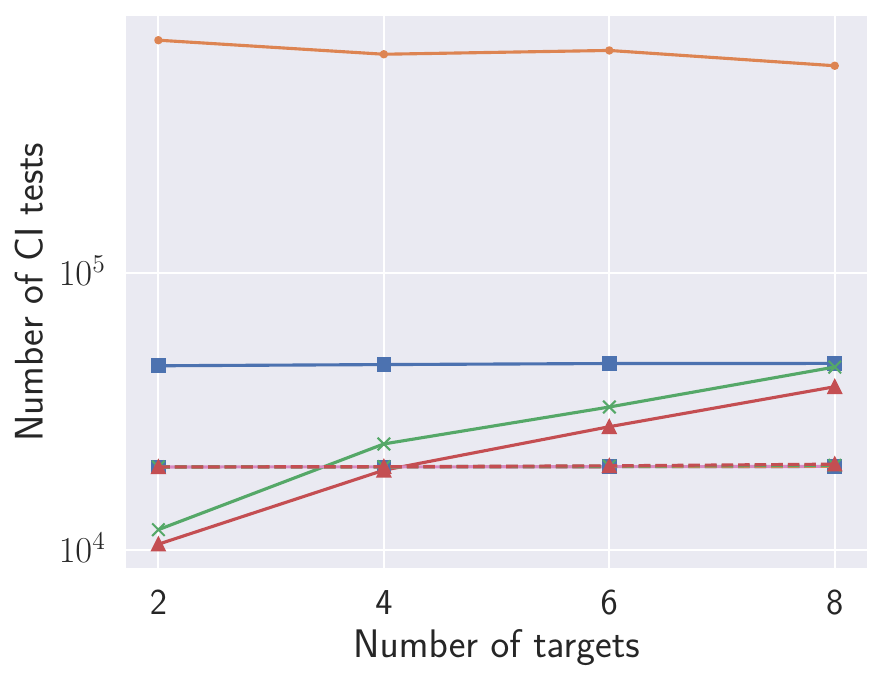}
            \caption*{Fisher-Z tests}
            \label{fig:test_per_target_ident_fshz}
        \end{subfigure}
        \begin{subfigure}[b]{0.24\linewidth}
            \includegraphics[width=\linewidth]{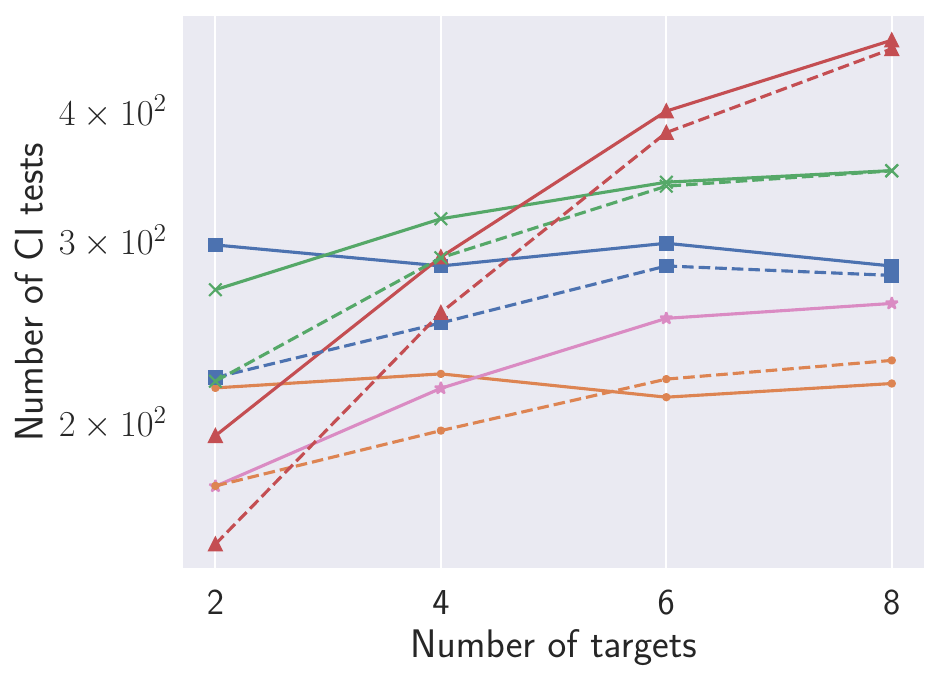}
            \caption*{KCI tests}
            \label{fig:test_per_target_ident_kci}
        \end{subfigure}
        \begin{subfigure}[b]{0.24\linewidth}
            \includegraphics[width=\linewidth]{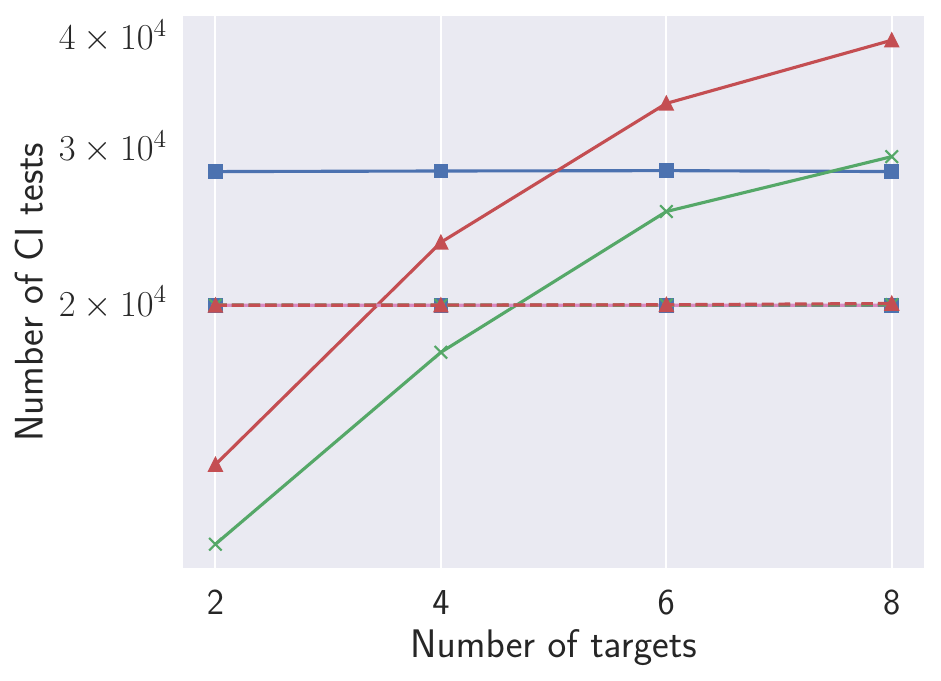}
            \caption*{$\chi^2$ tests}
            \label{fig:test_per_target_ident_chsq}
        \end{subfigure}
        \caption{Number of \ac{CI} tests.}
        \label{fig:test_per_target_ident_no}
    \end{subfigure}
    \begin{subfigure}[b]{\linewidth}
        \centering
        \begin{subfigure}[b]{0.24\linewidth}
            \includegraphics[width=\linewidth]{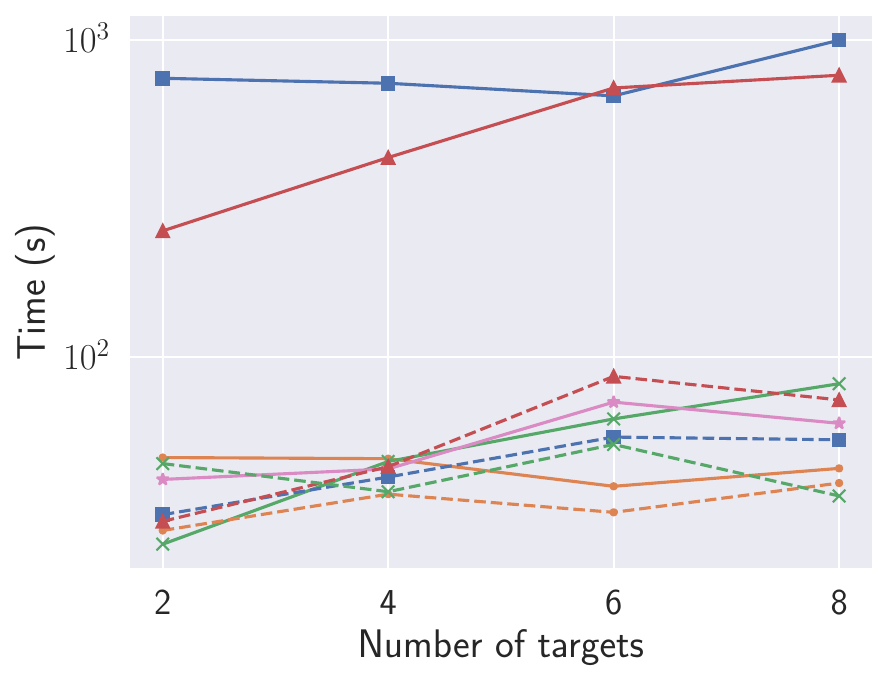}
            \caption*{d-separation tests}
            \label{fig:time_per_target_ident_dsep}
        \end{subfigure}
        \begin{subfigure}[b]{0.24\linewidth}
            \includegraphics[width=\linewidth]{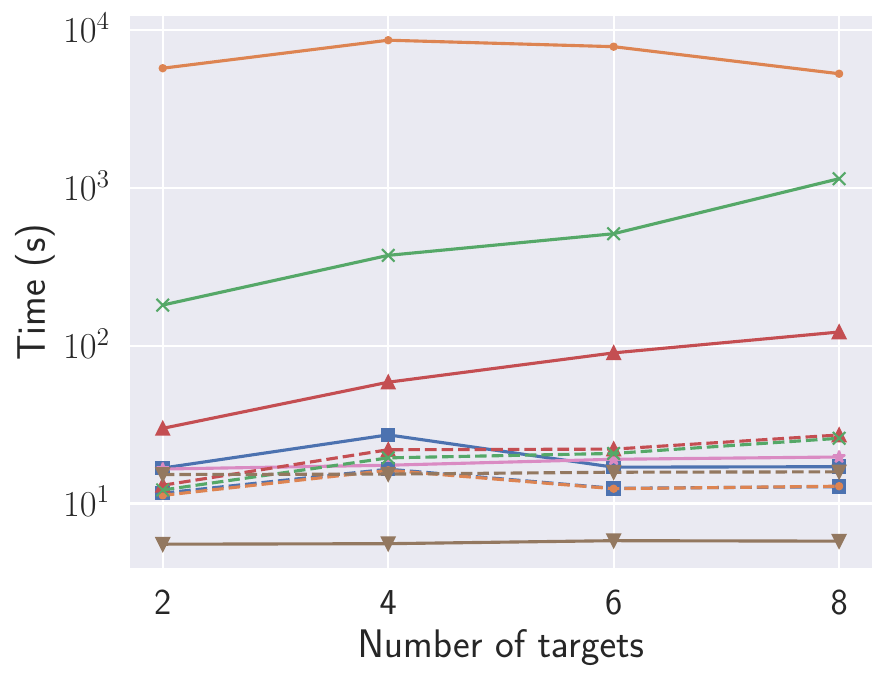}
            \caption*{Fisher-Z tests}
            \label{fig:time_per_target_ident_fshz}
        \end{subfigure}
        \begin{subfigure}[b]{0.24\linewidth}
            \includegraphics[width=\linewidth]{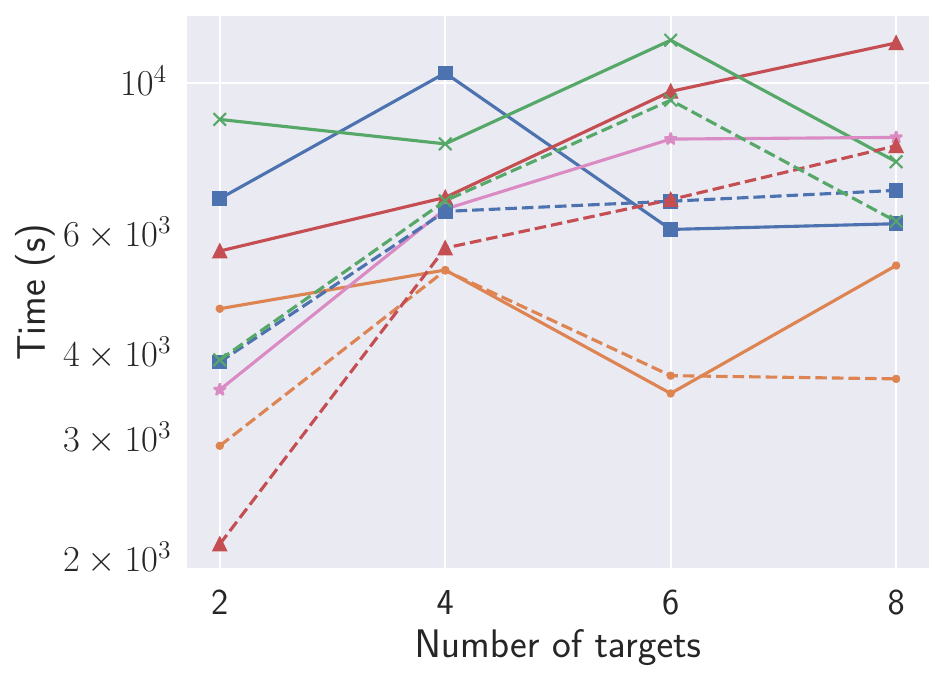}
            \caption*{KCI tests}
            \label{fig:time_per_target_ident_kci}
        \end{subfigure}
        \begin{subfigure}[b]{0.24\linewidth}
            \includegraphics[width=\linewidth]{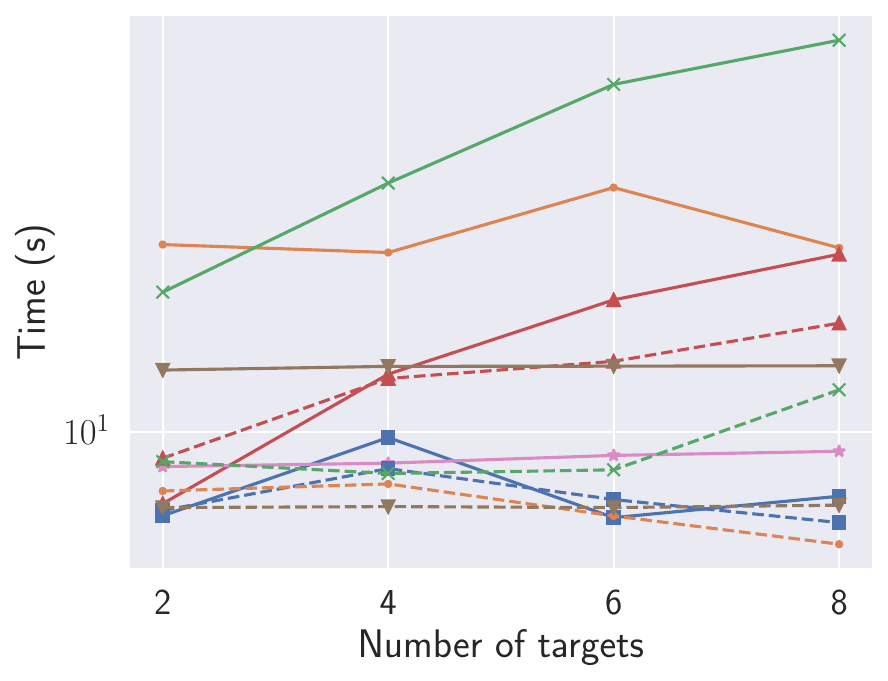}
            \caption*{$\chi^2$ tests}
            \label{fig:time_per_target_ident_chsq}
        \end{subfigure}
        \caption{Computation time.}
        \label{fig:time_per_target_ident}
    \end{subfigure}
    \caption{Number of \ac{CI} tests and computation time over number of identifiable targets, with $n_{\mathbf{V}}=10$ for KCI tests and $n_{\mathbf{V}}=200$ otherwise, $\overline{d} = 3, d_{\max}=10$ and $n_{\mathbf{D}} = 1000$ data-points. We also show baseline methods combined with SNAP$(0)$. We plot values on a $\log$ scale.}
    \label{fig:computation_per_target_ident}
\end{figure}

\begin{figure}
    \centering
    \includegraphics[width=.6\linewidth]{experiments/legend_big.pdf}
    \begin{subfigure}[b]{\linewidth}
        \centering
        \begin{subfigure}[b]{0.24\linewidth}
            \includegraphics[width=\linewidth]{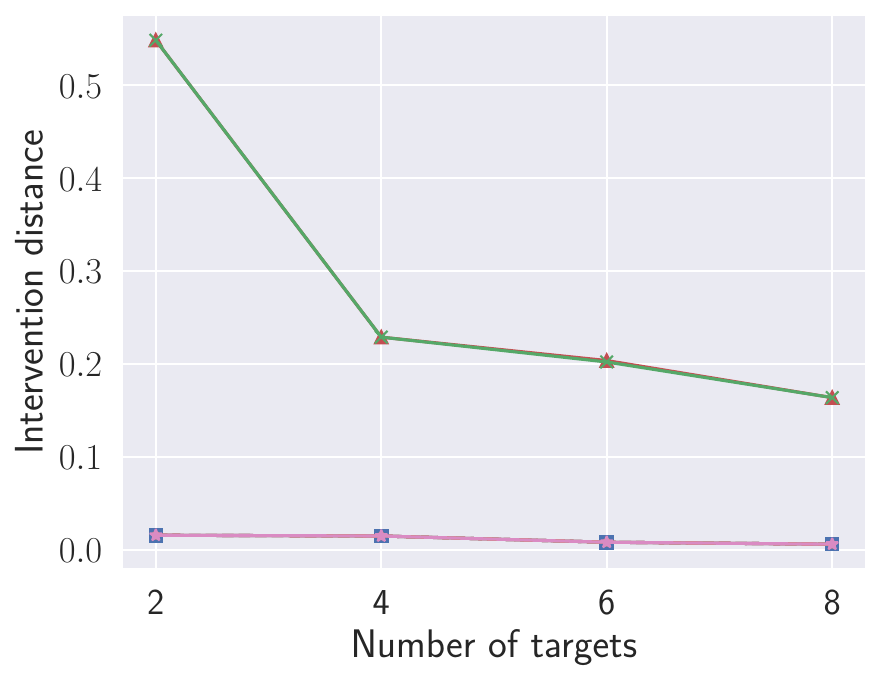}
            \caption*{d-separation tests}
            \label{fig:int_dist_per_target_ident_dsep_abs}
        \end{subfigure}
        \begin{subfigure}[b]{0.24\linewidth}
            \includegraphics[width=\linewidth]{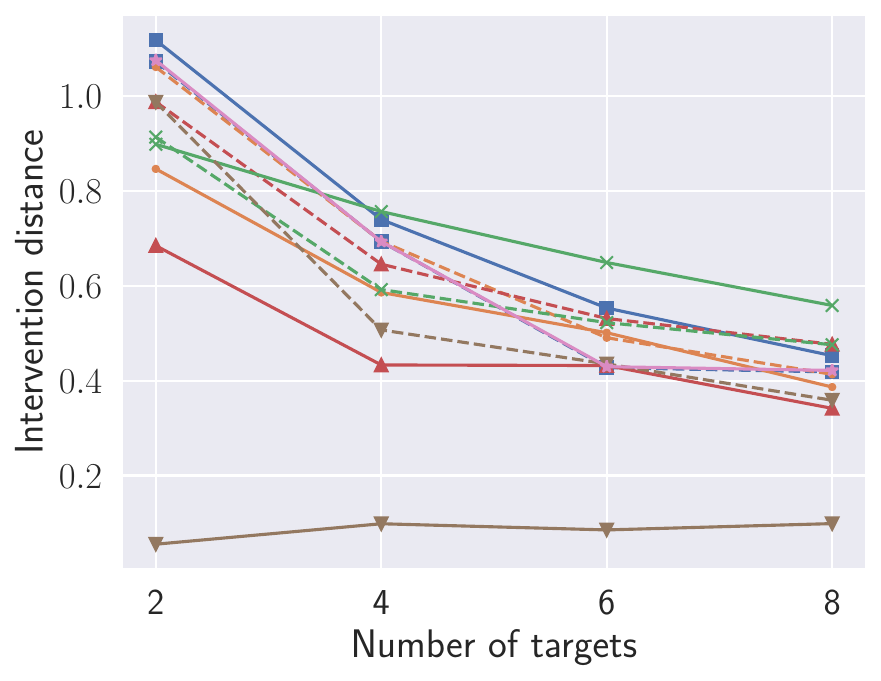}
            \caption*{Fisher-Z tests}
            \label{fig:int_dist_per_target_ident_fshz_abs}
        \end{subfigure}
        \begin{subfigure}[b]{0.24\linewidth}
            \includegraphics[width=\linewidth]{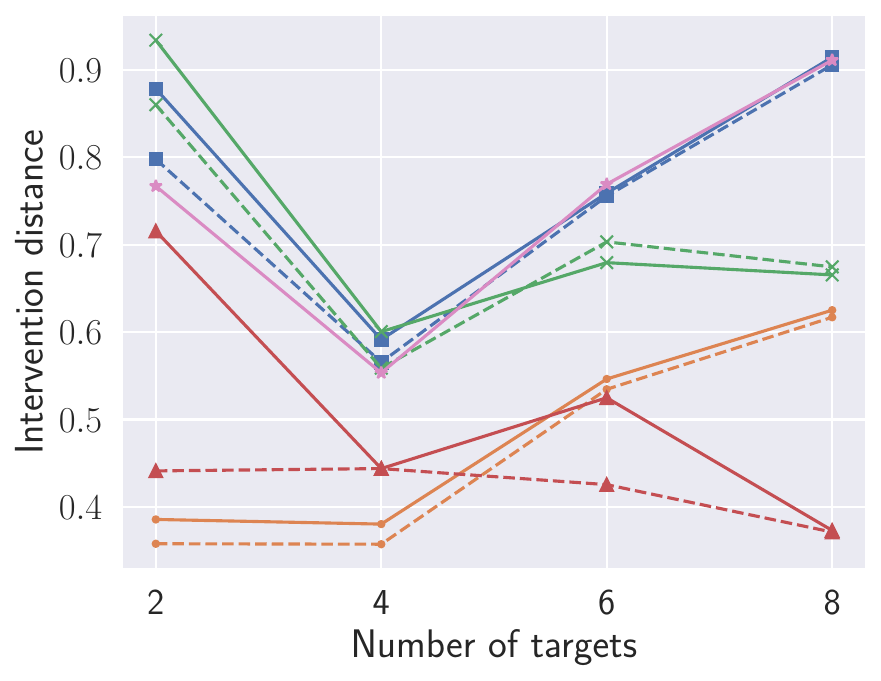}
            \caption*{KCI tests}
            \label{fig:int_dist_per_target_ident_kci_abs}
        \end{subfigure}
        \begin{subfigure}[b]{0.24\linewidth}
            \includegraphics[width=\linewidth]{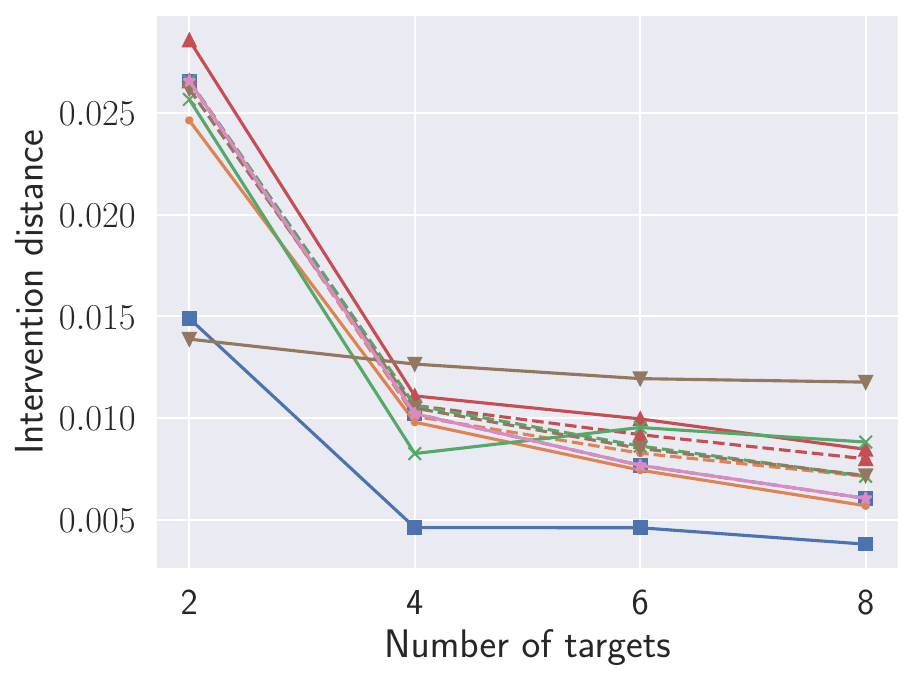}
            \caption*{$\chi^2$ tests}
            \label{fig:int_dist_per_target_ident_chsq_abs}
        \end{subfigure}
        \caption{Intervention distance}
        \label{fig:int_dist_per_target_ident_abs}
    \end{subfigure}
    \begin{subfigure}[b]{\linewidth}
        \centering
        \begin{subfigure}[b]{0.24\linewidth}
            \includegraphics[width=\linewidth]{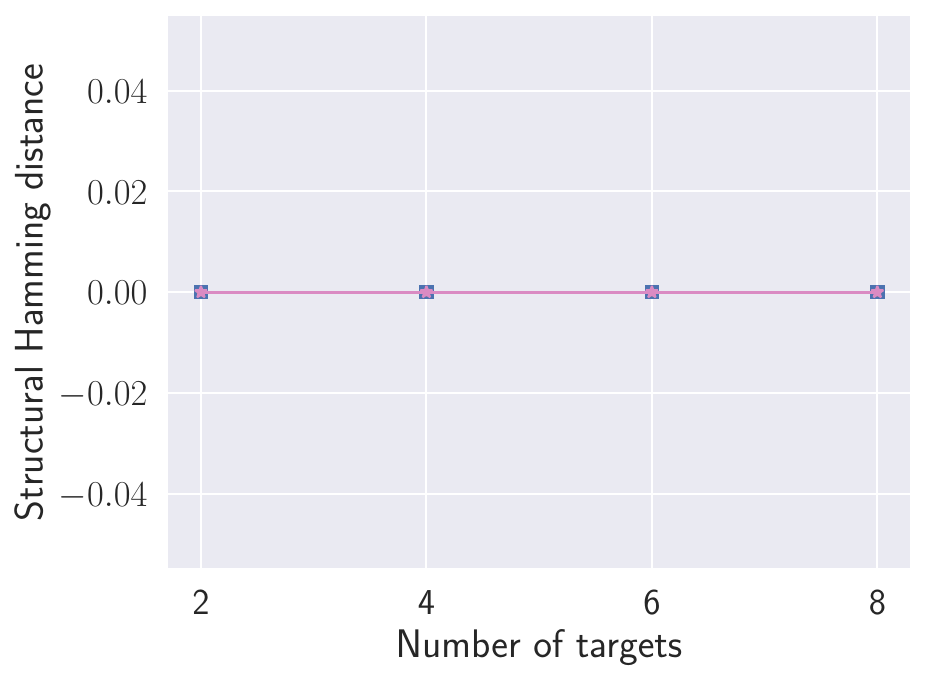}
            \caption*{d-separation tests}
            \label{fig:shd_per_target_ident_dsep}
        \end{subfigure}
        \begin{subfigure}[b]{0.24\linewidth}
            \includegraphics[width=\linewidth]{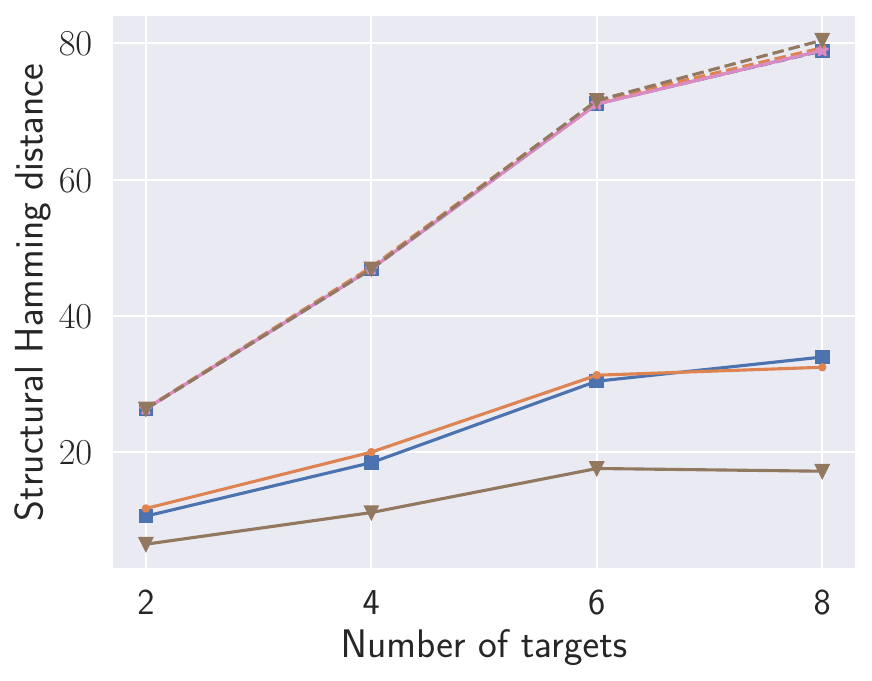}
            \caption*{Fisher-Z tests}
            \label{fig:shd_per_target_ident_fshz}
        \end{subfigure}
        \begin{subfigure}[b]{0.24\linewidth}
            \includegraphics[width=\linewidth]{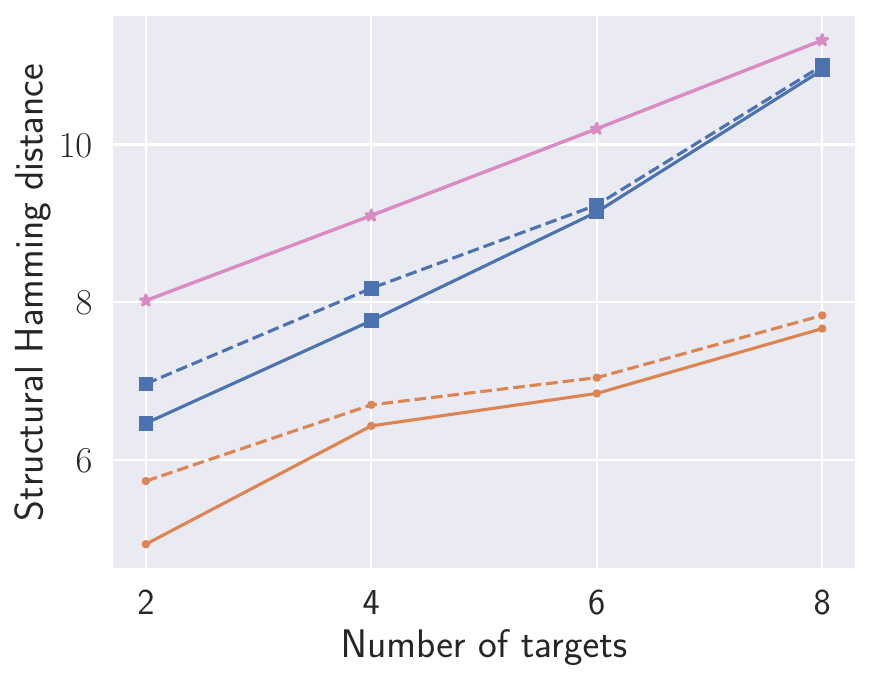}
            \caption*{KCI tests}
            \label{fig:shd_per_target_ident_kci}
        \end{subfigure}
        \begin{subfigure}[b]{0.24\linewidth}
            \includegraphics[width=\linewidth]{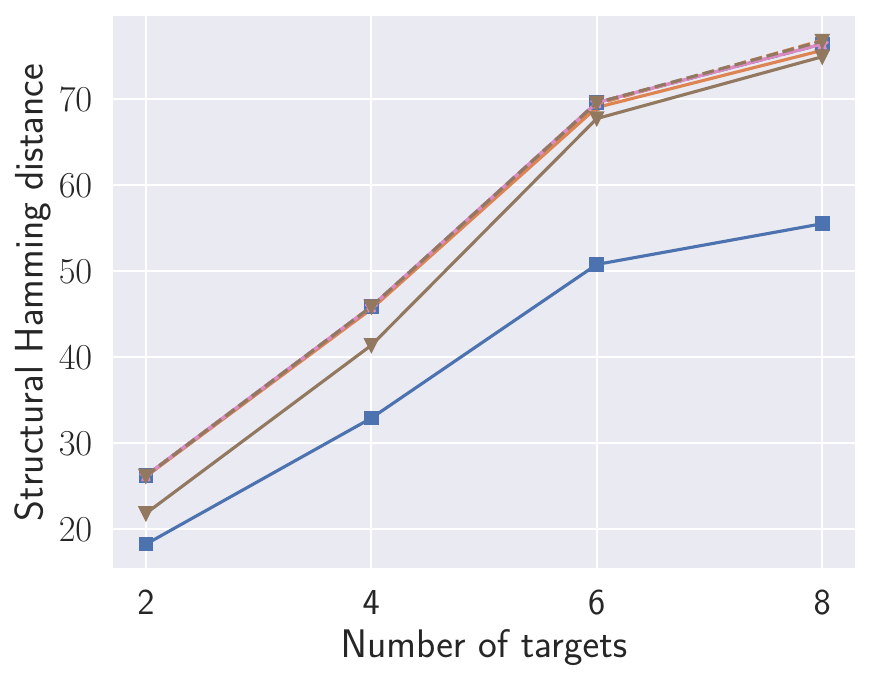}
            \caption*{$\chi^2$ tests}
            \label{fig:shd_per_target_ident_chsq}
        \end{subfigure}
        \caption{\Acl{SHD}}
        \label{fig:shd_per_target_ident}
    \end{subfigure}
    \caption{Estimation quality over number of identifiable targets, with $n_{\mathbf{V}}=10$ for KCI tests and $n_{\mathbf{V}}=200$ otherwise, $\overline{d} = 3, d_{\max}=10$ and $n_{\mathbf{D}} = 1000$ data-points. We also show baseline methods combined with SNAP$(0)$.}
    \label{fig:quality_per_target_ident}
\end{figure}


\begin{figure}
    \centering
    \includegraphics[width=.6\linewidth]{experiments/legend_small.pdf}
    \begin{subfigure}[b]{\linewidth}
        \begin{subfigure}[b]{0.24\linewidth}
            \includegraphics[width=\linewidth]{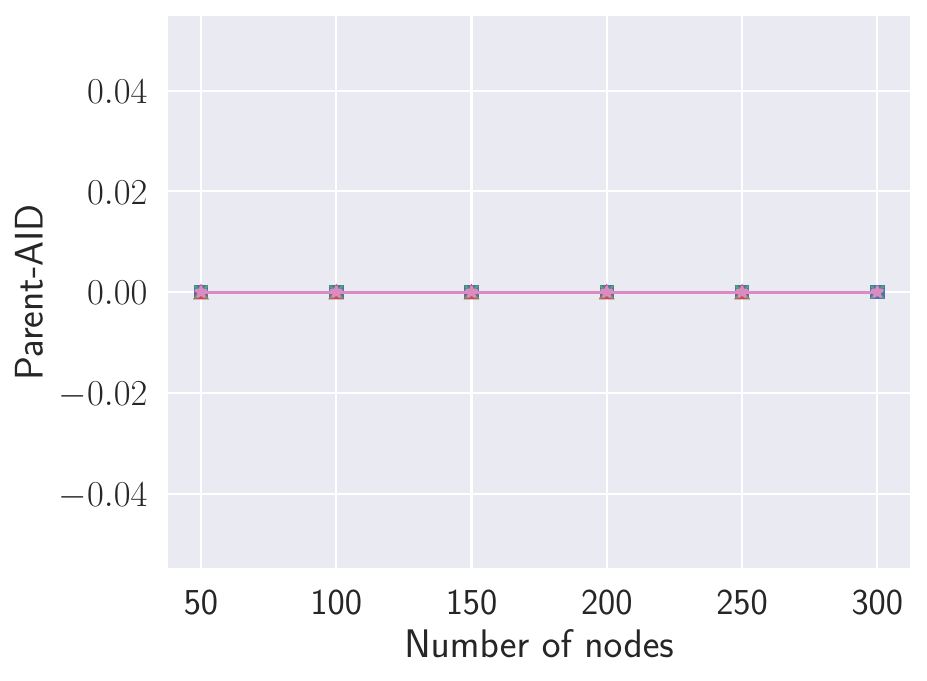}
            \caption*{d-separation tests}
            \label{fig:parent_aid_per_target_ident_dsep_std}
        \end{subfigure}
        \begin{subfigure}[b]{0.24\linewidth}
            \includegraphics[width=\linewidth]{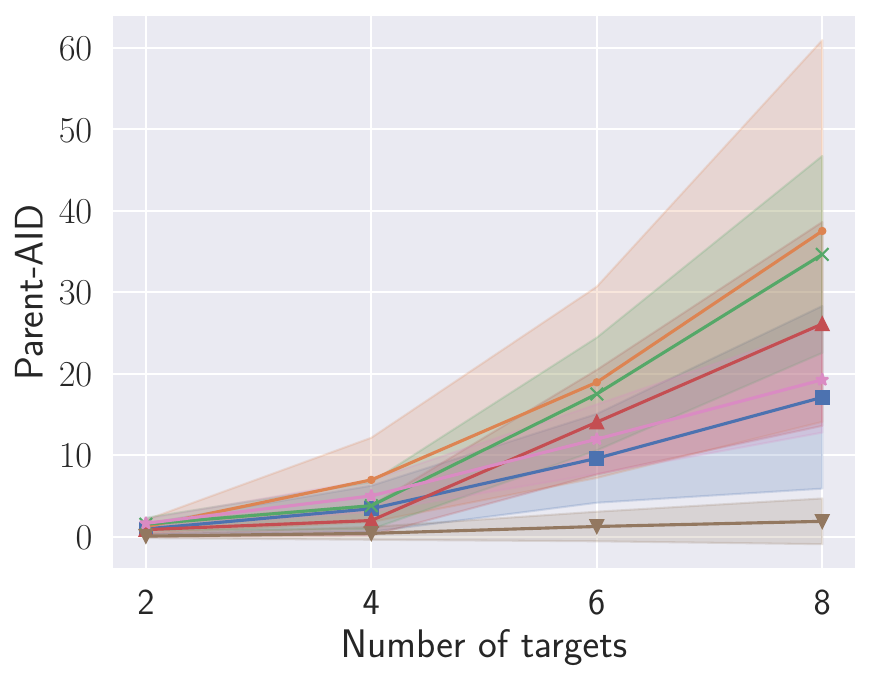}
            \caption*{Fisher-Z tests}
            \label{fig:parent_aid_per_target_ident_fshz_std}
        \end{subfigure}
        \begin{subfigure}[b]{0.24\linewidth}
            \includegraphics[width=\linewidth]{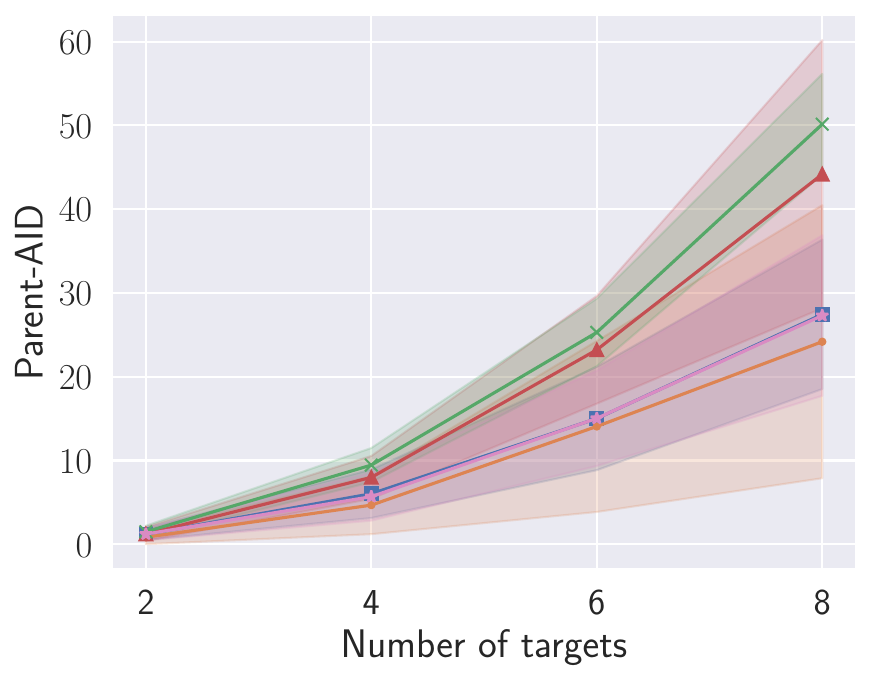}
            \caption*{KCI tests}
            \label{fig:parent_aid_per_target_ident_kci_std}
        \end{subfigure}
        \begin{subfigure}[b]{0.24\linewidth}
            \includegraphics[width=\linewidth]{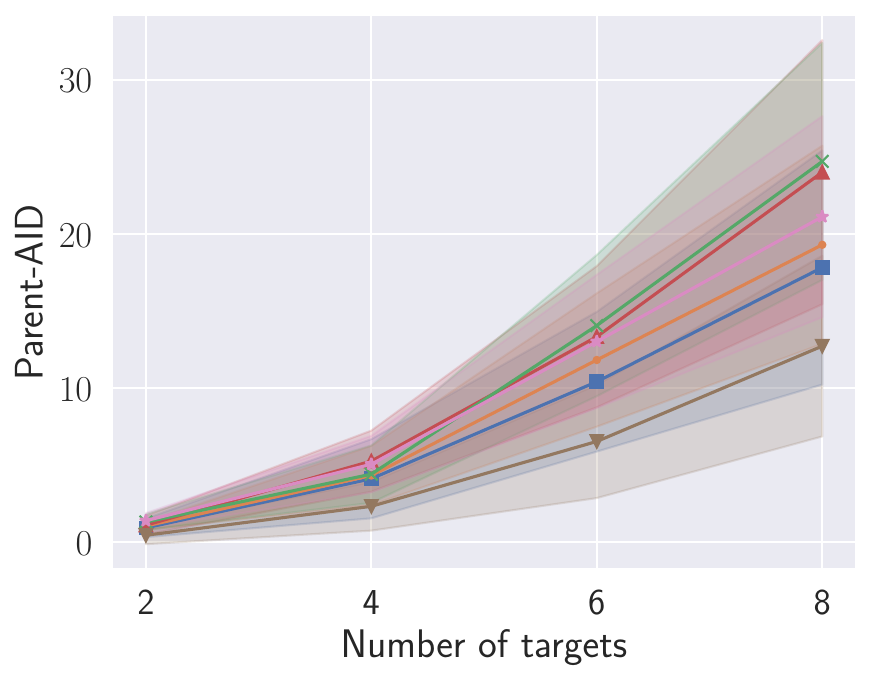}
            \caption*{$\chi^2$ tests}
            \label{fig:parent_aid_per_target_ident_chsq_std}
        \end{subfigure}
        \caption{Parent-AID.}
        \label{fig:parent_aid_per_target_ident_std}
    \end{subfigure}
    \begin{subfigure}[b]{\linewidth}
        \begin{subfigure}[b]{0.24\linewidth}
            \includegraphics[width=\linewidth]{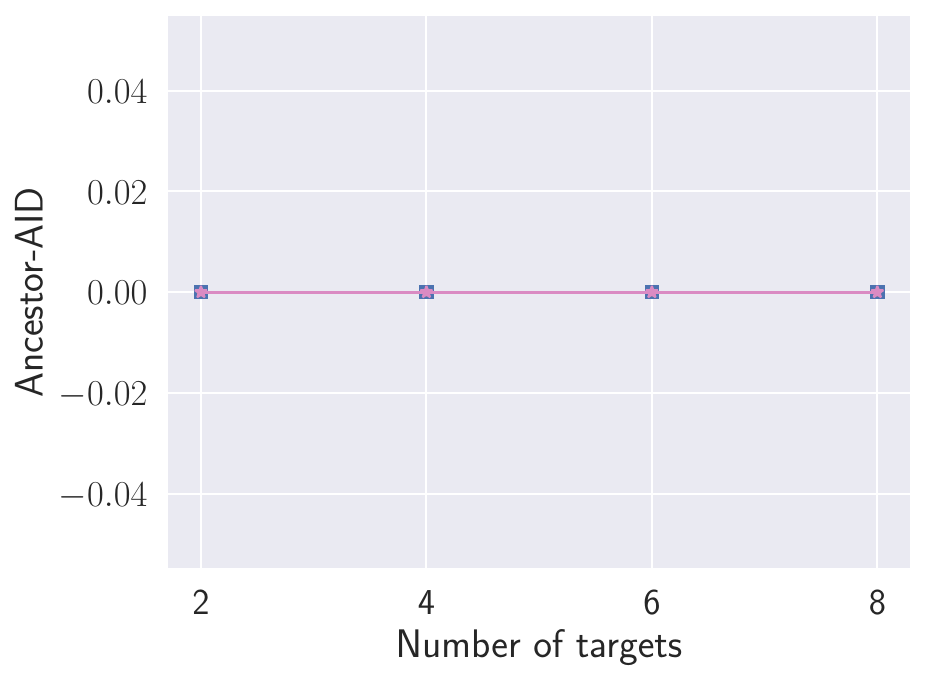}
            \caption*{d-separation tests}
            \label{fig:ancestor_aid_per_target_ident_dsep_std}
        \end{subfigure}
        \begin{subfigure}[b]{0.24\linewidth}
            \includegraphics[width=\linewidth]{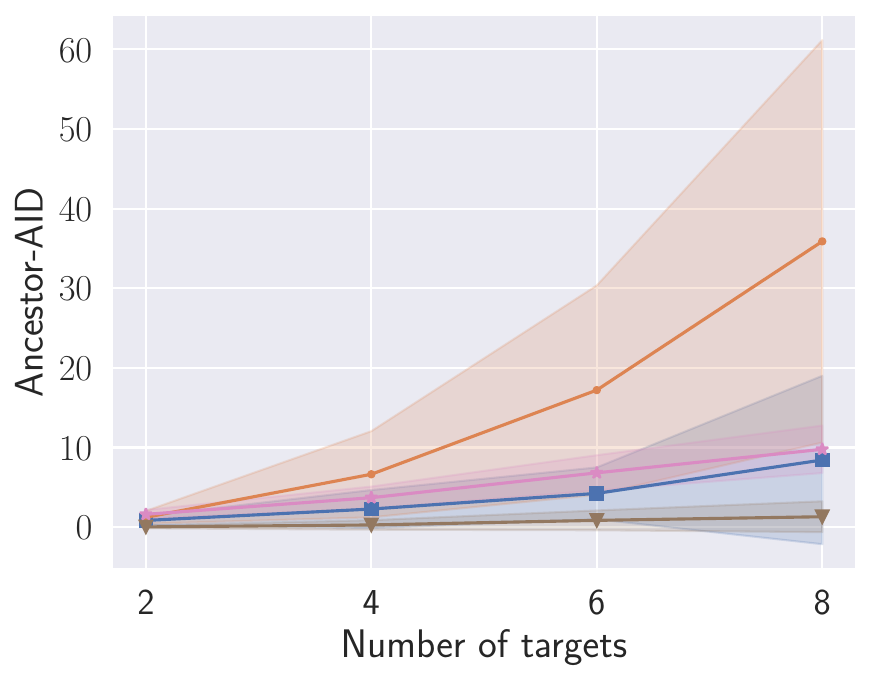}
            \caption*{Fisher-Z tests}
            \label{fig:ancestor_aid_per_target_ident_fshz_std}
        \end{subfigure}
        \begin{subfigure}[b]{0.24\linewidth}
            \includegraphics[width=\linewidth]{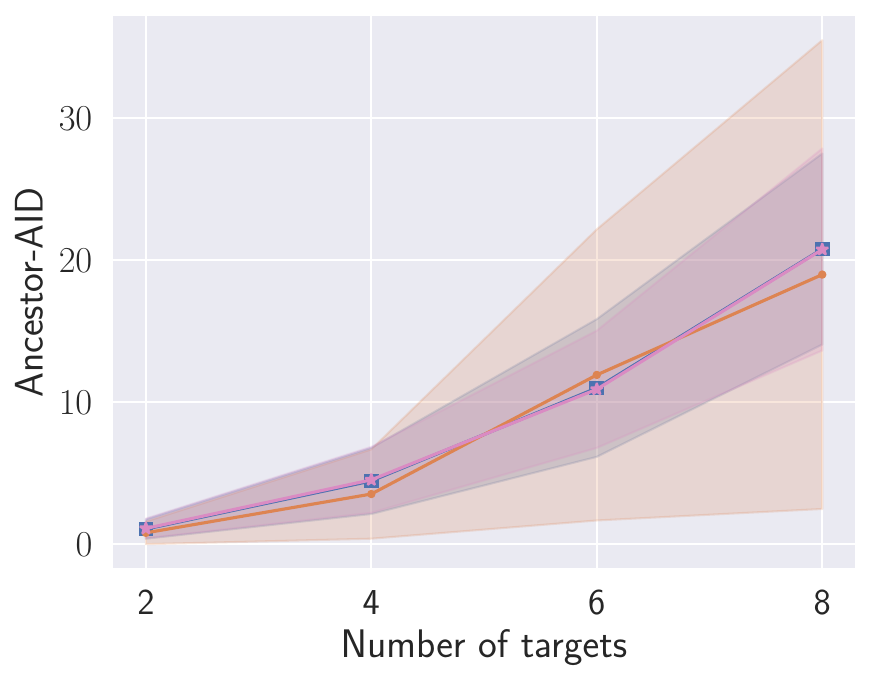}
            \caption*{KCI tests}
            \label{fig:ancestor_aid_per_target_ident_kci_std}
        \end{subfigure}
        \begin{subfigure}[b]{0.24\linewidth}
            \includegraphics[width=\linewidth]{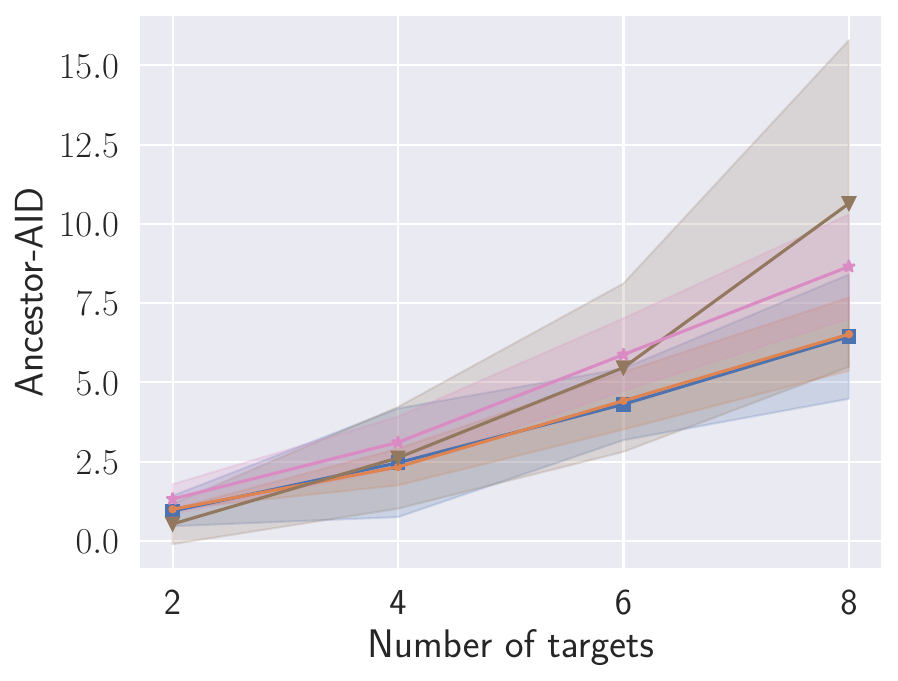}
            \caption*{$\chi^2$ tests}
            \label{fig:ancestor_aid_per_target_ident_chsq_std}
        \end{subfigure}
        \caption{Ancestor-AID.}
        \label{fig:ancestor_aid_per_target_ident_std}
    \end{subfigure}
    \begin{subfigure}[b]{\linewidth}
        \begin{subfigure}[b]{0.24\linewidth}
            \includegraphics[width=\linewidth]{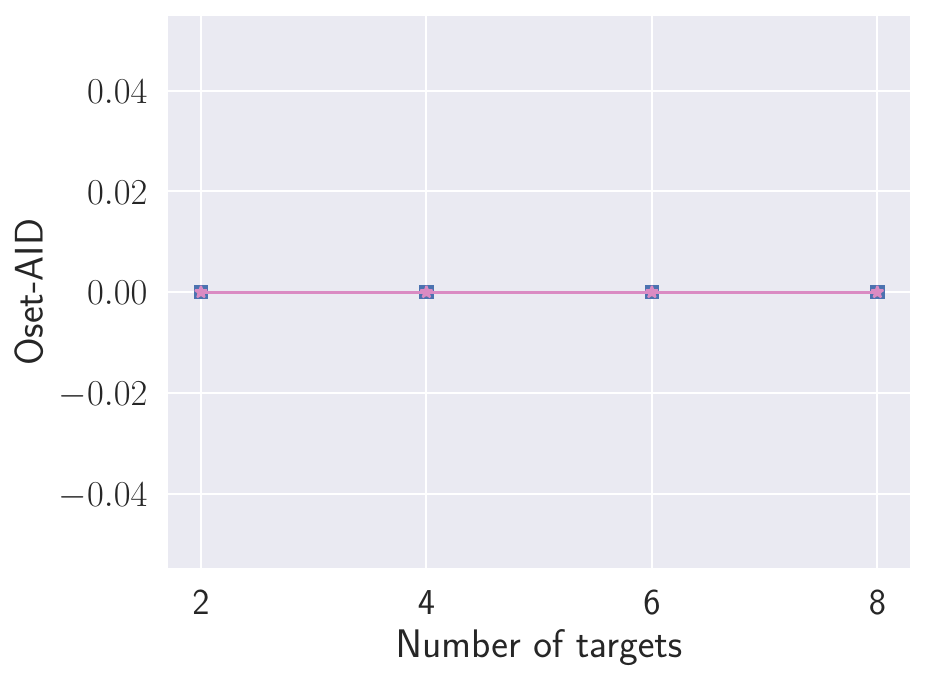}
            \caption*{d-separation tests}
            \label{fig:oset_aid_per_target_ident_dsep_std}
        \end{subfigure}
        \begin{subfigure}[b]{0.24\linewidth}
            \includegraphics[width=\linewidth]{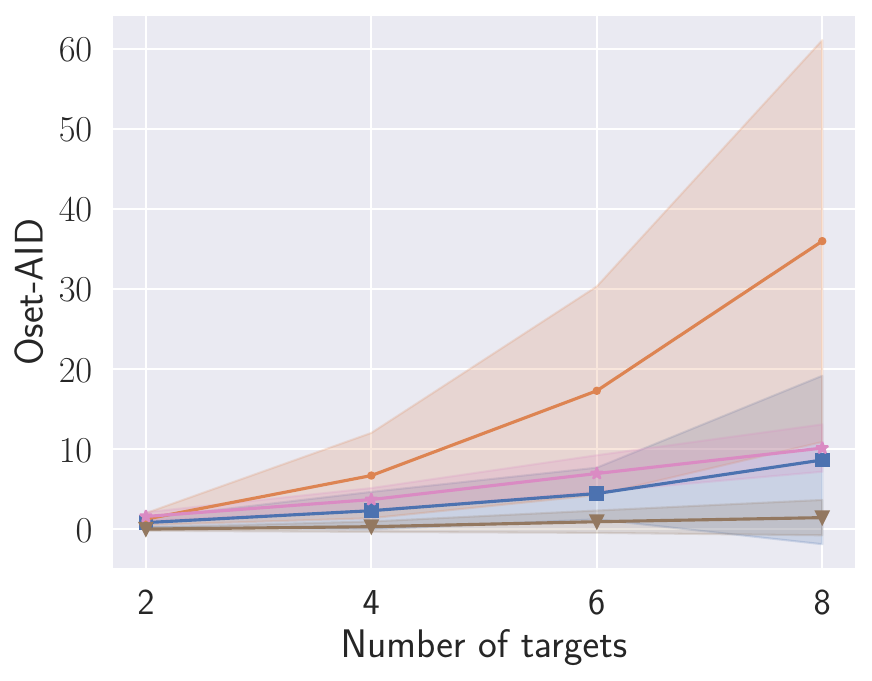}
            \caption*{Fisher-Z tests}
            \label{fig:oset_aid_per_target_ident_fshz_std}
        \end{subfigure}
        \begin{subfigure}[b]{0.24\linewidth}
            \includegraphics[width=\linewidth]{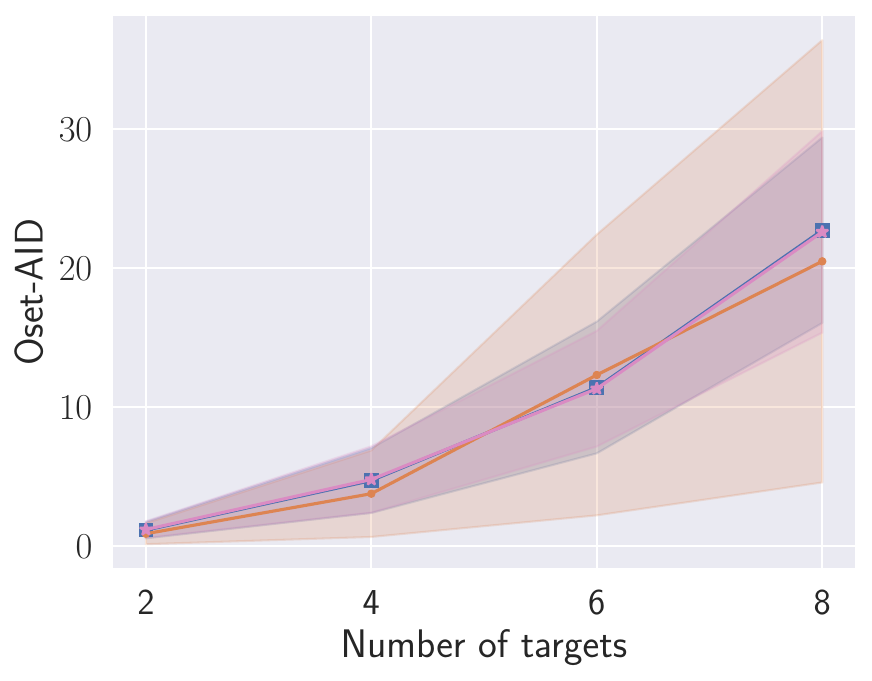}
            \caption*{KCI tests}
            \label{fig:oset_aid_per_target_ident_kci_std}
        \end{subfigure}
        \begin{subfigure}[b]{0.24\linewidth}
            \includegraphics[width=\linewidth]{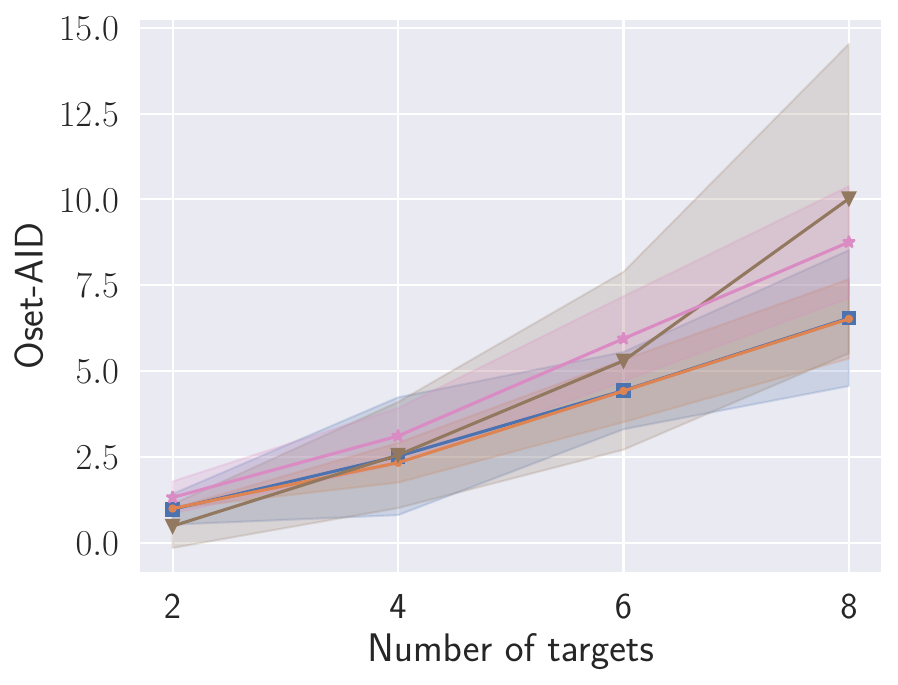}
            \caption*{$\chi^2$ tests}
            \label{fig:oset_aid_per_target_ident_chsq_std}
        \end{subfigure}
        \caption{Oset-AID.}
        \label{fig:oset_aid_per_target_ident_std}
    \end{subfigure}
    \caption{Adjustment identification distance over number of identifiable targets, with $n_{\mathbf{V}}=10$ for KCI tests and $n_{\mathbf{V}}=200$ otherwise, $\overline{d} = 3, d_{\max}=10$ and $n_{\mathbf{D}} = 1000$ data-points. The shadow area denotes the range of the standard deviation.}
    \label{fig:aid_per_target_ident_std}
\end{figure}

\begin{figure}
    \centering
    \includegraphics[width=.6\linewidth]{experiments/legend_big.pdf}
    \begin{subfigure}[b]{\linewidth}
        \begin{subfigure}[b]{0.24\linewidth}
            \includegraphics[width=\linewidth]{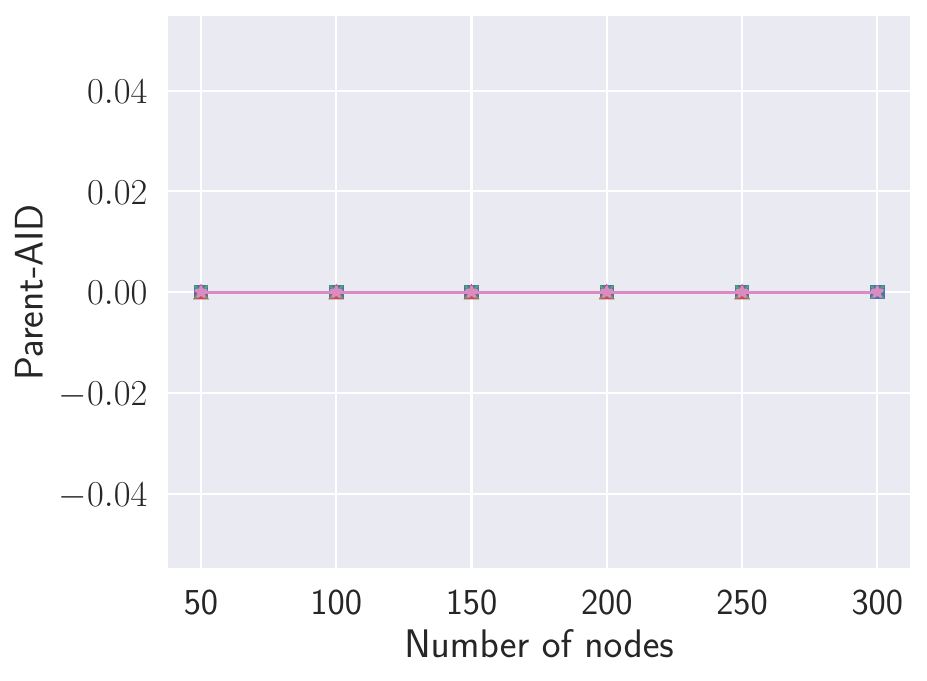}
            \caption*{d-separation tests}
            \label{fig:parent_aid_per_target_ident_dsep}
        \end{subfigure}
        \begin{subfigure}[b]{0.24\linewidth}
            \includegraphics[width=\linewidth]{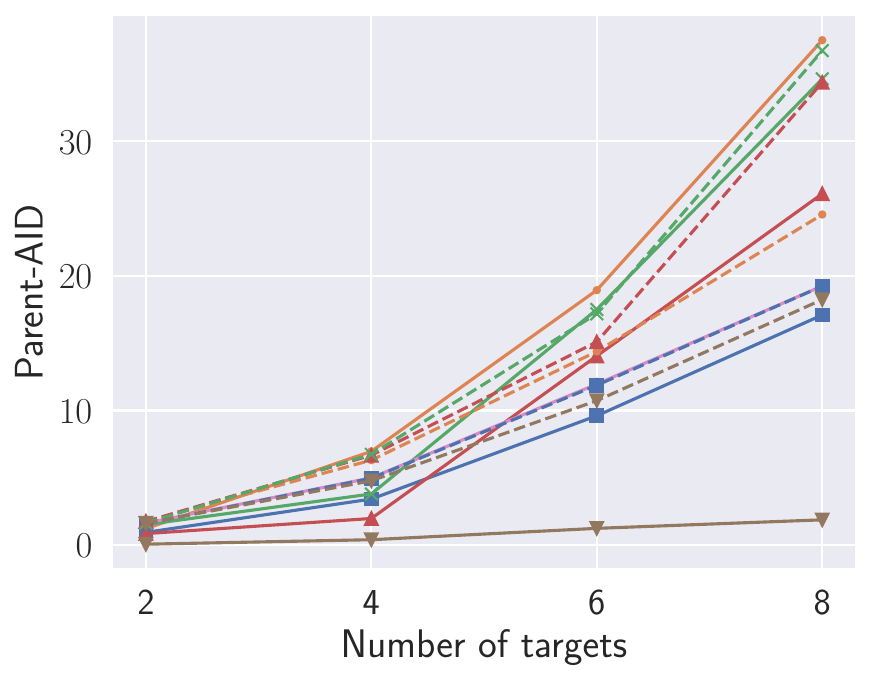}
            \caption*{Fisher-Z tests}
            \label{fig:parent_aid_per_target_ident_fshz}
        \end{subfigure}
        \begin{subfigure}[b]{0.24\linewidth}
            \includegraphics[width=\linewidth]{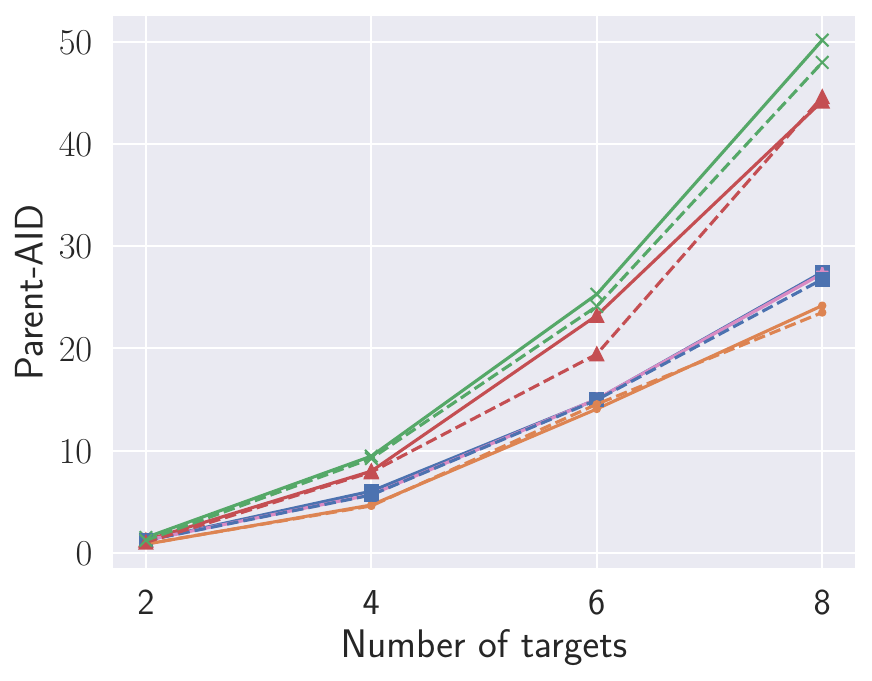}
            \caption*{KCI tests}
            \label{fig:parent_aid_per_target_ident_kci}
        \end{subfigure}
        \begin{subfigure}[b]{0.24\linewidth}
            \includegraphics[width=\linewidth]{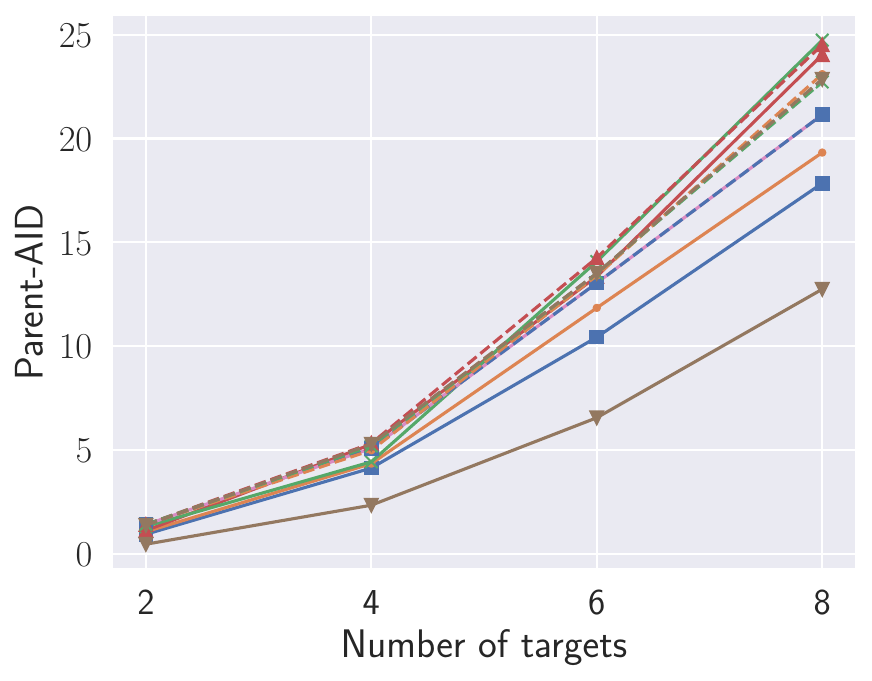}
            \caption*{$\chi^2$ tests}
            \label{fig:parent_aid_per_target_ident_chsq}
        \end{subfigure}
        \caption{Parent-AID.}
        \label{fig:parent_aid_per_target_ident}
    \end{subfigure}
    \begin{subfigure}[b]{\linewidth}
        \begin{subfigure}[b]{0.24\linewidth}
            \includegraphics[width=\linewidth]{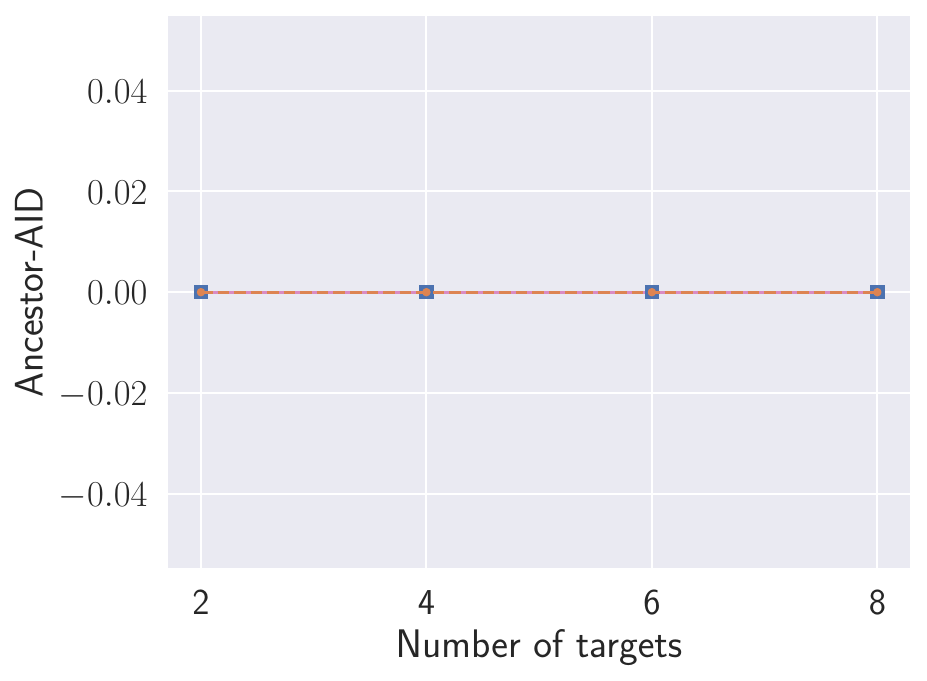}
            \caption*{d-separation tests}
            \label{fig:ancestor_aid_per_target_ident_dsep}
        \end{subfigure}
        \begin{subfigure}[b]{0.24\linewidth}
            \includegraphics[width=\linewidth]{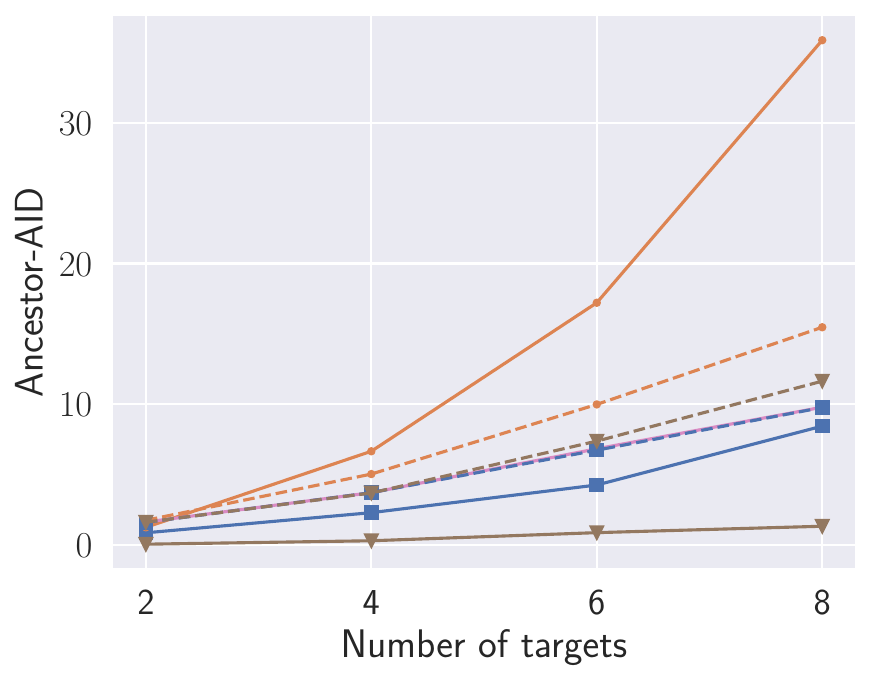}
            \caption*{Fisher-Z tests}
            \label{fig:ancestor_aid_per_target_ident_fshz}
        \end{subfigure}
        \begin{subfigure}[b]{0.24\linewidth}
            \includegraphics[width=\linewidth]{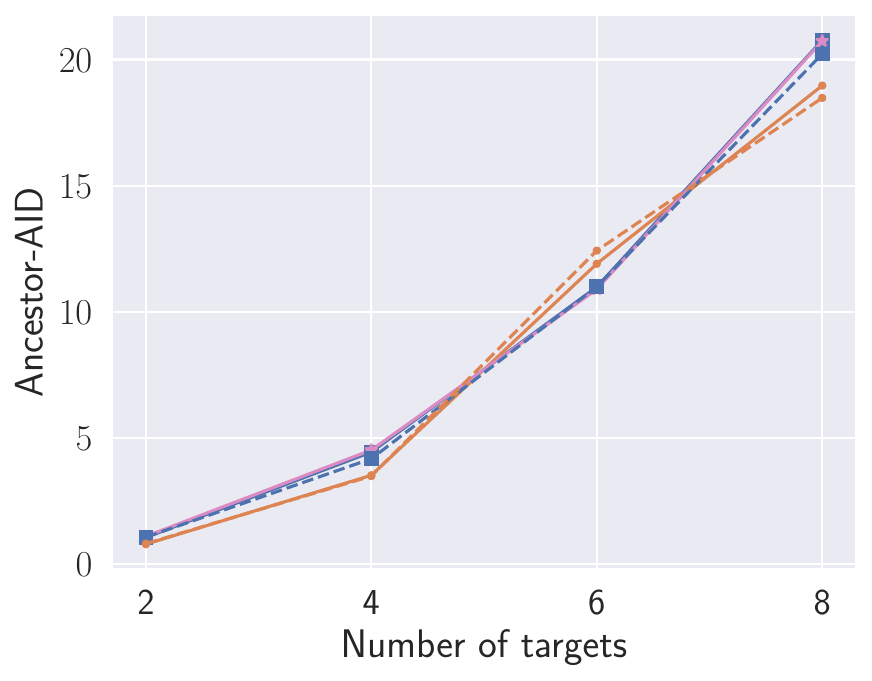}
            \caption*{KCI tests}
            \label{fig:ancestor_aid_per_target_ident_kci}
        \end{subfigure}
        \begin{subfigure}[b]{0.24\linewidth}
            \includegraphics[width=\linewidth]{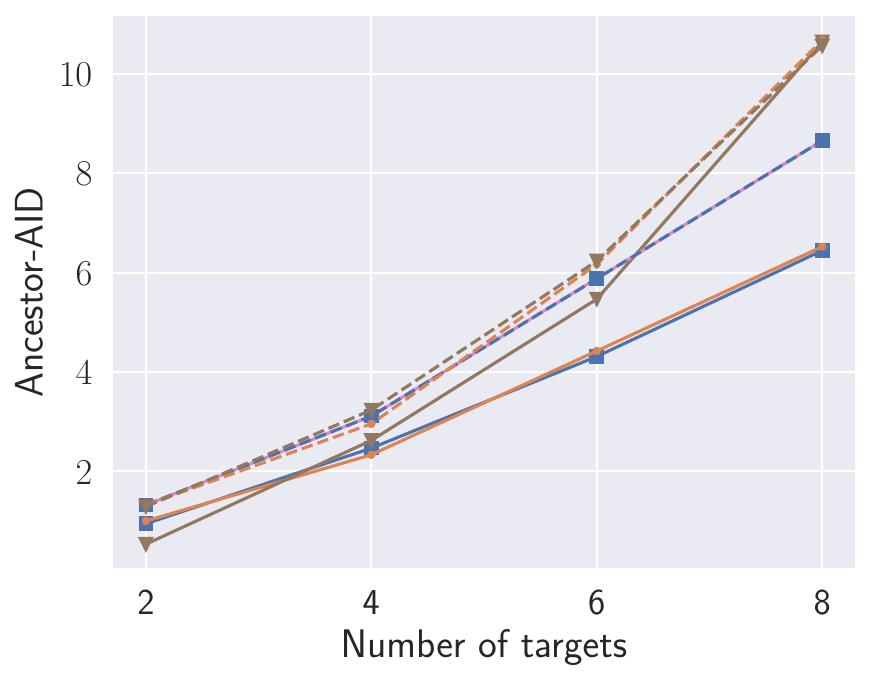}
            \caption*{$\chi^2$ tests}
            \label{fig:ancestor_aid_per_target_ident_chsq}
        \end{subfigure}
        \caption{Ancestor-AID.}
        \label{fig:ancestor_aid_per_target_ident}
    \end{subfigure}
    \begin{subfigure}[b]{\linewidth}
        \begin{subfigure}[b]{0.24\linewidth}
            \includegraphics[width=\linewidth]{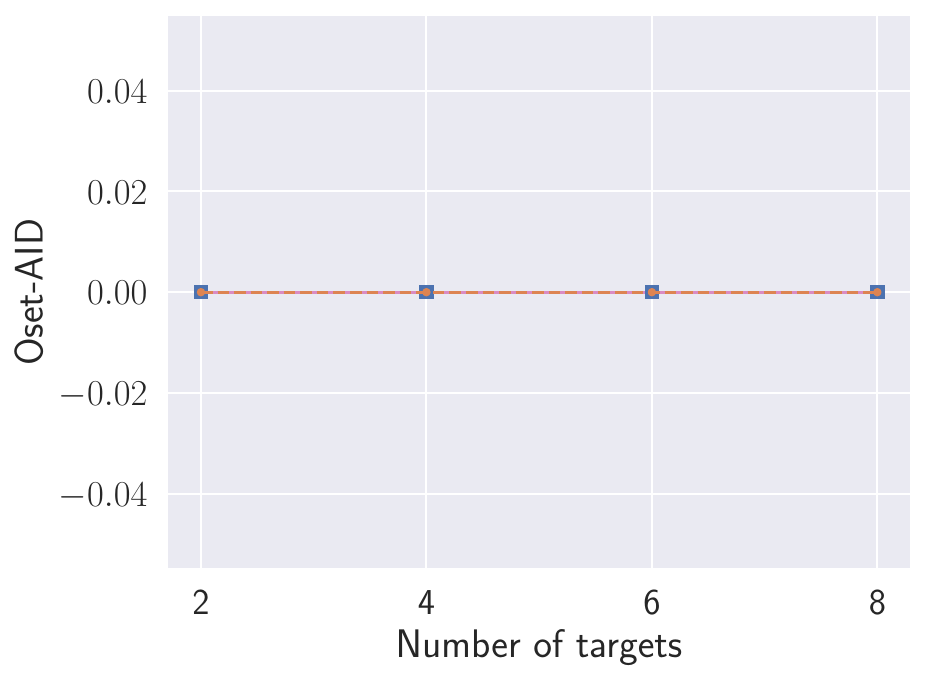}
            \caption*{d-separation tests}
            \label{fig:oset_aid_per_target_ident_dsep}
        \end{subfigure}
        \begin{subfigure}[b]{0.24\linewidth}
            \includegraphics[width=\linewidth]{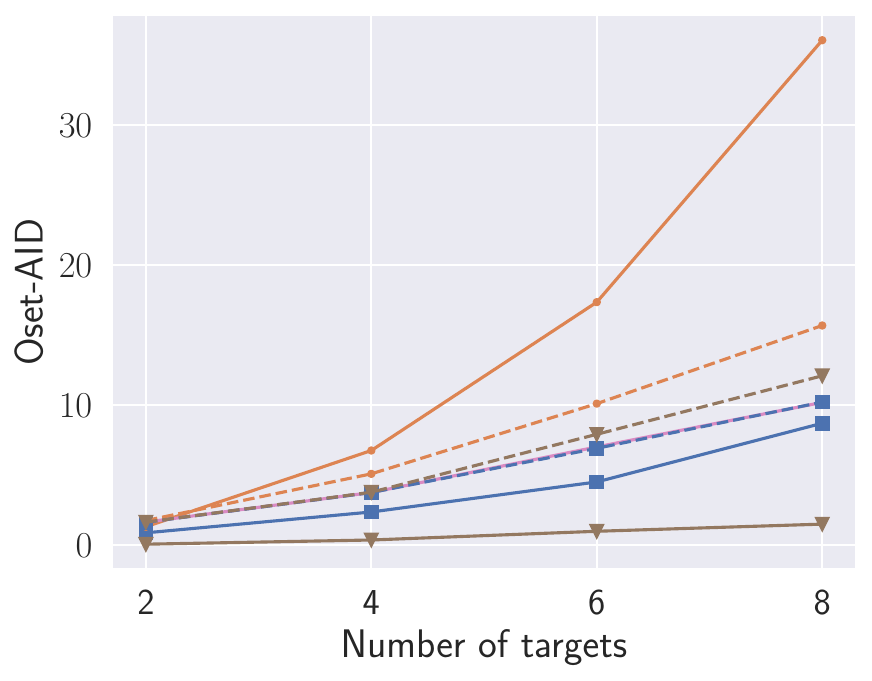}
            \caption*{Fisher-Z tests}
            \label{fig:oset_aid_per_target_ident_fshz}
        \end{subfigure}
        \begin{subfigure}[b]{0.24\linewidth}
            \includegraphics[width=\linewidth]{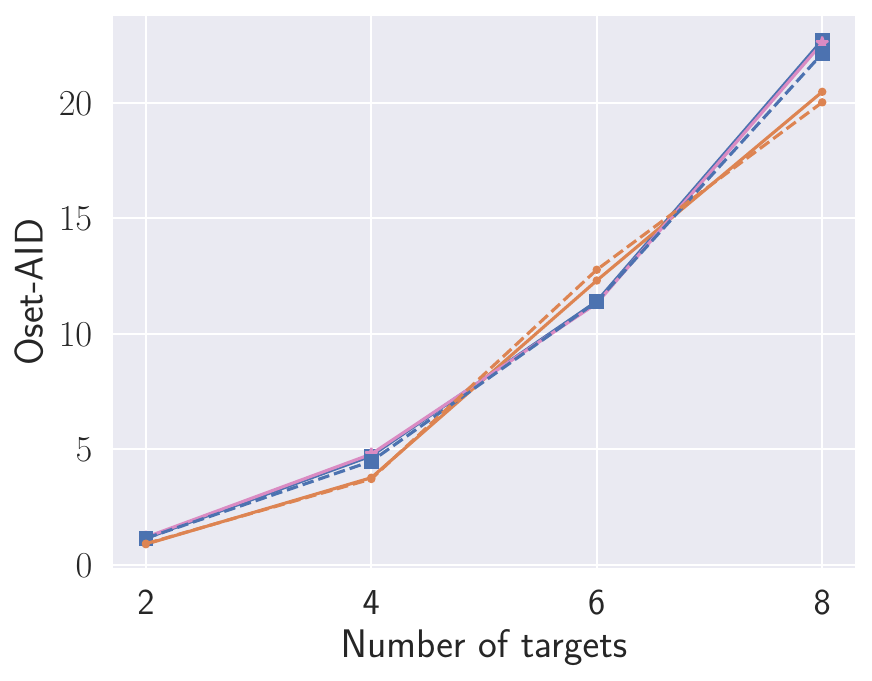}
            \caption*{KCI tests}
            \label{fig:oset_aid_per_target_ident_kci}
        \end{subfigure}
        \begin{subfigure}[b]{0.24\linewidth}
            \includegraphics[width=\linewidth]{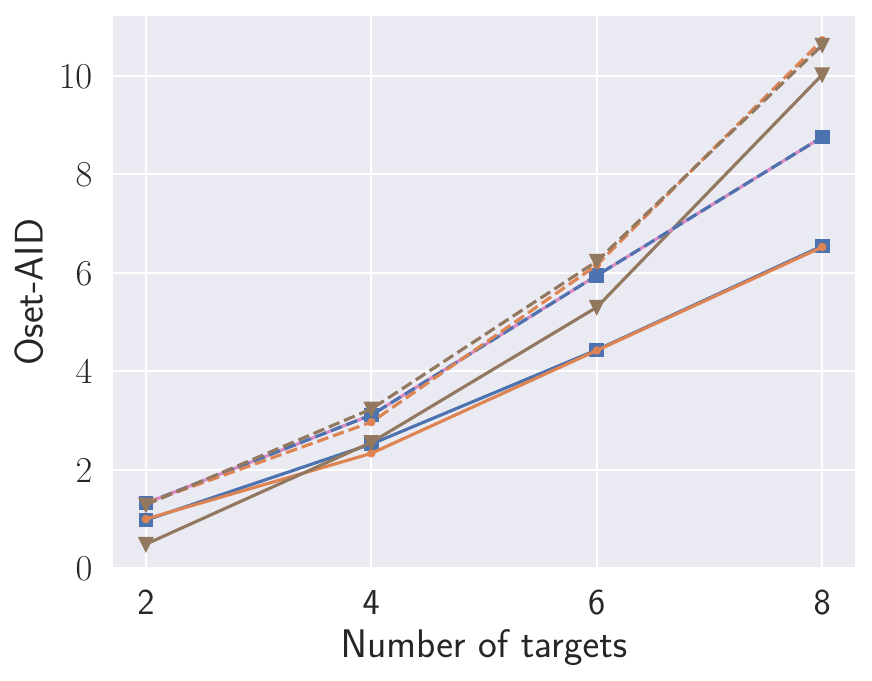}
            \caption*{$\chi^2$ tests}
            \label{fig:oset_aid_per_target_ident_chsq}
        \end{subfigure}
        \caption{Oset-AID.}
        \label{fig:oset_aid_per_target_ident}
    \end{subfigure}
    \caption{Adjustment identification distance for baseline methods combined with SNAP$(0)$ over number of identifiable targets, with $n_{\mathbf{V}}=10$ for KCI tests and $n_{\mathbf{V}}=200$ otherwise, $\overline{d} = 3, d_{\max}=10$ and $n_{\mathbf{D}} = 1000$ data-points.}
    \label{fig:aid_per_target_ident}
\end{figure}

\subsection{Various numbers of expected degree}
\label{app:over_degrees}

In this section, we evaluate methods on graphs with expected degrees of $2, \dots, 4$, $n_{\mathbf{V}}=10$ for KCI tests and $n_{\mathbf{V}}=200$ otherwise, $n_{\mathbf{T}}=4, d_{\max}=10$ and $n_{\mathbf{D}} = 1000$ data-points.
We evaluate baselines methods and SNAP$(\infty)$  on Fig.~\ref{fig:per_degree_std}.
On Fig.~\ref{fig:per_degree} we also include baseline methods combined with SNAP$(0)$ for prefiltering.
Our results show that SNAP variants are consistently among the top performers on most metrics except for \ac{SHD}, while the performance of other methods is highly dependent on the setting.
In particular, we observe that SNAP$(\infty)$ needs to do fewer KCI tests, while remaining comparable on other metrics.

\begin{figure}
    \centering
    \includegraphics[width=.6\linewidth]{experiments/legend_small.pdf}
    \begin{subfigure}[b]{\linewidth}
        \begin{subfigure}[b]{0.24\linewidth}
            \includegraphics[width=\linewidth]{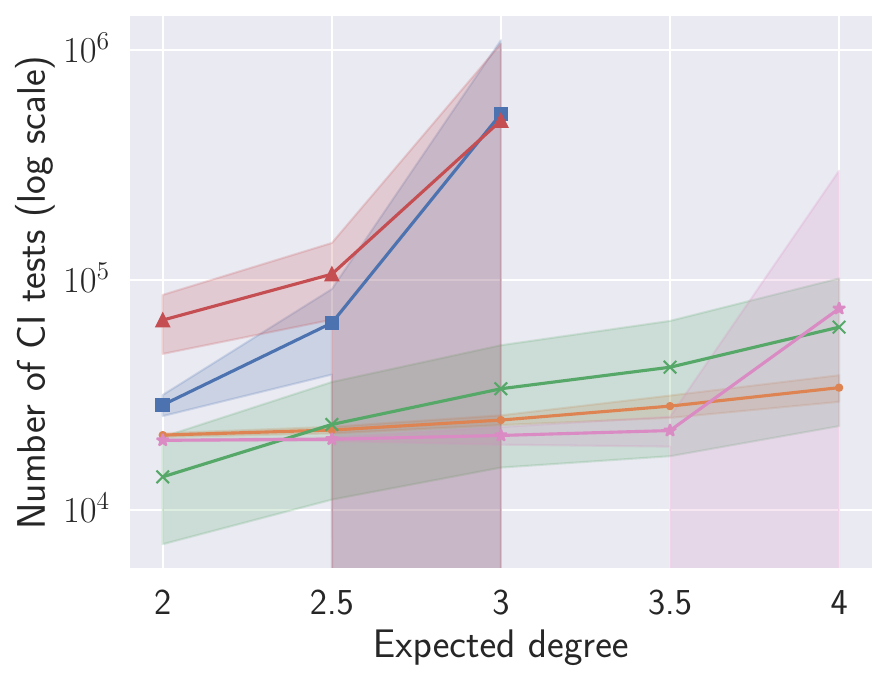}
            \caption*{d-separation tests}
            \label{fig:test_per_degree_unrest_dsep_std}
        \end{subfigure}
        \begin{subfigure}[b]{0.24\linewidth}
            \includegraphics[width=\linewidth]{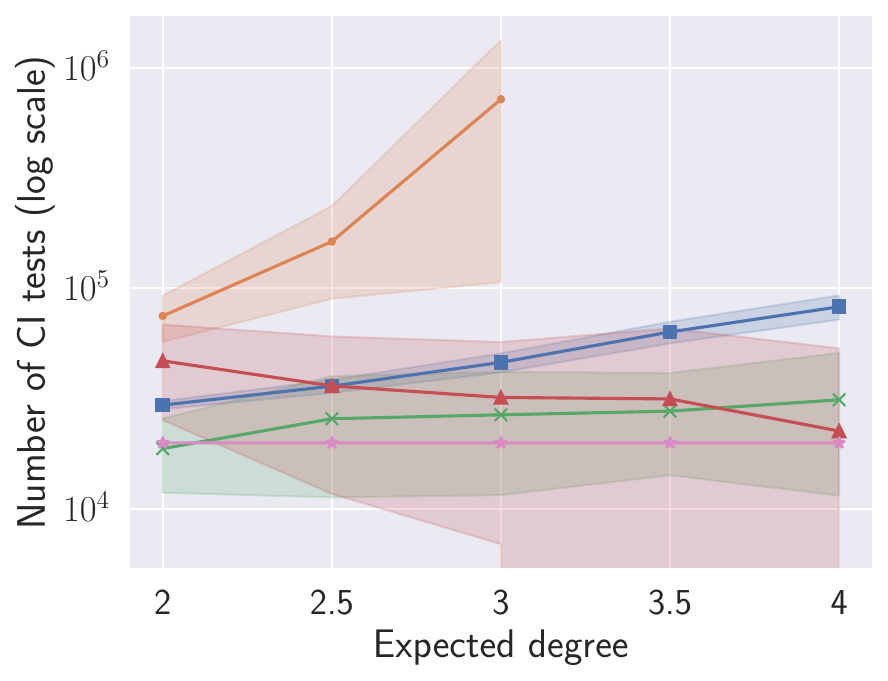}
            \caption*{Fisher-Z tests}
            \label{fig:test_per_degree_unrest_fshz_std}
        \end{subfigure}
        \begin{subfigure}[b]{0.24\linewidth}
            \includegraphics[width=\linewidth]{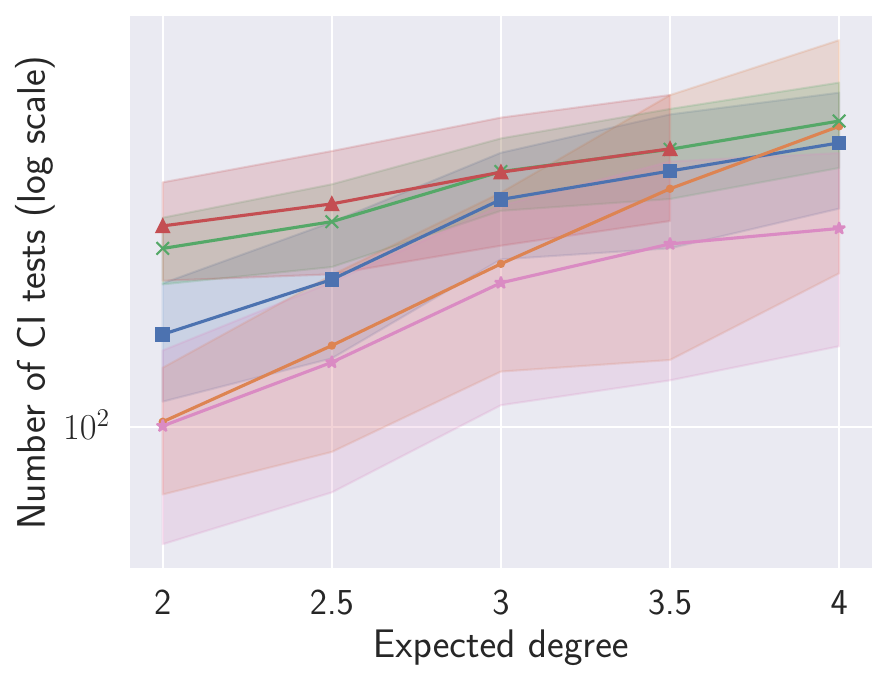}
            \caption*{KCI tests}
            \label{fig:test_per_degree_unrest_kci_std}
        \end{subfigure}
        \begin{subfigure}[b]{0.24\linewidth}
            \includegraphics[width=\linewidth]{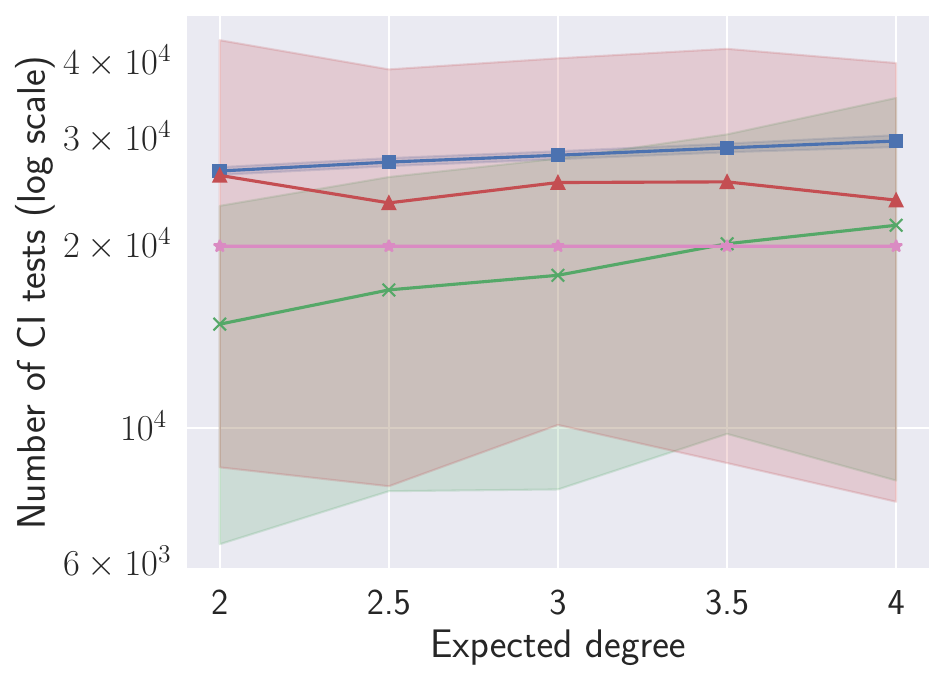}
            \caption*{$\chi^2$ tests}
            \label{fig:test_per_degree_unrest_chsq_std}
        \end{subfigure}
        \caption{Number of \ac{CI} tests.}
        \label{fig:test_per_degree_unrest_std}
    \end{subfigure}
    \begin{subfigure}[b]{\linewidth}
        \begin{subfigure}[b]{0.24\linewidth}
            \includegraphics[width=\linewidth]{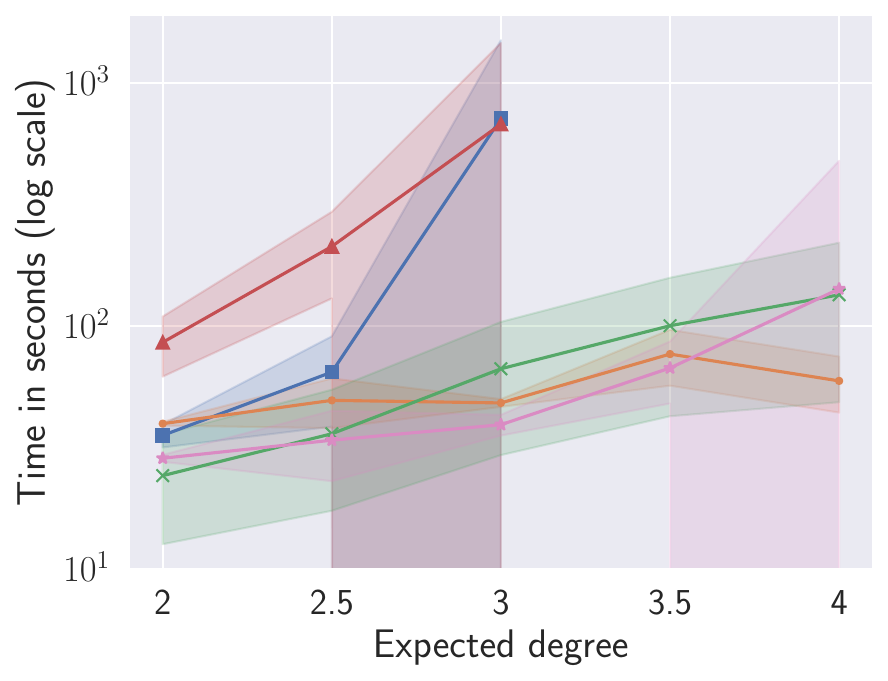}
            \caption*{d-separation tests}
            \label{fig:time_per_degree_unrest_dsep_std}
        \end{subfigure}
        \begin{subfigure}[b]{0.24\linewidth}
            \includegraphics[width=\linewidth]{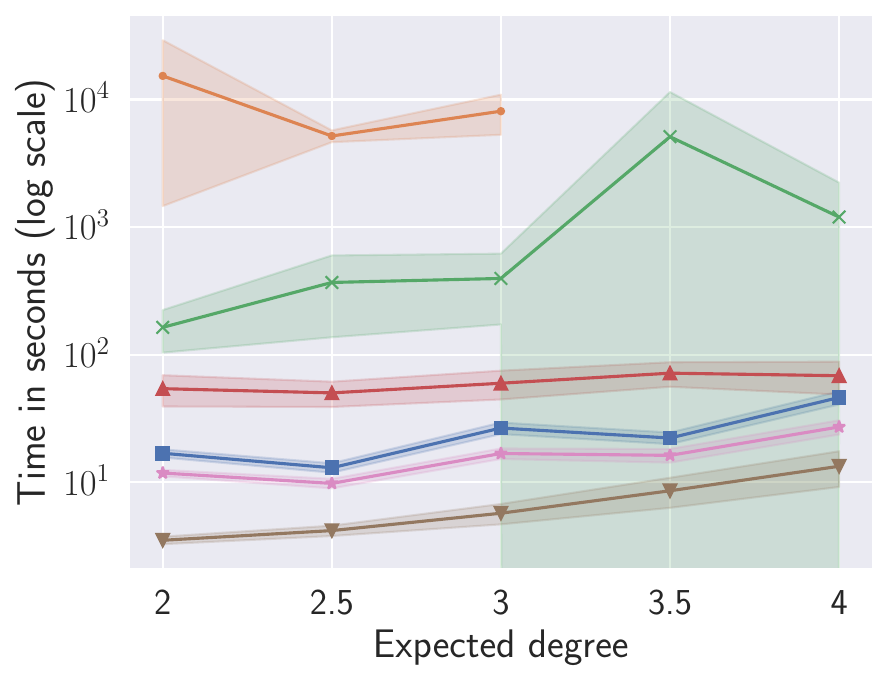}
            \caption*{Fisher-Z tests}
            \label{fig:time_per_degree_unrest_fshz_std}
        \end{subfigure}
        \begin{subfigure}[b]{0.24\linewidth}
            \includegraphics[width=\linewidth]{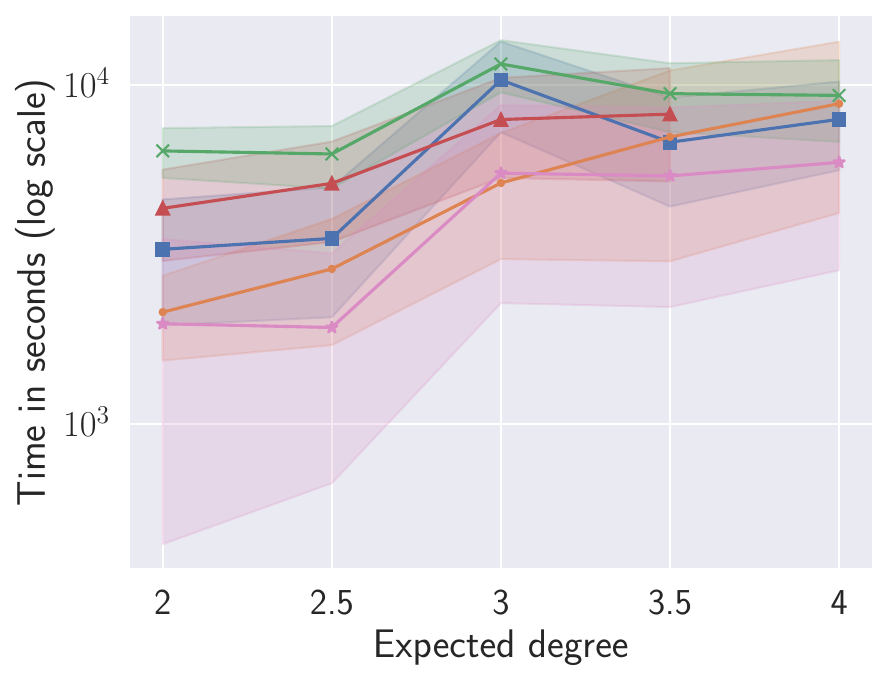}
            \caption*{KCI tests}
            \label{fig:time_per_degree_unrest_kci_std}
        \end{subfigure}
        \begin{subfigure}[b]{0.24\linewidth}
            \includegraphics[width=\linewidth]{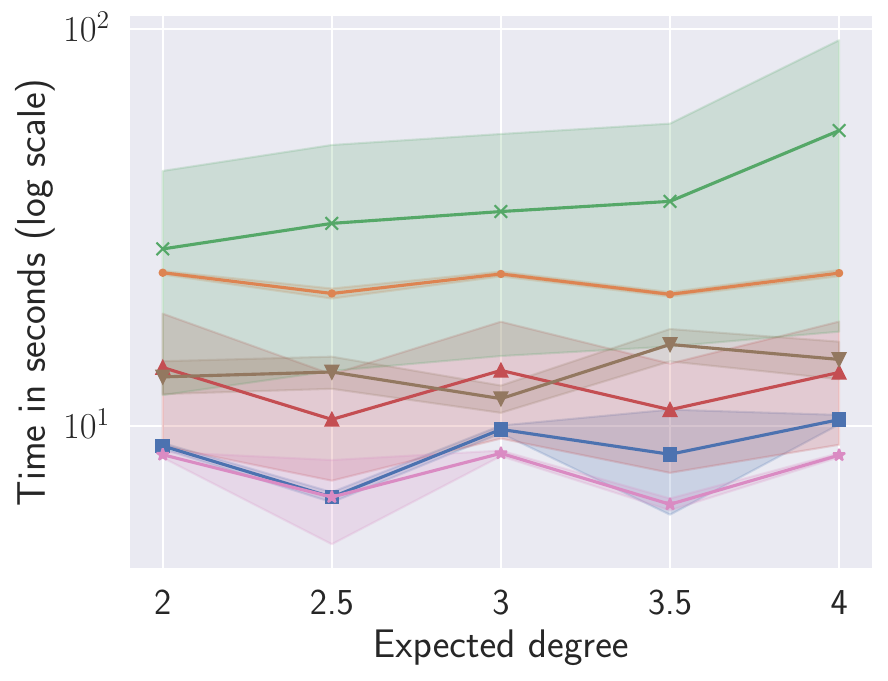}
            \caption*{$\chi^2$ tests}
            \label{fig:time_per_degree_unrest_chsq_std}
        \end{subfigure}
        \caption{Computation time.}
        \label{fig:time_per_degree_unrest_std}
    \end{subfigure}
    \begin{subfigure}[b]{\linewidth}
        \begin{subfigure}[b]{0.24\linewidth}
            \includegraphics[width=\linewidth]{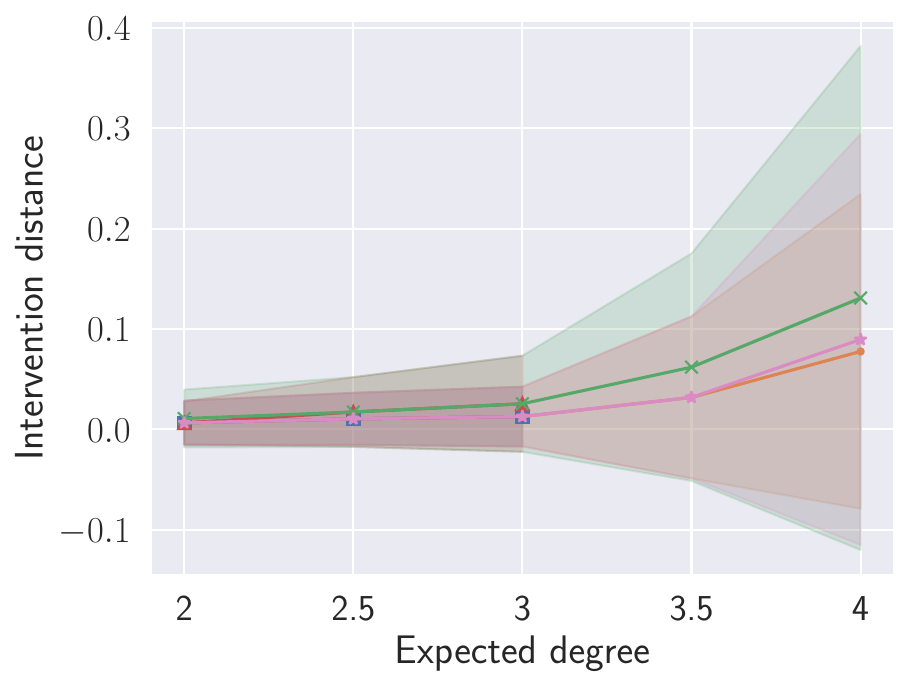}
            \caption*{d-separation tests}
            \label{fig:int_dist_per_degree_unrest_dsep_abs_std}
        \end{subfigure}
        \begin{subfigure}[b]{0.24\linewidth}
            \includegraphics[width=\linewidth]{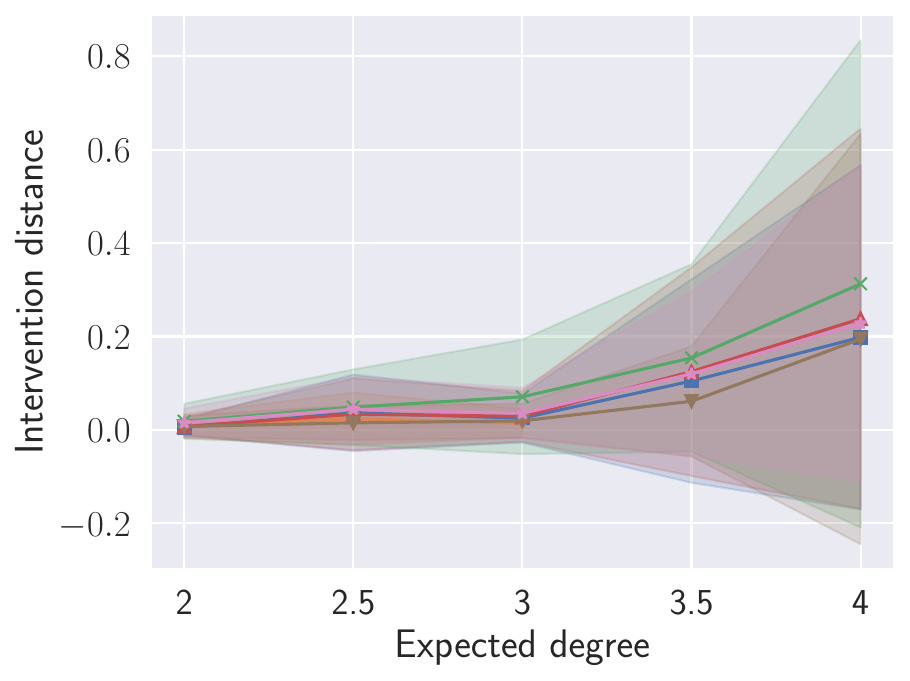}
            \caption*{Fisher-Z tests}
            \label{fig:int_dist_per_degree_unrest_fshz_abs_std}
        \end{subfigure}
        \begin{subfigure}[b]{0.24\linewidth}
            \includegraphics[width=\linewidth]{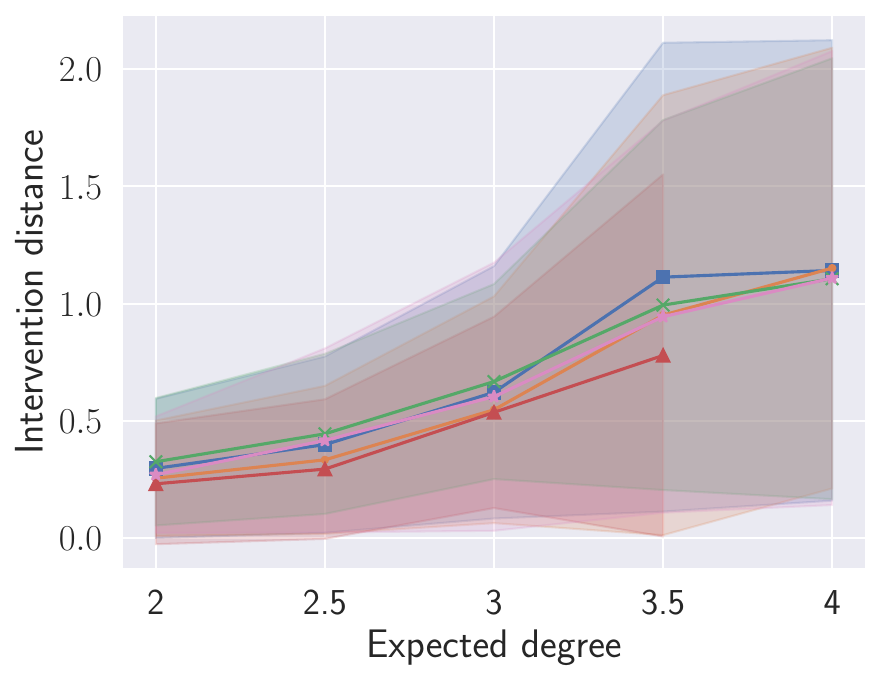}
            \caption*{KCI tests}
            \label{fig:int_dist_per_degree_unrest_kci_abs_std}
        \end{subfigure}
        \begin{subfigure}[b]{0.24\linewidth}
            \includegraphics[width=\linewidth]{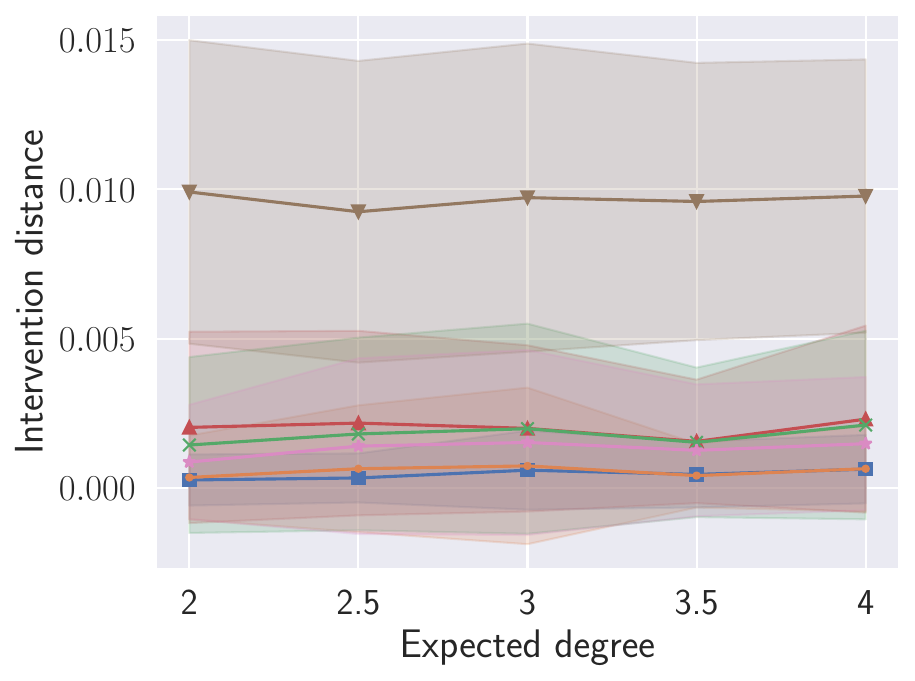}
            \caption*{$\chi^2$ tests}
            \label{fig:int_dist_per_degree_unrest_chsq_abs_std}
        \end{subfigure}
        \caption{Intervention distance.}
        \label{fig:int_dist_per_degree_unrest_std}
    \end{subfigure}
    \begin{subfigure}[b]{\linewidth}
        \begin{subfigure}[b]{0.24\linewidth}
            \includegraphics[width=\linewidth]{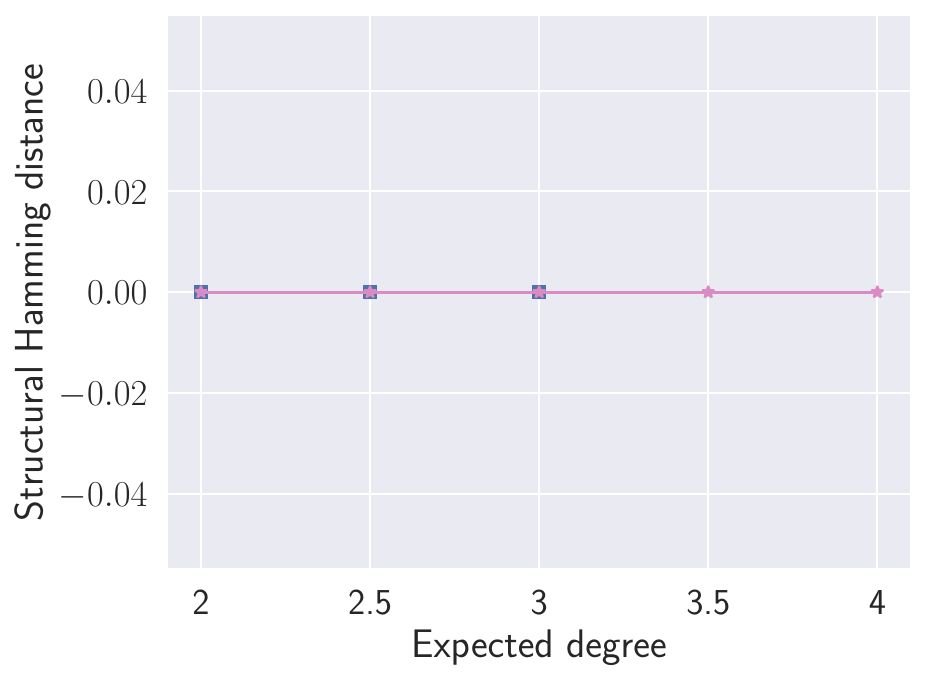}
            \caption*{d-separation tests}
            \label{fig:shd_per_degree_unrest_dsep_std}
        \end{subfigure}
        \begin{subfigure}[b]{0.24\linewidth}
            \includegraphics[width=\linewidth]{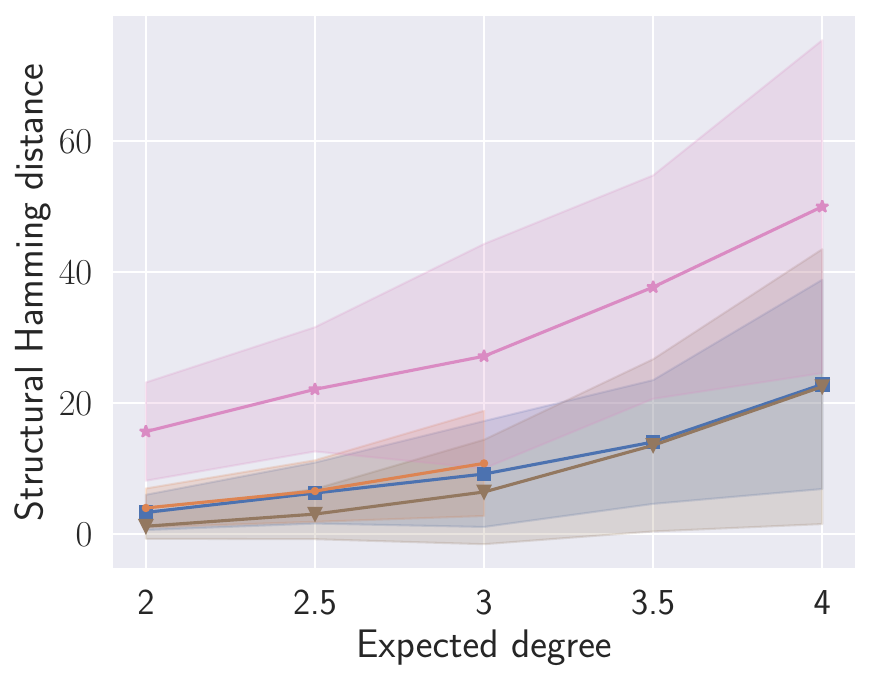}
            \caption*{Fisher-Z tests}
            \label{fig:shd_per_degree_unrest_fshz_std}
        \end{subfigure}
        \begin{subfigure}[b]{0.24\linewidth}
            \includegraphics[width=\linewidth]{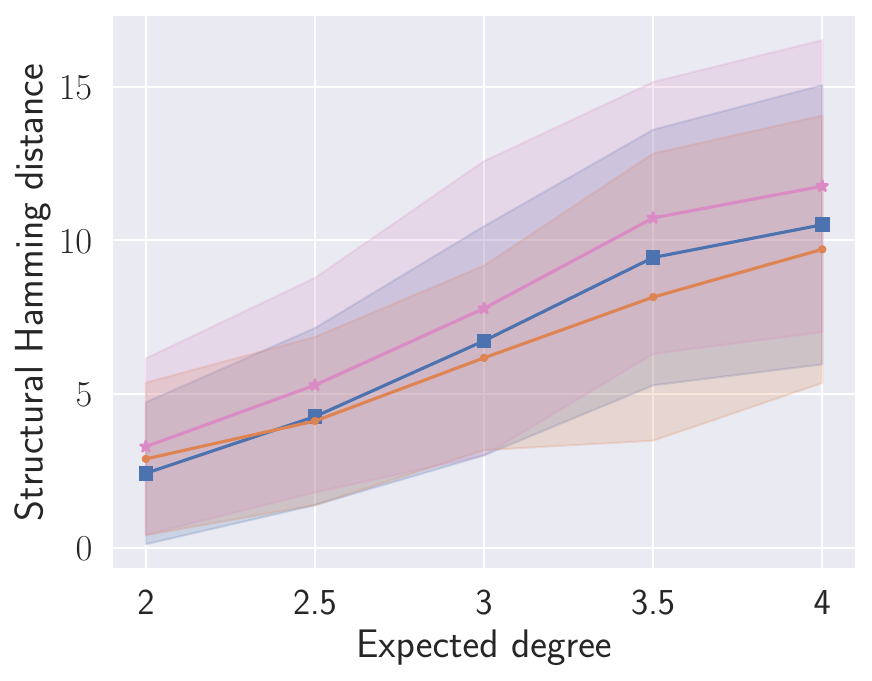}
            \caption*{KCI tests}
            \label{fig:shd_per_degree_unrest_kci_std}
        \end{subfigure}
        \begin{subfigure}[b]{0.24\linewidth}
            \includegraphics[width=\linewidth]{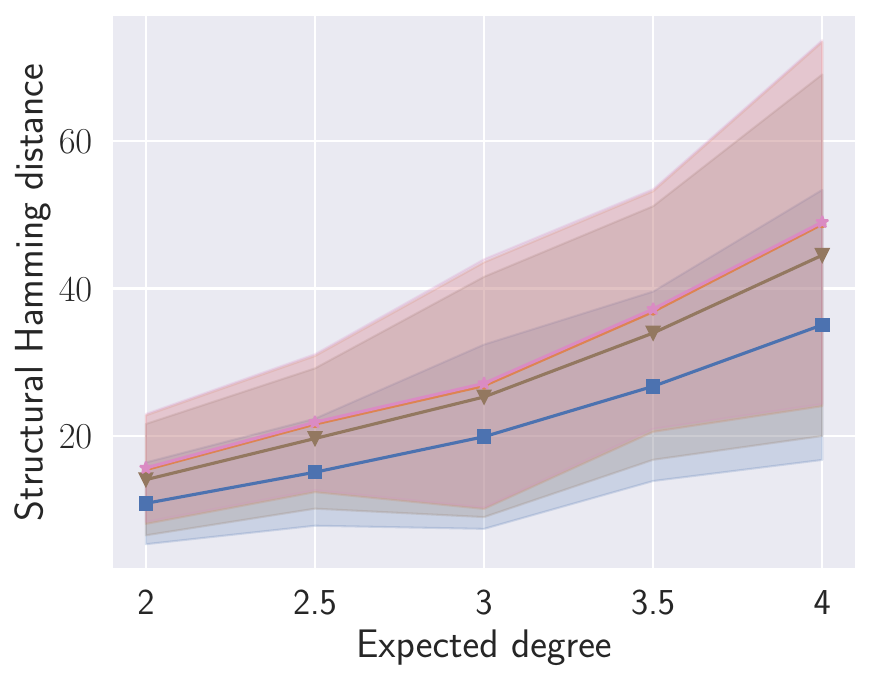}
            \caption*{$\chi^2$ tests}
            \label{fig:shd_per_degree_unrest_chsq_std}
        \end{subfigure}
        \caption{\Acf{SHD}.}
        \label{fig:shd_per_degree_unrest_std}
    \end{subfigure}
    \caption{Additional results over expected degree, with $n_{\mathbf{V}}=10$ for KCI tests and $n_{\mathbf{V}}=200$ otherwise, $n_{\mathbf{T}}=4, d_{\max}=10$ and $n_{\mathbf{D}} = 1000$ data-points. The shadow area denotes the range of the standard deviation. We compute the intervention distance in the d-separation tests case using random linear Gaussian data according to the discovered structure.}
    \label{fig:per_degree_std}
\end{figure}

\begin{figure}
    \centering
    \includegraphics[width=.6\linewidth]{experiments/legend_small.pdf}
    \begin{subfigure}[b]{\linewidth}
        \begin{subfigure}[b]{0.24\linewidth}
            \includegraphics[width=\linewidth]{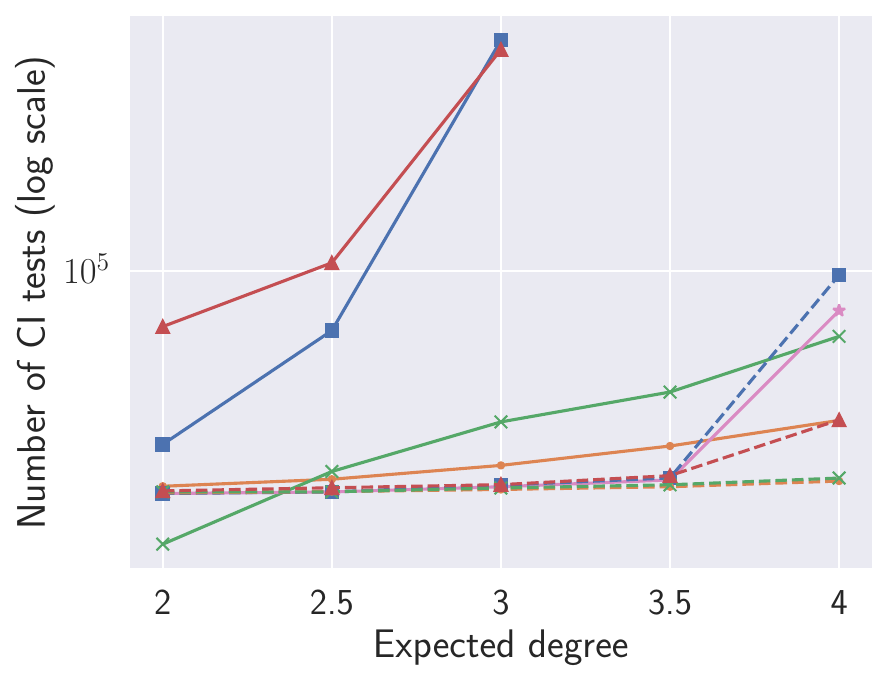}
            \caption*{d-separation tests}
            \label{fig:test_per_degree_unrest_dsep}
        \end{subfigure}
        \begin{subfigure}[b]{0.24\linewidth}
            \includegraphics[width=\linewidth]{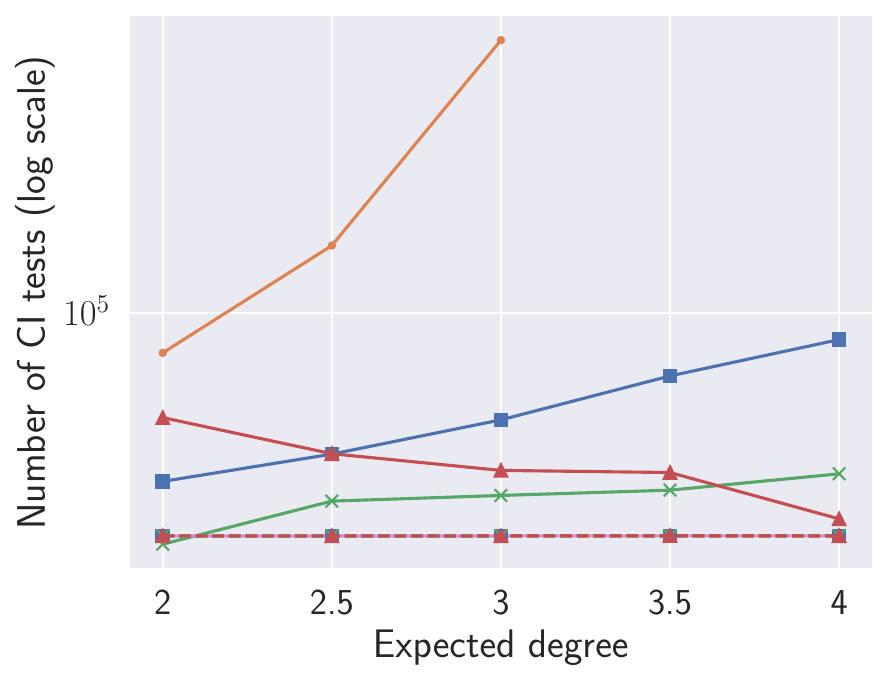}
            \caption*{Fisher-Z tests}
            \label{fig:test_per_degree_unrest_fshz}
        \end{subfigure}
        \begin{subfigure}[b]{0.24\linewidth}
            \includegraphics[width=\linewidth]{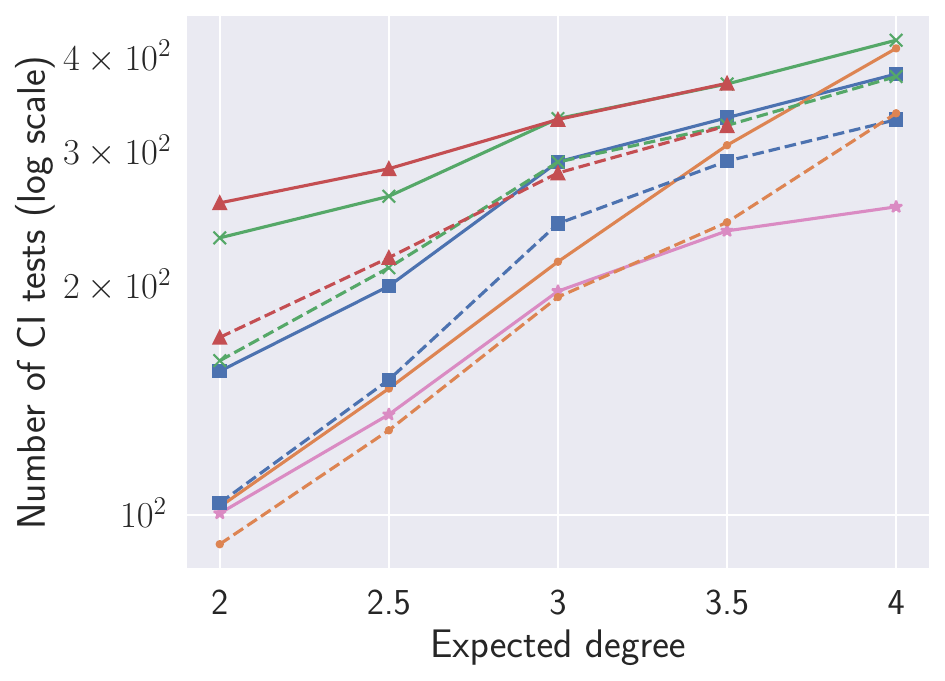}
            \caption*{KCI tests}
            \label{fig:test_per_degree_unrest_kci}
        \end{subfigure}
        \begin{subfigure}[b]{0.24\linewidth}
            \includegraphics[width=\linewidth]{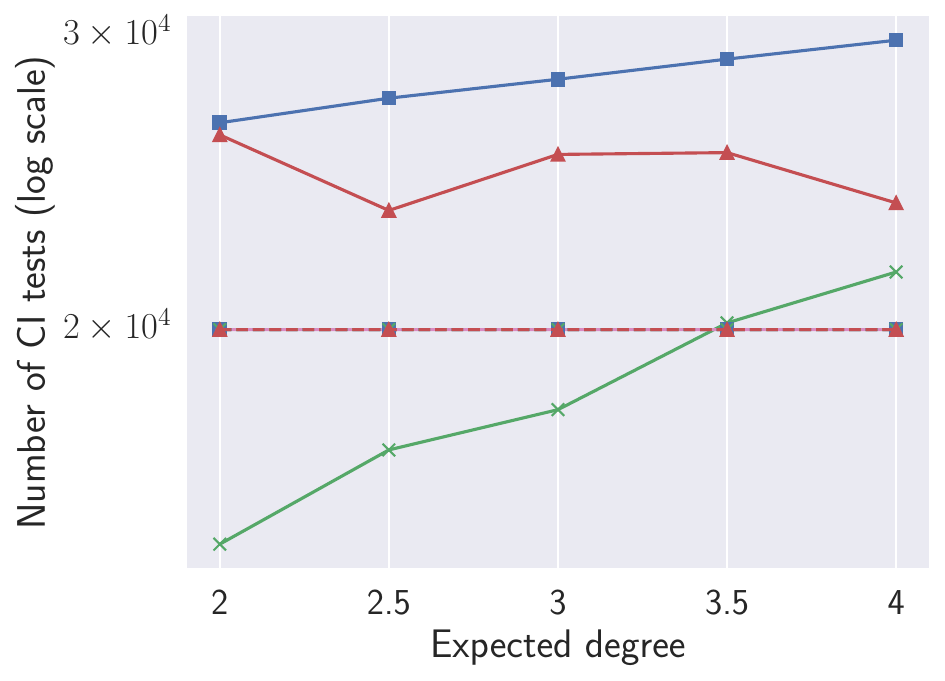}
            \caption*{$\chi^2$ tests}
            \label{fig:test_per_degree_unrest_chsq}
        \end{subfigure}
        \caption{Number of \ac{CI} tests.}
        \label{fig:test_per_degree_unrest}
    \end{subfigure}
    \begin{subfigure}[b]{\linewidth}
        \begin{subfigure}[b]{0.24\linewidth}
            \includegraphics[width=\linewidth]{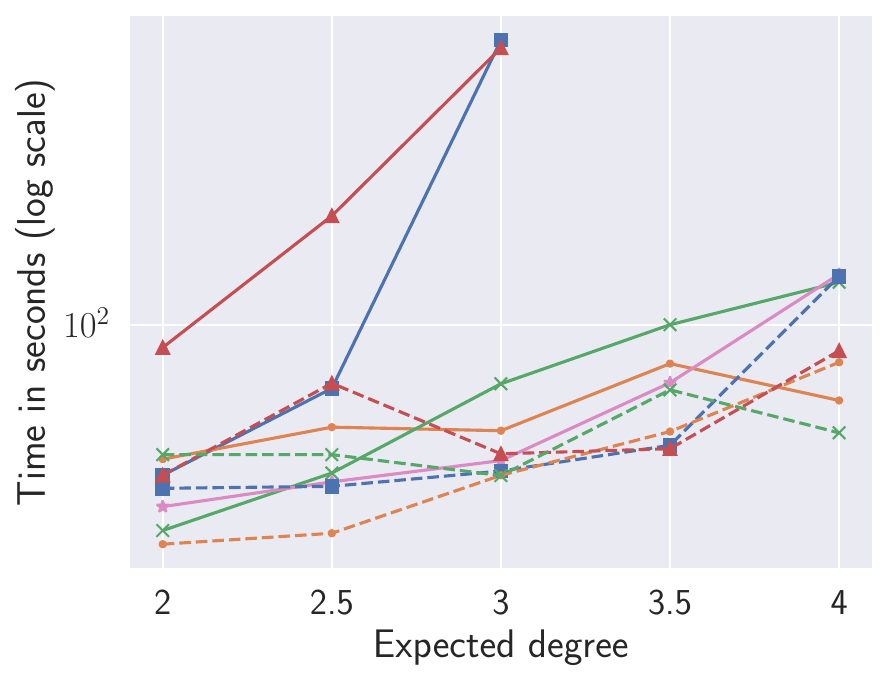}
            \caption*{d-separation tests}
            \label{fig:time_per_degree_unrest_dsep}
        \end{subfigure}
        \begin{subfigure}[b]{0.24\linewidth}
            \includegraphics[width=\linewidth]{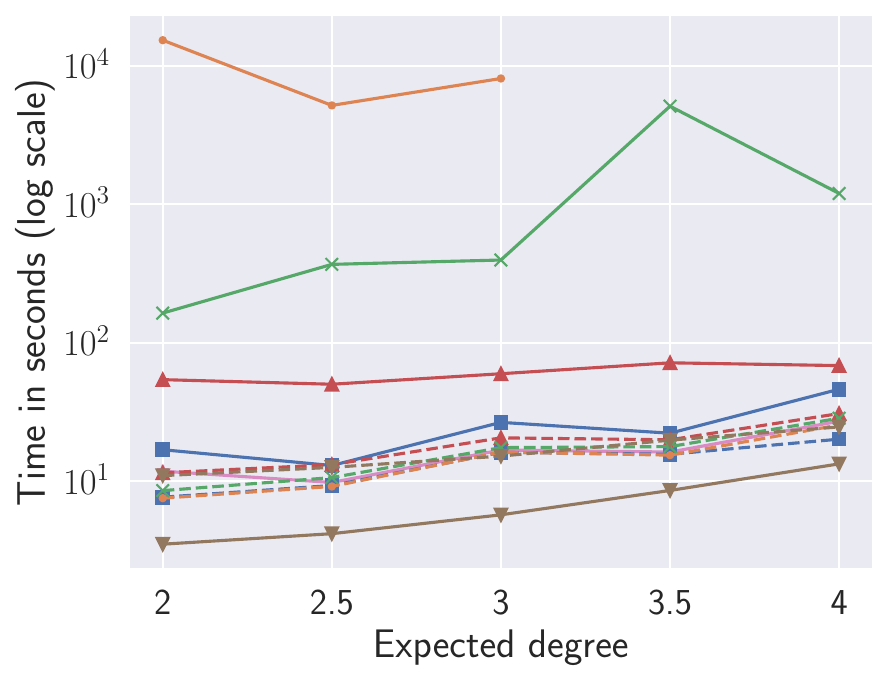}
            \caption*{Fisher-Z tests}
            \label{fig:time_per_degree_unrest_fshz}
        \end{subfigure}
        \begin{subfigure}[b]{0.24\linewidth}
            \includegraphics[width=\linewidth]{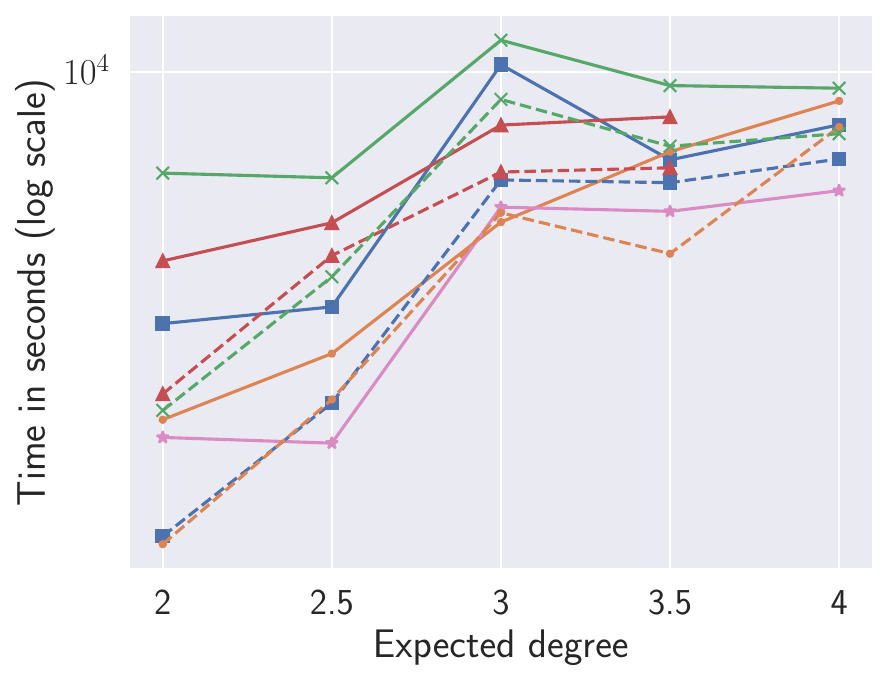}
            \caption*{KCI tests}
            \label{fig:time_per_degree_unrest_kci}
        \end{subfigure}
        \begin{subfigure}[b]{0.24\linewidth}
            \includegraphics[width=\linewidth]{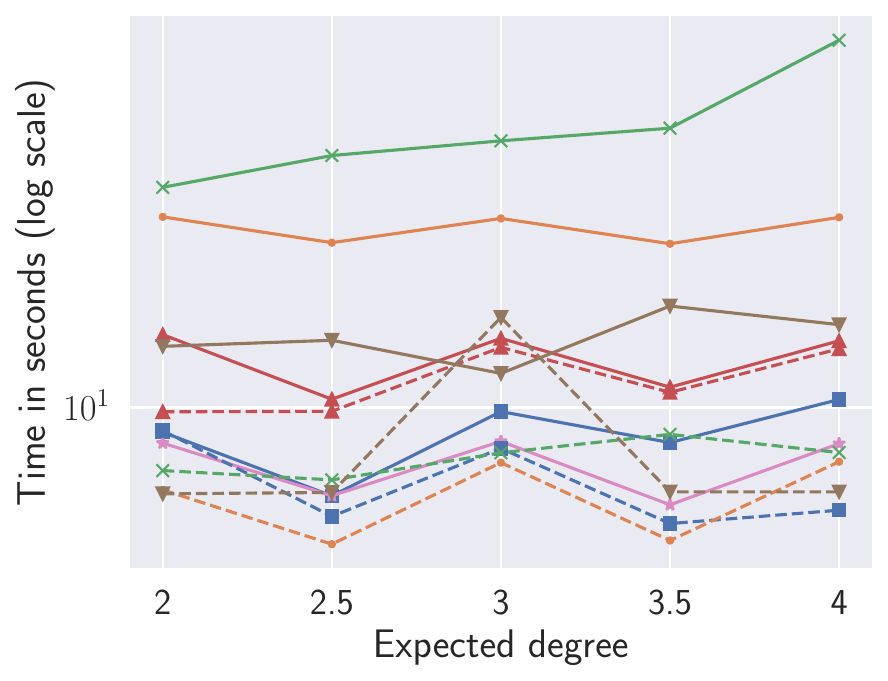}
            \caption*{$\chi^2$ tests}
            \label{fig:time_per_degree_unrest_chsq}
        \end{subfigure}
        \caption{Computation time.}
        \label{fig:time_per_degree_unrest}
    \end{subfigure}
    \begin{subfigure}[b]{\linewidth}
        \begin{subfigure}[b]{0.24\linewidth}
            \includegraphics[width=\linewidth]{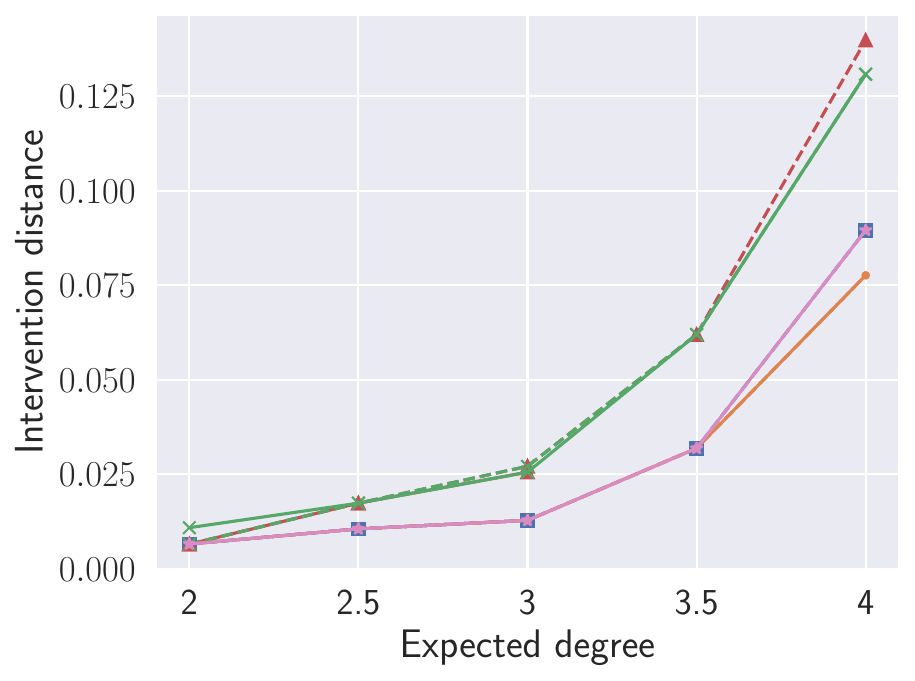}
            \caption*{d-separation tests}
            \label{fig:int_dist_per_degree_unrest_dsep_abs}
        \end{subfigure}
        \begin{subfigure}[b]{0.24\linewidth}
            \includegraphics[width=\linewidth]{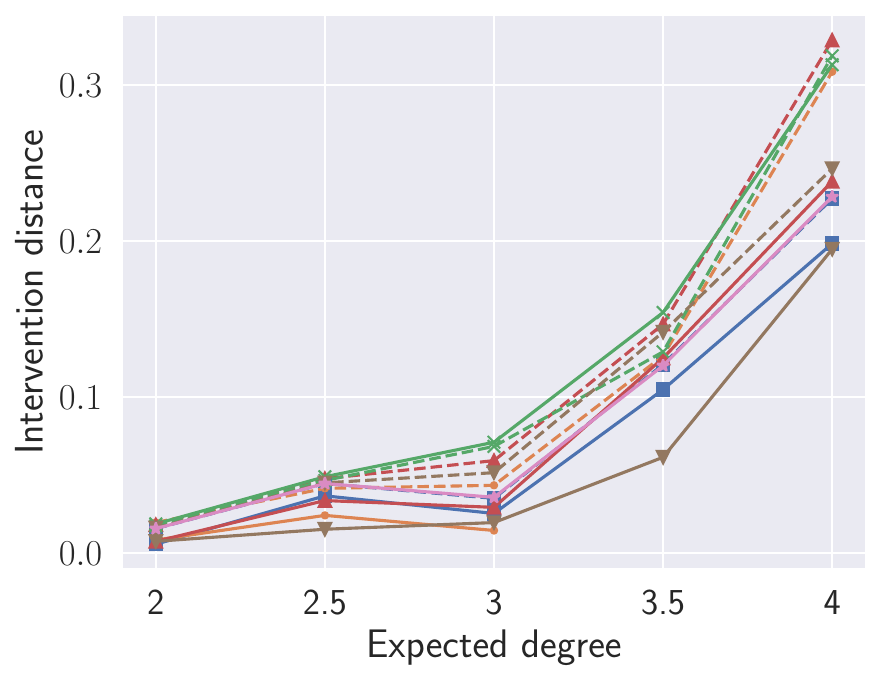}
            \caption*{Fisher-Z tests}
            \label{fig:int_dist_per_degree_unrest_fshz_abs}
        \end{subfigure}
        \begin{subfigure}[b]{0.24\linewidth}
            \includegraphics[width=\linewidth]{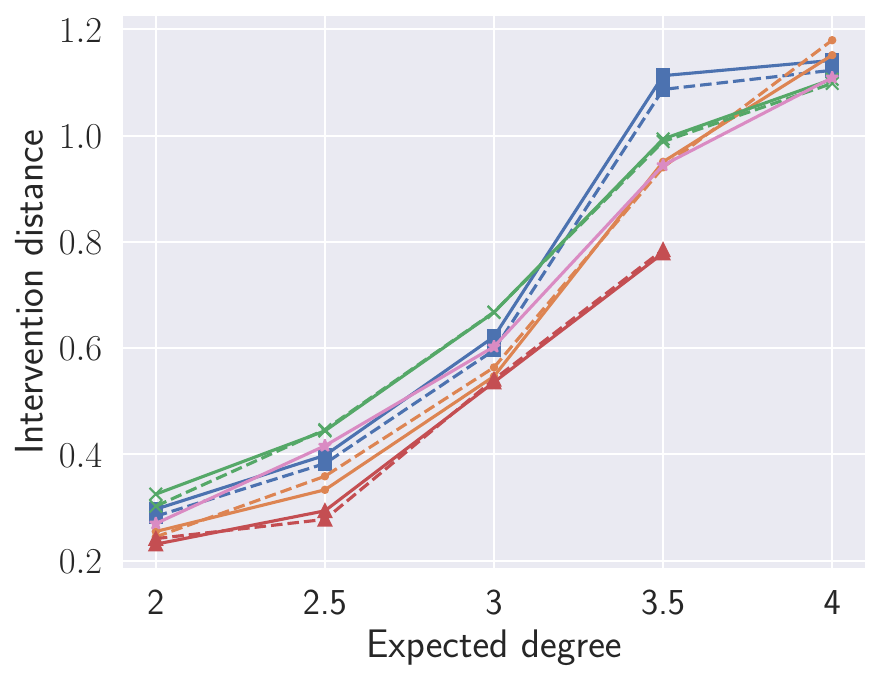}
            \caption*{KCI tests}
            \label{fig:int_dist_per_degree_unrest_kci_abs}
        \end{subfigure}
        \begin{subfigure}[b]{0.24\linewidth}
            \includegraphics[width=\linewidth]{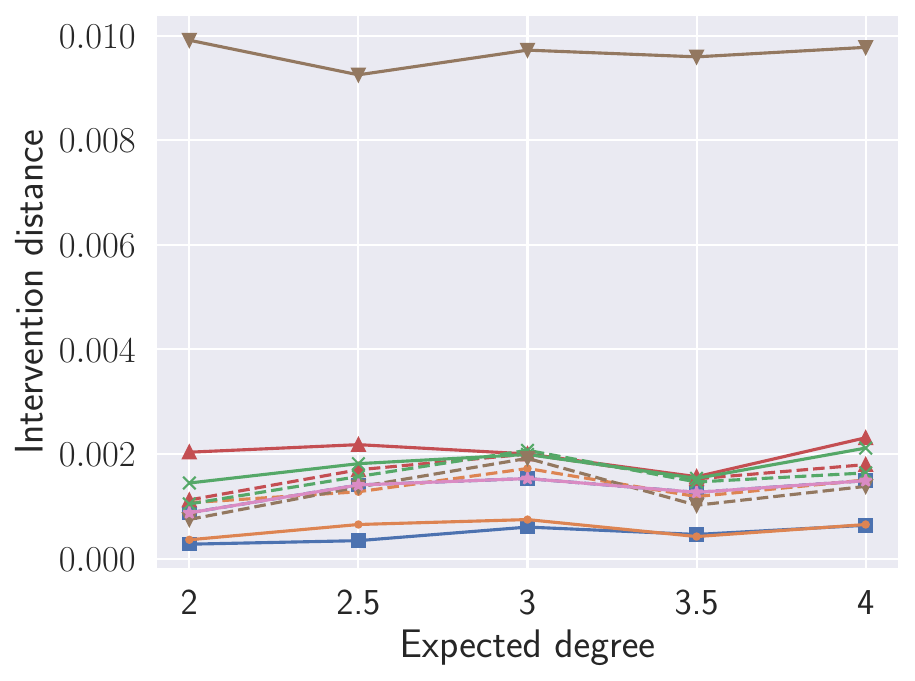}
            \caption*{$\chi^2$ tests}
            \label{fig:int_dist_per_degree_unrest_chsq_abs}
        \end{subfigure}
        \caption{Intervention distance.}
        \label{fig:int_dist_per_degree_unrest}
    \end{subfigure}
    \begin{subfigure}[b]{\linewidth}
        \begin{subfigure}[b]{0.24\linewidth}
            \includegraphics[width=\linewidth]{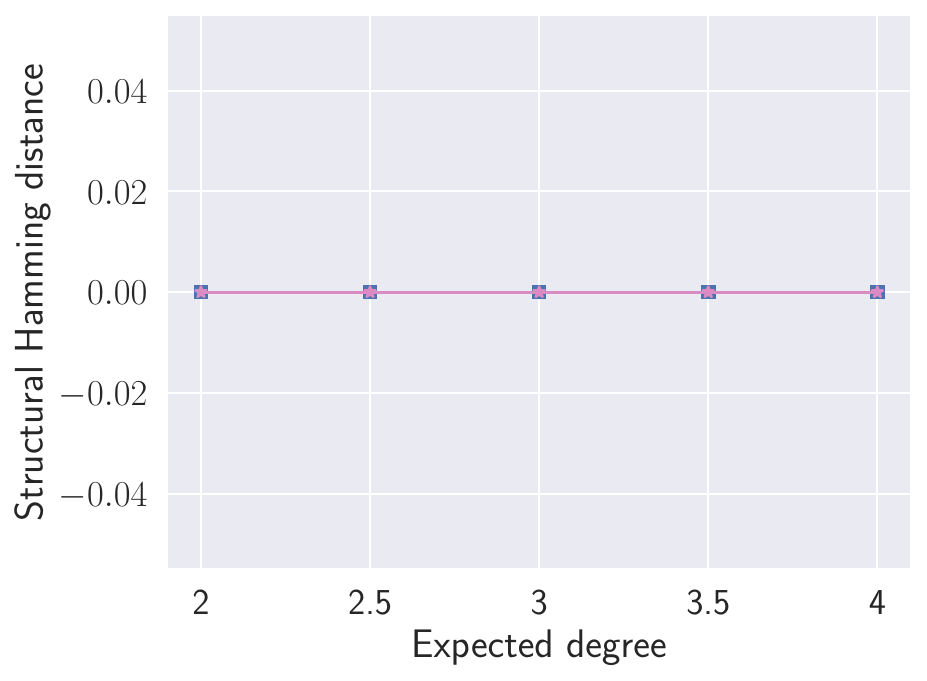}
            \caption*{d-separation tests}
            \label{fig:shd_per_degree_unrest_dsep}
        \end{subfigure}
        \begin{subfigure}[b]{0.24\linewidth}
            \includegraphics[width=\linewidth]{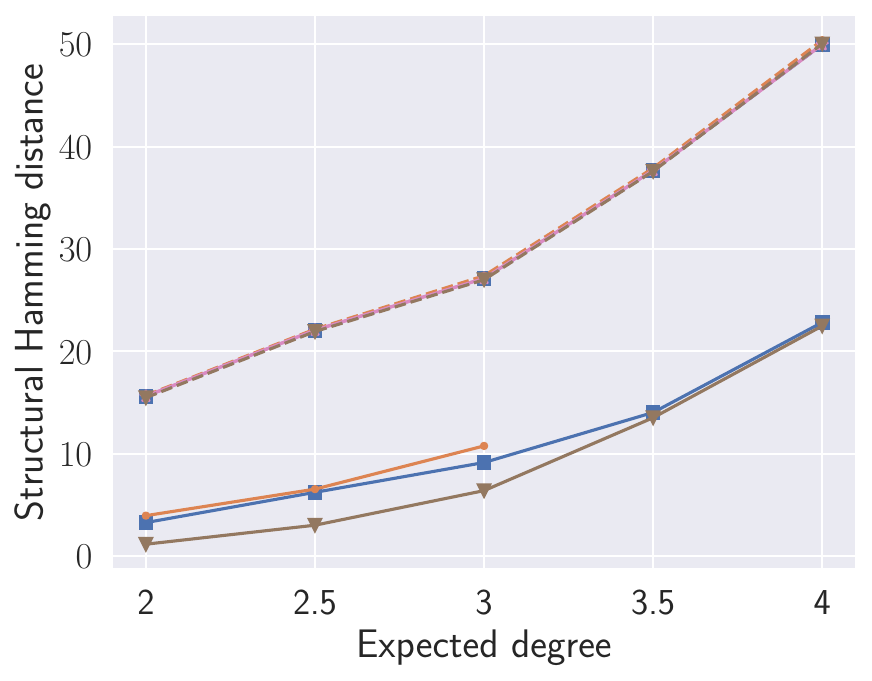}
            \caption*{Fisher-Z tests}
            \label{fig:shd_per_degree_unrest_fshz}
        \end{subfigure}
        \begin{subfigure}[b]{0.24\linewidth}
            \includegraphics[width=\linewidth]{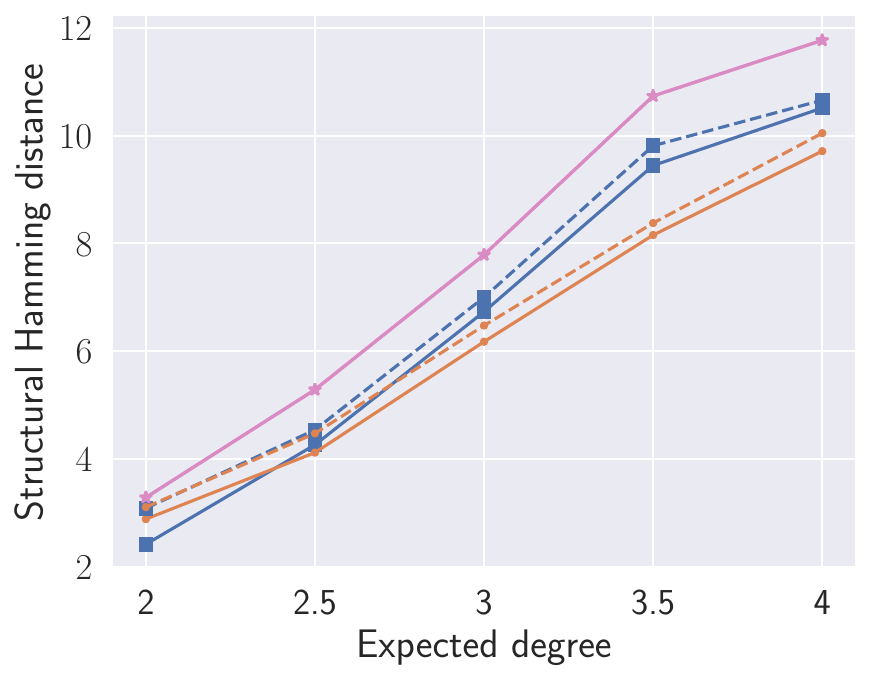}
            \caption*{KCI tests}
            \label{fig:shd_per_degree_unrest_kci}
        \end{subfigure}
        \begin{subfigure}[b]{0.24\linewidth}
            \includegraphics[width=\linewidth]{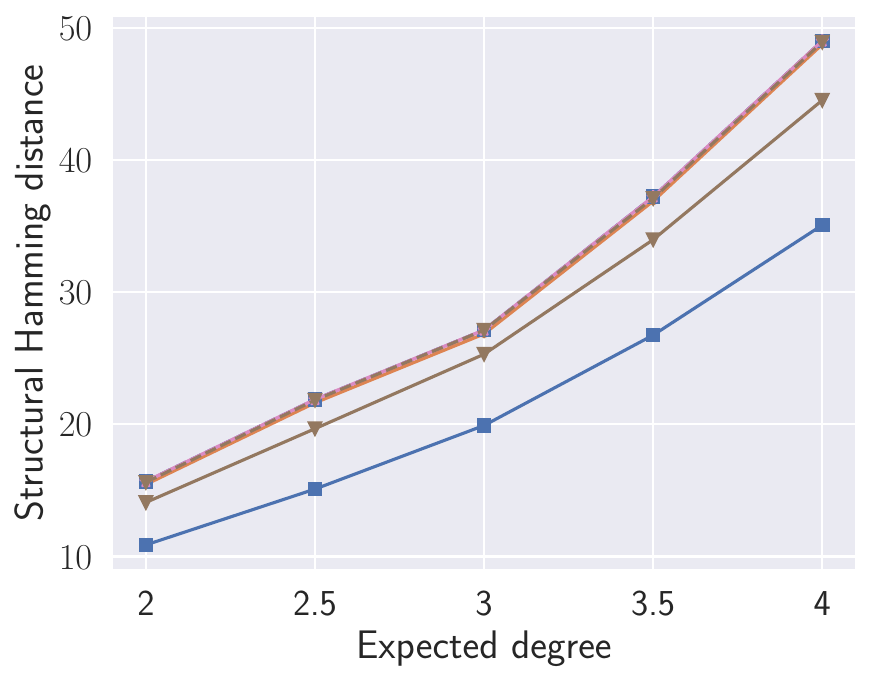}
            \caption*{$\chi^2$ tests}
            \label{fig:shd_per_degree_unrest_chsq}
        \end{subfigure}
        \caption{\Acf{SHD}.}
        \label{fig:shd_per_degree_unrest}
    \end{subfigure}
    \caption{Additional results for baseline methods combined with SNAP$(0)$ over expected degree, with $n_{\mathbf{V}}=10$ for KCI tests and $n_{\mathbf{V}}=200$ otherwise, $n_{\mathbf{T}}=4, d_{\max}=10$ and $n_{\mathbf{D}} = 1000$ data-points. We compute the intervention distance in the d-separation tests case using random linear Gaussian data according to the discovered structure.}
    \label{fig:per_degree}
\end{figure}

\subsection{Various numbers of samples}
\label{app:over_samples}
Figures~\ref{fig:per_samples_std} and \ref{fig:per_samples} show how different metrics change over different number of data samples.
More samples slightly increase the number of \ac{CI} tests and execution time for all methods.
For KCI tests, the execution time increases drastically instead.
Surprisingly, for $\chi^2$ tests, execution time increases considerably for MARVEL and MB-by-MB, but not for other methods.
While \ac{SHD} decreases for global methods with more samples, intervention distance seems mostly unaffected.
This suggests that most structural improvements are irrelevant for causal effect estimation.

\begin{figure}
    \centering
    \includegraphics[width=.8\linewidth]{experiments/legend_small.pdf}
    \begin{subfigure}[b]{\linewidth}
        \centering
        \begin{subfigure}[b]{0.28\linewidth}
            \includegraphics[width=\linewidth]{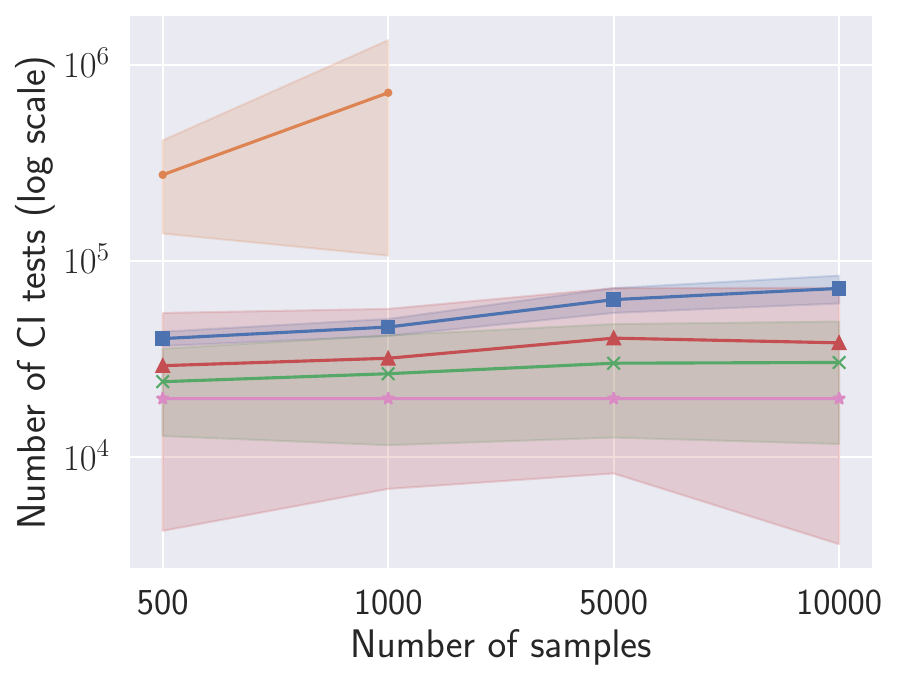}
            \caption*{Fisher-Z tests}
            \label{fig:test_per_samples_unrest_fshz_std}
        \end{subfigure}
        \begin{subfigure}[b]{0.28\linewidth}
            \includegraphics[width=\linewidth]{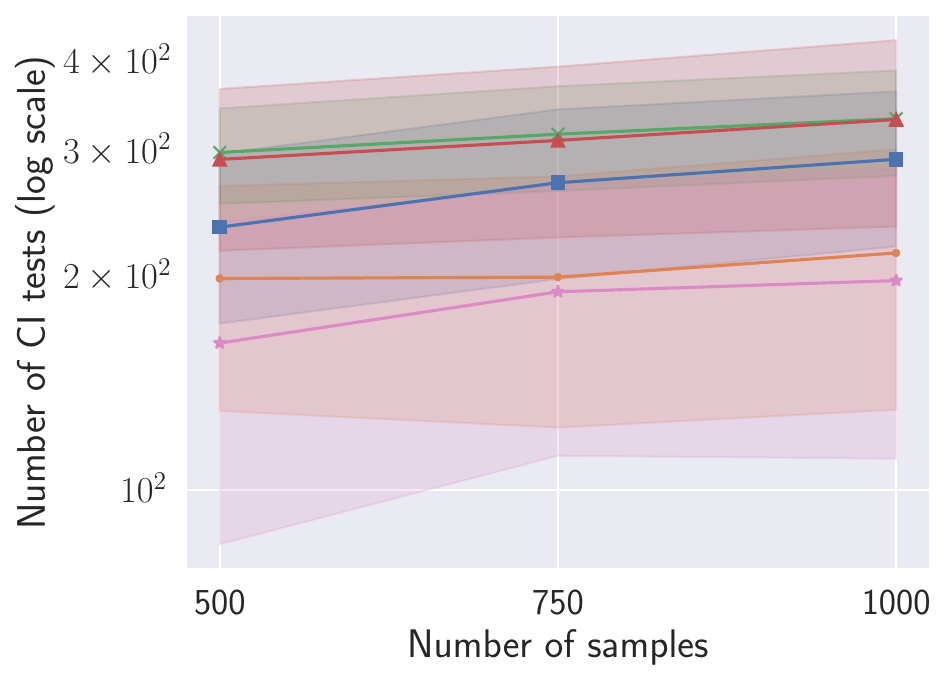}
            \caption*{KCI tests}
            \label{fig:test_per_samples_unrest_kci_std}
        \end{subfigure}
        \begin{subfigure}[b]{0.28\linewidth}
            \includegraphics[width=\linewidth]{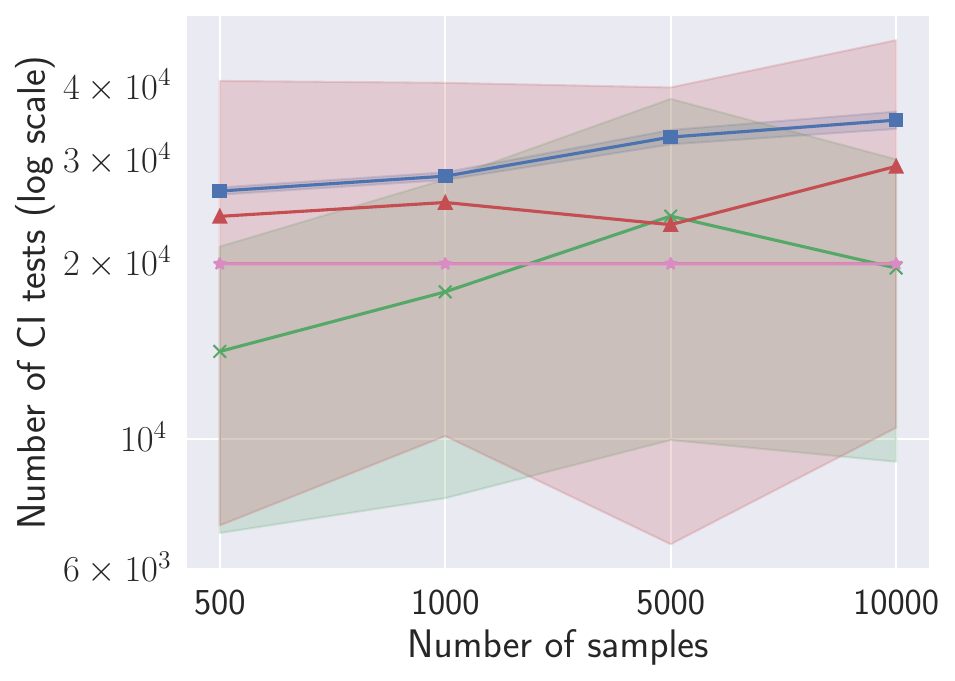}
            \caption*{$\chi^2$ tests}
            \label{fig:test_per_samples_unrest_chsq_std}
        \end{subfigure}
        \caption{Number of \ac{CI} tests.}
        \label{fig:test_per_samples_unrest_std}
    \end{subfigure}
    \begin{subfigure}[b]{\linewidth}
        \centering
        \begin{subfigure}[b]{0.28\linewidth}
            \includegraphics[width=\linewidth]{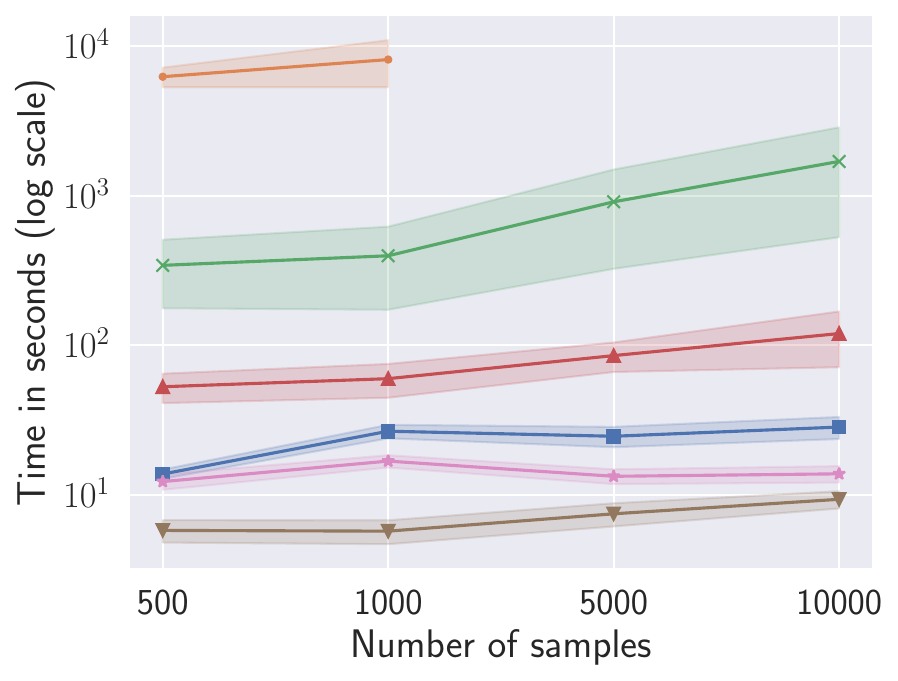}
            \caption*{Fisher-Z tests}
            \label{fig:time_per_samples_unrest_fshz_std}
        \end{subfigure}
        \begin{subfigure}[b]{0.28\linewidth}
            \includegraphics[width=\linewidth]{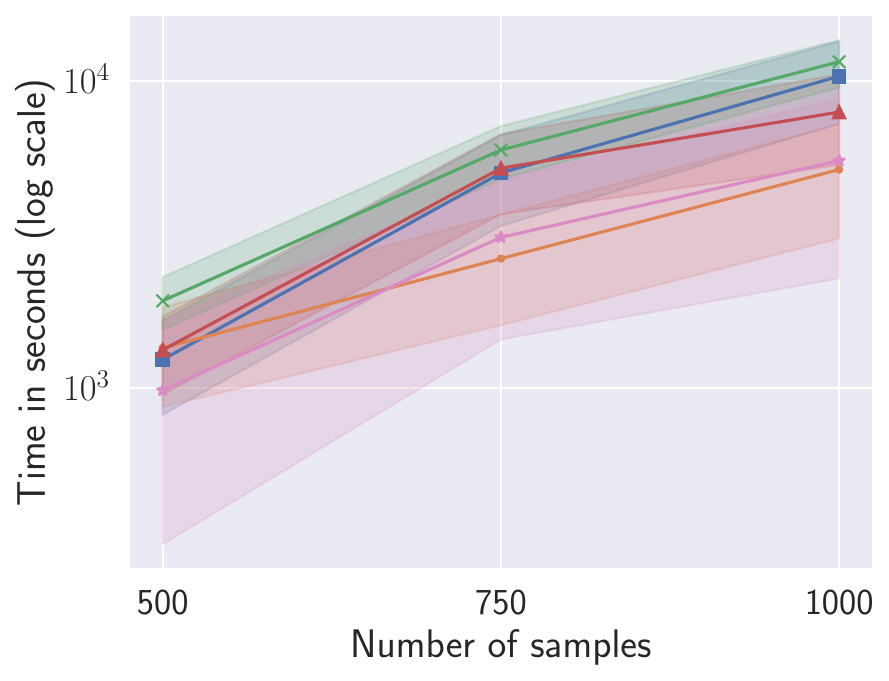}
            \caption*{KCI tests}
            \label{fig:time_per_samples_unrest_kci_std}
        \end{subfigure}
        \begin{subfigure}[b]{0.28\linewidth}
            \includegraphics[width=\linewidth]{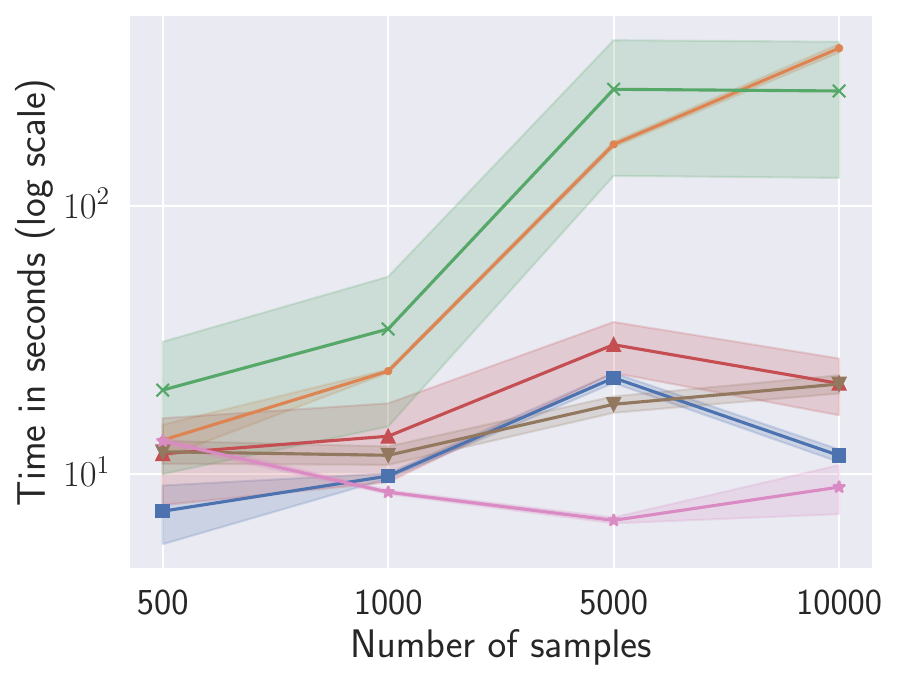}
            \caption*{$\chi^2$ tests}
            \label{fig:time_per_samples_unrest_chsq_std}
        \end{subfigure}
        \caption{Computation time.}
        \label{fig:time_per_samples_unrest_std}
    \end{subfigure}
    \begin{subfigure}[b]{\linewidth}
        \centering
        \begin{subfigure}[b]{0.28\linewidth}
            \includegraphics[width=\linewidth]{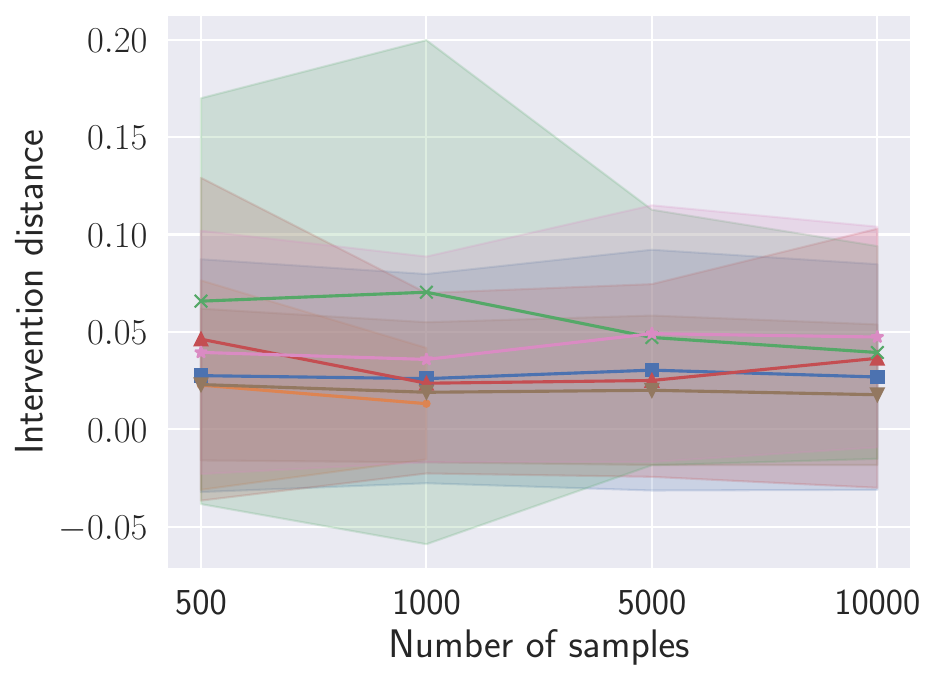}
            \caption*{Fisher-Z tests}
            \label{fig:int_dist_per_samples_unrest_fshz_abs_std}
        \end{subfigure}
        \begin{subfigure}[b]{0.28\linewidth}
            \includegraphics[width=\linewidth]{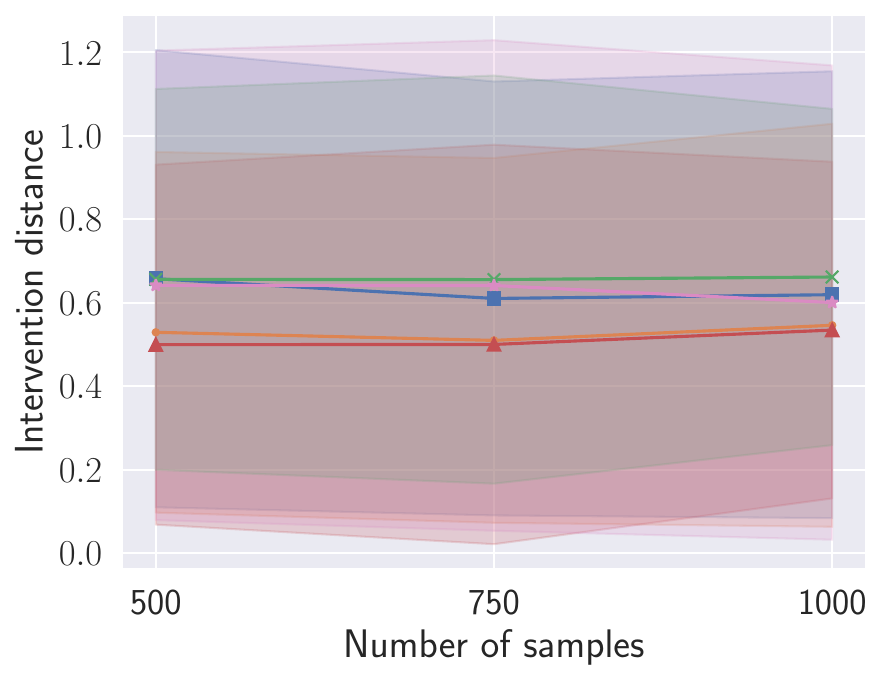}
            \caption*{KCI tests}
            \label{fig:int_dist_per_samples_unrest_kci_abs_std}
        \end{subfigure}
        \begin{subfigure}[b]{0.28\linewidth}
            \includegraphics[width=\linewidth]{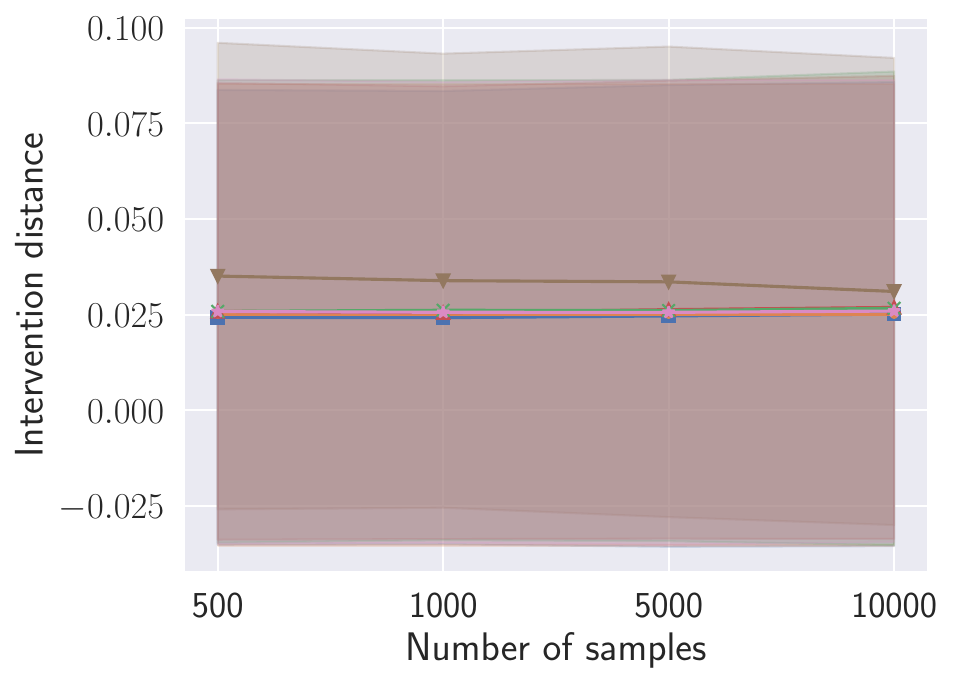}
            \caption*{$\chi^2$ tests}
            \label{fig:int_dist_per_samples_unrest_chsq_abs_std}
        \end{subfigure}
        \caption{Intervention distance.}
    \end{subfigure}
    \begin{subfigure}[b]{\linewidth}
        \centering
        \begin{subfigure}[b]{0.28\linewidth}
            \includegraphics[width=\linewidth]{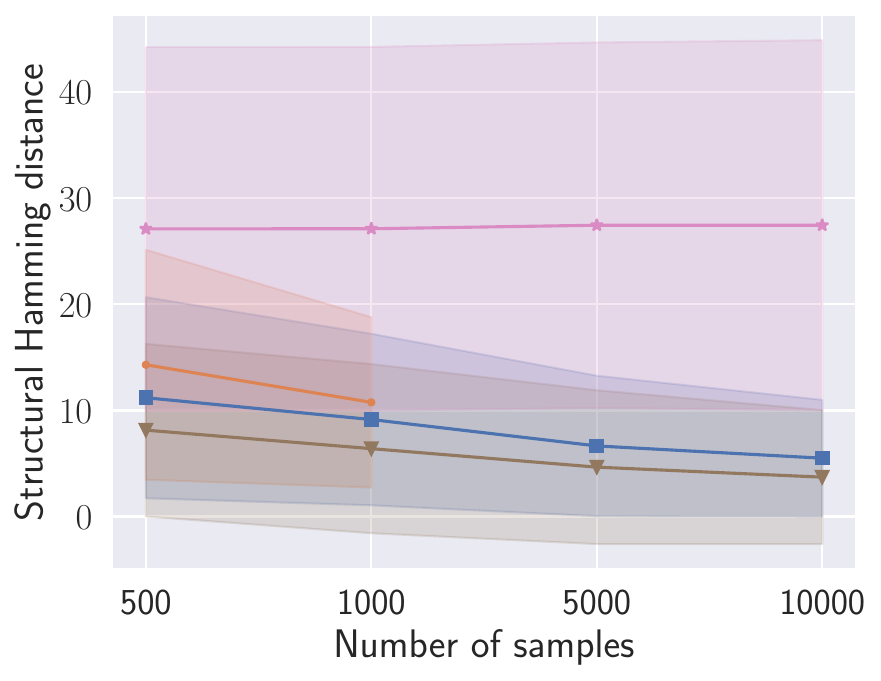}
            \caption*{Fisher-Z tests}
            \label{fig:shd_per_samples_unrest_fshz_std}
        \end{subfigure}
        \begin{subfigure}[b]{0.28\linewidth}
            \includegraphics[width=\linewidth]{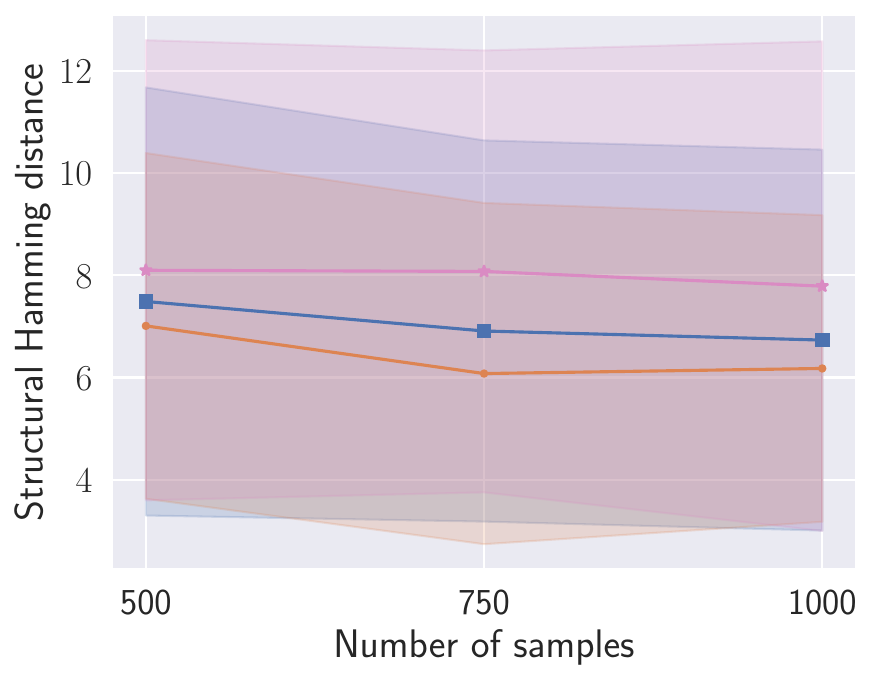}
            \caption*{KCI tests}
            \label{fig:shd_per_samples_unrest_kci_std}
        \end{subfigure}
        \begin{subfigure}[b]{0.28\linewidth}
            \includegraphics[width=\linewidth]{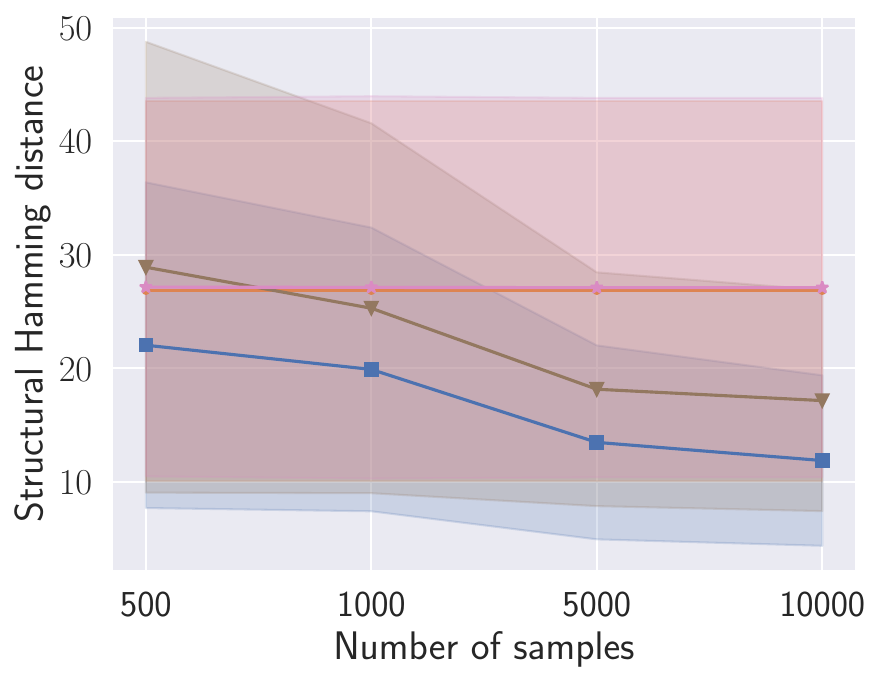}
            \caption*{$\chi^2$ tests}
            \label{fig:shd_per_samples_unrest_chsq_std}
        \end{subfigure}
        \caption{\Acf{SHD}.}
        \label{fig:shd_per_samples_unrest_std}
    \end{subfigure}
    \caption{Additional results over number of samples, with $n_{\mathbf{V}}=10$ for KCI tests and $n_{\mathbf{V}}=200$ otherwise, $n_{\mathbf{T}}=4, \overline{d} = 3$ and $d_{\max}=10$.
    The shadow area denotes the range of the standard deviation.
    We compute the intervention distance in the d-separation tests case using random linear Gaussian data according to the discovered structure.
    Results for MARVEL with Fisher-Z tests are not shown for 5000 and 10000 samples because the runs did not finish in two days.}
    \label{fig:per_samples_std}
\end{figure}

\begin{figure}
    \centering
    \includegraphics[width=.8\linewidth]{experiments/legend_small.pdf}
    \begin{subfigure}[b]{\linewidth}
        \centering
        \begin{subfigure}[b]{0.3\linewidth}
            \includegraphics[width=\linewidth]{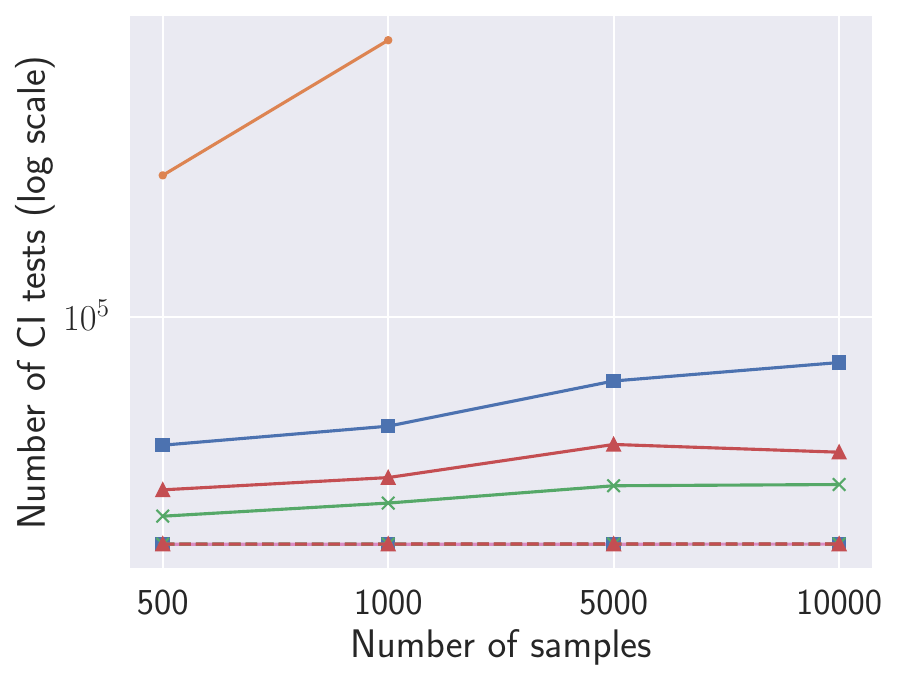}
            \caption*{Fisher-Z tests}
            \label{fig:test_per_samples_unrest_fshz}
        \end{subfigure}
        \begin{subfigure}[b]{0.3\linewidth}
            \includegraphics[width=\linewidth]{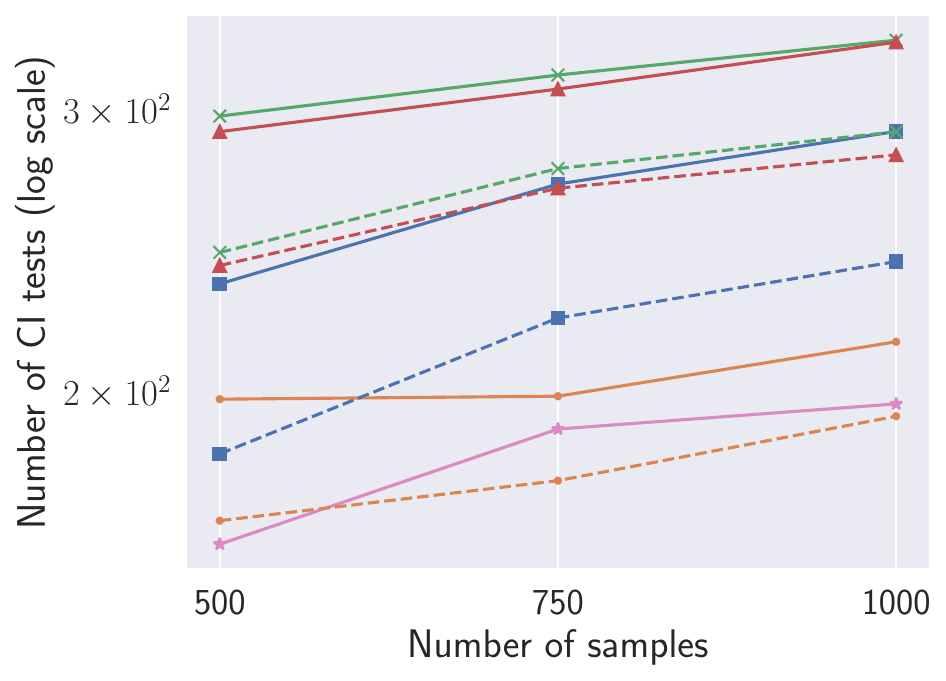}
            \caption*{KCI tests}
            \label{fig:test_per_samples_unrest_kci}
        \end{subfigure}
        \begin{subfigure}[b]{0.3\linewidth}
            \includegraphics[width=\linewidth]{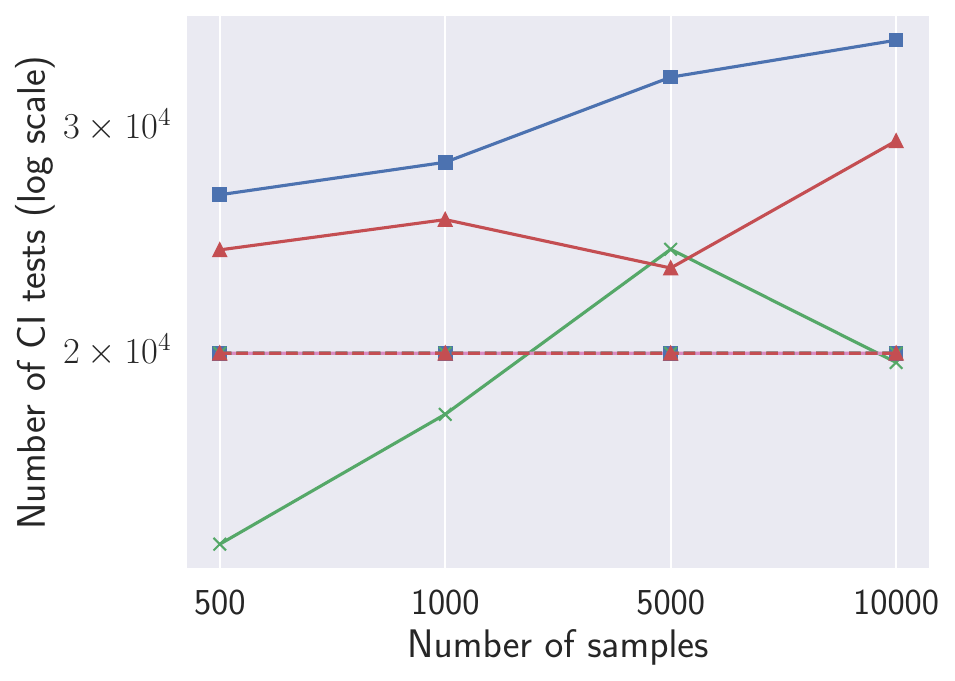}
            \caption*{$\chi^2$ tests}
            \label{fig:test_per_samples_unrest_chsq}
        \end{subfigure}
        \caption{Number of \ac{CI} tests.}
        \label{fig:test_per_samples_unrest}
    \end{subfigure}
    \begin{subfigure}[b]{\linewidth}
        \centering
        \begin{subfigure}[b]{0.3\linewidth}
            \includegraphics[width=\linewidth]{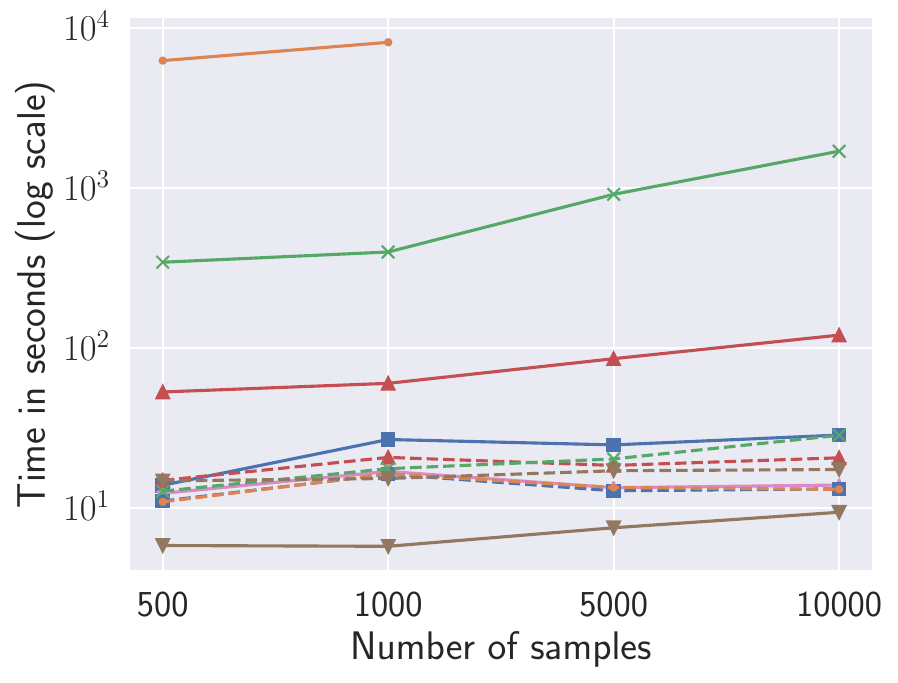}
            \caption*{Fisher-Z tests}
            \label{fig:time_per_samples_unrest_fshz}
        \end{subfigure}
        \begin{subfigure}[b]{0.3\linewidth}
            \includegraphics[width=\linewidth]{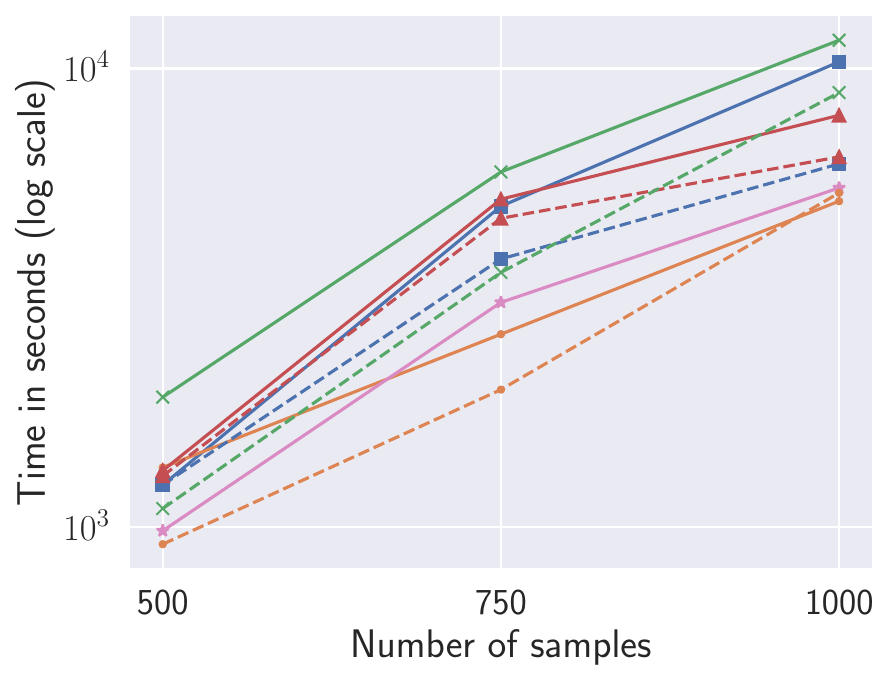}
            \caption*{KCI tests}
            \label{fig:time_per_samples_unrest_kci}
        \end{subfigure}
        \begin{subfigure}[b]{0.3\linewidth}
            \includegraphics[width=\linewidth]{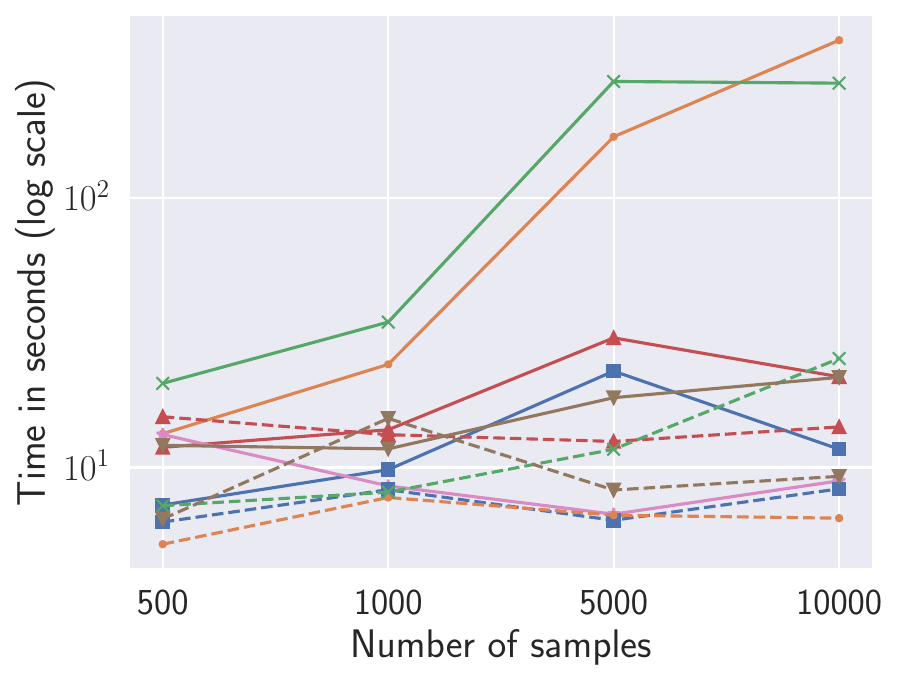}
            \caption*{$\chi^2$ tests}
            \label{fig:time_per_samples_unrest_chsq}
        \end{subfigure}
        \caption{Computation time.}
        \label{fig:time_per_samples_unrest}
    \end{subfigure}
    \begin{subfigure}[b]{\linewidth}
        \centering
        \begin{subfigure}[b]{0.3\linewidth}
            \includegraphics[width=\linewidth]{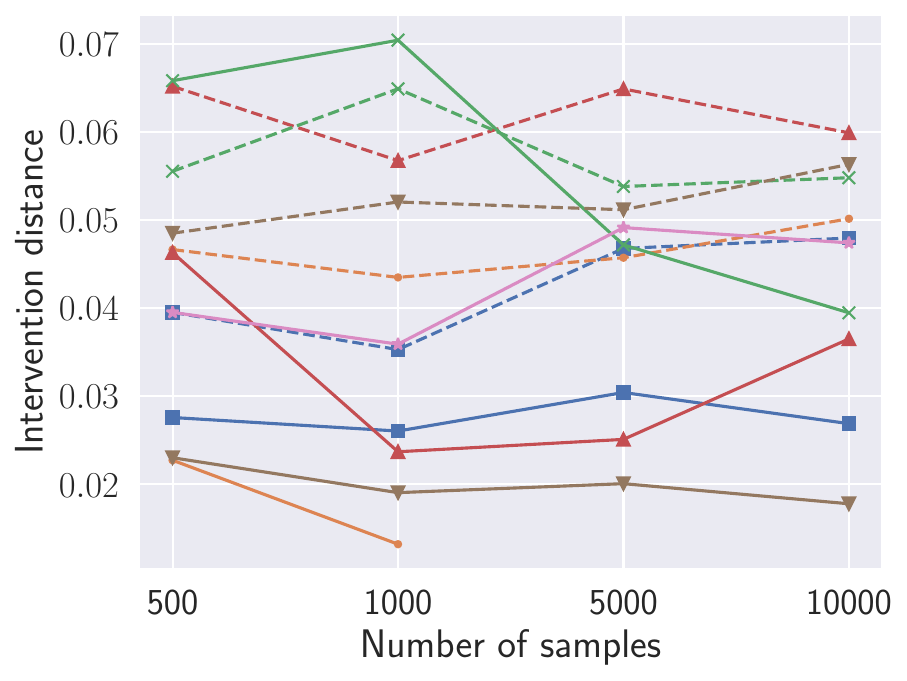}
            \caption*{Fisher-Z tests}
            \label{fig:int_dist_per_samples_unrest_fshz_abs}
        \end{subfigure}
        \begin{subfigure}[b]{0.3\linewidth}
            \includegraphics[width=\linewidth]{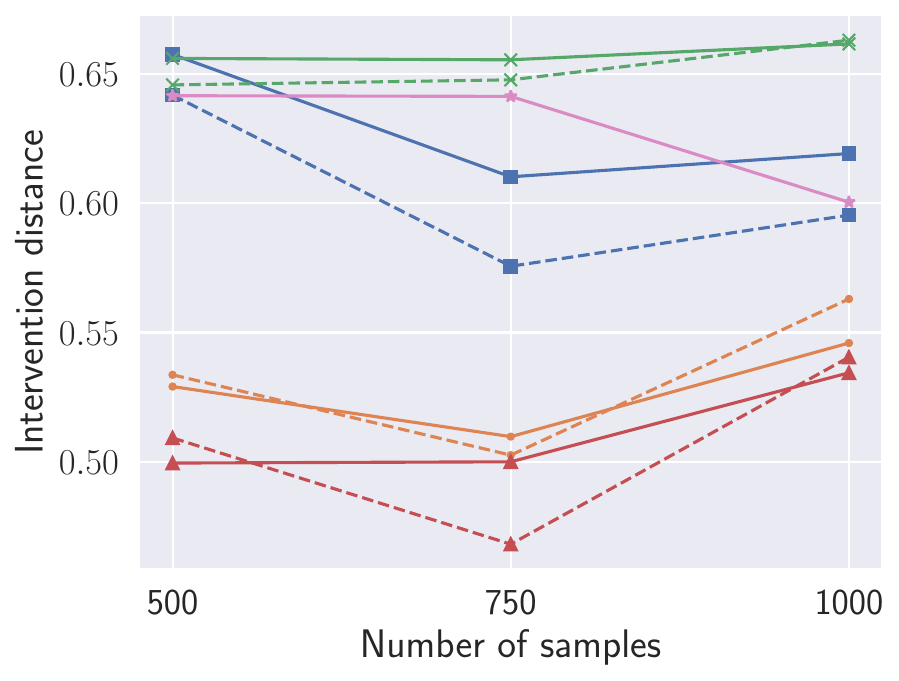}
            \caption*{KCI tests}
            \label{fig:int_dist_per_samples_unrest_kci_abs}
        \end{subfigure}
        \begin{subfigure}[b]{0.3\linewidth}
            \includegraphics[width=\linewidth]{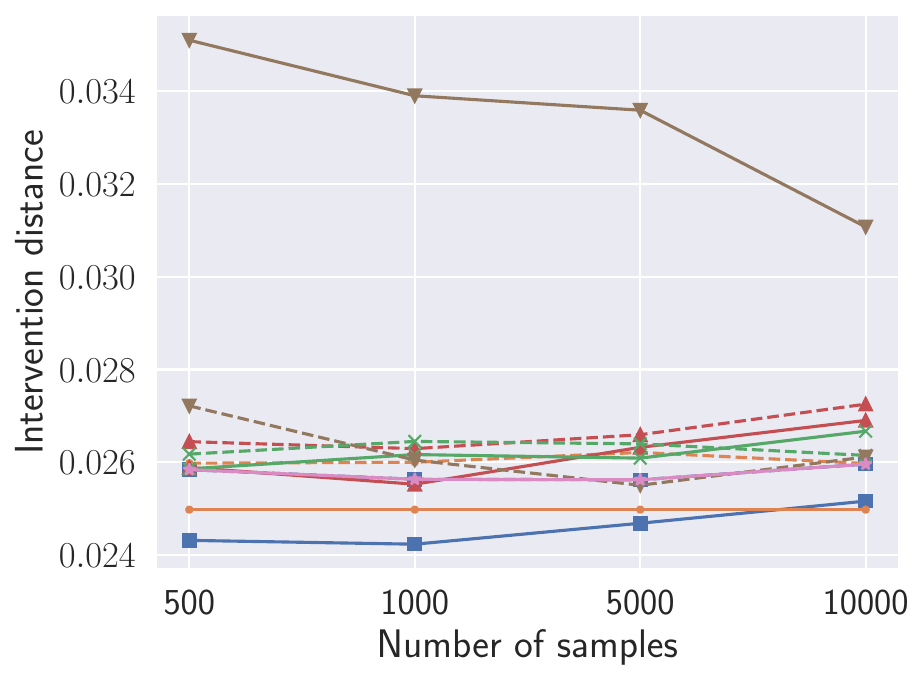}
            \caption*{$\chi^2$ tests}
            \label{fig:int_dist_per_samples_unrest_chsq_abs}
        \end{subfigure}
        \caption{Intervention distance.}
        \label{fig:int_dist_per_samples_unrest}
    \end{subfigure}
    \begin{subfigure}[b]{\linewidth}
        \centering
        \begin{subfigure}[b]{0.3\linewidth}
            \includegraphics[width=\linewidth]{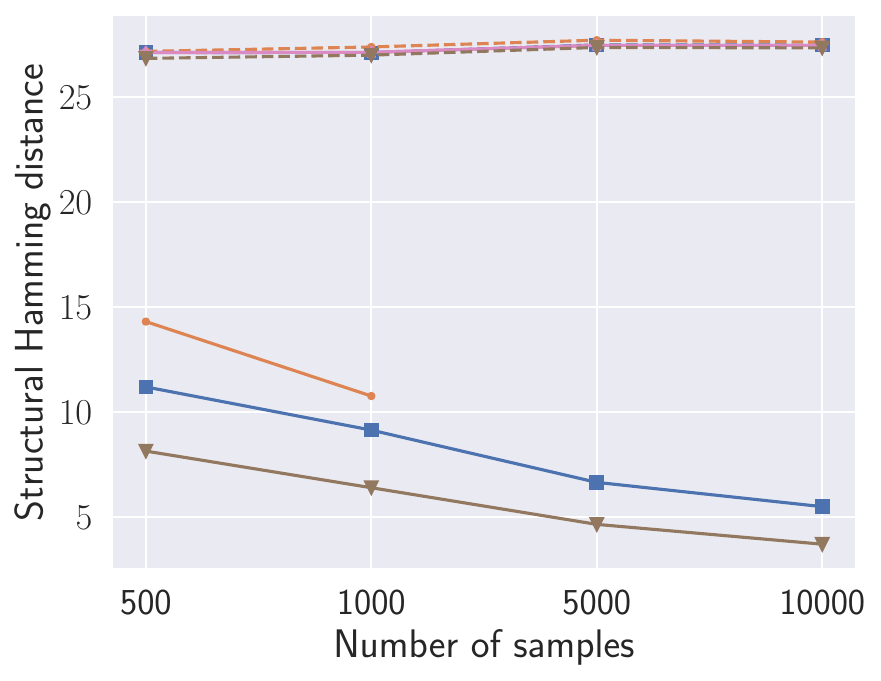}
            \caption*{Fisher-Z tests}
            \label{fig:shd_per_samples_unrest_fshz}
        \end{subfigure}
        \begin{subfigure}[b]{0.3\linewidth}
            \includegraphics[width=\linewidth]{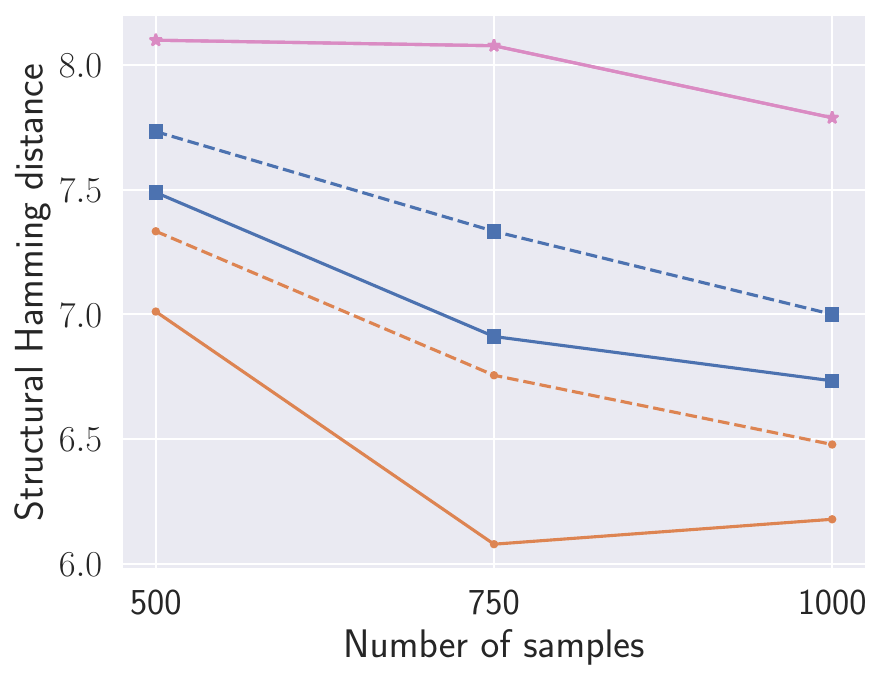}
            \caption*{KCI tests}
            \label{fig:shd_per_samples_unrest_kci}
        \end{subfigure}
        \begin{subfigure}[b]{0.3\linewidth}
            \includegraphics[width=\linewidth]{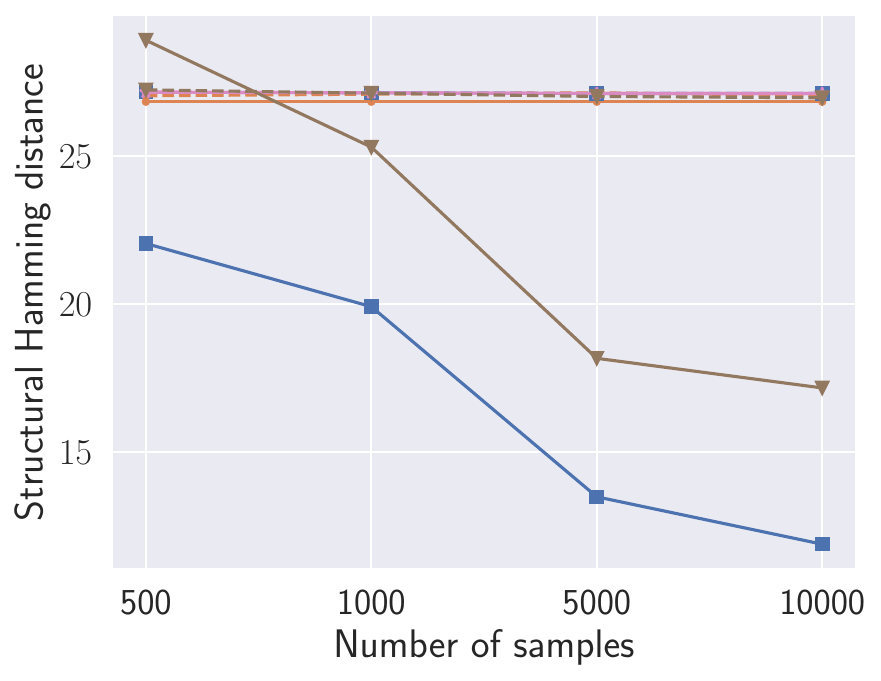}
            \caption*{$\chi^2$ tests}
            \label{fig:shd_per_samples_unrest_chsq}
        \end{subfigure}
        \caption{\Acf{SHD}.}
        \label{fig:shd_per_samples_unrest}
    \end{subfigure}
    \caption{Additional results for baseline methods combined with SNAP$(0)$ over number of samples, with $n_{\mathbf{V}}=10$ for KCI tests and $n_{\mathbf{V}}=200$ otherwise, $n_{\mathbf{T}}=4, \overline{d} = 3$ and $d_{\max}=10$.
    We compute the intervention distance in the d-separation tests case using random linear Gaussian data according to the discovered structure.
    Results for MARVEL with Fisher-Z tests are not shown for 5000 and 10000 samples because the runs did not finish in two days.
    }
    \label{fig:per_samples}
\end{figure}

\subsection{SNAP(k) with higher order k}
\label{app:snapk}
We run baseline methods combined with SNAP$(k)$ for prefiltering with $k=0,\dots,2$, on graphs with various number of nodes, $n_{\mathbf{T}}=4$, $\overline{d} = 3, d_{\max}=10$ and $n_{\mathbf{D}} = 1000$ data-points, to compare the impact of maximum iterations $k$.
Our results in \Cref{fig:test_snapk,fig:time_snapk,fig:shd_snapk,fig:int_dist_snapk} show that increasing $k$ only by one or two does not have a large impact on computation or estimation quality.

\begin{figure}
    \centering
    \includegraphics[width=.6\linewidth]{experiments/legend_big.pdf}
    \begin{subfigure}[b]{\linewidth}
        \centering
        \begin{subfigure}[b]{0.24\linewidth}
            \includegraphics[width=\linewidth]{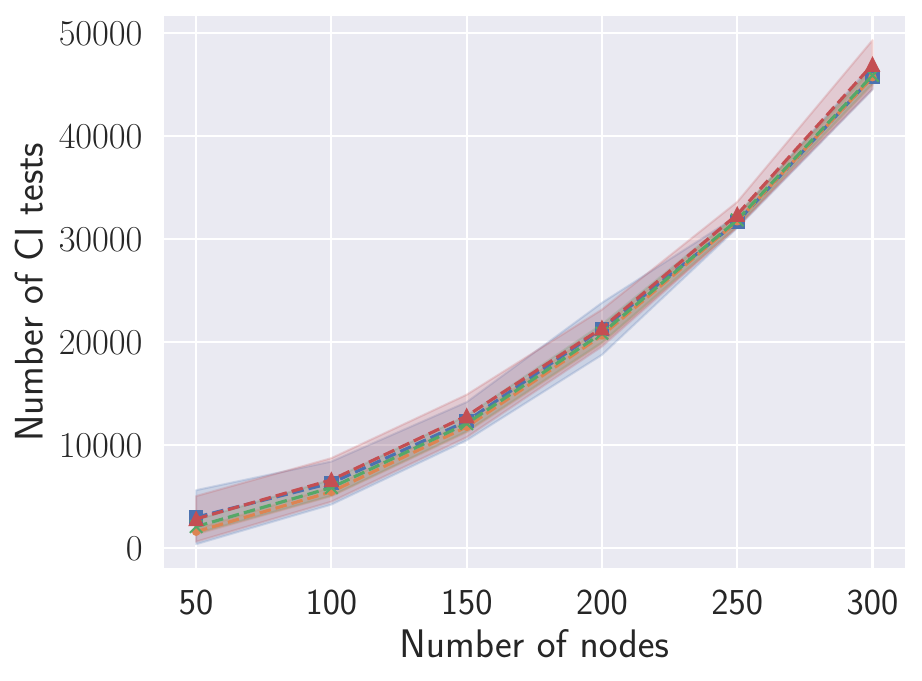}
            \caption*{d-separation tests}
        \end{subfigure}
        \begin{subfigure}[b]{0.24\linewidth}
            \includegraphics[width=\linewidth]{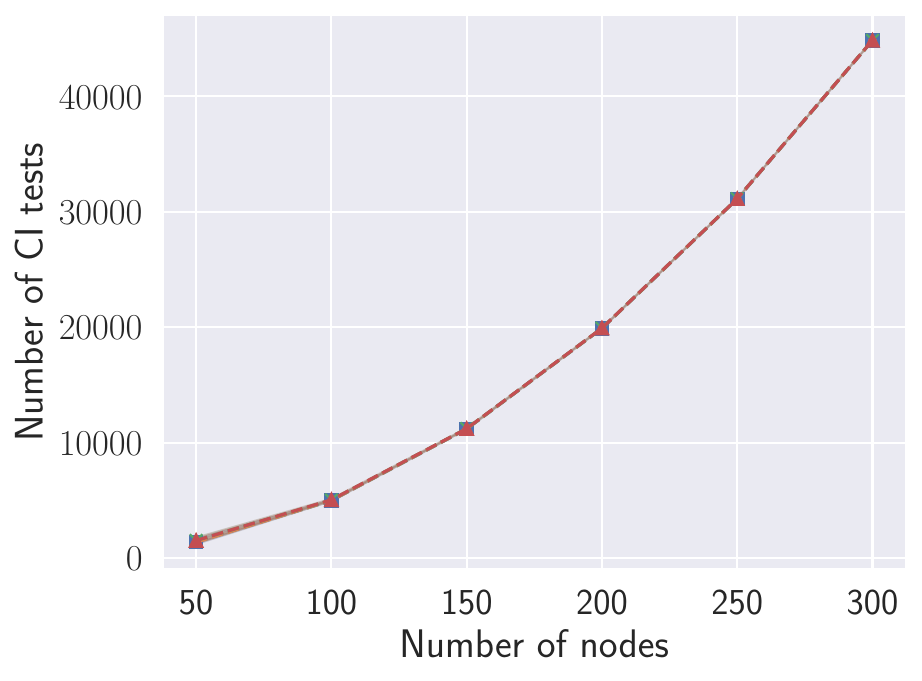}
            \caption*{Fisher-Z tests}
        \end{subfigure}
        \begin{subfigure}[b]{0.24\linewidth}
            \includegraphics[width=\linewidth]{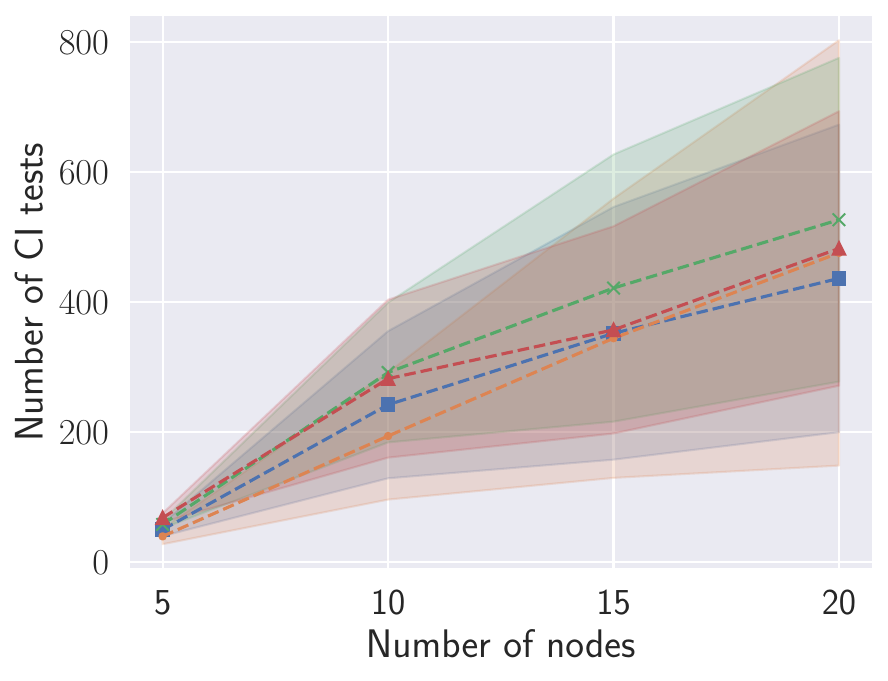}
            \caption*{KCI tests}
        \end{subfigure}
        \begin{subfigure}[b]{0.24\linewidth}
            \includegraphics[width=\linewidth]{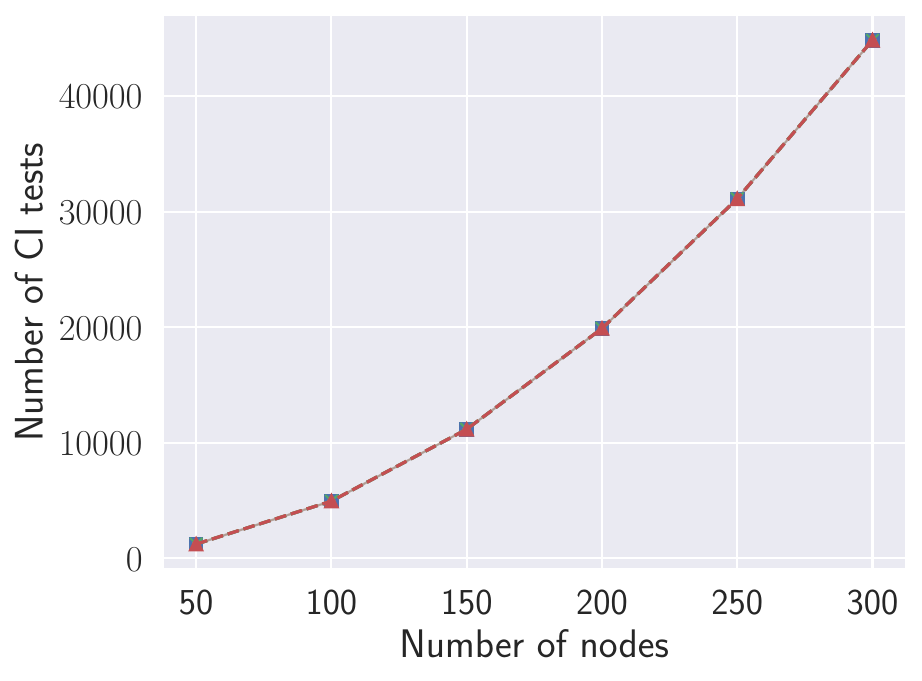}
            \caption*{$\chi^2$ tests}
        \end{subfigure}
        \caption{SNAP$(0)$.}
        \label{fig:test_snap0}
    \end{subfigure}
    \begin{subfigure}[b]{\linewidth}
        \centering
        \begin{subfigure}[b]{0.24\linewidth}
            \includegraphics[width=\linewidth]{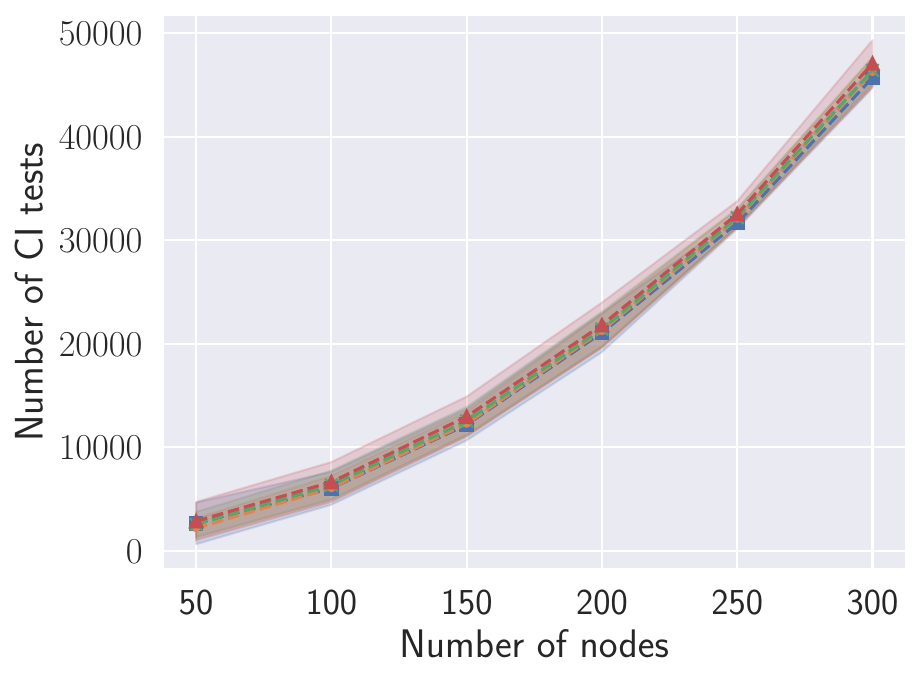}
            \caption*{d-separation tests}
        \end{subfigure}
        \begin{subfigure}[b]{0.24\linewidth}
            \includegraphics[width=\linewidth]{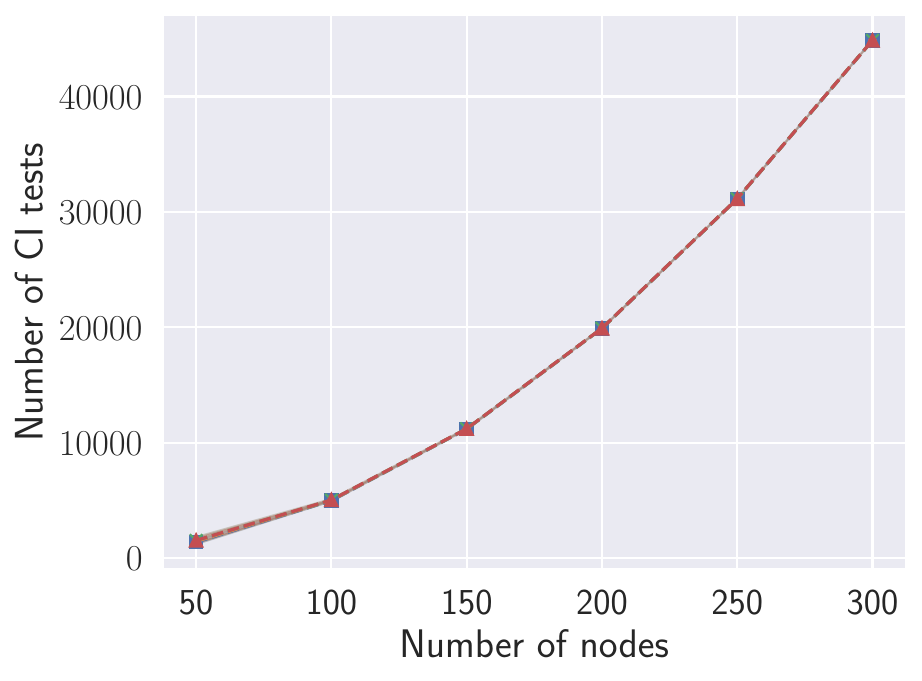}
            \caption*{Fisher-Z tests}
        \end{subfigure}
        \begin{subfigure}[b]{0.24\linewidth}
            \includegraphics[width=\linewidth]{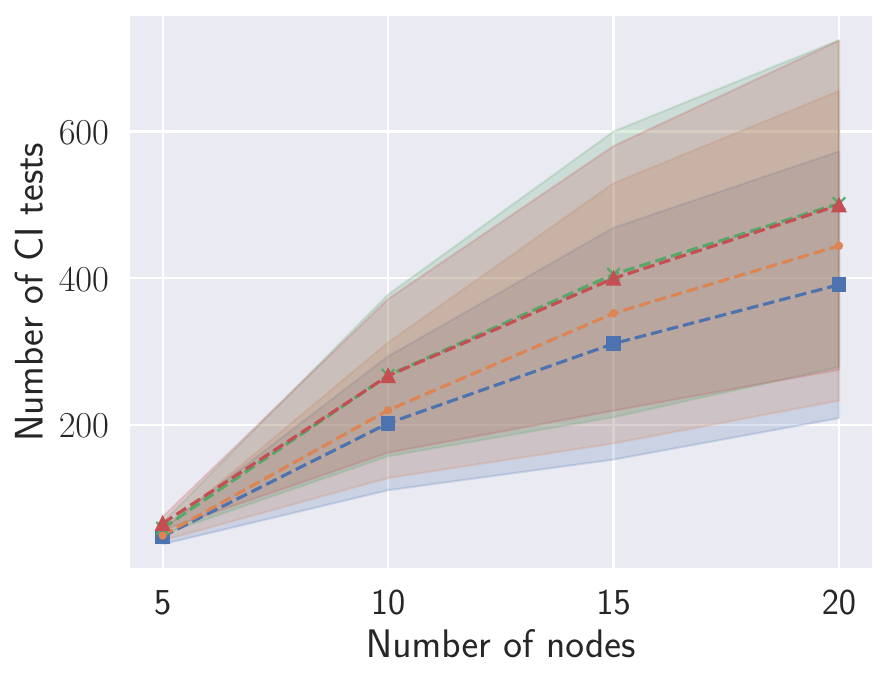}
            \caption*{KCI tests}
        \end{subfigure}
        \begin{subfigure}[b]{0.24\linewidth}
            \includegraphics[width=\linewidth]{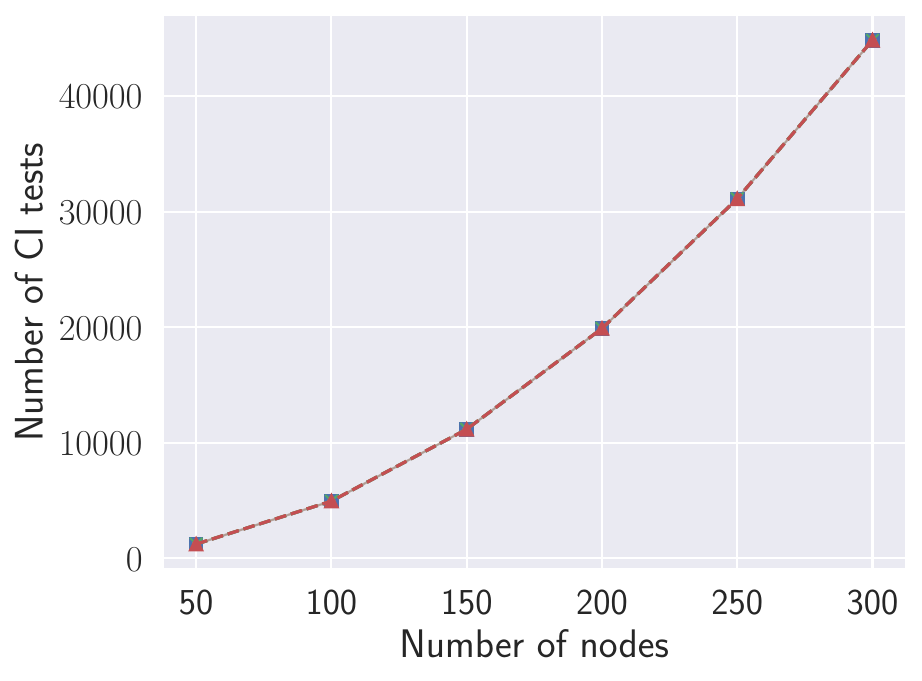}
            \caption*{$\chi^2$ tests}
        \end{subfigure}
        \caption{SNAP$(1)$.}
        \label{fig:test_snap1}
    \end{subfigure}
    \begin{subfigure}[b]{\linewidth}
        \centering
        \begin{subfigure}[b]{0.24\linewidth}
            \includegraphics[width=\linewidth]{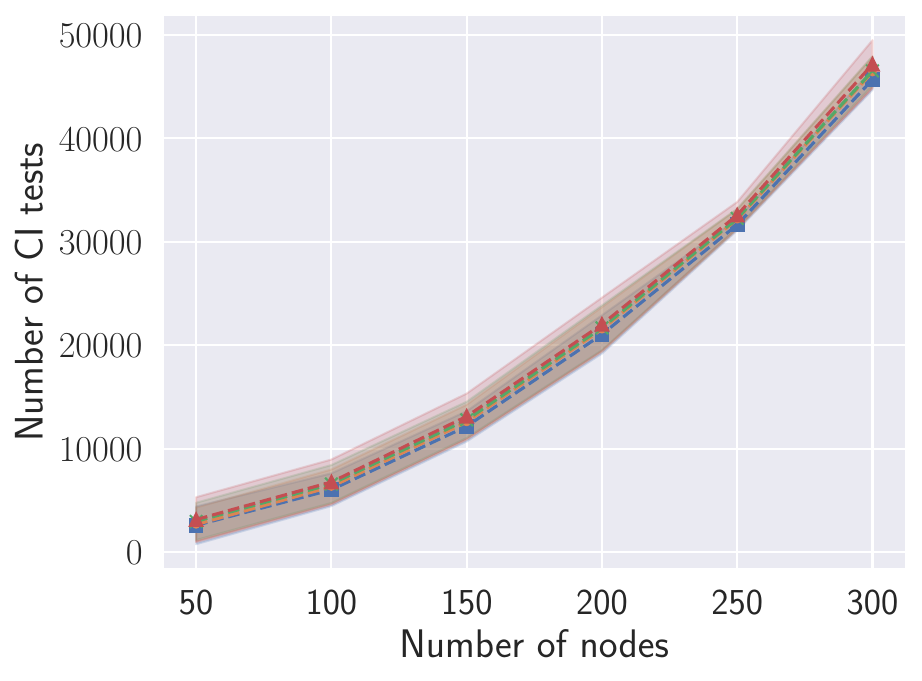}
            \caption*{d-separation tests}
        \end{subfigure}
        \begin{subfigure}[b]{0.24\linewidth}
            \includegraphics[width=\linewidth]{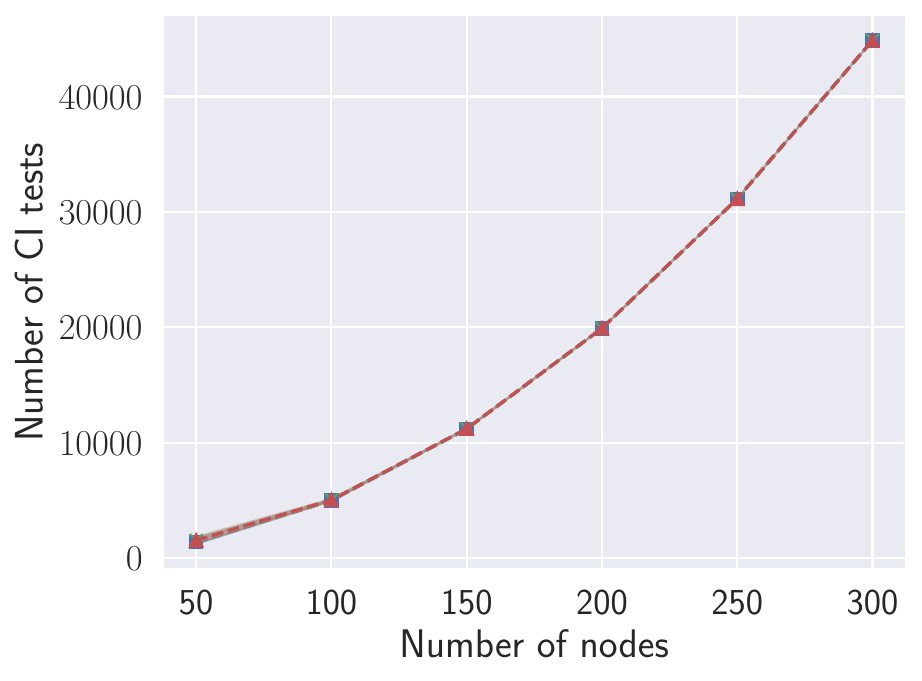}
            \caption*{Fisher-Z tests}
        \end{subfigure}
        \begin{subfigure}[b]{0.24\linewidth}
            \includegraphics[width=\linewidth]{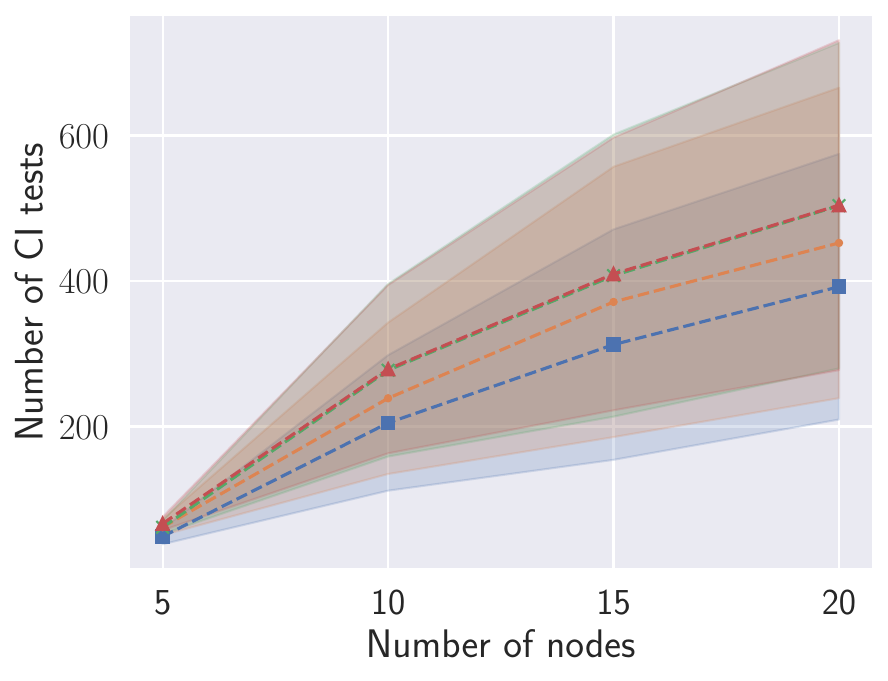}
            \caption*{KCI tests}
        \end{subfigure}
        \begin{subfigure}[b]{0.24\linewidth}
            \includegraphics[width=\linewidth]{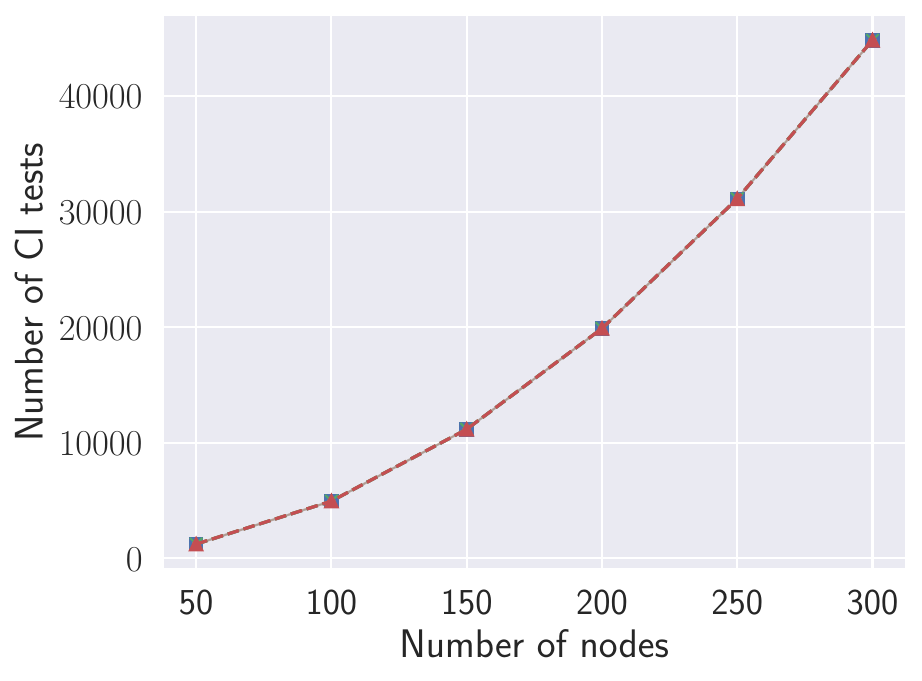}
            \caption*{$\chi^2$ tests}
        \end{subfigure}
        \caption{SNAP$(2)$.}
        \label{fig:test_snap2}
    \end{subfigure}
    \caption{Number of \ac{CI} tests over number of nodes for SNAP($k$) with $k= 0,\dots,2$, with $n_{\mathbf{T}}=4$, $\overline{d} = 3, d_{\max}=10$ and $n_{\mathbf{D}} = 1000$ data-points.}
    \label{fig:test_snapk}
\end{figure}

\begin{figure}
    \centering
    \includegraphics[width=.6\linewidth]{experiments/legend_big.pdf}
    \begin{subfigure}[b]{\linewidth}
        \centering
        \begin{subfigure}[b]{0.24\linewidth}
            \includegraphics[width=\linewidth]{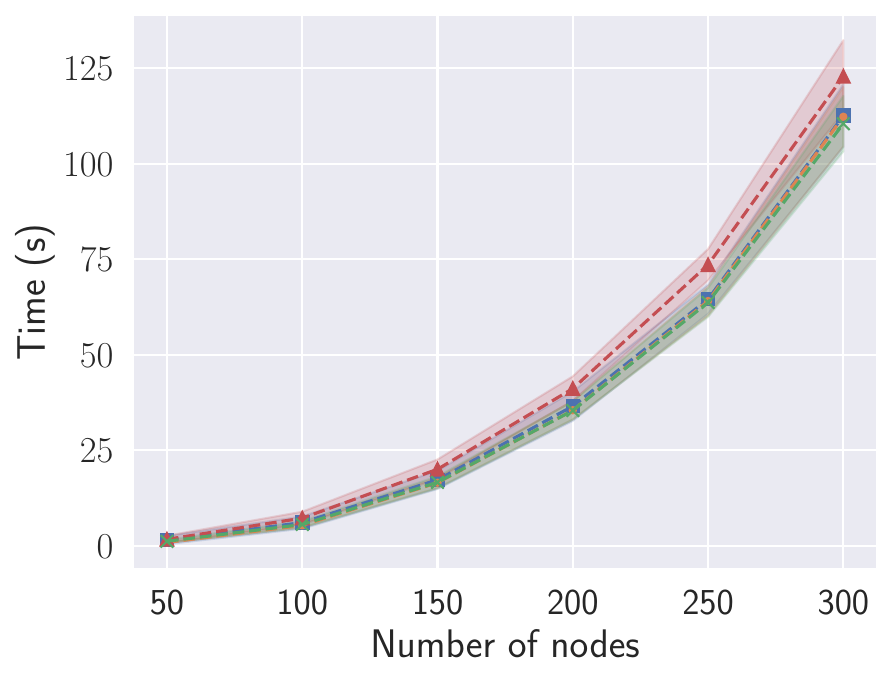}
            \caption*{d-separation tests}
        \end{subfigure}
        \begin{subfigure}[b]{0.24\linewidth}
            \includegraphics[width=\linewidth]{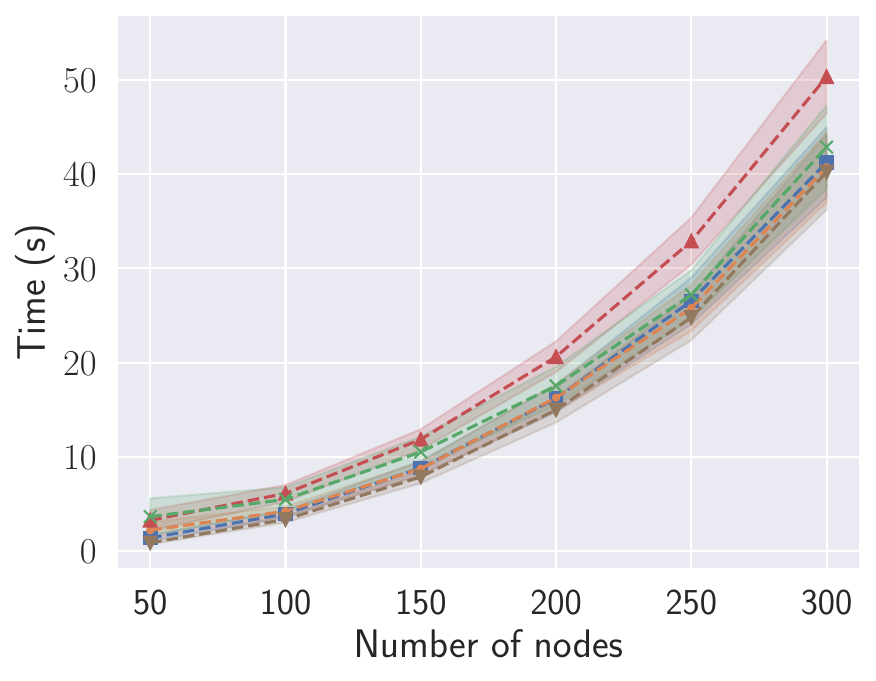}
            \caption*{Fisher-Z tests}
        \end{subfigure}
        \begin{subfigure}[b]{0.24\linewidth}
            \includegraphics[width=\linewidth]{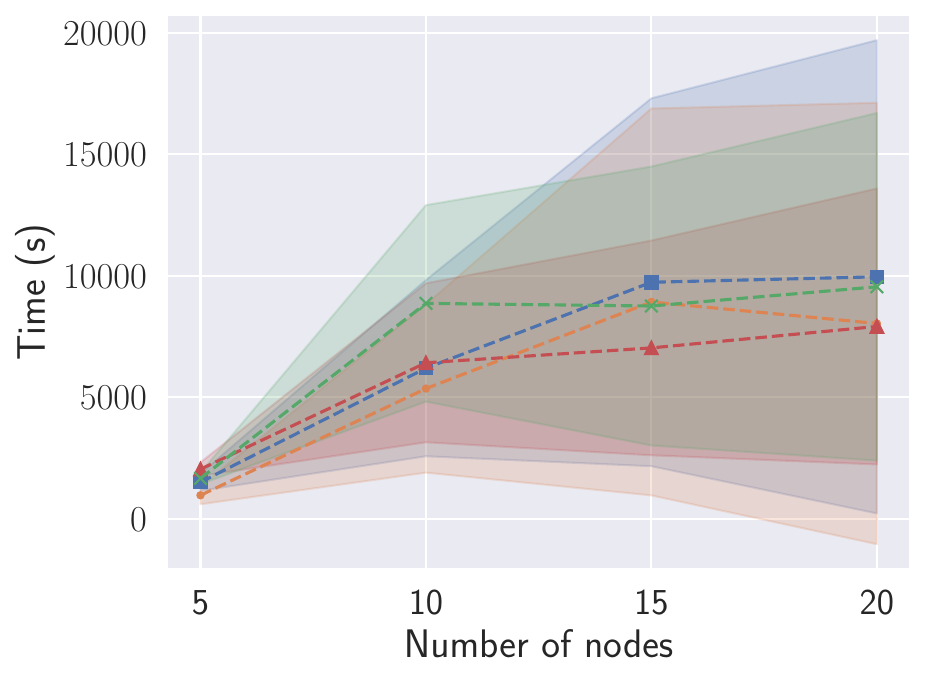}
            \caption*{KCI tests}
        \end{subfigure}
        \begin{subfigure}[b]{0.24\linewidth}
            \includegraphics[width=\linewidth]{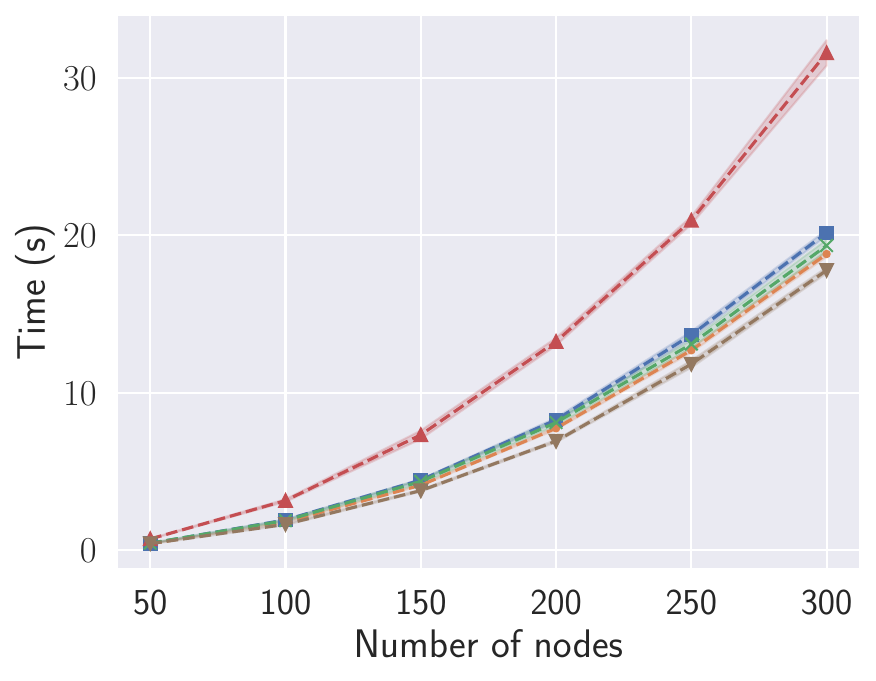}
            \caption*{$\chi^2$ tests}
        \end{subfigure}
        \caption{SNAP$(0)$.}
        \label{fig:time_snap0}
    \end{subfigure}
    \begin{subfigure}[b]{\linewidth}
        \centering
        \begin{subfigure}[b]{0.24\linewidth}
            \includegraphics[width=\linewidth]{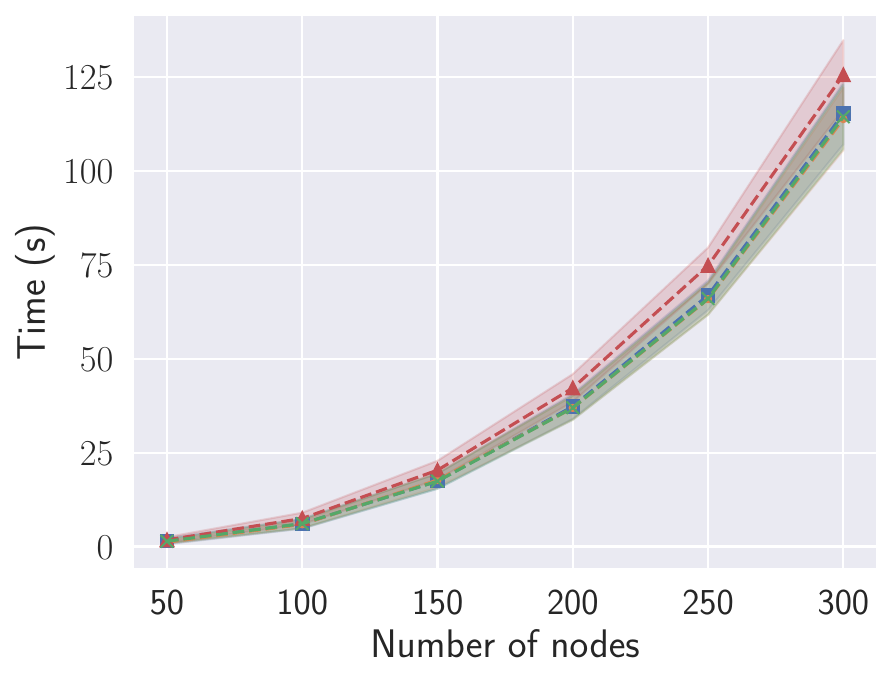}
            \caption*{d-separation tests}
        \end{subfigure}
        \begin{subfigure}[b]{0.24\linewidth}
            \includegraphics[width=\linewidth]{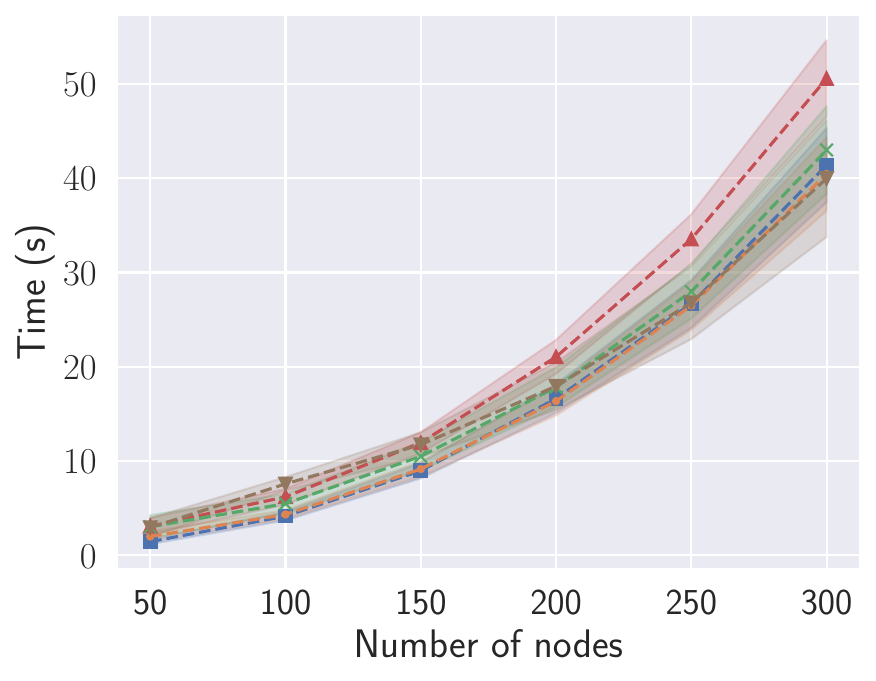}
            \caption*{Fisher-Z tests}
        \end{subfigure}
        \begin{subfigure}[b]{0.24\linewidth}
            \includegraphics[width=\linewidth]{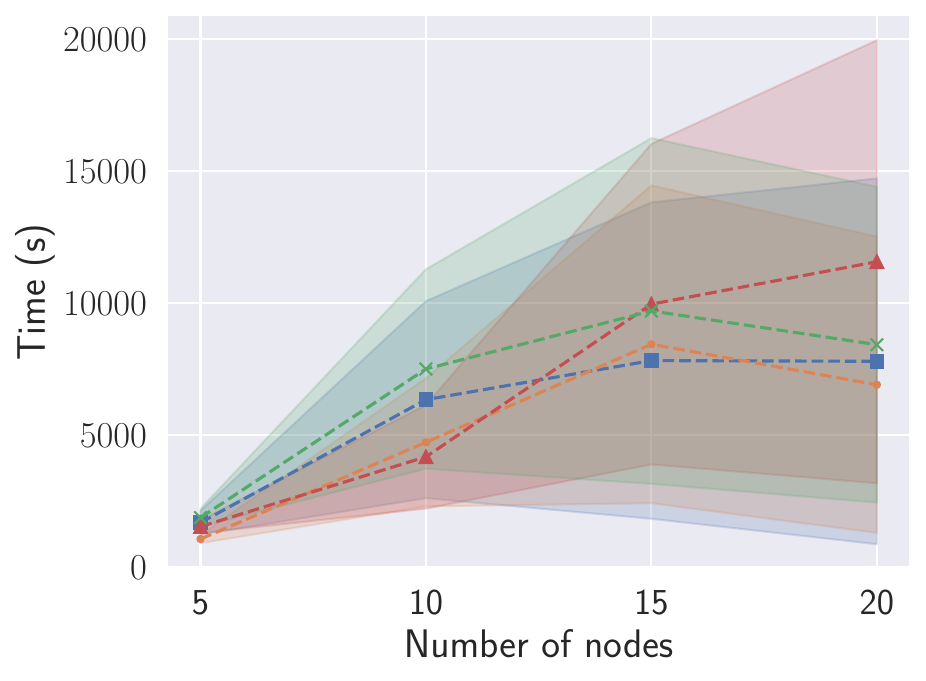}
            \caption*{KCI tests}
        \end{subfigure}
        \begin{subfigure}[b]{0.24\linewidth}
            \includegraphics[width=\linewidth]{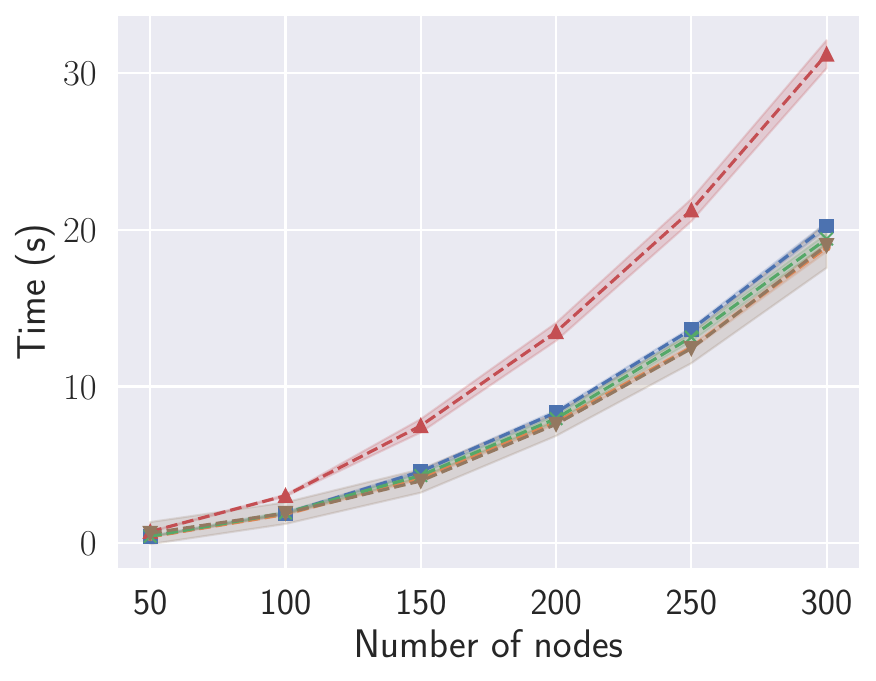}
            \caption*{$\chi^2$ tests}
        \end{subfigure}
        \caption{SNAP$(1)$.}
        \label{fig:time_snap1}
    \end{subfigure}
    \begin{subfigure}[b]{\linewidth}
        \centering
        \begin{subfigure}[b]{0.24\linewidth}
            \includegraphics[width=\linewidth]{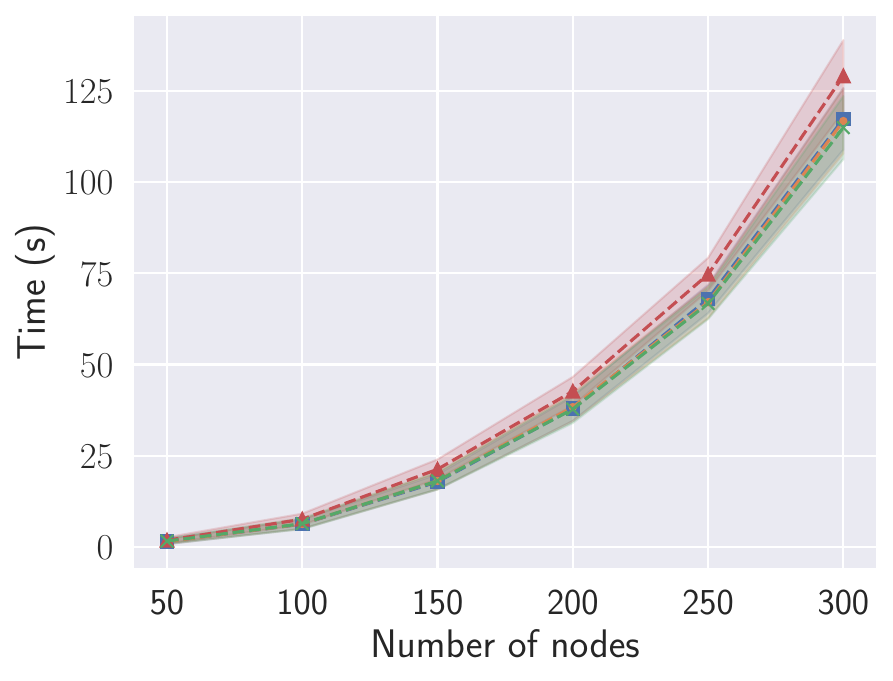}
            \caption*{d-separation tests}
        \end{subfigure}
        \begin{subfigure}[b]{0.24\linewidth}
            \includegraphics[width=\linewidth]{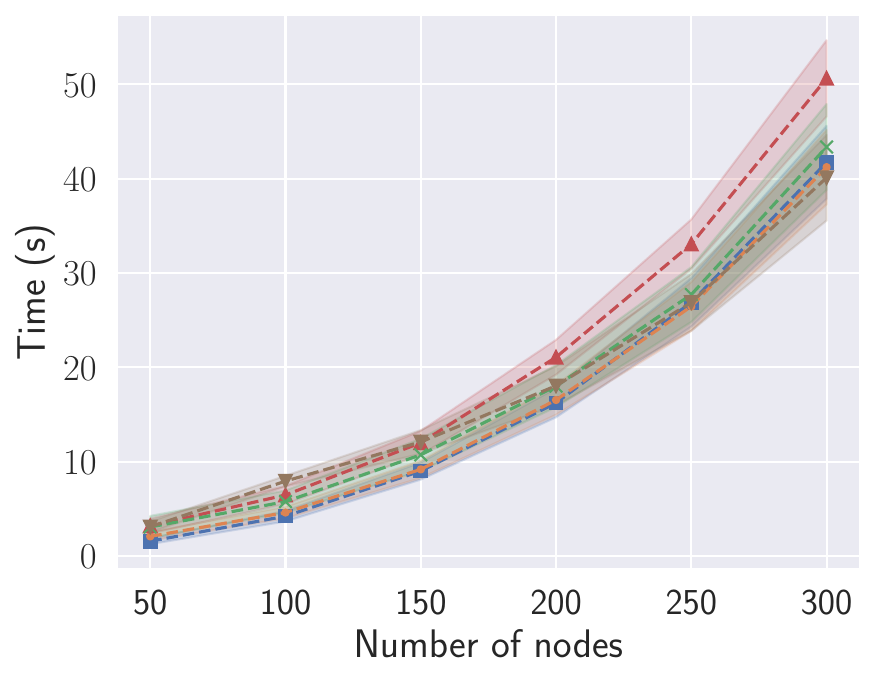}
            \caption*{Fisher-Z tests}
        \end{subfigure}
        \begin{subfigure}[b]{0.24\linewidth}
            \includegraphics[width=\linewidth]{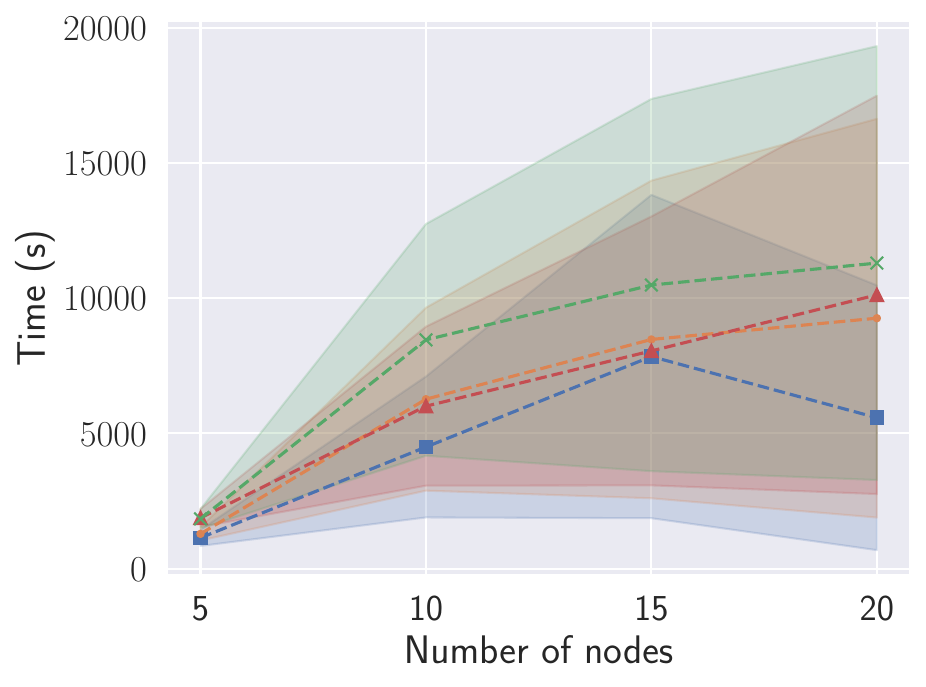}
            \caption*{KCI tests}
        \end{subfigure}
        \begin{subfigure}[b]{0.24\linewidth}
            \includegraphics[width=\linewidth]{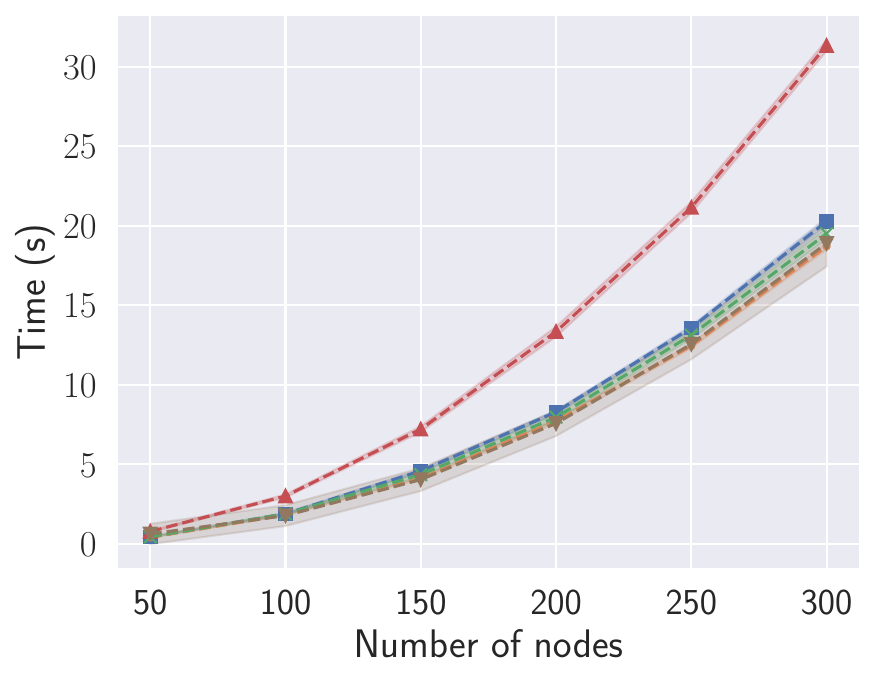}
            \caption*{$\chi^2$ tests}
        \end{subfigure}
        \caption{SNAP$(2)$.}
        \label{fig:time_snap2}
    \end{subfigure}
    \caption{Computation time over number of nodes for SNAP($k$) with $k= 0,\dots,2$, with $n_{\mathbf{T}}=4$, $\overline{d} = 3, d_{\max}=10$ and $n_{\mathbf{D}} = 1000$ data-points.}
    \label{fig:time_snapk}
\end{figure}

\begin{figure}
    \centering
    \includegraphics[width=.6\linewidth]{experiments/legend_big.pdf}
    \begin{subfigure}[b]{\linewidth}
        \centering
        \begin{subfigure}[b]{0.24\linewidth}
            \includegraphics[width=\linewidth]{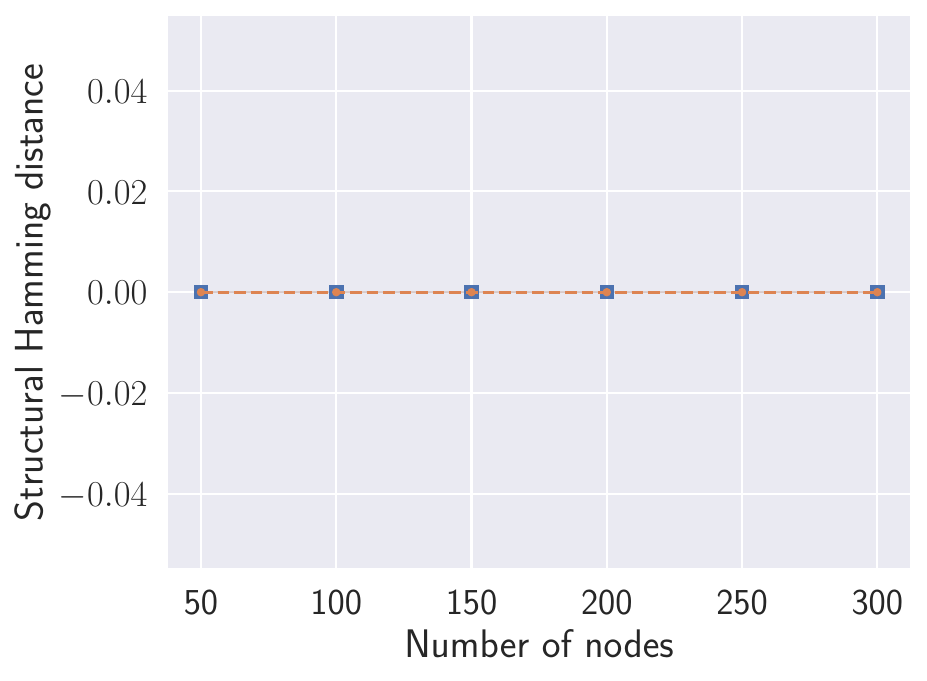}
            \caption*{d-separation tests}
        \end{subfigure}
        \begin{subfigure}[b]{0.24\linewidth}
            \includegraphics[width=\linewidth]{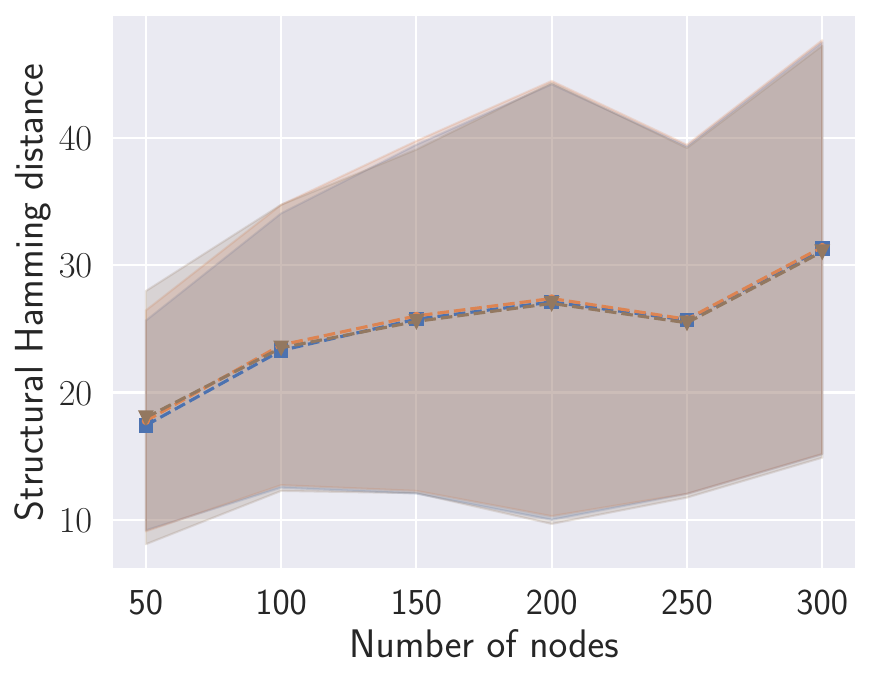}
            \caption*{Fisher-Z tests}
        \end{subfigure}
        \begin{subfigure}[b]{0.24\linewidth}
            \includegraphics[width=\linewidth]{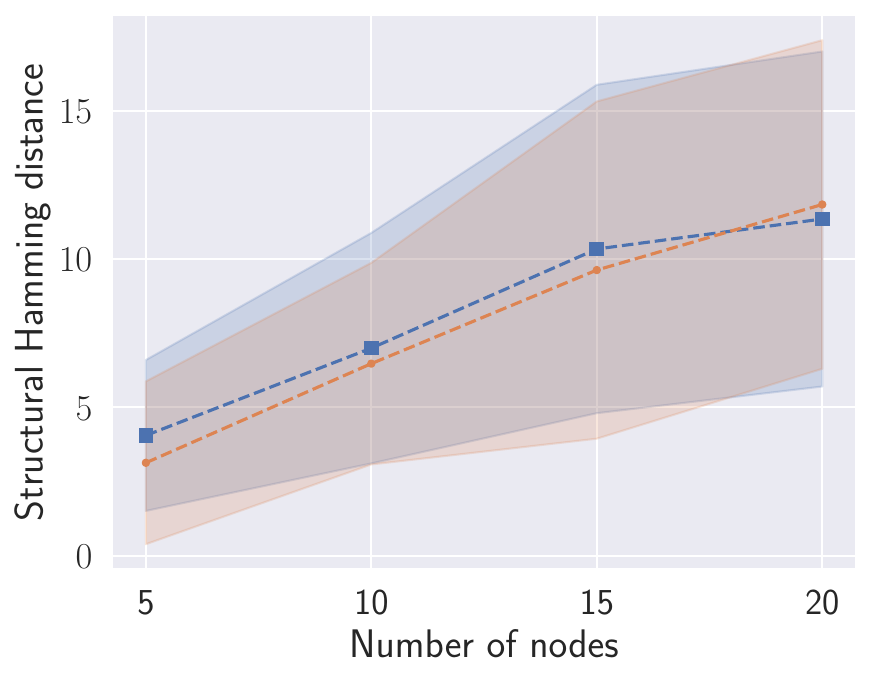}
            \caption*{KCI tests}
        \end{subfigure}
        \begin{subfigure}[b]{0.24\linewidth}
            \includegraphics[width=\linewidth]{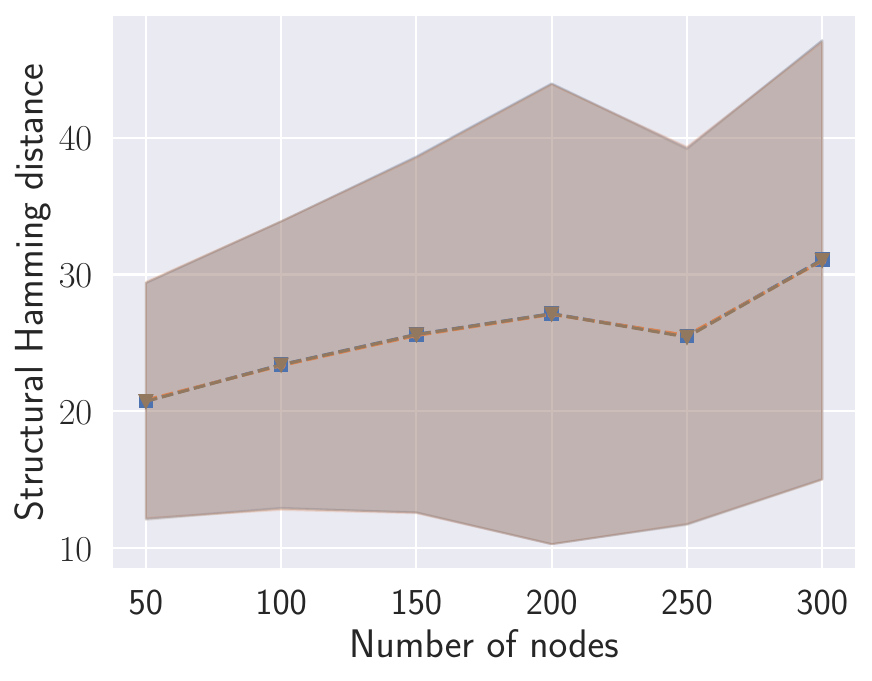}
            \caption*{$\chi^2$ tests}
        \end{subfigure}
        \caption{SNAP$(0)$.}
        \label{fig:shd_snap0}
    \end{subfigure}
    \begin{subfigure}[b]{\linewidth}
        \centering
        \begin{subfigure}[b]{0.24\linewidth}
            \includegraphics[width=\linewidth]{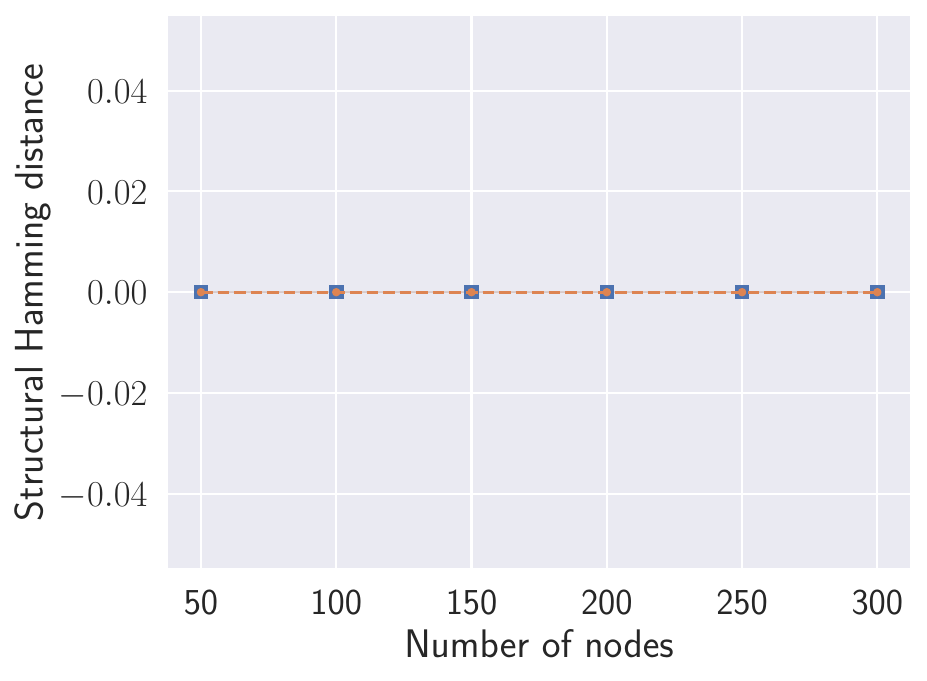}
            \caption*{d-separation tests}
        \end{subfigure}
        \begin{subfigure}[b]{0.24\linewidth}
            \includegraphics[width=\linewidth]{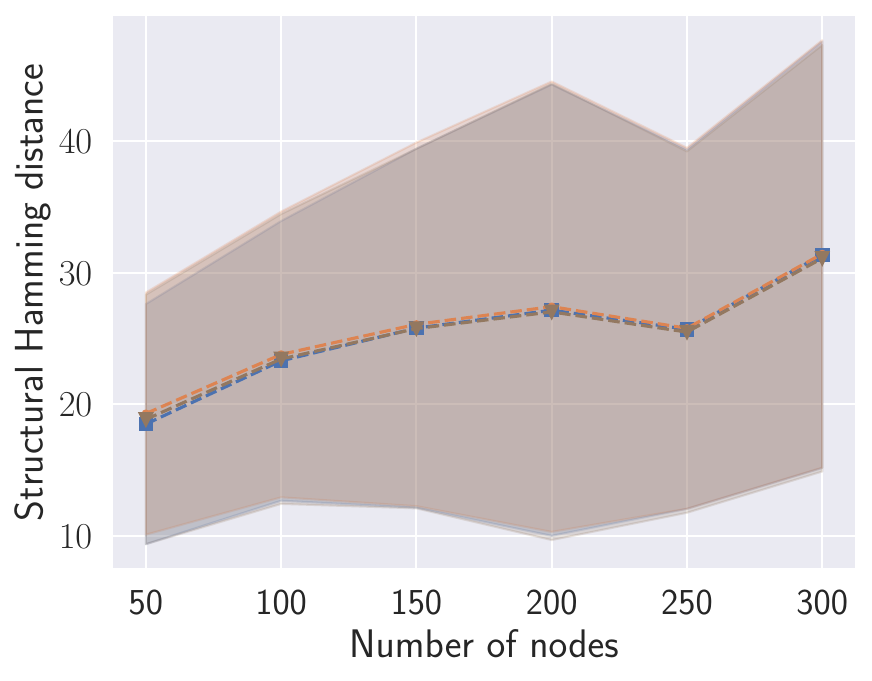}
            \caption*{Fisher-Z tests}
        \end{subfigure}
        \begin{subfigure}[b]{0.24\linewidth}
            \includegraphics[width=\linewidth]{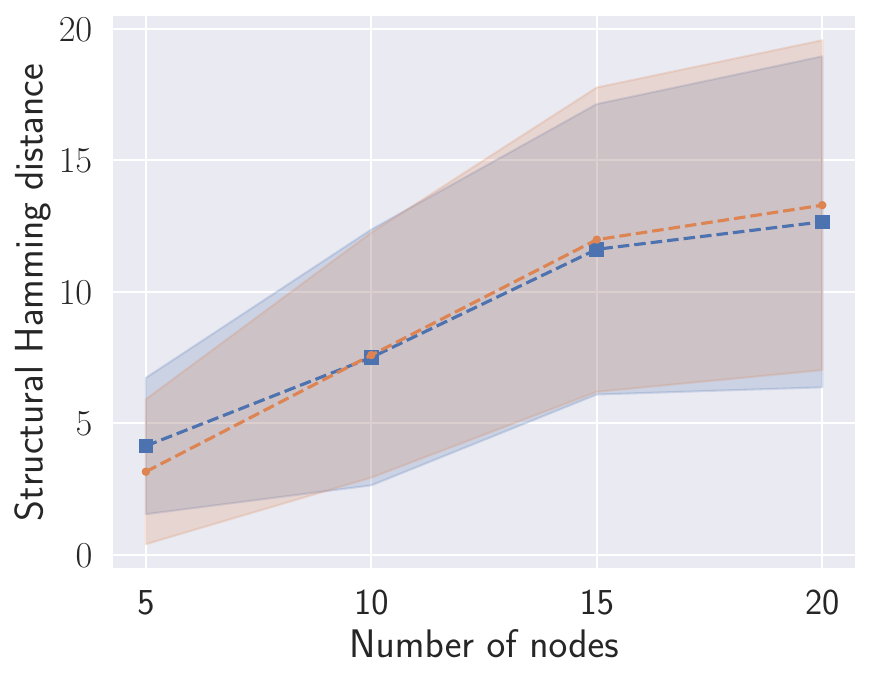}
            \caption*{KCI tests}
        \end{subfigure}
        \begin{subfigure}[b]{0.24\linewidth}
            \includegraphics[width=\linewidth]{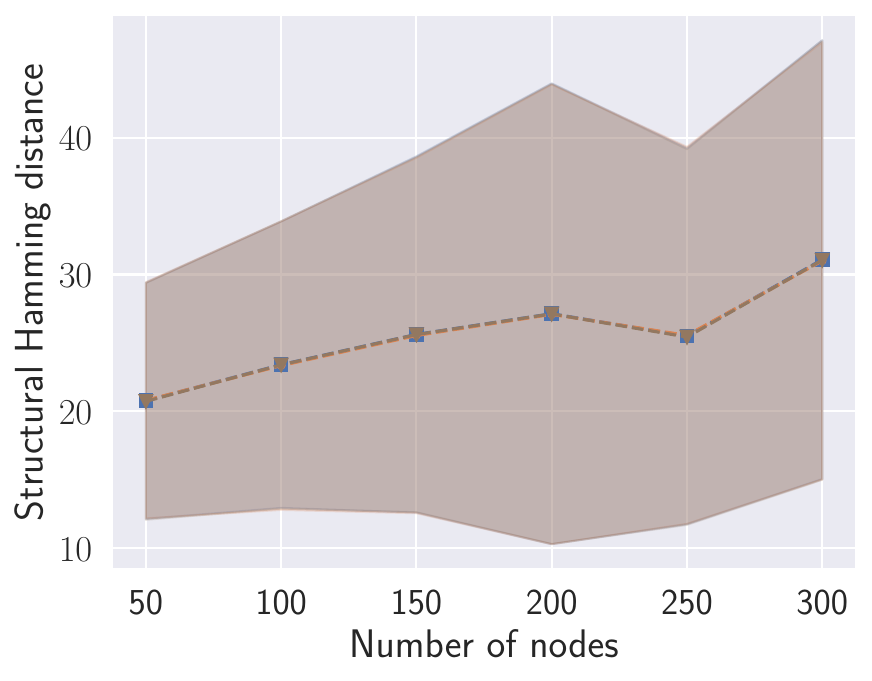}
            \caption*{$\chi^2$ tests}
        \end{subfigure}
        \caption{SNAP$(1)$.}
        \label{fig:shd_snap1}
    \end{subfigure}
    \begin{subfigure}[b]{\linewidth}
        \centering
        \begin{subfigure}[b]{0.24\linewidth}
            \includegraphics[width=\linewidth]{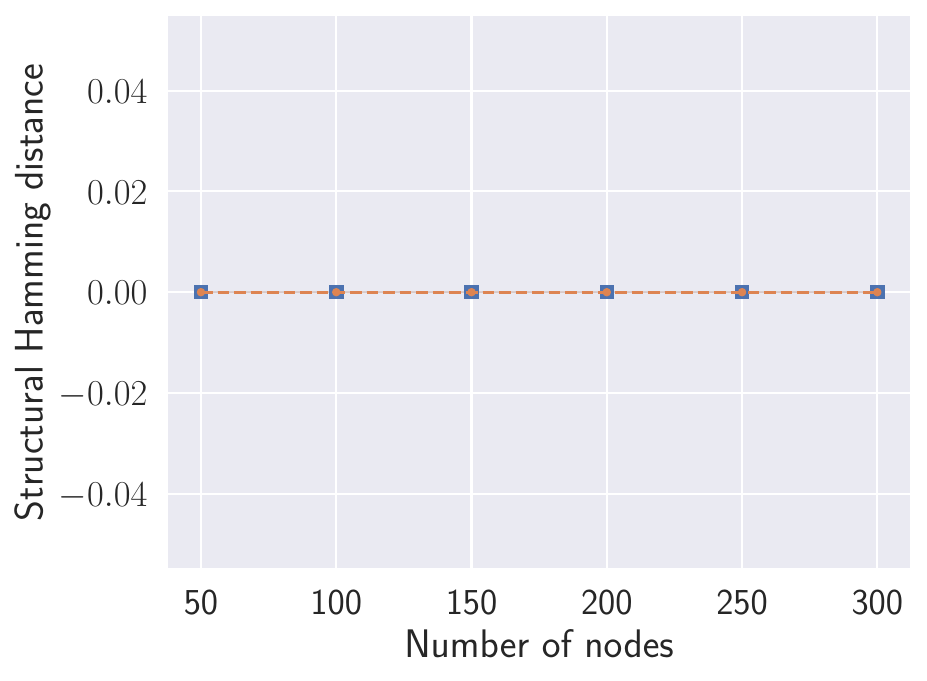}
            \caption*{d-separation tests}
        \end{subfigure}
        \begin{subfigure}[b]{0.24\linewidth}
            \includegraphics[width=\linewidth]{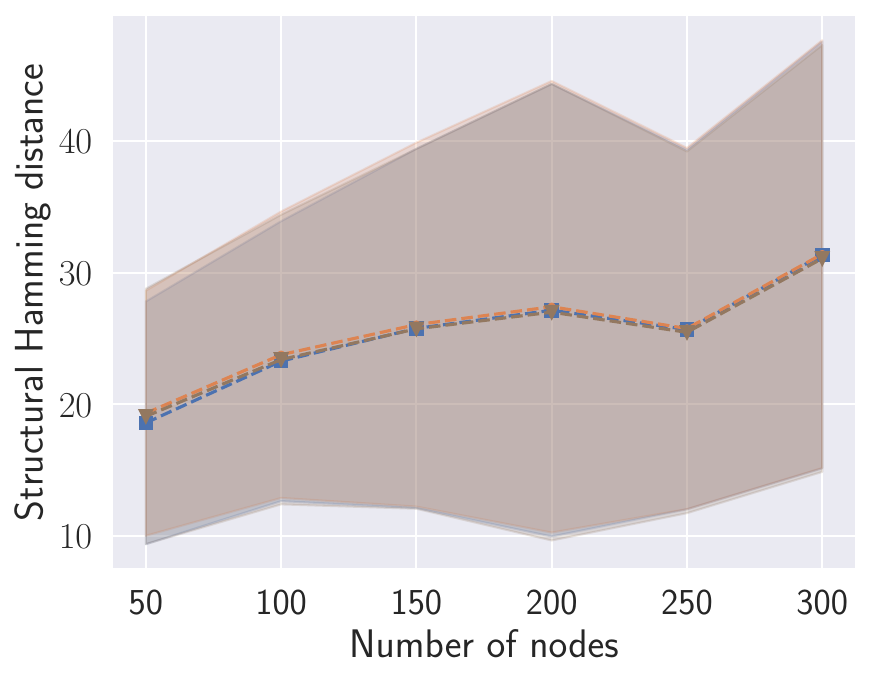}
            \caption*{Fisher-Z tests}
        \end{subfigure}
        \begin{subfigure}[b]{0.24\linewidth}
            \includegraphics[width=\linewidth]{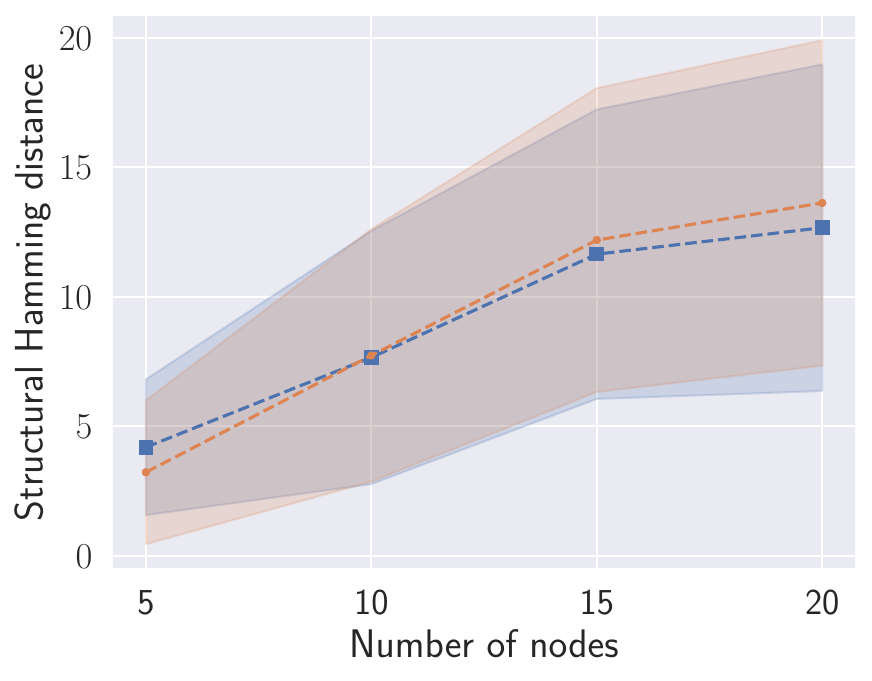}
            \caption*{KCI tests}
        \end{subfigure}
        \begin{subfigure}[b]{0.24\linewidth}
            \includegraphics[width=\linewidth]{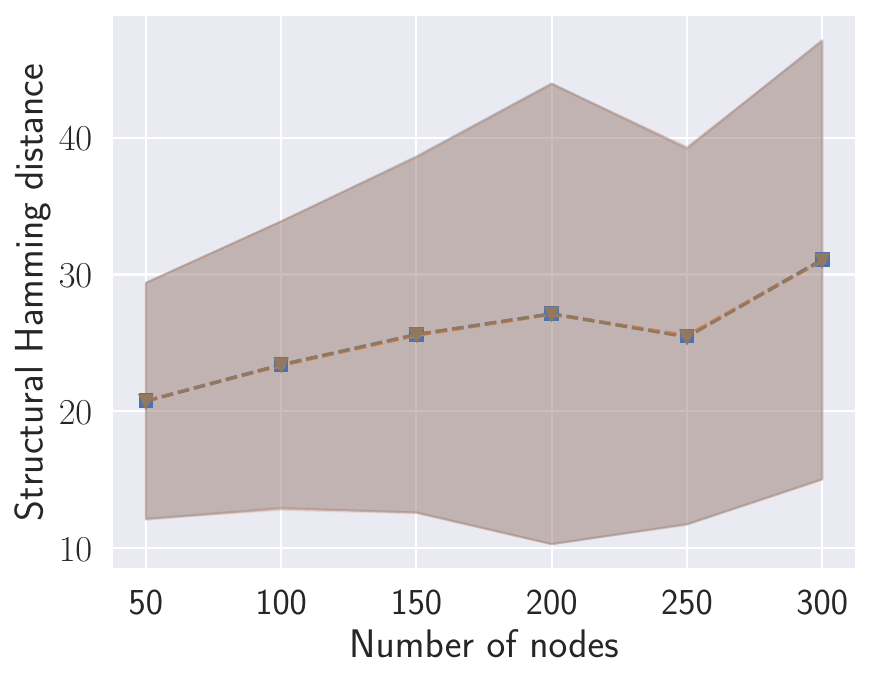}
            \caption*{$\chi^2$ tests}
        \end{subfigure}
        \caption{SNAP$(2)$.}
        \label{fig:shd_snap2}
    \end{subfigure}
    \caption{Structural Hamming Distance (SHD) over number targets for SNAP($k$) with $k= 0,\dots,2$, with $n_{\mathbf{T}}=4$, $\overline{d} = 3, d_{\max}=10$ and $n_{\mathbf{D}} = 1000$ data-points.}
    \label{fig:shd_snapk}
\end{figure}

\begin{figure}
    \centering
    \includegraphics[width=.6\linewidth]{experiments/legend_big.pdf}
    \begin{subfigure}[b]{\linewidth}
        \centering
        \begin{subfigure}[b]{0.24\linewidth}
            \includegraphics[width=\linewidth]{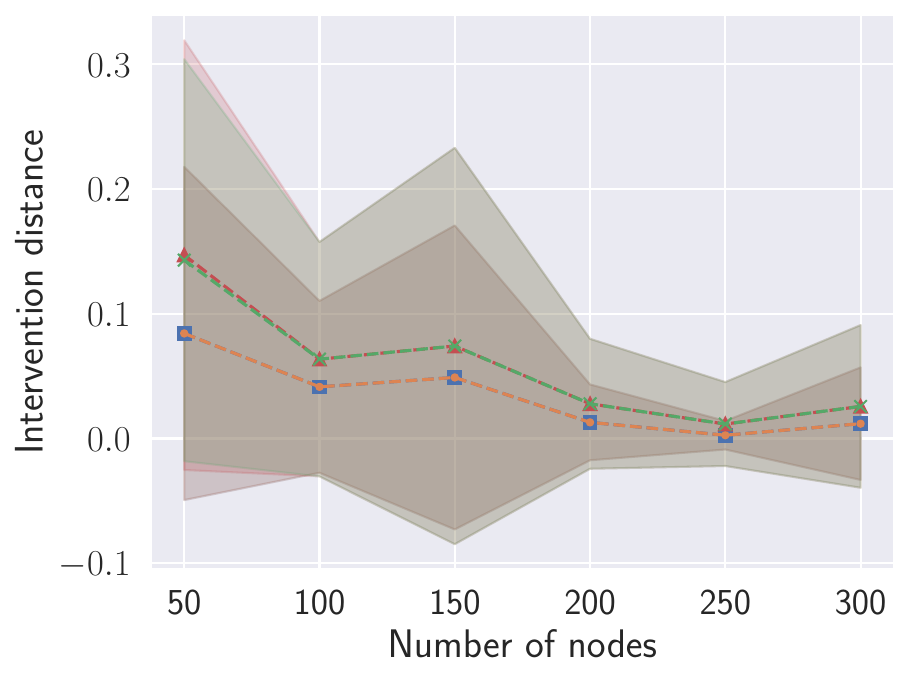}
            \caption*{d-separation tests}
        \end{subfigure}
        \begin{subfigure}[b]{0.24\linewidth}
            \includegraphics[width=\linewidth]{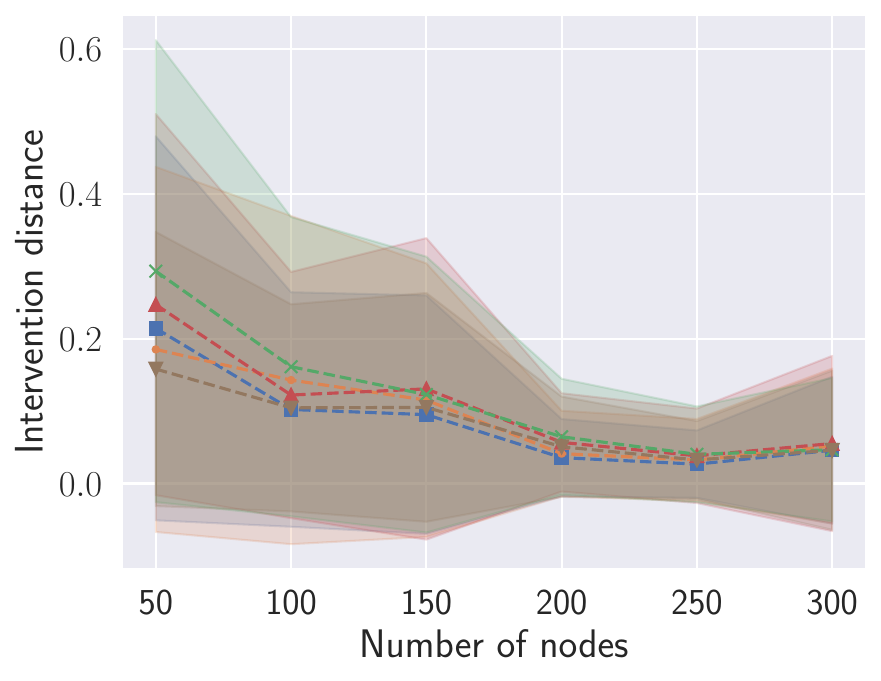}
            \caption*{Fisher-Z tests}
        \end{subfigure}
        \begin{subfigure}[b]{0.24\linewidth}
            \includegraphics[width=\linewidth]{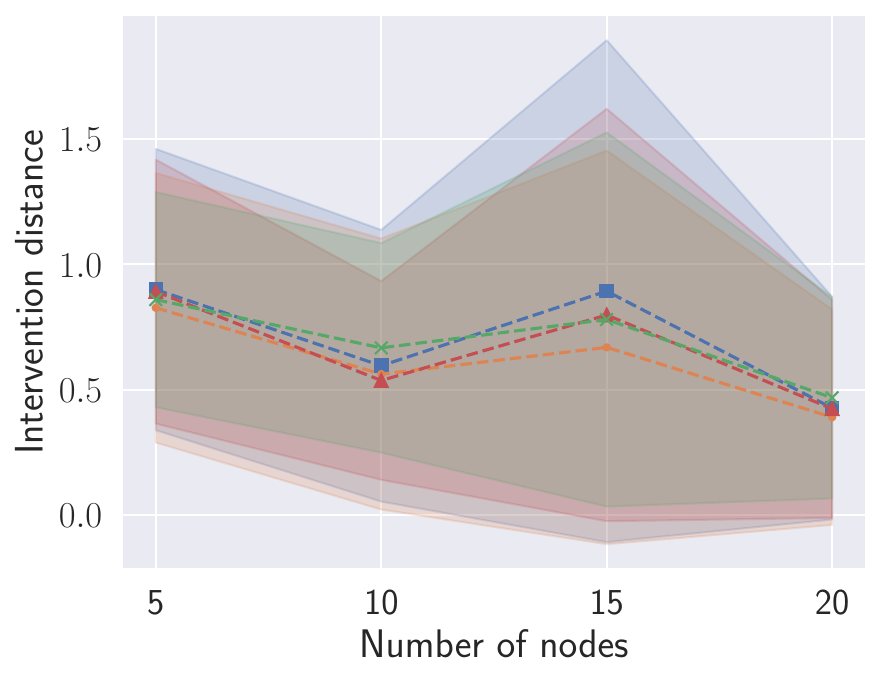}
            \caption*{KCI tests}
        \end{subfigure}
        \begin{subfigure}[b]{0.24\linewidth}
            \includegraphics[width=\linewidth]{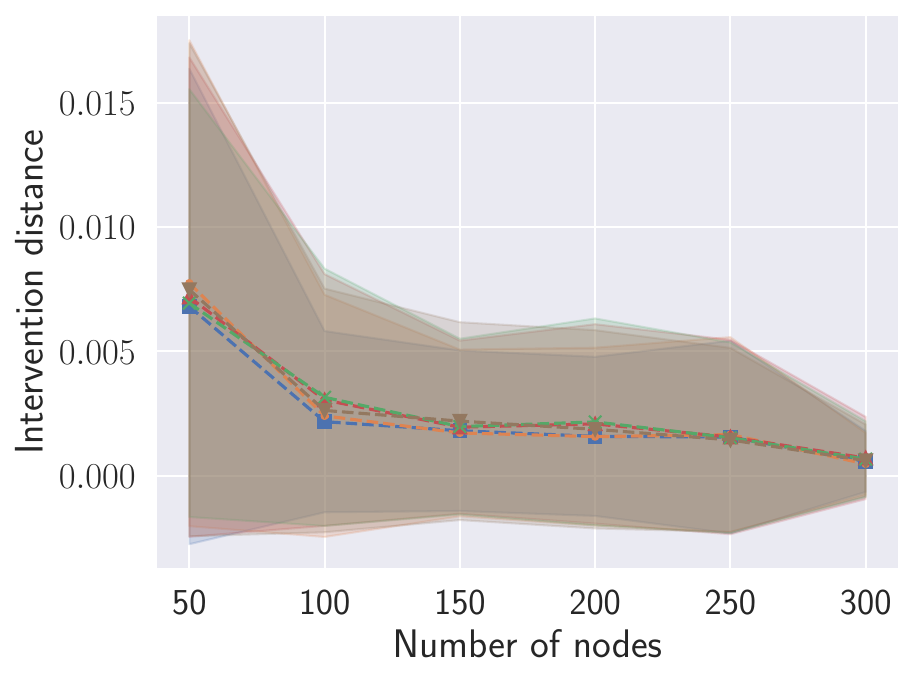}
            \caption*{$\chi^2$ tests}
        \end{subfigure}
        \caption{SNAP$(0)$.}
        \label{fig:int_dist_snap0}
    \end{subfigure}
    \begin{subfigure}[b]{\linewidth}
        \centering
        \begin{subfigure}[b]{0.24\linewidth}
            \includegraphics[width=\linewidth]{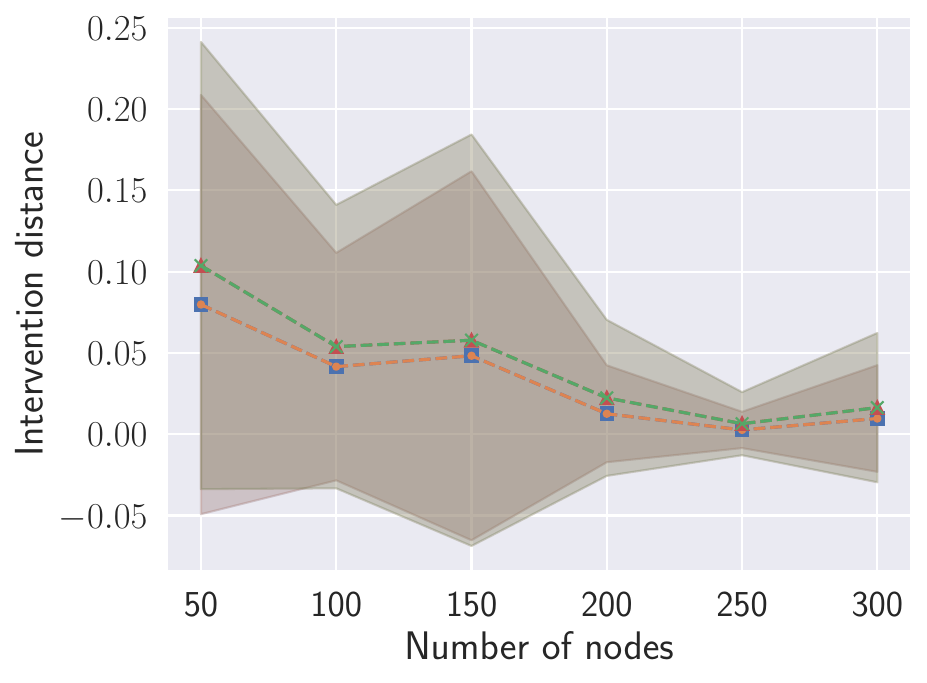}
            \caption*{d-separation tests}
        \end{subfigure}
        \begin{subfigure}[b]{0.24\linewidth}
            \includegraphics[width=\linewidth]{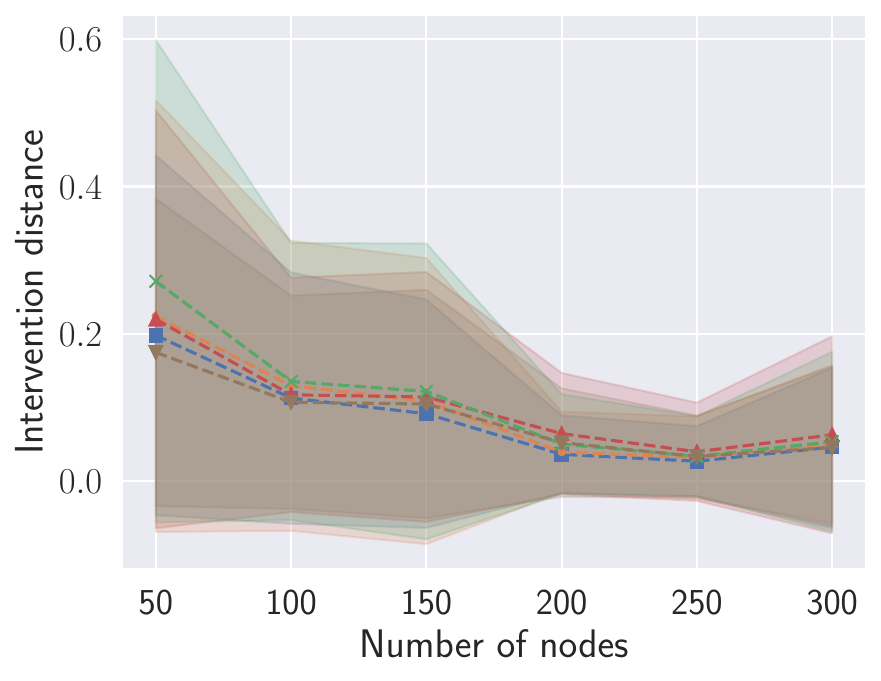}
            \caption*{Fisher-Z tests}
        \end{subfigure}
        \begin{subfigure}[b]{0.24\linewidth}
            \includegraphics[width=\linewidth]{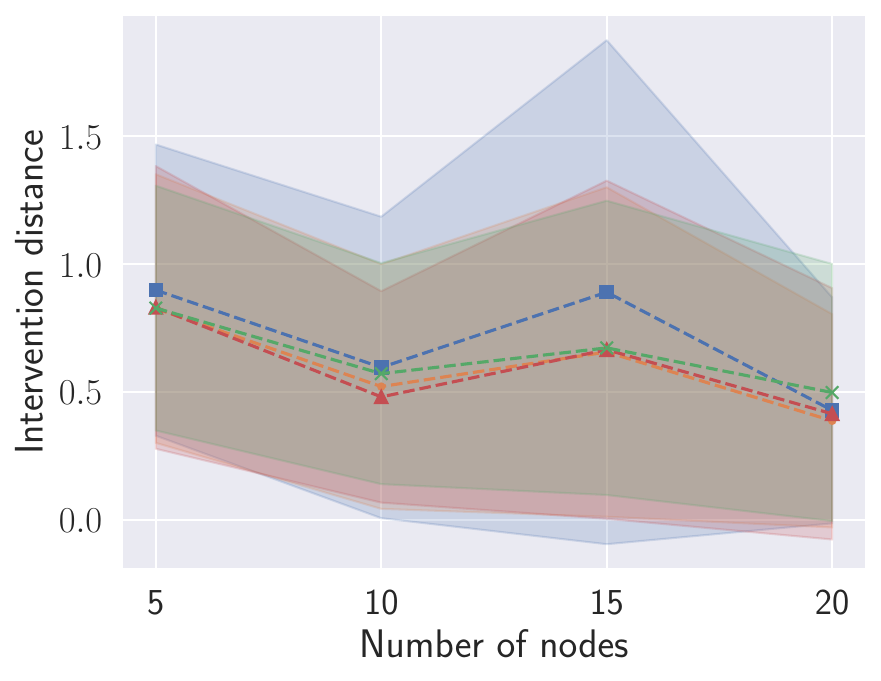}
            \caption*{KCI tests}
        \end{subfigure}
        \begin{subfigure}[b]{0.24\linewidth}
            \includegraphics[width=\linewidth]{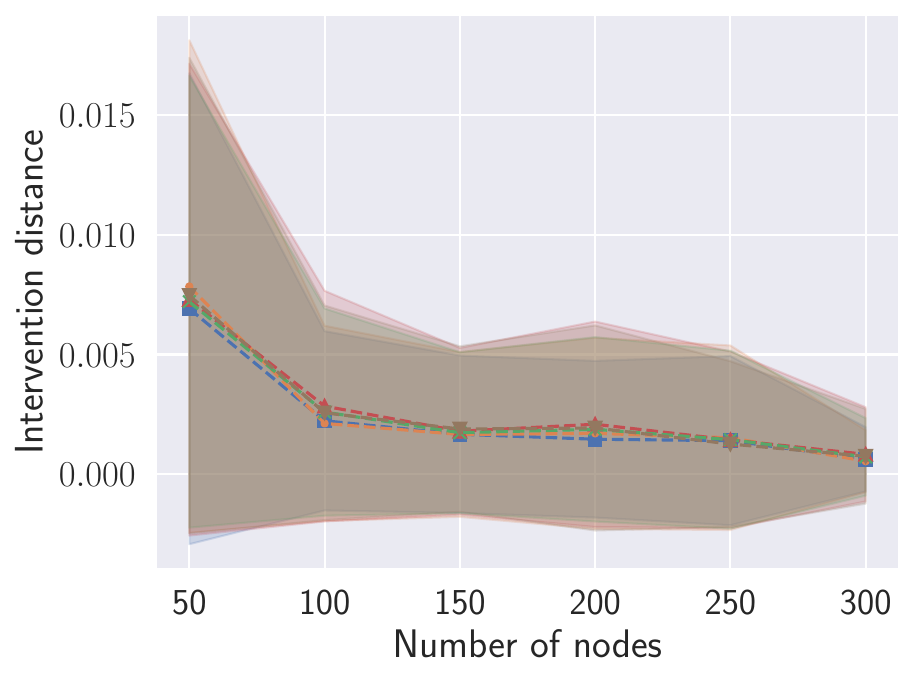}
            \caption*{$\chi^2$ tests}
        \end{subfigure}
        \caption{SNAP$(1)$.}
        \label{fig:int_dist_snap1}
    \end{subfigure}
    \begin{subfigure}[b]{\linewidth}
        \centering
        \begin{subfigure}[b]{0.24\linewidth}
            \includegraphics[width=\linewidth]{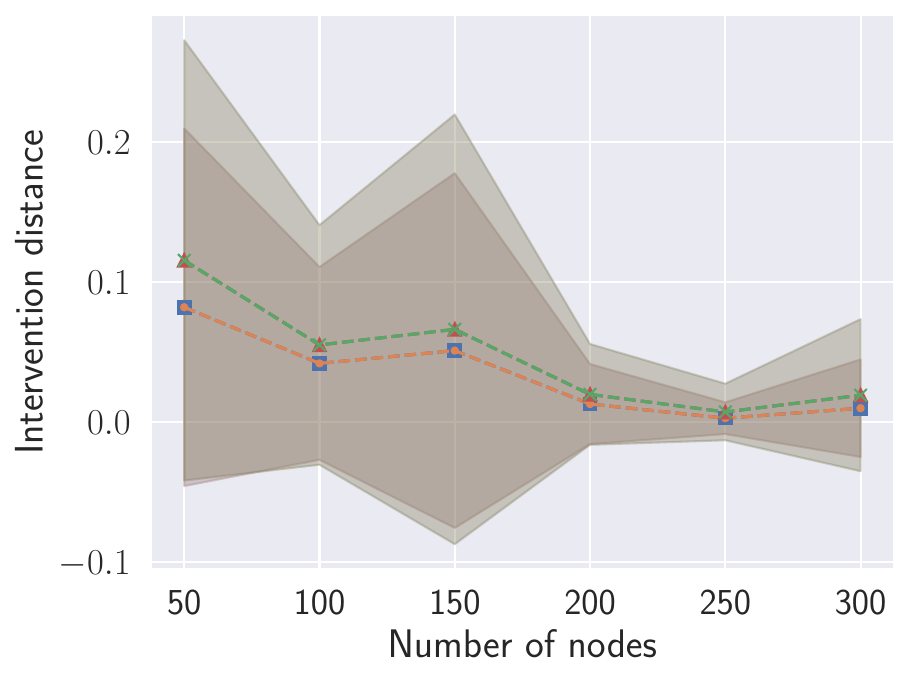}
            \caption*{d-separation tests}
        \end{subfigure}
        \begin{subfigure}[b]{0.24\linewidth}
            \includegraphics[width=\linewidth]{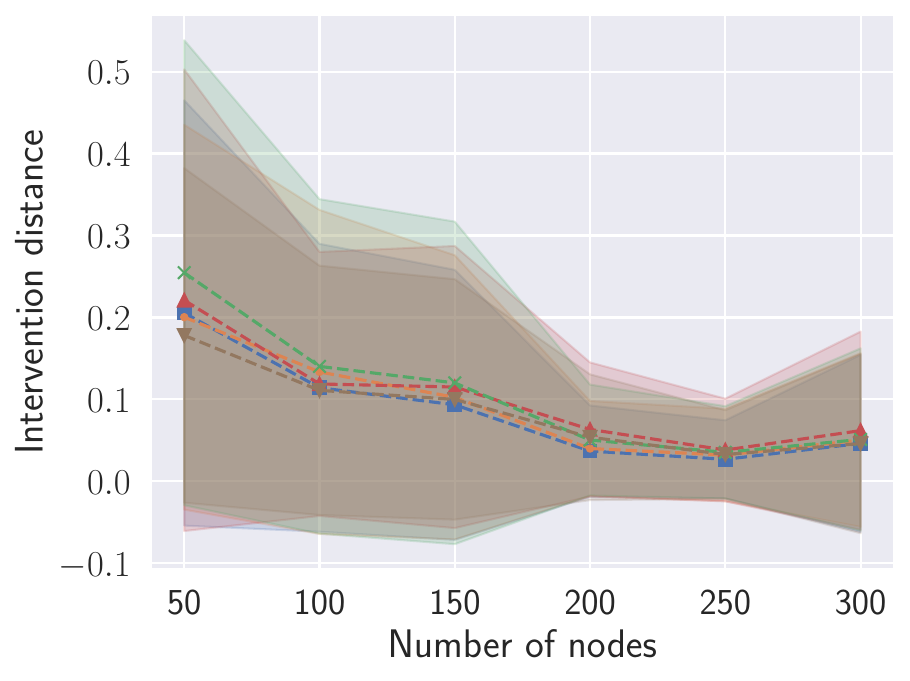}
            \caption*{Fisher-Z tests}
        \end{subfigure}
        \begin{subfigure}[b]{0.24\linewidth}
            \includegraphics[width=\linewidth]{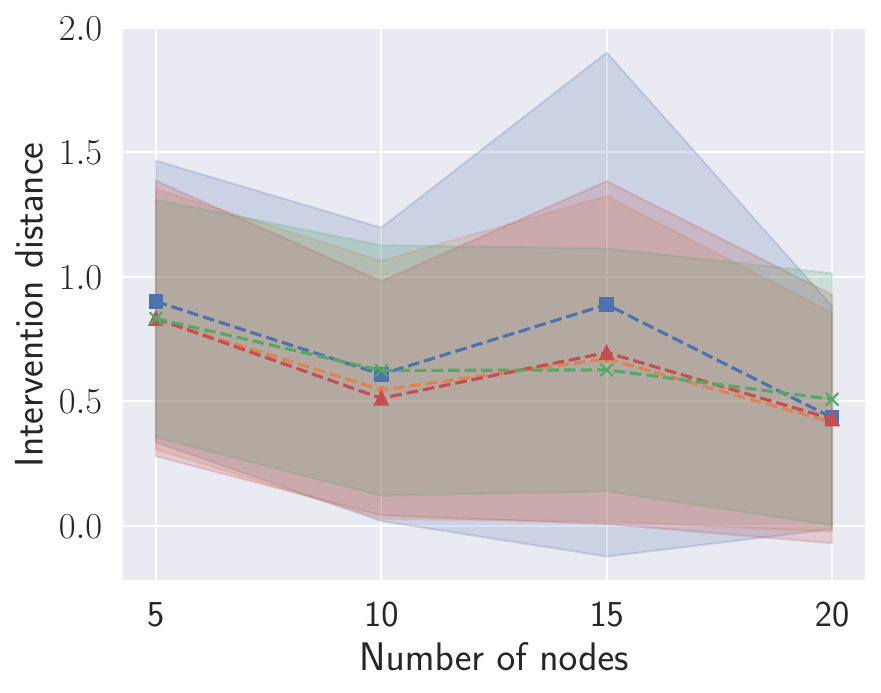}
            \caption*{KCI tests}
        \end{subfigure}
        \begin{subfigure}[b]{0.24\linewidth}
            \includegraphics[width=\linewidth]{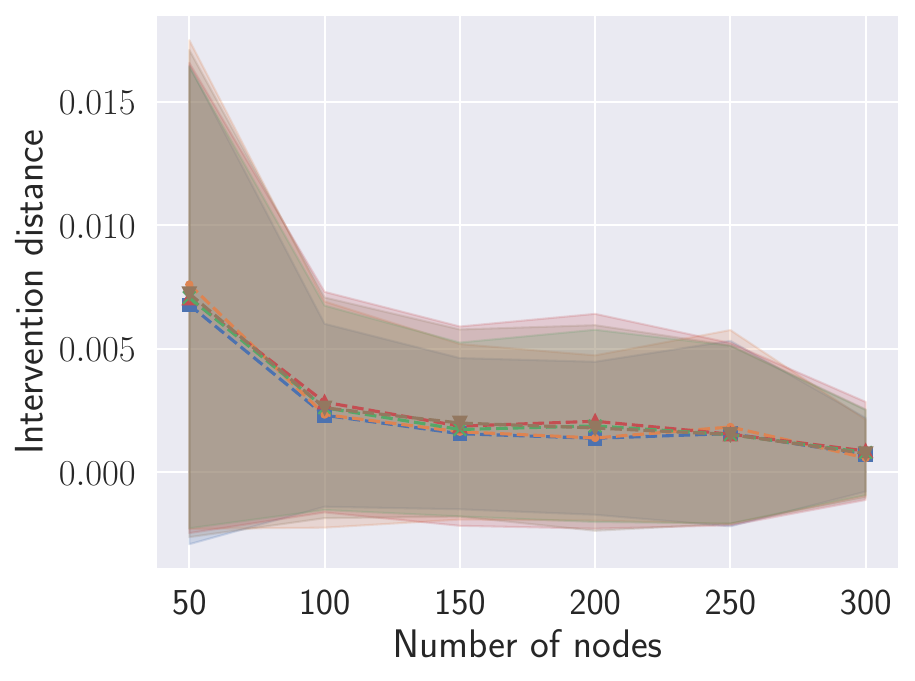}
            \caption*{$\chi^2$ tests}
        \end{subfigure}
        \caption{SNAP$(2)$.}
        \label{fig:int_dist_snap2}
    \end{subfigure}
    \caption{Intervention distance over number targets for SNAP($k$) with $k= 0,\dots,2$, with $n_{\mathbf{T}}=4$, $\overline{d} = 3, d_{\max}=10$ and $n_{\mathbf{D}} = 1000$ data-points.}
    \label{fig:int_dist_snapk}
\end{figure}

\subsection{Order of CI tests}
\label{app:order}
Fig.~\ref{fig:test_per_order} highlights the fact that SNAP avoids higher order independence tests.
We plot the number of d-separation \ac{CI} tests performed over the order of the tests, averaged over graphs with a fixed parameter of $n_{\mathbf{T}}=4, n_{\mathbf{V}}=20, \overline{d} = 3$ and $d_{\max}=10$.
All SNAP variants do the same amount of marginal tests.
At higher orders, combining any baseline, except for MARVEL, with SNAP$(0)$ for prefiltering allows it to perform fewer higher order tests.
MARVEL exchanges low order \ac{CI} tests to tests of maximum order (18), due to total conditioning.
Combining MARVEL with SNAP$(0)$ often allows for lower order total conditioning over only the remaining variables.

\begin{figure}
    \centering
    \includegraphics[width=.6\linewidth]{experiments/legend_big.pdf}
    \begin{subfigure}[b]{.49\linewidth}
        \centering
        \includegraphics[width=0.8\linewidth]{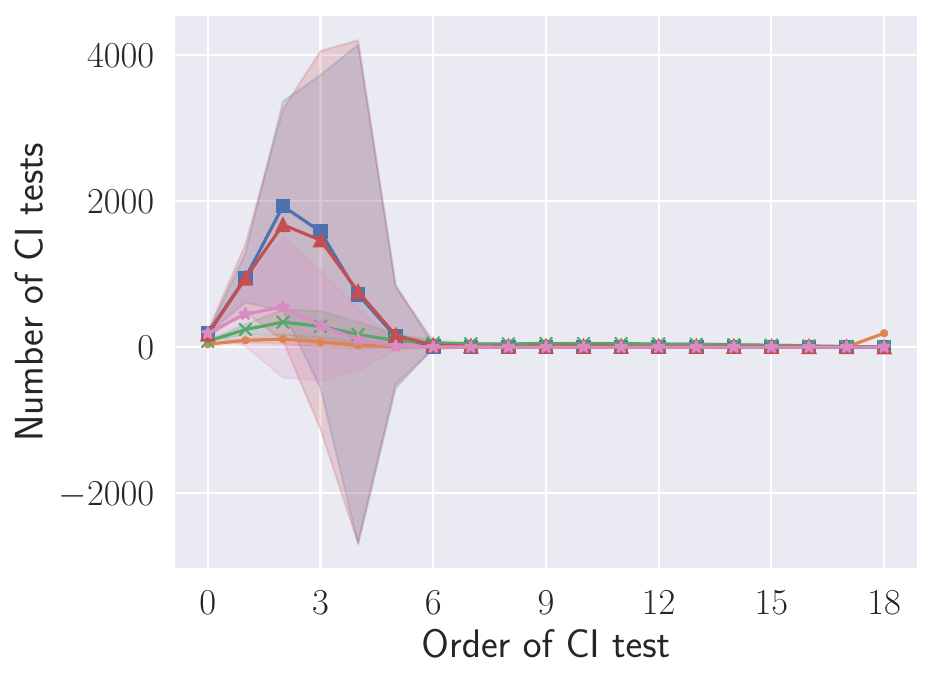}
    \end{subfigure}
    \begin{subfigure}[b]{.49\linewidth}
        \centering
        \includegraphics[width=0.8\linewidth]{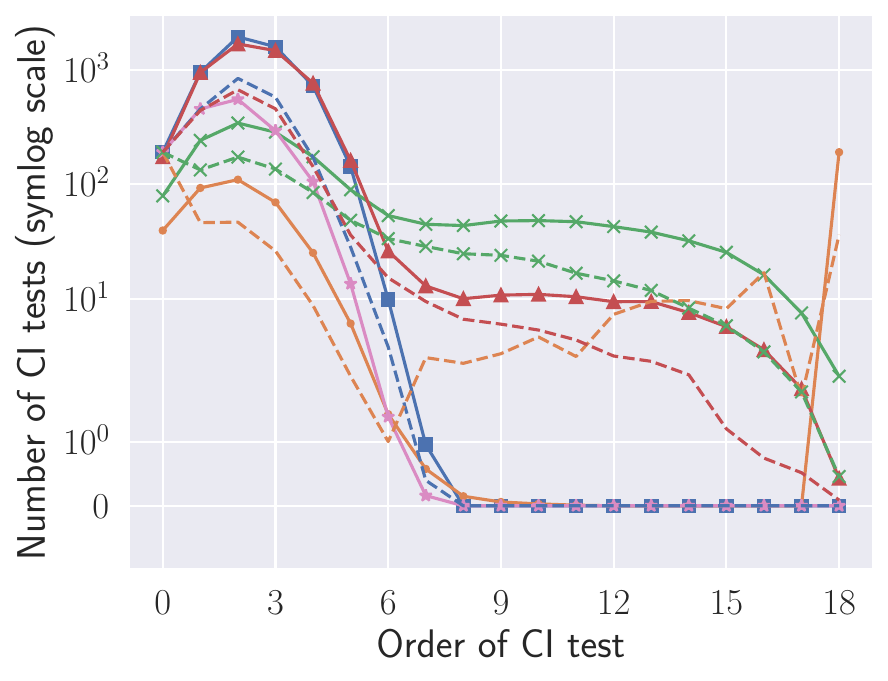}
    \end{subfigure}
    \caption{Tests per order for d-separation \ac{CI} tests with $n_{\mathbf{T}}=4, n_{\mathbf{V}}=20, \overline{d} = 3$ and $d_{\max}=10$. SNAP variants use lower order \ac{CI} tests than baselines. The best method is indicated in \textbf{bold}.}
    \label{fig:test_per_order}
\end{figure}

\subsection{Real data}
\label{app:real-network-results}
We show all results for the MAGIC-NIAB network in Tab.~\ref{tab:magic-niab_app}, which shows that the different in intervention distance between most methods in not significant due to high variance.

\begin{table*}
\centering
\begin{tabular}{lrrrr}
\hline
\multicolumn{1}{c}{} & \multicolumn{1}{c}{CI tests} & \multicolumn{1}{c}{Time} & \multicolumn{1}{c}{Int. dist.} & \multicolumn{1}{c}{SHD} \\ \hline
PC                   & $12806.95 (\pm2085.96)$           & $79.27 (\pm24.75)$            & $\mathbf{0.009 (\pm 0.007)}$           & $22.33 (\pm6.02)$   \\
PC-SNAP$(0)$           & $\mathbf{955.15 (\pm9.50)}$       & $0.47 (\pm0.12)$              & $0.012 (\pm0.011)$                    & $44.84 (\pm9.03)$            \\ \hline
MARVEL               & $8873.08 (\pm3055.54)$   & $27.33 (\pm9.41)$                        & $0.019 (\pm0.020)$                    & $40.81 (\pm9.11)$            \\
MARVEL-SNAP$(0)$       & $959.89 (\pm5.49)$                & $0.58 (\pm0.15)$              & $0.021 (\pm0.015)$                    & $46.39 (\pm8.95)$            \\ \hline
LDECC*                & $18142.34 (\pm2608.34)$           & $19.22 (\pm4.02)$             & $0.022 (\pm0.015)$                    & -            \\
LDECC*-SNAP$(0)$        & $980.55 (\pm22.52)$               & $0.85 (\pm0.15)$              & $0.020 (\pm0.015)$                    & -            \\ \hline
MB-by-MB*             & $11464.10 (\pm1995.24)$           & $25.72 (\pm4.70)  $           & $0.024 (\pm0.021)$                    & -            \\
MB-by-MB*-SNAP$(0)$     & $971.74 (\pm17.21)  $             & $0.73 (\pm0.22)$              & $0.019 (\pm0.021)$                    & -            \\ \hline
FGES                 & -                            & $0.72 (\pm0.13)$              & $0.011 (\pm0.010)$         & $\mathbf{18.64 (\pm5.03)}$   \\
FGES-SNAP$(0)$         & -                            & $\mathbf{0.40 (\pm0.07)}$     & $0.015 (\pm0.011)$                    & $44.83 (\pm8.66)$            \\ \hline
SNAP($\infty$)       & $\mathbf{955.15 (\pm9.50)}$       & $0.61 (\pm0.16)$              & $0.012 (\pm0.011)$                    & $44.84 (\pm9.03)$            \\ \hline
\end{tabular}
\caption{Results for the MAGIC-NIAB network, with $n_{\mathbf{V}} = 66, \overline{d} = 3, n_{\mathbf{T}} = 4$ identifiable targets and $n_{\mathbf{D}} = 1000$ linear Gaussian data-points. We repeat each experiment 100 times and remove the best and worst 5 results for each method. The best method is indicated in \textbf{bold}.}
\label{tab:magic-niab_app}
\end{table*}

We also try the Andes network from bnlearn \citep{scutari2010bnlearn} and sample binary data according to its parameters.
Tab.~\ref{tab:andes} shows that SNAP variants always improve running time and are comparable in terms of \ac{CI} tests.
Furthermore, SNAP variants achieve comparable intervention distance to local methods, and greatly improves the intervention distance of SHD.

\begin{table*}
\begin{tabular}{lrrrr}
\hline
\multicolumn{1}{c}{} & \multicolumn{1}{c}{CI tests} & \multicolumn{1}{c}{Time} & \multicolumn{1}{c}{Int. dist.} & \multicolumn{1}{c}{SHD} \\ \hline
PC                   & $46917.81 (\pm1582.70)$           & $6.65 (\pm0.39)$              & $\mathbf{0.0030 (\pm0.0114)}$        & $\mathbf{74.33 (\pm10.15)}$  \\
PC-SNAP$(0)$           & $24753.07 (\pm0.33)$              & $5.63 (\pm0.25)$              & $0.0066 (\pm0.0255)$                & $168.93 (\pm15.12)$          \\ \hline
MARVEL               & $\mathbf{24753.00 (\pm0.00)}$     & $19.02 (\pm0.38)$             & $0.0043 (\pm0.0245)$                & $168.66 (\pm15.12)$          \\
MARVEL-SNAP$(0)$       & $24759.09 (\pm0.53)$              & $4.96 (\pm0.17)$              & $0.0076 (\pm0.0270)$                 & $169.04 (\pm15.10)$          \\ \hline
LDECC*                & $25588.54 (\pm20192.74)$          & $7.60 (\pm2.57)$              & $0.0063 (\pm0.0241)$                & -          \\
LDECC*-SNAP$(0)$        & $24755.58 (\pm3.58)$              & $8.65 (\pm0.43)$              & $0.0074 (\pm0.0262)$                 & -          \\ \hline
MB-by-MB*             & $29520.48 (\pm14909.30)$          & $31.06 (\pm17.81)$            & $0.0069 (\pm0.0287)$                 & -          \\
MB-by-MB*-SNAP$(0)$     & $24754.68 (\pm2.32)$              & $5.25 (\pm0.23)$              & $0.0073 (\pm0.0287)$                 & -          \\ \hline
FGES                 & -                            & $19.69 (\pm2.40)$             & $0.0216 (\pm0.0545)$        & $204.71 (\pm24.13)$ \\
FGES-SNAP$(0)$         & -                            & $\mathbf{4.51 (\pm0.59)}$     & $0.0080 (\pm0.0283)$                 & $169.90 (\pm22.53)$         \\ \hline
SNAP($\infty$)       & $24753.07 (\pm0.33)$              & $5.45 (\pm0.23)$              & $0.0066 (\pm0.0255)$                 & $168.93 (\pm15.12)$          \\ \hline
\end{tabular}
\caption{Results for the Andes network, with $n_{\mathbf{V}} = 223, \overline{d} = 3.03, n_{\mathbf{T}} = 4$ identifiable targets and $n_{\mathbf{D}} = 1000$ binary data-points. We repeat each experiment 100 times and remove the best and worst 5 results for each method. The best method is indicated in \textbf{bold}.}
\label{tab:andes}
\end{table*}

\subsection{CI test error rates}

\label{app:error_rates}
In this section, we analyze the type I and type II error rates of the CI tests used by SNAP($\infty$), and their relation to missing and extra edges in the output graph.
Figure \ref{fig:error_rate_per_node} shows the type I and type II error rates for different CI tests over the number of nodes.
Even though we use both Fisher-Z and KCI tests on linear Gaussian data with a significance threshold $\alpha = 0.05$, Fisher-Z tests have a consistently higher type I error rate compared to type II on these data, while KCI tests show the opposite performance with consistently higher type II error rate compared to type I.
On the other hand, the type II error rate of $\chi^2$ tests is significantly higher and decreases as the number of nodes grow.

\begin{figure}
    \centering
    \begin{subfigure}[b]{\linewidth}
        \centering
        \begin{subfigure}[b]{0.3\linewidth}
            \includegraphics[width=\linewidth]{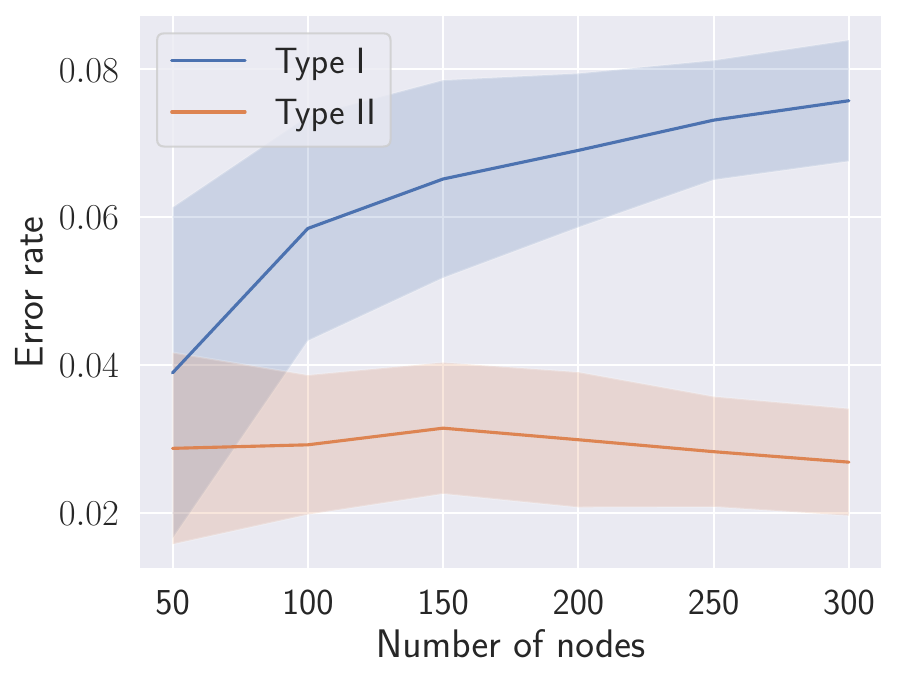}
            \caption*{Fisher-Z tests}
            \label{fig:error_rate_per_node_fshz}
        \end{subfigure}
        \begin{subfigure}[b]{0.3\linewidth}
            \includegraphics[width=\linewidth]{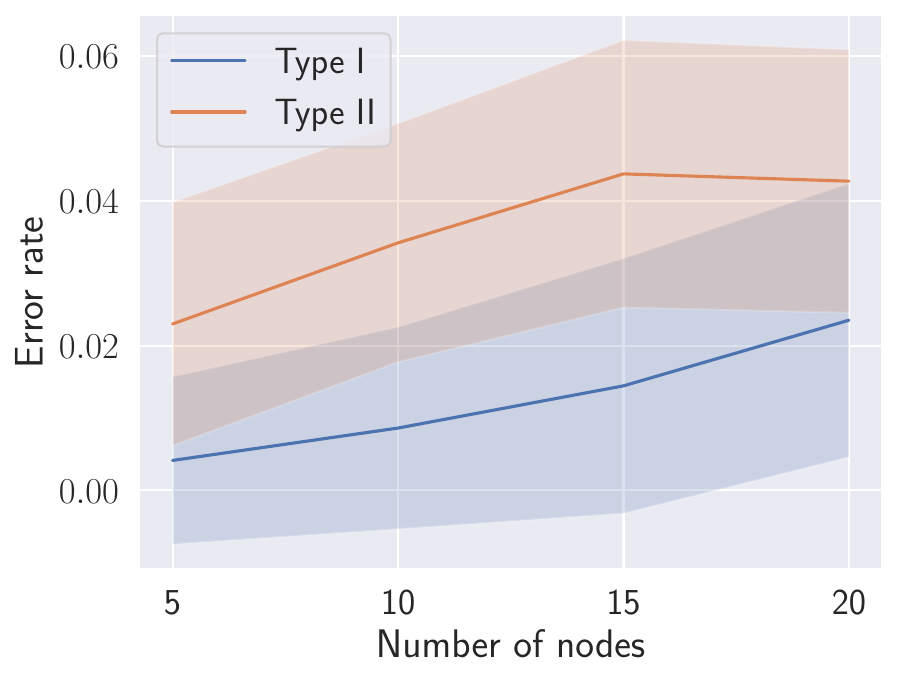}
            \caption*{KCI tests}
            \label{fig:error_rate_per_node_kci}
        \end{subfigure}
        \begin{subfigure}[b]{0.3\linewidth}
            \includegraphics[width=\linewidth]{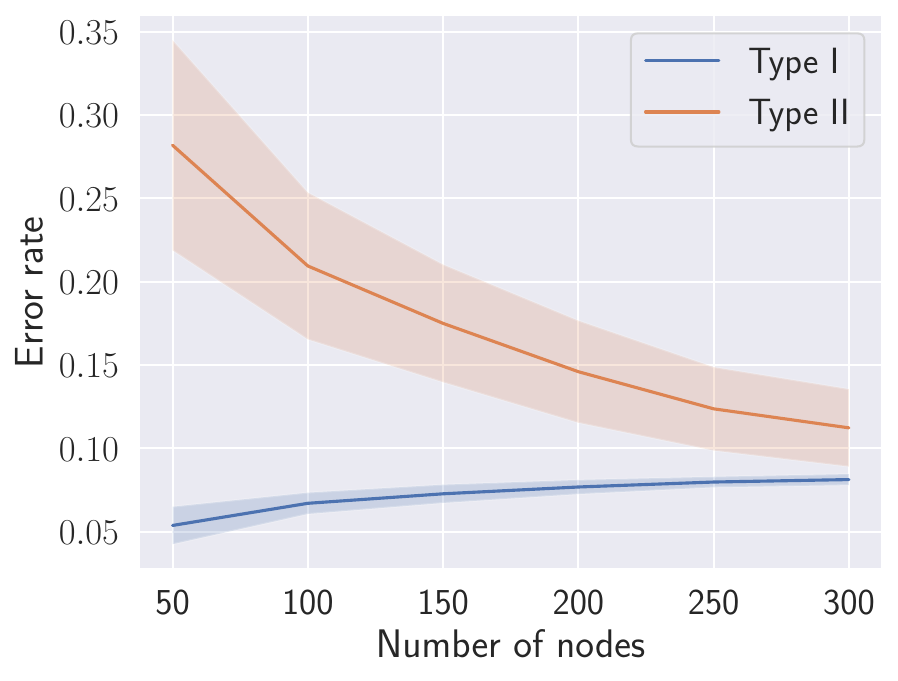}
            \caption*{$\chi^2$ tests}
            \label{fig:error_rate_per_node_chsq}
        \end{subfigure}
    \end{subfigure}
    \caption{Type I and type II error rates of different CI tests performed by SNAP($\infty$) over number of nodes, with $n_{\mathbf{T}}=4, \overline{d} = 3, d_{\max}=10$ and $n_{\mathbf{D}} = 1000$ data-points.}
    \label{fig:error_rate_per_node}
\end{figure}

Figure~\ref{fig:extra_missing_edges} shows the number of extra and missing edges in the graph returned by SNAP($\infty$) compared to the true CPDAG over all the possible ancestors of the targets.
Our results show that for all CI tests, there are much more missing edges than extra ones.
In particular, this holds for Fisher-Z tests, where the type II error rate is lower.
This indicates that missing edges arise not from type II CI testing errors, but from pruning the wrong nodes.

\begin{figure}
    \centering
    \begin{subfigure}[b]{\linewidth}
        \centering
        \begin{subfigure}[b]{0.3\linewidth}
            \includegraphics[width=\linewidth]{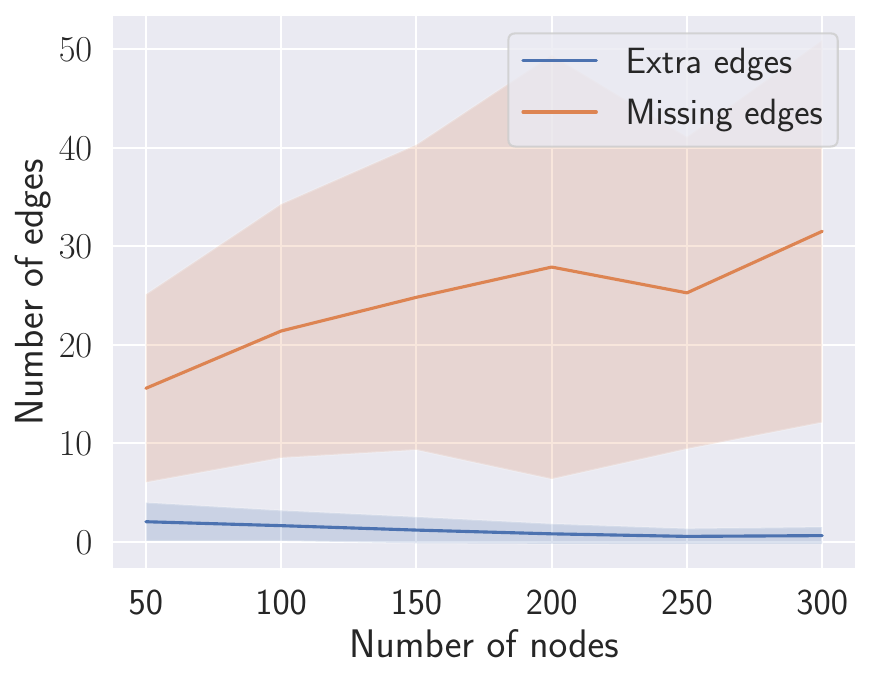}
            \caption*{Fisher-Z tests}
            \label{fig:extra_missing_edges_fshz}
        \end{subfigure}
        \begin{subfigure}[b]{0.3\linewidth}
            \includegraphics[width=\linewidth]{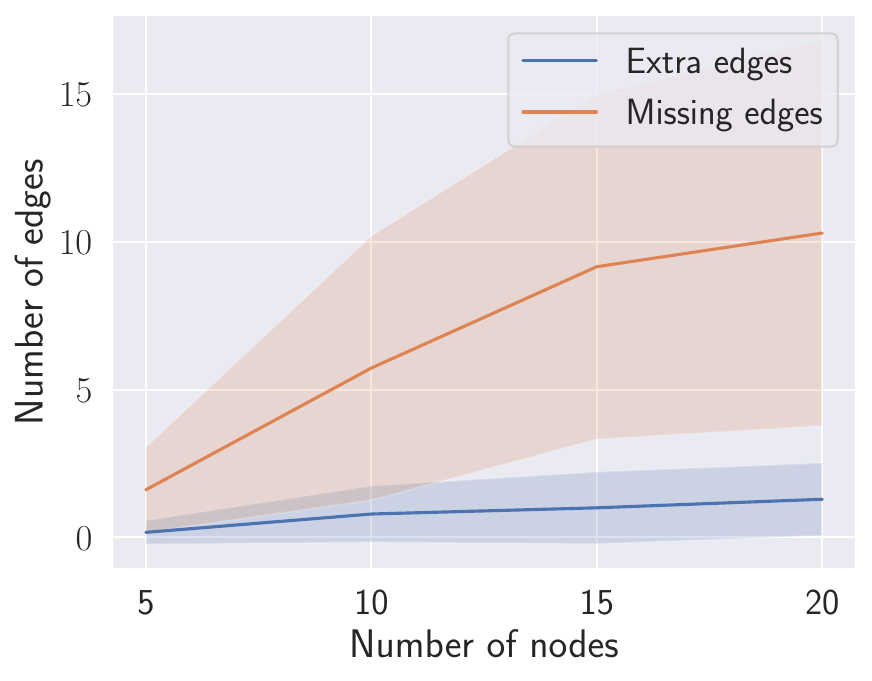}
            \caption*{KCI tests}
            \label{fig:extra_missing_edges_kci}
        \end{subfigure}
        \begin{subfigure}[b]{0.3\linewidth}
            \includegraphics[width=\linewidth]{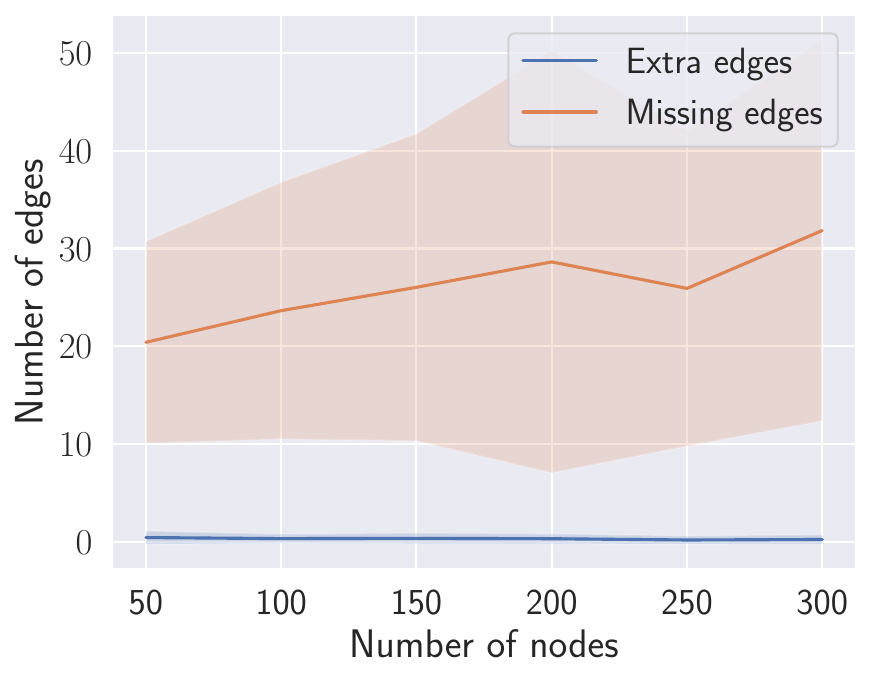}
            \caption*{$\chi^2$ tests}
            \label{fig:extra_missing_edges_chsq}
        \end{subfigure}
    \end{subfigure}
    \caption{Number of extra and missing edges in the output graphs by SNAP($\infty$) over number of nodes, with $n_{\mathbf{T}}=4, \overline{d} = 3, d_{\max}=10$ and $n_{\mathbf{D}} = 1000$ data-points for different CI tests.}
    \label{fig:extra_missing_edges}
\end{figure}

\subsection{SHD on the optimal adjustment set}
\label{app:shd_opt}

As shown in Figures~\ref{fig:quality_per_node_unrest_std} and \ref{fig:appendix_per_node_ident_std}, even though SNAP($\infty$) achieves slightly higher \ac{SHD} than the baselines in the finite data cases, especially when using Fisher-Z tests, its intervention distance remains comparable.
This could be due to measuring \ac{SHD} on the induced subgraph of the true CPDAG over all possible ancestors of the targets, even though a large portion of these possible ancestors might not be part of the optimal adjustment set.

Thus, in this section, we evaluate \ac{SHD} on the induced subgraph of the true CPDAG over only the nodes included in the optimal adjustment sets for some target pair.
In particular, we consider the same setting as in Figure~\ref{fig:appendix_per_node_ident_std}, namely graphs with varying number of nodes, expected degree of 3, maximum degree of 10 and 4 identifiable targets, which ensures that an optimal adjustment set exists.
Compared to the results shown in Figure~\ref{fig:shd_per_node_ident_std}, our results in Figure~\ref{fig:shd_per_node_ident_std_optadjset} show that the difference in \ac{SHD} between SNAP and baselines is considerably smaller when only considering optimal adjustment sets, especially in the case of Fisher-Z and KCI tests.

\begin{figure}
    \centering
    \includegraphics[width=.8\linewidth]{experiments/legend_small.pdf}
    \begin{subfigure}[b]{\linewidth}
        \centering
        \begin{subfigure}[b]{0.3\linewidth}
            \includegraphics[width=\linewidth]{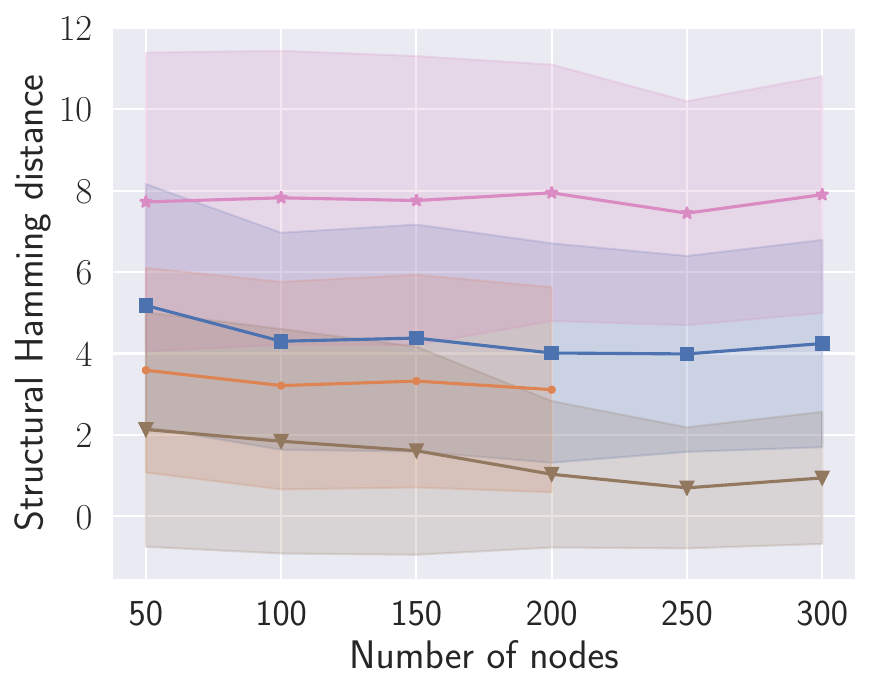}
            \caption*{Fisher-Z tests}
            \label{fig:shd_per_node_ident_fshz_opt_adj_set}
        \end{subfigure}
        \begin{subfigure}[b]{0.3\linewidth}
            \includegraphics[width=\linewidth]{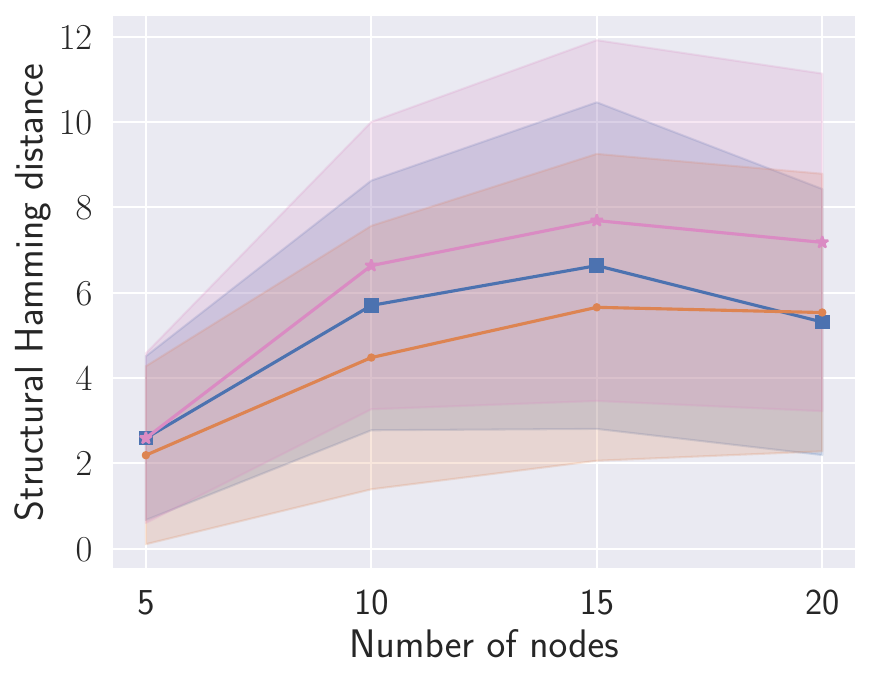}
            \caption*{KCI tests}
            \label{fig:shd_per_node_ident_kci_opt_adj_set}
        \end{subfigure}
        \begin{subfigure}[b]{0.3\linewidth}
            \includegraphics[width=\linewidth]{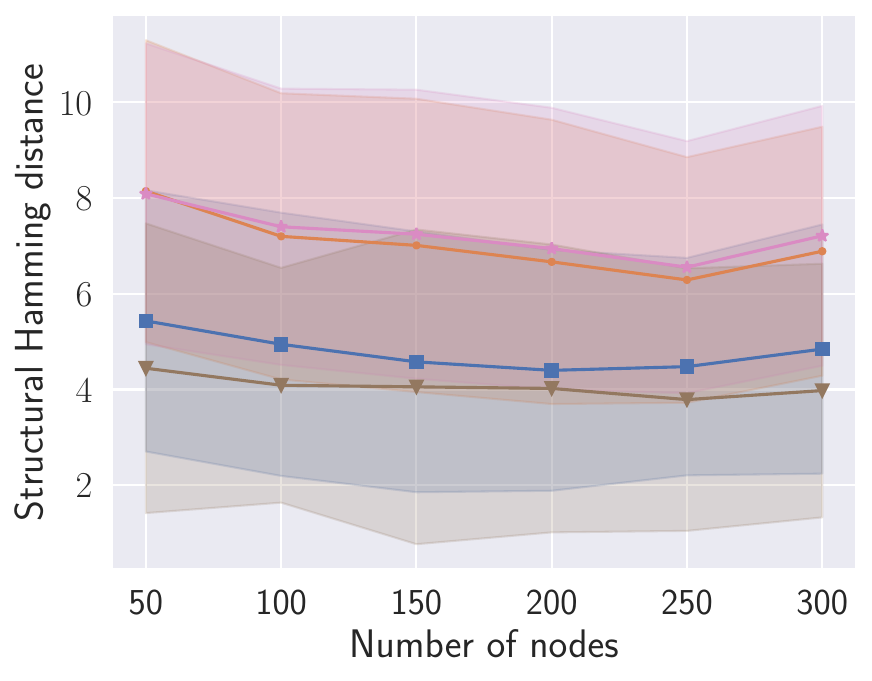}
            \caption*{$\chi^2$ tests}
            \label{fig:shd_per_node_ident_chsq_opt_adj_set}
        \end{subfigure}
    \end{subfigure}
    \caption{\ac{SHD} on the induced subgraph of the true CPDAG over only the nodes included in the optimal adjustment sets for some target pair. Results are shown over number of nodes for identifiable targets, with $n_{\mathbf{T}}=4, \overline{d} = 3, d_{\max}=10$ and $n_{\mathbf{D}} = 1000$ data-points. The shadow area denotes the range of the standard deviation.}
    \label{fig:shd_per_node_ident_std_optadjset}
\end{figure}

We also show \ac{SHD} results over the optimal adjustment sets for the MAGIC-NIAB and Andes networks on Figure~\ref{fig:shd_ident_real_data_opt_adj_set}.
Our results show that SNAP($\infty$) performs on par with most baselines.
For MAGIC-NIAB network, the average number of nodes involved in optimal adjustment sets among target pairs is $21.78$.
While there is still a noticeable difference in SHD, it is smaller than when considering all possible ancestors.
For the Andes network, the average number of nodes involved in optimal adjustment sets is more than twice of that at $48.38$, while the difference in SHD becomes much smaller compared to the results for SHD over all possible ancestors shown in Table~\ref{tab:andes}.

\begin{figure}
    \centering
    \begin{subfigure}[b]{\linewidth}
        \centering
        \begin{subfigure}[b]{0.45\linewidth}
            \includegraphics[width=\linewidth]{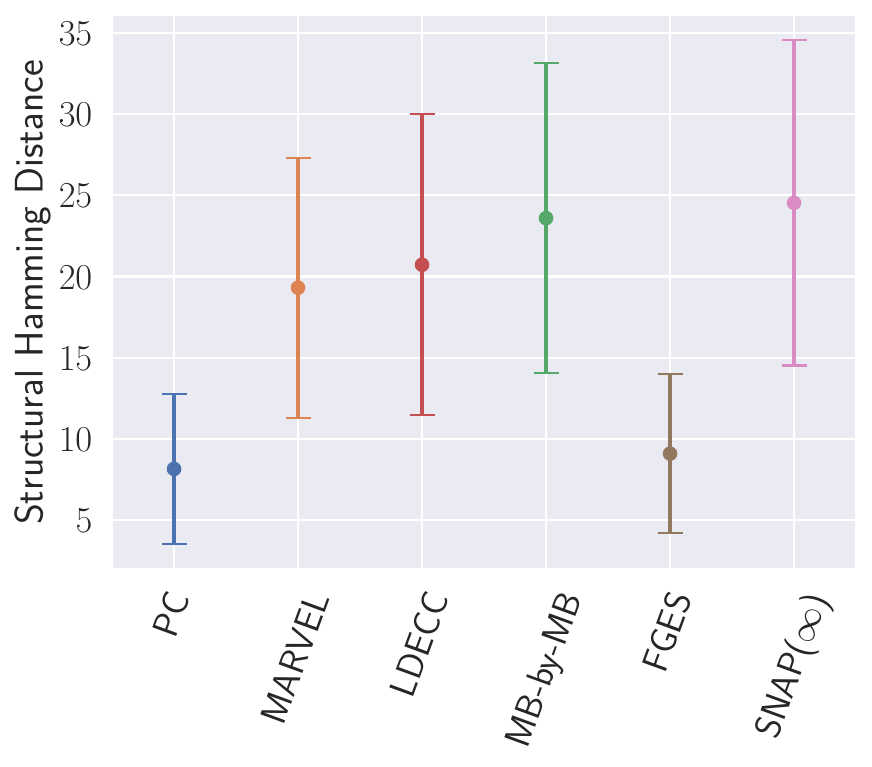}
            \caption*{MAGIC-NIAB}
            \label{fig:shd_ident_magic_niab_opt_adj_set}
        \end{subfigure}
        \begin{subfigure}[b]{0.45\linewidth}
            \includegraphics[width=\linewidth]{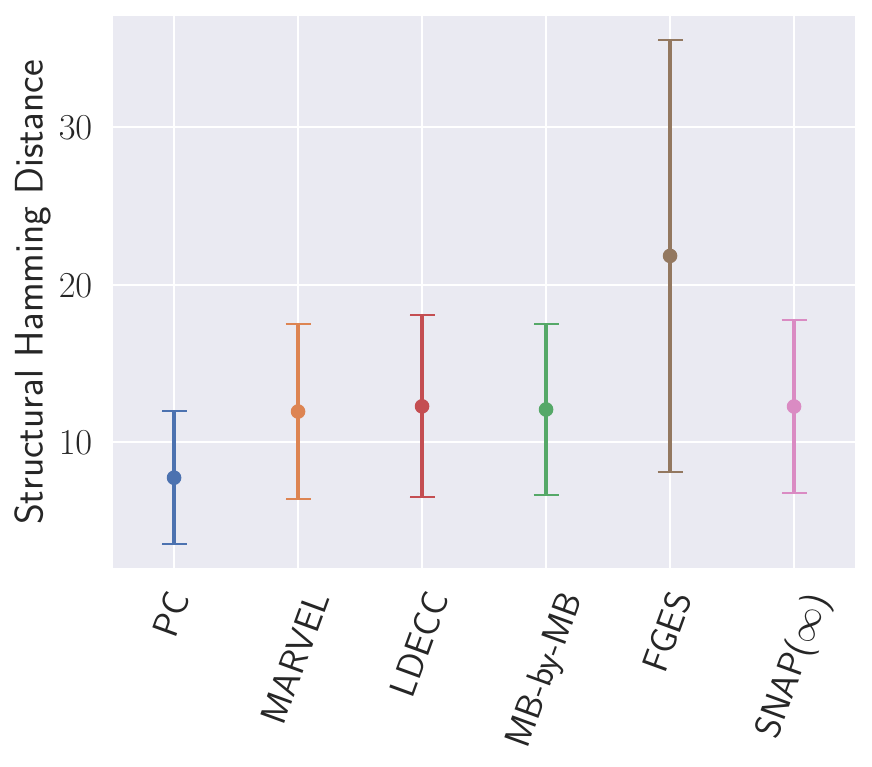}
            \caption*{Andes}
            \label{fig:shd_ident_andes_opt_adj_set}
        \end{subfigure}
    \end{subfigure}
    \caption{\ac{SHD} on the induced subgraph of the true CPDAG over only the nodes included in the optimal adjustment sets for some target pair, for the MAGIC-NIAB and Andes networks with $n_{\mathbf{T}} = 4$ identifiable targets and $n_{\mathbf{D}} = 1000$ data-points. The dots denote the mean, while the error bars indicate standard deviation.}
    \label{fig:shd_ident_real_data_opt_adj_set}
\end{figure}

\end{document}